%% file: main.tex
\documentclass{article}

\usepackage{hyena}
\input{math_commands.tex}

\usepackage{main}

\usepackage{hyperref}       
\usepackage{url}            
\usepackage{graphicx}
\usepackage{subfiles}
\usepackage{tabularx}
\usepackage{etoolbox}
\usepackage{multirow}
\usepackage{enumitem}
\usepackage{amssymb}
\usepackage{amsmath}
\usepackage{amsthm}
\usepackage{xspace}
\usepackage{float}
\usepackage{booktabs}       
\usepackage{amsfonts}       
\usepackage{nicefrac}       
\usepackage{microtype}      
\usepackage{color, soul}

\newtoggle{arxiv}
\toggletrue{arxiv}

\usepackage{listings}
\definecolor{codegreen}{rgb}{0,0.6,0}
\definecolor{codegray}{rgb}{0.5,0.5,0.5}
\definecolor{codepurple}{rgb}{0.58,0,0.82}
\definecolor{backcolour}{rgb}{1,1,1}
\lstdefinestyle{mystyle}{
    backgroundcolor=\color{backcolour},   
    commentstyle=\color{codegreen},
    keywordstyle=\color{magenta},
    numberstyle=\tiny\color{codegray},
    stringstyle=\color{codepurple},
    basicstyle=\ttfamily\footnotesize,
    breakatwhitespace=false,         
    breaklines=true,                 
    captionpos=b,                    
    keepspaces=true,                 
    numbers=left,                    
    numbersep=5pt,                  
    showspaces=false,                
    showstringspaces=false,
    showtabs=false,                  
    tabsize=2,
}
\usepackage{cleveref, algorithm, listings}
\algrenewcommand\algorithmicrequire{\textbf{Input:}}
\algrenewcommand\algorithmicensure{\textbf{Output:}}

\newcommand*\samethanks[1][\value{footnote}]{\footnotemark[#1]}

\iftoggle{arxiv}{
    \usepackage{natbib} 
    \setlength{\textwidth}{6.5in}
    \setlength{\textheight}{9in}
    \setlength{\oddsidemargin}{0in}
    \setlength{\evensidemargin}{0in}
    \setlength{\topmargin}{-0.5in}
    \newlength{\defbaselineskip}
    \setlength{\defbaselineskip}{\baselineskip}
    \setlength{\marginparwidth}{0.8in}
}

\newcolumntype{Y}{>{\hsize=.7\hsize}X}
\newcolumntype{Z}{>{\hsize=1.3\hsize}X}

\makeatletter
\def\@copyrightspace{\relax}
\makeatother

\makeatletter
\def\@myauthornotes{}
\def\myauthornote#1{%
  \if@ACM@anonymous\else
    \g@addto@macro\addresses{}%
    \g@addto@macro\@myauthornotes{%
      \stepcounter{footnote}\footnotetext{#1}}%
  \fi}
\makeatother

\iftoggle{arxiv}{
    \title{Zoology: Measuring and Improving  Recall in Efficient Language Models}
    \usepackage{authblk}
    \author[$\dagger$]{Simran Arora\thanks{Equal contribution, Random ordering by coin toss.}}
    \author[$\dagger$]{Sabri Eyuboglu\samethanks}
    \author[$\triangle$]{Aman Timalsina}
    \author[$\ddagger$]{Isys Johnson}
    \author[$\dagger$]{Michael Poli}
    \author[$\dagger$]{James Zou}
    \author[$\ddagger$]{Atri Rudra}
    \author[$\dagger$]{Christopher R{\'e}}
    \affil[$\dagger$]{Stanford University} 
    \affil[$\ddagger$]{University at Buffalo} \affil[$\triangle$]{Purdue University}\vspace{4pt}
    \affil[$\ddagger$]{\texttt{\{simran, eyuboglu, poli, jamesz, chrismre\}@cs.stanford.edu}}
    \affil[$\ddagger$]{\texttt{\{isysjohn,atri\}@buffalo.edu}}
    \affil[$\triangle$]{\texttt{\{atimalsi\}@purdue.edu}}
}

\begin{document}
\maketitle

\begin{abstract}  
Attention-free language models that combine \textit{gating} and \textit{convolutions} are growing in popularity due to their efficiency and increasingly competitive performance.
To better understand these architectures, we pretrain a suite of 17 attention and \textit{gated-convolution} language models, finding that SoTA gated-convolution architectures still underperform attention by up to 2.1 perplexity points on the Pile. 
In fine-grained analysis, we find 82\% of the gap is explained by each model's ability to recall information that is previously mentioned in-context, \textit{e.g.} \textit{Hakuna Matata means no worries Hakuna Matata it means no $\rightarrow$ \texttt{??}}.
On this task, termed \textit{associative recall}, we find that attention outperforms gated-convolutions by a large margin: a 70M parameter attention model outperforms a 1.4 billion parameter gated-convolution model on associative recall.
This is surprising because prior work shows gated convolutions can perfectly solve synthetic tests for AR capability.  
To close the gap between synthetics and real language, we develop a new formalization of the task called multi-query associative recall (${\Task}$) that better reflects actual language.
We perform an empirical and theoretical study of ${\Task}$ that elucidates differences in the parameter-efficiency of attention and gated-convolution recall.
Informed by our analysis, we evaluate simple convolution-attention hybrids and show that hybrids with input-dependent sparse attention patterns can close 97.4\% of the gap to attention, while maintaining sub-quadratic scaling. Our code is accessible at: \url{https://github.com/HazyResearch/zoology}.
\end{abstract}

\section{Introduction}
\input{Sections/sec_1_intro}

\section{Background and Preliminaries}

\input{Sections/sec_2_setup}

\section{Identifying the associative recall problem}
\input{Sections/sec_3_downstream}

\section{Explaining the associative recall problem}
\input{Sections/sec_4_framework}

\section{Closing the Associative Recall Gap}
\input{Sections/sec_5_method}

\section{Discussion and Conclusion}
\input{Sections/sec_6_conclusion}

\section*{Acknowledgments}
We thank Tri Dao, Daniel Fu, Neel Guha, Stefano Massaroli, Eric Nguyen, and Michael Zhang for helpful feedback and discussion during this work. We are grateful to Together Computer for making this work possible. We gratefully acknowledge the support of DARPA under Nos. FA86501827865 (SDH) and FA86501827882 (ASED); NIH under No. U54EB020405 (Mobilize), NSF under Nos. CCF1763315 (Beyond Sparsity), CCF1563078 (Volume to Velocity), and 1937301 (RTML); ONR under No. N000141712266 (Unifying Weak Supervision); the Moore Foundation, NXP, Xilinx, LETI-CEA, Intel, IBM, Microsoft, NEC, Toshiba, TSMC, ARM, Hitachi, BASF, Accenture, Ericsson, Qualcomm, Analog Devices, the Okawa Foundation, American Family Insurance, Google Cloud, Microsoft Azure, Swiss Re,
Brown Institute for Media Innovation,
Department of Defense (DoD) through the National Defense Science and
Engineering Graduate Fellowship (NDSEG) Program, 
Fannie and John Hertz Foundation,
National Science Foundation Graduate Research Fellowship Program,
Texas Instruments Stanford Graduate Fellowship in Science and Engineering,
and members of the Stanford DAWN project: Teradata, Facebook, Google, Ant Financial, NEC, VMWare, and Infosys. The U.S. Government is authorized to reproduce and distribute reprints for Governmental purposes notwithstanding any copyright notation thereon. Any opinions, findings, and conclusions or recommendations expressed in this material are those of the authors and do not necessarily reflect the views, policies, or endorsements, either expressed or implied, of DARPA, NIH, ONR, or the U.S. Government.
AR's work is supported by NSF grant\# CCF-2247014.
IJ's work is supported by an NSF Graduate Fellowship.

\bibliographystyle{unsrtnat}
\bibliography{references} \newpage

\appendix
\section*{Appendix}
The appendix includes the following content:
\begin{enumerate}
    \item \Cref{app:related-work} provides an extended discussion of related work and concepts.
    \item \Cref{app:implementation} provides a code implementation of the {\Coyote} architecture.
    \item \Cref{app:expdetails} gives details for the experiments, including model architectures and hyperparameters.
    \item \Cref{app:mqar-downsream} provides additional analysis of how {$\Task$} appears in real data across a variety of language distributions.  
    \item \Cref{app:synthetic} provides a formal definition of the {$\Task$} problem, synthetic construction procedure, and experimental details.
    \item \Cref{app:synth-results} provides additional synthetic experiments and analysis.
    \item \Cref{sec:scaling} provides experiments and analysis for how the {$\Task$} gap changes as we scale the gated convolution and attention architectures.s
    \item \Cref{app:thry} gives proofs and additional discussion for the theoretical analysis in our work.
\end{enumerate}

\input{Sections/appendix/related}

\input{Sections/appendix/code}
\section{Downstream Experimental Details}
\label{app:expdetails}

\input{Sections/appendix/experiment}

\input{Sections/appendix/mqar_downstream}

\input{Sections/appendix/mqar_framework}

\input{Sections/appendix/scaling}
\section{Details on Theoretical Analysis}
This section provides proofs and extensions to the theoretical results in the main paper. 
\label{app:thry}

\subsection{Preliminaries and Notation}
\input{Sections/appendix/intro/setup}
\input{Sections/appendix/intro/layer}

\subsection{Primitives}
\input{Sections/appendix/circuit/primitives/setup}
\input{Sections/appendix/circuit/primitives/primitives}
\input{Sections/appendix/circuit/primitives/equivalence}

\input{Sections/appendix/circuit/linear_ac}
\input{Sections/appendix/circuit/general_ac}
\input{Sections/appendix/circuit/retnet_equivalency}

\subsection{The Multiple-Query Associative Recall Problem}
\label{sec: gen-ar}
\input{Sections/appendix/general_ar/setup_iclr}
\input{Sections/appendix/general_ar/seq-ar_iclr}
\input{Sections/appendix/general_ar/pbs_iclr}
\input{Sections/appendix/general_ar/reduction-to-ac_iclr}
\input{Sections/appendix/general_ar/ac-to-coyote_iclr}

\input{Sections/appendix/general_ar/data-dependency/setup}

\input{Sections/appendix/hyperparameters/attention}
\input{Sections/appendix/hyperparameters/hyena}
\input{Sections/appendix/hyperparameters/h3}
\input{Sections/appendix/hyperparameters/rwkv}

\input{Sections/appendix/hyperparameters/retnet}

\input{Sections/appendix/hyperparameters/longconv}
\input{Sections/appendix/hyperparameters/coyote}

\end{document}

%% file: math_commands.tex
\def\eqref#1{equation~\ref{#1}}

\def\floor#1{\lfloor #1 \rfloor}
\def\1{\bm{1}}

\DeclareMathAlphabet{\mathsfit}{\encodingdefault}{\sfdefault}{m}{sl}
\SetMathAlphabet{\mathsfit}{bold}{\encodingdefault}{\sfdefault}{bx}{n}

\newcommand{\softmax}{\mathrm{softmax}}

\newcommand{\Coyote}{\textsc{BaseConv}}
\newcommand{\Task}{\textsc{Mqar}}

%% file: Sections/sec_1_intro.tex
Two advances -- gating and long convolutions -- have catalyzed a wave of excitement around \textit{gated-convolution} language models~\citep[inter alia.]{dao2022hungry, ma2022mega, wang2022pretraining,poli2023hyena}. These architectures combine gating (\textit{i.e.} element-wise multiplication) with \textit{long} convolutional filters (\textit{i.e.} the length of the sequence) to enable interactions between distant tokens~\citep{dauphin2017language,gu2021efficiently}. Recent work suggests that these models, which exhibit better asymptotic scaling in input sequence length than attention, can match attention in language modeling quality~\citep{poli2023hyena, peng2023rwkv, fu2023monarch}.

We pretrain and evaluate 17 language models across 4 scales (70M - 1.4Bn) and 5 architectures on the same data and infrastructure setup. Surprisingly, we find that there is still a perplexity gap of up to 2.1 points between state-of-the-art convolution-based architectures and strong Transformer baselines  
in language modeling on the Pile (Table \ref{table:ppl-slices}). Through fine-grained analysis, we find a single, simple capability is responsible for much of the gap: recalling information seen in-context.
Consider the example below, where some tokens can be predicted by recalling an earlier association:
\input{Figures/Main/Figure1/figure}
\vspace{-1mm}
$$
\underbrace{\text{Hakuna Matata!}}_{\mathclap{\textbf{Key-Value}}} \text{It means} \underbrace{\text{no  worries}}_{\mathclap{\textbf{Key-Value}}} \text{for the rest of your days!} \underbrace{\text{Hakuna}}_{\mathclap{\textbf{Query}}} \underbrace{\text{Matata}}_{\mathclap{\textbf{AR Hit}}}
\text{means} \underbrace{\text{no}}_{\mathclap{\textbf{Query}}} \rightarrow \underbrace{\text{worries}}_{\mathclap{\textbf{AR Hit}}}
$$

We find that errors on ``AR Hits'' (\textit{e.g.} \textit{worries} above) account for 82\% of the perplexity gap to attention on average, despite only representing $6.4\%$ of all tokens in the Pile dataset.\footnote{We measure AR on real data using a simple heuristic: $n$-gram tokens that are repeated in context. 
} A 70M parameter Transformer can predict AR Hits better than a 1.4Bn parameter Hyena gated convolution model ($20\times$ larger) (Table \ref{table:ppl-slices}, Table \ref{table:ppl-slices-large}). The AR gap persists at the 7Bn parameter scale when comparing RWKV and Llama-2 (\Cref{sec:scaling}). 

This task, \textit{associative recall} (AR), has  a long history in machine learning  ~\cite[inter alia.]{graves2014neural,ba2016using} (\cref{sec:related-mech}). 
Prior work argues that a model's ability to perform AR is predictive of in-context learning quality \citep{elhage2021mathematical, olsson2022context}. As a result, AR has been adopted as a tool in designing new  architectures and prior work has shown that gated convolution architectures match attention on synthetic AR tasks used as proxies for real language modeling \citep{dao2022hungry, poli2023hyena, lutati2023focus}. For this reason, the downstream AR perplexity gaps are surprising.

Through measuring recall on real data, we learn the key disparity is that prior synthetic formulations assume there is \textit{one query} per input, at a \textit{fixed position} in the sequence, where tokens come from a small vocabulary size (\textit{e.g.} $|V| <$ 50, less than model dimension). Yet,  
language modeling often requires performing  \textit{multiple recalls} (\textit{e.g.} for both ``Hakuna Matata'' and ``no worries'' above, in a single forward pass), at \textit{varying positions}, with tokens from a large vocabulary  (larger than model dimension). 
We thus propose the study of \textit{multi-query} AR (${\Task}$). 
Compared to the prior AR formulations, ${\Task}$ better captures the persisting quality gaps on synthetic and real world data (Section \ref{sec:theory-explanation}). However, it is not clear why {$\Task$} elucidates the gap.

\textbf{We formalize the the {$\Task$} gap across the architecture classes.} Gated-convolutions process variable sequences using \textit{fixed} filters defined by the model weights rather than as functions of the input data (See Figure \ref{fig:main}).
We find it is inefficient for gated-convolutions to perform the variable distance token-to-token interactions (e.g. at a distance of $10$ tokens for \textit{Hakuna Matata} and $9$ for \textit{no worries}) in a parameter and FLOPs efficient way compared to attention. Attention achieves input-dependence since it computes all token-to-token interactions when determining how to mix information in the sequence. We formally describe the limitation of gated convolutions via theory and experiments in the rest of the paper. 

We first introduce a simple operator, ${\Coyote}$, that can provably simulate the class of architectures built from gating and convolution primitives. This includes architectures such as H3, Hyena, RWKV, and RetNet. 
We show with theory and experiments that the model dimension for ${\Coyote}$ (and thus the aforementioned architectures) to solve {$\Task$} grows with the input sequence length
(Theorem \ref{thm: indep-genar}) while attention can solve {$\Task$} with model dimension independent of sequence length (Proposition \ref{prop: attention-ar}, \Cref{sec:empirical-capacity}).\footnote{Note that the {\em runtime} of Attention is quadratic, but gated convolutions is near linear.} In practice, gated-convolutions appear to encode approximate solutions to AR that support only a subset of token-interaction distances (which affects the ability to identify matching keys given $\Task$ queries). While theoretically analyzing deep learning architectures is challenging \citep{hahn2020theoretical, merrill2022saturated, keles2023on}, the fact that gating and convolutions are polynomial operations facilitates our precise analysis.

\textbf{We show that input-dependent sequence mixing is important to solve ${\Task}$ efficiently.} The scaling for gated convolutions with input-independent convolutions is undesirable so we next ask which architectural choices close the gap. We show that data-dependent sequence mixing helps an architecture solve ${\Task}$ efficiently (Theorem \ref{thm: dep-genar}). The model needs to adapt the sequence mixing weights based on the token-interaction distances required for each new example. 

Several architectural modifications could satisfy the input-dependence property. Based on our analysis, we evaluate \textit{minimal modifications} that only add input-dependent operations on exact-match repeated bigram tokens (our heuristic for measuring tokens that require recall). Note that language models often need to perform recall between fuzzier substrings (e.g. synonymous bigrams) or higher-dimensional concepts. However, we show that simply inserting input-dependent operator --- e.g., a convolution filter that shifts the sequence based on the bigram positions or sparse attention placed only on repeated bigram positions --- to the {\Coyote} architecture at $<10\%$ of layers suffices to outperform the Transformer baseline on Pile language modeling (\Cref{sec:sec5-method}) and succeed on {$\Task$} synthetic tasks (\Cref{sec:empirical-capacity}). Moreover, this closes $>80\%$ of the {$\Task$} perplexity gap on the Pile validation data. Finally, we prototype solutions that \textit{learn} the positions at which to use input-dependent operations, validating that they also close the gap to attention on the Pile.

In this work, we analyze the increasingly-popular convolution-based architectures and identify fundamental limitations. We hope the ${\Task}$ framework and our analysis of the role of input-dependence for ${\Task}$ inform the design of future architectures. We release our code for reproducability: \url{https://github.com/HazyResearch/zoology}.

%% file: Figures/Main/Figure1/figure.tex
\begin{figure}[t]
    \includegraphics[width=\linewidth]{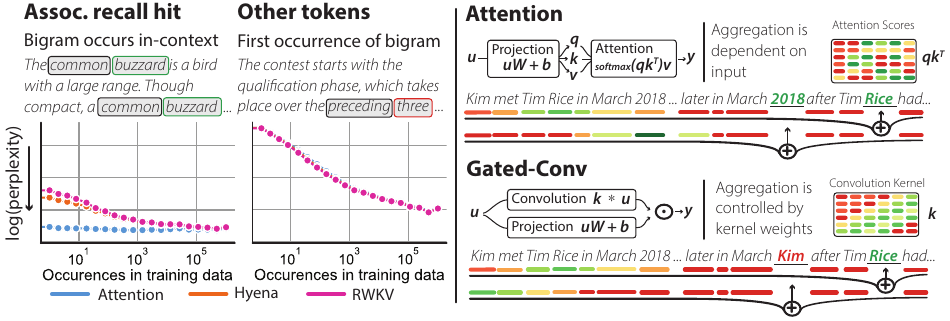}
    \caption{\textbf{The associative recall gap.} We stratify Pile validation data for models from each architecture class by whether or not the predicted token is a previously seen bigram in the example context. We plot validation perplexity versus the bigram's frequency in the training data. We can clearly see that the gap is localized to examples where bigrams occur, and are seen rarely during in training. }
    \label{fig:main}
    \vspace{-3mm}
\end{figure}

%% file: Sections/sec_2_setup.tex
\label{sec:ar-background}
\input{Tables/downstream_ar}

In this section, we describe the setting, introduce notation, and discuss important related work. See \cref{app:related-work} for a broader discussion of related work. 

\textbf{Language modeling.} We study auto-regressive language models trained on the task of next token prediction~\citep{pile}. Given a sequence of $N$ tokens $\bm{x} = \{x_0, ..., x_{N-1}\}$ drawn from a vocabulary $C$, the model outputs a probability distribution over $C$ for each of $x_i$ given the preceding tokens  $\mathrm{P}(x_{i}|x_0, ..., x_{i-1})$. 
The language models in this work share the same high-level architecture. First, each token $x_i$ in the input is embedded in $d$-dimensional space yielding a matrix $\bm{u} \in \mathbb{R}^{N \times d}$. Next, $\bm{u}$ is passed through a stack of $L$ layers, with layer $\ell$ outputting $\bm{u}^\ell \in \mathbb{R}^{N \times d}$. Finally, the embeddings $\bm{u}^L$ output by the last layer are mapped back to logits over $C$ with a linear projection. Each layer transforms $\bm{u}$ with a \textit{sequence mixer} (\textit{e.g.} attention) followed by a \textit{state mixer} (\textit{e.g.} MLP). Unless specified, our models adhere to the implementation details of the LLaMA architecture (except  the sequence mixer, which we vary throughout)~\citep{touvron2023llama}.

\textbf{Sequence mixers.}
\label{sec:prelim-architectures}
Our work evaluates how the choice of sequence mixer affects the quality and behavior of language models. 
Most sequence mixers aggregate the token embeddings in a sequence via a weighted sum. For example, $\bm{y}[i, :] = \sum_{j=0}^{N - 1} \omega(i, j) \bm{u}[j, :])$, where $\omega$ is a function outputting scalar weights. 
We study the differences between two classes of sequence mixers, discussed next.

\textit{Attention.}~\citep{vaswani2018attention} We review attention, the \textit{de facto} language model sequence mixer. 
An attention layer is parameterized by three learnable projection matrices $\mathbf{Q}, \mathbf{K}, \mathbf{V} \in \mathbb{R}^{N \times d}$. 
To compute the output $\bm{y}$ given inputs $\bm{u}$, attention applies the projections to the input: $\bm{q} = \mathbf{Q}\mathbf{u}$, $\bm{k} = \mathbf{K}\mathbf{u}$, $\bm{v} = \mathbf{V}\mathbf{u}$. The projected embeddings are aggregated according to:  $\bm{y} = \text{softmax}(\frac{1}{\sqrt{d}}\bm{q}\bm{k}^\top) \bm{v}$ (shown in \cref{fig:main}) in $\mathcal{O}(N^2d)$ time, which is expensive for long sequences (large $N$).

\textit{Gated-convolutions.}
\label{sec:models}
A more efficient alternative to attention is the convolution, which is defined as $\bm{y}[i, :] = \sum_{j=0}^{N-1} \mathbf{k}[j, :] \odot \bm{u}[i - j, :]$  where the kernel $\mathbf{k} \in \mathbb{R}^{N \times d}$ is a learnable weight matrix.
Convolutions can be computed in time $O(Nd \log N)$ using the Fast Fourier Transform (FFT) and the convolution theorem: $\bm{y} = \mathbf{u} \ast \mathbf{k} = \mathrm{FFT^{-1}(FFT(\mathbf{u}) \odot FFT(\mathbf{k}))}$ ~\citep{cooley1965algorithm}. 
While purely convolutional architectures match or outperform attention in certain domains (\textit{e.g.} vision \citep{tay2022efficient, gu2021efficiently}, audio \citep{goel2022its}, and time-series \citep{zhang2023effectively}), they trail by a large margin on language~\citep{dao2022hungry}. 

Recent work has closed much of this gap by combining convolutions with \textit{gating} (\textit{i.e.} elementwise multiplication of the input with a transformed version of itself). \textit{Gating} was first proposed in the context of convolutions by \citet{dauphin2017language}, but more recently it has played a central role in state-of-the-art sub-quadratic language models, some of which claim attention-level quality~\cite[inter alia]{dao2022hungry, wang2022pretraining, poli2023hyena, peng2023rwkv,zhai2021attention,fu2023monarch}.
Though they may appear different on the surface, these sub-quadratic architectures can all be expressed in terms of convolutions and gating. See \cref{sec:models_appen} for detailed descriptions of these architectures and their similarities. This work analyzes the differences between attention and the broad class of \textit{gated convolution} mixers.

%% file: Tables/downstream_ar.tex
\begin{table}[]
\centering
\small
\begin{tabular}{@{}lcccccc@{}}
\toprule
 &  &  & \textbf{Overall} & \multicolumn{2}{c}{\textbf{Slices}} & \textbf{\% of gap due to} \\ 

Model  & Param (M) & TFLOPs  & & AR Hits & Other Tokens & AR Hits\\
\midrule
Attention & 125 & 2.46 & \textbf{11.01} (2.40) & \textbf{2.16 (0.77)} & 12.45 (2.52) & ---  \\
Long Conv & 128 & 1.74 &  16.98 (2.83) & 25.62 (3.24) & 16.46 (2.80) & 40.1\% \\
H3        & 168 & 2.55 &12.06 (2.49) & 6.75 (1.91) & 12.60 (2.53) & 88.4\% \\
Hyena     & 158 & 2.41 &   11.60 (2.45) & 5.00 (1.61) & 12.28 (2.51) & 100.0\% \\
RWKV      & 169 & 2.08 &  11.64 (2.45) & 5.70 (1.74) & 12.29 (2.51) & 100.0\% \\
\midrule
Attention & 360 & 6.23 &  \textbf{9.44} (2.25) & \textbf{1.98 (0.69)} & 10.62 (2.36) &  ---  \\
Long Conv & 360  & 4.08 & 13.13 (2.57) & 13.27 (2.59) & 13.12 (2.57) & 40.5\% \\
H3  & 357 & 4.85 & 10.38 (2.34) & 4.81 (1.57) & 11.00 (2.40) & 65.8\% \\
Hyena     & 358 & 5.03 & 10.07 (2.31) & 3.83 (1.34) & 10.75 (2.38) & 98.2\% \\
RWKV      & 351 & 4.31 & 9.79 (2.28) & 3.82 (1.34) & 10.51 (2.35) & 100.0\%  \\
\bottomrule
\end{tabular}
\caption{\textbf{Language modeling validation perplexity on the Pile.} After pretraining on 10B tokens of Pile data, we report log perplexity with negative log-likelihood in parentheses.  We report overall scores, and for the AR vs. non-AR token slices defined in \Cref{sec:sec3-downstream}. FLOPs are computed for inputs of 2048 tokens based on the equations in Appendix \ref{app:expdetails}. A table with additional results is in \Cref{table:full-ppl-slices}.
}
\label{table:ppl-slices}
\vspace{-5mm}
\end{table}

%% file: Sections/sec_3_downstream.tex
\label{sec:sec3-downstream}

In this section, we measure the perplexity gap between gated convolutions and attention and show that single skill termed \textit{associative recall} accounts for 82\% of the gap on average. This is surprising because prior work shows gated convolutions solve a synthetic version of associative recall perfectly. 
Informed by our analysis, we define a new synthetic formulation that better reflects real data. This task facilitates our analysis of why the gap occurs (Section \ref{sec:theory-explanation}) and how to fix it (Section \ref{sec:sec5-method}).

\subsection{Fine-grained analysis of downstream quality}
\paragraph{Perplexity Gap} We pretrain a suite of large language models with different sequence mixers across 3  scales (70M-360M) for 10B tokens on the standard Pile language modeling setting using the EleutherAI GPT-NeoX training infrastructure~\citep{pile, gpt-neox-library}. In the main paper, we compare attention to three state-of-the-art gated convolution sequence mixers:  H3, Hyena, and RWKV~\citep{dao2022hungry,poli2023hyena,peng2023rwkv}.\footnote{RWKV is commonly referred to as an RNN. We show it can be viewed as a convolution (\Cref{sec:rwkv-model}).} 
We further include a pure long-convolution model to underscore the importance of gating. 
In \cref{app:synth-results}, we include results on additional sequence mixers~\citep{hasani2022liquid,sun2023retentive}.
We use a strong Transformer attention baseline with rotary embeddings and SwiGLU MLPs following the Llama architecture \citep{touvron2023llama}. We also use this strong training recipe when training the attention-free sequence mixers. 
 For experimental details and hyperparameters for all architectures see \cref{app:expdetails}.

Across scales, we find that attention outperforms the gated convolutions by at least a third of a perplexity point on average: the minimum gaps are $+2.14$, $+0.59$, $+0.35$ PPL at 70M, 160M, and 360M parameter scales, respectively. We report overall test perplexity in \Cref{table:ppl-slices}. Though these gaps are relatively small on average, models may still perform very differently on different subsets of data~\citep{Eyuboglu2022-qz}.

\paragraph{Associative Recall Perplexity} 
To better understand the differences between attention and gated convolutions, we perform a fine-grained analysis of next token predictions and observe that convolution-based models struggle to recall associations previously seen in context. For example, in \cref{fig:main}, the model must recall the association between ``\textit{Tim}" and the last name ``\textit{Rice}".
Following a long line of prior work, we call this skill \textit{associative recall} (AR)~\citep{willshaw1969non,hopfield1982neural} (see \cref{app:related-work} for an extended discussion of AR's history in machine learning). 
In \cref{sec:pile-examples}, we provide annotated Pile examples to demonstrate the phenomenon qualitatively.

\textit{Quantifying AR performance.} It is challenging to derive a quantitative measure of associative recall performance on the Pile because we don't know which next token predictions in raw text require associative recall.
We use a simple heuristic to identify these tokens, which we refer to as \textit{AR Hits}. An AR Hit is the last token of an $n$-gram repeated in context (\textit{e.g.} the second occurence of $``Rice"$ in \cref{fig:main}). However, some common $n$-grams (\textit{e.g.} \textit{``of the"}) could have been memorized during training, so we factor in the frequency with which an $n$-grams appeared in the training data. This heuristic enables us to scale our analysis to over 10 million tokens of Pile validation data.  

We stratify Pile tokens into two slices based on this heuristic and report perplexity on each in \cref{table:ppl-slices}:
\begin{enumerate}
    \item \textbf{AR Hits}: (6.4\% of tokens) Tokens in the final position of a bigram (a pair of consecutive tokens) which previously appeared in context, but $\leq 1250 \times$ during training. 
    \item \textbf{Other tokens:} (93.6\% of tokens)  Tokens in the final position of a bigram which did not previously appear in context or it appeared  $> 1,250$ times during training. 
\end{enumerate}

In Figure \ref{fig:main}, we visualize these slices by plotting log-perplexity against the frequency with which bigrams appear during training. Strikingly, the gap between attention and gated convolutions is the largest on AR hits with the fewest occurrences. On the other tokens, there is no gap. 

In \Cref{table:ppl-slices}, we also compute the percentage of the difference in perplexity between attention and each model that is due to AR tokens: $\frac{\Delta \log(\phi_\text{AR}) \cdot |T_\text{AR}|}{\Delta \log(\phi) \cdot |T|}$, where $\phi$ is the perplexity and $T$ is the set of tokens in the test set. 
This quantity can also be interpreted as the fraction of the overall gap that would close if a model matched attention on the AR slice. 
We find that the AR slice accounts for 82\% of the average quality gap between the gated convolutions and attention. 

To evaluate the AR capacity of larger models, we train two 1.4 billion parameter attention and Hyena models for 50 billion tokens on the Pile and repeat this analysis (see \cref{table:ppl-slices-large}). 
Strikingly, a 70 million parameter attention model is a full perplexity point better in the AR slice than this 1.4B Hyena model that is 20$\times$ its size (2.41 vs. 3.43 ppl.). 
In \cref{sec:scaling}, we also evaluate open-source RWKV and attention models trained up to 7 billion parameters. We use a controlled semi-synthetic dataset to measure AR capacity and show that RWKV's performance degrades sharply as we increase the number of queries in an example while attention performs consistently well \cref{sec:scaling}.

\subsection{Formalizing the problem: Multi-Query Associative Recall}
This gap in associative recall perplexity is very surprising because prior work shows gated-convolutions can perfectly solve a formalized version of the task~\citep{dao2022hungry,poli2023hyena,olsson2022context}.  
In this synthetic task, the input $\bm{x}$ contains a sequence of bigrams representing \textit{key-value} pairs from a random dictionary followed by a \textit{single} query token. For example, the correct output for the input below would be $3$: 
\vspace{-1mm}
$$
\text{A 4} \underbrace{\text{B 3}}_{\mathclap{\textbf{Key-Value}}}  \text{C 6 E 2 F 1 C 6 G 8 } \rightarrow
 \underbrace{\text{B ?}}_{\mathclap{\textbf{Query}}}
$$
Gated convolutions (\textit{e.g.} H3, Hyena, RWKV) can solve this task perfectly for most sequence lengths. 

These conclusions are inconsistent with our findings on the Pile, as described above, so we ask how this formulation of AR differs from the way AR manifests in real language.
We identify a major difference. In real world inputs, the language model often needs to perform multiple associative recalls in a single forward pass, at varying positions in the sequence (\textit{e.g.} \textit{``Tim Rice"} and \textit{``March 2018"} in \cref{fig:main}. We refer to this as Multi-Query AR ({$\Task$}).
\label{sec:theory-ar-def}
We formally define the {$\Task$} problem as follows: \footnote{In \Cref{sec: intro-general-ar} we define a formal, general form of this definition, used in the theoretical analysis.}

\begin{definition} [Multi-Query-{AR} $\paren{\Task}$]
    \label{def: general-AR} 
    We are given an input sequence 
    $\bm{x} = \{x_0, \ldots, x_{N-1}\}$ where each $x_i \in C$ is a token drawn from a vocabulary of size $c = |C|$. 
    The task 
    is to check, for every query
    $1\le i<N$, whether there exists a 
    $0 \le j < i$ such that $\bm{u}_i \equiv \bm{u}_j$. If so, output $\bm{u}_{j+1}$. 
\end{definition}

For example, the correct output for input below would be $\text{4, 6, 1, 2, 3}$:
$$
\text{A 4 B 3 C 6}\underbrace{\text{F 1}}_{\mathclap{\textbf{Key-Value}}}  \text{E 2} \rightarrow \text{A ? C ?} \underbrace{\text{F ?}}_{\mathclap{\textbf{Query}}} \text{E ? B ?}
$$
\vspace{-1mm}
In Section \ref{sec:theory-explanation}, we use {$\Task$} to explain the quality gap between gated convolutions and attention.

%% file: Sections/sec_4_framework.tex
\label{sec:theory-explanation}
In this section, we provide an explanation for the gap in associative recall performance by analyzing the formal ${\Task}$ task theoretically and empirically. In \cref{sec:theory-coyote-def}, we define a simple gated-convolution architecture, called \Coyote, which we show can simulate a broad class of architectures built from gating and convolutions. This allows us to make general statements that apply to popular gated-convolution architectures like Hyena, RWKV, or H3. 
In \cref{sec:theory-capacity}, we show that there exist theoretical solutions to {$\Task$} that could in principle be learned by \Coyote, and we analyze their complexity in terms of model width and depth. 
In \cref{sec:empirical-capacity}, we use experiments on synthetic data to show that solving ${\Task}$ with ${\Coyote}$ (and other gated-convolution architectures) requires model dimension to scale linearly with the sequence length. In contrast, attention solves {\Task} consistently in our experiments with model dimension scaling independently of sequence length. These empirical scaling laws provide a potential explanation for the AR gap and, alongside our theoretical analysis, point to the potential solutions discussed in \Cref{sec:sec5-method}.

\subsection{\Coyote: a minimal gated convolution operator}
\label{sec:theory-coyote-def}
In this section, we define our minimal gated-convolution architecture, called \Coyote.
Given a function, we would like to know the most efficient model (\textit{e.g.} parameters, FLOPs) that can represent the solution. In this work, we show we can precisely reason about this question for representing \textit{polynomial functions} with gated convolutions as gating and convolutions are both polynomial operations. The standard model defining computational complexity for polynomials is by the size of the smallest arithmetic circuit that can compute the polynomial. We define the {\Coyote} gated convolution operator which is exciting because (1) it is universal in that it that can simulate any arithmetic circuit $\calC$ (with only a poly-log blowup in the corresponding parameters) and (2) it is simple to implement efficiently (19 lines of pure PyTorch including imports, see \cref{app:implementation}).

\begin{definition}[{\Coyote} Operator] 
    Given an input $\hyenaInput\in \R^{N \times d}$, the {\Coyote} operator for layer $\ell$ is defined as:
    \begin{equation}
        \label{eq: coyote-recursion}
        \begin{aligned}
            \bm{y}
            &:= 
            \underbrace{\paren{\bm{u} \cdot \bm{W}^\ell+\bm{b}_1^{\ell}}}_{\mathclap{\textbf{Linear Projection}}}
            \odot 
            \underbrace{\paren{\bm{h}^{\ell} \ast \bm{u}+\bm{b}_2^{\ell}}}_{\mathclap{\textbf{Convolution}}}
            \quad 
        \end{aligned}
    \end{equation}
    where the layer is parameterized by learnable filters $\bm{h}\in\R^{N\times d}$, a linear projection $\bm{W}^\ell\in\R^{d\times d}$, 
    and `bias' matrices $\bm{b}_1,\bm{b}_2\in\R^{N\times d}$. The $\odot$ is component-wise product and convolution of two matrices is computed as convolution of the corresponding columns.
\end{definition}

In our experiments, each ${\Coyote}$ layer uses $\tilde{\calO}(Nd + d^2)$ parameters\footnote{We use $\tilde{\calO}{(\cdot)}$ to hide poly-log factors.}
and can be computed in $\tilde{\calO}(Nd^2)$ operations. For our theoretical results, we can assume the weight matrix $\bm{W}^\ell$ is restricted to a class of matrices that support near-linear time matrix multiplication (\textit{e.g.} Kaleidoscope matrices, see \cref{def: W-kmat}). Under this assumption, ${\Coyote}$ uses $\tilde{\calO}(Nd)$ parameters and $\tilde{\calO}(Nd)$ FLOPs~(\cref{prop: single-baseconv}). 
We now state the equivalency result between arithmetic circuits and \Coyote\,  a ``canonical" representation of arithmetic circuits~(\cref{thm: gen-ac} in \cref{sec: arithmetic}):

\begin{theorem}[Equivalency to Arithmetic Circuits]
    \label{thm:equiv}
    For an arithmetic circuit $\calC$ of size $s$ and depth $\Delta$ that takes $\bm{u} \in \R^{N \times d}$ as input, there exists an equivalent ${\Coyote}$ operator that uses $\tilde{\calO}(s \Delta)$ parameters and $\tilde{\calO}(\Delta)$ layers.\footnote{The formal statement in the Appendix has a sharper version of this result in terms of the circuit `width'.}
\end{theorem}
In other words, any gated convolution model with small number of layers can be simulated by \Coyote\ with only a (poly)logarithmic blowup in parameters and layers. We note that arithmetic circuits are a very well studied computation model in computational complexity \cite{bürgisser1996algebraic}. Many well-known efficient algorithms on matrices (e.g. the FFT or the current best known matrix-matrix multiplication algorithm) in fact give small arithmetic circuits. However, arithmetic circuits are inherently discrete objects -- we cannot learn them via gradient descent. \cref{thm:equiv} shows that (up to poly-log loss in parameters), we can \textit{instead} learn over {\Coyote} models. This result generalizes a similar result from~\cite{dao2020kaleidoscope} for the special class of linear functions: we generalize the earlier result to the class of {\em all} polynomials.

For specific gated convolution layers, we can get rid of the poly-logarithmic factor blowup--we observe in the appendix that \Coyote\ and Hyena models can simulate each other with only a small constant blowup in parameters~(\cref{prop: coyote-hyena-equiv} in \cref{sec: arithmetic}).

\subsection{Theoretical analysis of gated convolution capacity and associative recall}
\label{sec:theory-capacity}

In this section, we provide theoretical ${\Task}$ solutions that could in principle be learned by each architecture and analyze their complexity in terms of model width and depth. 
First, we note that attention solves $\Task$ with parameters independent of sequence length (\cref{prop: app-attention}). 
\begin{proposition}[Attention]
\label{prop: attention-ar}
    Given an input $\hyenaInput\in \{0,1\}^{N \times 3c}$, Attention (even without using soft-max) solves $\Task$ for $\hyenaInput$ using {${\calO}(c^2)$} parameters, $\calO(Nc^2 + N^2c)$ time complexity and ${\calO}(1)$ layers.
\end{proposition}
It is natural to wonder then if \textit{all} pairwise comparisons among tokens are necessary to solve $\Task$. Indeed, in the RAM setting, a sequential algorithm can simply utilize $N$ logarithmic insertion and membership queries to solve $\Task$ in subquadratic time. Unfortunately, any model attempting to emulate this would require $\Omega(N)$ layers. Instead, we observe that we can parallelize this algorithm using dyadic intervals and achieve a depth of $\tilde{\mathcal{O}}(1)$~(\cref{prop: pram-gen-ar}). We then convert this algorithm into an arithmetic circuit and apply \cref{thm:equiv} to derive an equivalent \Coyote\ model. This allows us to prove new upper bounds for \Coyote\ models applied to $\Task$, which improves upon the quadratic time complexity of attention to near-linear runtime at the cost of using poly-log layers (\cref{thm: gen-ar-coyote} in \cref{sec: gen-ar}). 

\begin{theorem}[Data-Independent Filters\footnote{
We note here that existing architectures also use data-independent convolution filters, meaning the filter is defined as a function of the model parameters, independent of the input.}]
    \label{thm: indep-genar}
    Given an input $\bm{u} \in \{0,1\}^{N \times \calO(\log{c})}$ to $\Task$ (where we assume that distinct tokens are embedded into distinct vectors in $\{0,1\}^{\calO(\log{c})}$), there exists a ${\Coyote}$ operator that solves $\Task$ for $\hyenaInput$ using $\tilde{\calO}(N\log{c})$ parameters as well as time complexity and $\tilde{\calO}(1)$ layers. 
\end{theorem}
Nevertheless, the poly-logarithmic number of layers in the above result is undesirable in practice. But, we show that using {\em input-dependent} convolution filters, one can get constant many layers (for a sub-class of inputs). Towards that end, we define the interaction distance between a query $\bm{q}_i$ and the matching key $\bm{k}_j$ as $i-j$. This then allows us to present the corresponding upper bound for data-dependent mixing~(\cref{thm: input-dep-genar} in \cref{sec: data-dep-ar}).

\begin{theorem}[Input-Dependent Filters]
    \label{thm: dep-genar}
    Given an input $\hyenaInput\in \{0,1\}^{N \times c}$ to $\Task$ (where we assume that the tokens are embedded as one-hot encoding in $\{0,1\}^c$ and there exists at most $t$ distinct interaction distances),\footnote{Note that the interaction distances can be arbitrary: there is just a bounded number of distinct distances.} there exists a ${\Coyote}$ operator that uses {\em input-dependent kernels} to solve the above case of $\Task$  using $\calO(t \cdot Nc)$ parameters and ${\calO}(1)$ layers.
\end{theorem}

\subsection{Empirical analysis of gated convolution capacity and associative recall}
\label{sec:empirical-capacity}
\input{Figures/Main/Figure2/figure}

In this section, we measure empirically how model dimension must scale in order for different sequence mixers to solve \Task. 

\paragraph{Setup} We train and evaluate models on a synthetic ${\Task}$ with vocabulary size $8,192$, varying model dimension and sequence length from $64$ to $512$. \Cref{app:synthetic} provides further details on the formulation and construction of this synthetic task. 
Following \cite{olsson2022context}, we train two layer models with a Transformer backbone that interleaves sequence mixing and state mixing (MLPs). For each architecture, we sweep four learning rates from $\log(-4)$ to $\log(-2)\}$ for each architecture, and report maximum test accuracy. 

Our results, which are summarized in \Cref{fig:capacity}, support two main claims:

\textbf{Claim 1 (Gated-convolutions and attention).} \textit{Gated-convolution models with two layers require model dimension to scale at least linearly in sequence length in order to solve associative recall, while attention models can solve it with near-constant dimensionality.}
We compare attention and ${\Coyote}$ as well as three popular instantiations of gated-convolution architectures: RWKV, H3, and Hyena \citep{peng2023rwkv, fu2023simple, poli2023hyena}.
In the top row of \cref{fig:capacity}, attention solves ${\Task}$ perfectly at all sequence lengths using a constant model dimension of 64. In contrast, ${\Task}$ does not achieve accuracy $>0.9$ unless $d \geq N$.

\label{sec:data-dependent-synthetics}

\textbf{Claim 2 (Input-dependent filters).} \textit{Using input-dependent filters in gated-convolution models can close some of the gap to attention.}
In \cref{thm: dep-genar}, we show that ${\Coyote}$ with input-dependent filters could solve ${\Task}$ with improved scaling. In this solution, we construct a filter that spikes at position $j$ if matching keys are separated by $j$ tokens. We evaluate two approaches for constructing this filter: (1) programatically (\textit{i.e.} hard-coded comparisons between token ids) or (2) with autocorrelation, which could learn to perform fuzzy-matches (see \cref{sec: data-dep-ar}). In the bottom row of \cref{fig:capacity}, we see that $\Coyote$ with programmatic input-dependent filters achieves near-constant scaling in model dimension and that $\Coyote$ with autocorrelation input-dependent filters achieves improved scaling over $\Coyote$ with input-independent filters.

These input-dependent filters cannot easily be made to satisfy causality and using an $O(N \log N)$ filter per gap could be expensive if each gap applies only to a small number of bigrams. 
A simpler and perhaps more efficient way to solve the problem would be to introduce a small amount of attention to an otherwise {\Coyote{} model~\citep{dao2022hungry}. As the first natural baseline, we evaluate an attention hybrid on synthetic {\Task} and show that it achieves improved scaling in \Cref{fig:capacity}. Next, we put these empirical and theoretical insights into practice on Pile language modeling.

%% file: Figures/Main/Figure2/figure.tex
\begin{figure}
    \includegraphics[width=\linewidth]{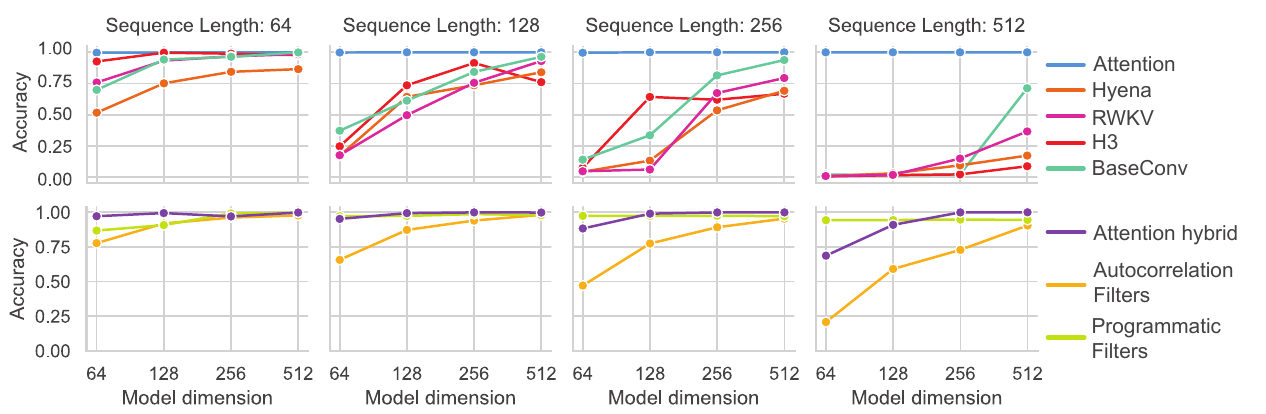}
    \caption{The x-axis is the model dimension and the y-axis is accuracy on {$\Task$}. Increasing the sequence length correlates with increased task difficulty. Both rows use the same experimental setup, but evaluate different sequence mixers. \textbf{(Top, Claim 1) Gated convolutions and attention.} We observe the gated convolutions require larger dimensionality than attention to solve the task. \textbf{(Bottom, Claim 2) Input-dependent filters.} Models with input-independent aggregation achieve improved scaling over the gated-convolutions with input-independent aggregation.}
    \label{fig:capacity}
    \vspace{-3mm}
\end{figure}

%% file: Sections/sec_5_method.tex
\label{sec:sec5-method}

In this section, we evaluate hybrid \Coyote-Attention models that leverage different sparsity patterns. 
We show that hybrids with input-dependent sparsity patterns can close most of the gap to attention, while maintaining sub-quadratic scaling. 
We also show that ${\Coyote}$ hybrids can outperform attention-only models by up to a full perplexity point, all while being dramatically simpler to implement and analyze than prior hybrids~\citep{dao2022hungry} (see implementation in \cref{app:implementation}. 
We describe the architectures in \Cref{sec:method-arch} and results on Pile language modeling in  \Cref{sec:data-dependent-pile}.

\paragraph{Sparse ${\Coyote}$-Attention Hybrids}
\label{sec:method-arch}
We evaluate hybrids composed primarily of ${\Coyote}$ layers and three attention layers (6.3\% of layers at 354M parameters and 10\% at 168M parameters).  
We augment the attention layers with operators that \textit{selectively} apply attention to some tokens based on a selection function $f:\mathbb{R}^{N\times d} \rightarrow \{0,1\}^N$. They take as input $\bm{u} \in \mathbb{R}^{N \times d}$ and output $\bm{y} \in \mathbb{R}^{N \times d}$:
\begin{equation}
    \label{eq:selective}
    \bm{y}[i, :] = \softmax(\frac{1}{\sqrt{d}}\bm{q}[i, :]\bm{k}^\top)\bm{v} \cdot f(\bm{u})[i] 
\end{equation}
\vspace{-2mm}

where $\bm{q}, \bm{k}, \bm{v}$ are query, key, and value projections, as in attention. We evaluate four choices for $f$:

(1) \textit{Full attention.} First, we evaluate the performance of full attention hybrids. These hybrids correspond to fixing $f(\bm{u})[i] = 1$ for all $i$ in \cref{eq:selective}. Prior work has shown full attention hybrids to be effective, but they do not explain why supplementing gated convolutions with attention is necessary~\citep{dao2022hungry}. Based on our findings in \cref{sec:sec3-downstream}, we hypothesize that sparse attention applied only to associative recall hits may suffice.

(2) \textit{Random selection.} As a control, we evaluate sparse attention randomly applied to tokens. This corresponds to a stochastic selection function where $f(\bm{u})[i]$ is drawn from a Bernoulli. Many prior works have employed sparse attention patterns like this one, which are independent of the input~\citep[inter alia.]{zaheer2020bigbird,child2019generating,beltagy2020longformer}.

(3) \textit{Programmatic selection.} Next, to evaluate our hypothesis that attention is needed for associative recall, we prototype a programmatic selection function that selects only those tokens that might be associative recall hits.  Specifically, $f(\bm{x}[i, :])$ is $1$ if the token $x_i$ previously occurred in the sequence. In practice, we compare raw token ids, not token embeddings.
\begin{equation}
    \label{eq:programmatic}
    f(\bm{x})[i] = 
    \begin{cases} 
        1 & \text{if there exists}  \quad j < i  \quad \text{such that} \quad x_i = x_j \\
        0 & \text{otherwise } 
    \end{cases}
\end{equation}
\vspace{-2mm}

\input{Tables/highways}

(4) \textit{Learned selection.} Finally, we prototype a learned selection function $f(\bm{u})[i] = \sigma(\bm{u}[i, :] \cdot \bm{W})$ parameterized as a simple linear layer with sigmoid activation. We fix a hyperparameter $k$ and select the top-$k$ tokens with the highest score in each batch. This allows us to compute attention in $\calO(ndk)$ time. During training, we add small amount of Gaussian noise to the top-$k$ calculation to encourage exploration and use an auxiliary loss to encourage sparse selection: $\ell_f(\mathbf{u}) = \frac{1}{N}\max(0, \sum_{i=1}^N f(\bm{u})[i] - k )$. This approach is most similar to the recently proposed SeqBoat architecture~\citep{ren2023sparse}. We discuss the differences in \Cref{app:related-work}. 

\paragraph{Downstream Evaluations}
\label{sec:data-dependent-pile}
We evaluate the prototypes on Pile language modeling. We take ${\Coyote}$ architectures at the 150M and 360M parameter scales and add input-dependent selection to three layers (details in \Cref{app:exp-gated-convs}). Results are in \Cref{table:pile-input-dependent}. We validate that the prototypes close the overall and AR quality gaps to attention when added to the ${\Coyote}$ backbone. 

At 360M parameters, ${\Coyote}$ with just $3$ attention layers can outperform the Transformer, while requiring fewer FLOPs. Our hybrid attention-${\Coyote}$ models outperform attention only models by 0.85 perplexity points while enabling an 18\% reduction in total FLOPs vs attention. However, this uses full quadratic attention. We next show that sparse attention localized to potential AR tokens, is also sufficient to close the gap, validating our insights on the role of input-dependence for MQAR. At 360M parameters, programmatic selection closes 85\% of the gap between pure ${\Coyote}$ and attention on the AR slice, in contrast to the random selection control. Learned selection closes 72\% a of the gap using just $k=256$ (sub-quadratic) attention positions per example.

%% file: Tables/highways.tex
\begin{table}[]
\centering
\small
\begin{tabular}{@{}lccccc@{}}
\toprule
 & & & \textbf{Overall} & \multicolumn{2}{c}{\textbf{Slices}} \\ 

Model  & Param (M) & TFLOPs &  & AR  Hits & Other Tokens\\
\midrule
\midrule
Attention & 125 & 2.46 & 11.01 (2.40) & 2.16 (0.77) & 12.45 (2.52) \\
\midrule
${\Coyote}$ & 168 & 2.46 & 12.90 (2.56) & 8.68 (2.16) & 13.29 (2.59)  \\
+ Random selection  & 162 & 2.44 & 12.13 (2.50) & 5.33 (1.67) & 12.83 (2.55) \\
+ Programmatic selection  & 162 & 2.44 & 11.06 (2.40) & 2.70 (0.99) & 12.19 (2.50) \\
+ Learned selection  & 166  & 2.44 & 11.06 (2.40) & 3.20 (1.16) & 12.14 (2.50)   \\
+ Full attention  & 166 & 2.58 & 9.57 (2.26) & 2.01 (0.70) & 10.76 (2.38) \\

\midrule
\midrule
Attention & 360  & 6.23 & 9.44 (2.25) & 1.98 (0.69) & 10.62 (2.36) \\
\midrule
${\Coyote}$ & 354 & 4.81 &  11.01 (2.40) & 5.98 (1.79) & 11.52 (2.44)  \\
+ Random selection  & 365 & 5.06 & 12.94 (2.56) & 6.17 (1.82) & 13.62 (2.61) \\
+ Programmatic selection  & 365 & 5.06 & 9.54 (2.26) &  2.35 (0.86) & 10.50 (2.35) \\
+ Learned selection  & 351 & 5.06 & 9.59 (2.26) & 2.61 (0.96) & 10.58 (2.36)  \\
+ Full attention  & 351 & 5.10 & 8.59 (2.15) & 1.95 (0.67) & 8.91 (2.19) \\

\bottomrule
\end{tabular}
\caption{\textbf{Language model perplexity on slices of the PILE.} 
We evaluate Hyena and {\Coyote} with Hybridization and Selective look-up at 160 and 355M parameters. We validate that the methods enable the gated convolutions to outperform attention.
}
\label{table:pile-input-dependent}
\vspace{-3mm}
\end{table}

%% file: Sections/sec_6_conclusion.tex
We present an extensive analysis of gated convolution architectures in light of their recent popularity. We identify a persisting quality gap between efficient convolution and inefficient attention based architectures, largely due to a single failure mode associative recall. We design a new multi-query associative recall (${\Task}$) analysis tool, which correlates with downstream AR quality. We theoretically and empirically explain the gap is due to insufficient data-dependent mixing in gated convolutions and we show minimal architectures that  close the gap on the Pile.

Our results in analyzing gated convolutions go beyond the conventional wisdom that attention is the ``\textit{right}'' model. A significant amount of work focuses on improving the efficiency of attention \citep{dao2022flashattention, dao2023flashattention2, katharopoulos-et-al-2020} and theoretically studying the exact power of attention \citep{hahn2020theoretical, merrill2022saturated, keles2023on}. Attention is often used as the goalpost for what is needed downstream. We hope our contributions highlight the value of {$\Task$}, and more broadly tasks tied to real language modeling, as a proxy to study.

%% file: Sections/appendix/related.tex
\section{Extended Related Work}
\label{app:related-work}
We are inspired by and build on prior work in mechanistic interpretability (\Cref{sec:related-mech}), efficient architecture design (\Cref{sec:related-eff}), and input-dependent architectures (\Cref{sec:related-inp}).

\subsection{Mechanistic Interpretability, Synthetic Languages, and Associative Recall} 
\label{sec:related-mech}
Work in mechanistic interpretability aims to decompose the capabilities of a neural network into human-understandable algorithms that can be attributed to specific parameters in the model~\citep{olah2022mechanistic, power2022grokking,elhage2021mathematical,cammarata2020thread}. Some of these works, use synthetic data to validate mechanistic interpretations of neural networks~\citep{olsson2022context}. This relates to a broader line of work using synthetic languages to study language model architectures~\citep{wang2016galactic,white2021examining,allen2023physics,ravfogel2019studying,xie2021incontext}. 
\textit{Mechanistic design} puts mechanistic interpretations to use in designing new architectures and learning algorithms.  Several works in architecture research have used synthetic tasks to validate architecture designs~\citep{kitaev2020reformer}.

Our work is focused on one particular synthetic task: \textit{associative recall}. Motivated by psychological models of how humans associate and retrieve information, work in the early days of neural network research focused on developing systems capable of associative recall~\citep{willshaw1969non,feldman1981parallel,hopfield1982neural}.  For example, Hopfield networks, proposed in 1982, are a recurrent neural network explicitly designed to support associative (``content-addressable") memory~\citep{hopfield1982neural}. 
More recently, several notable recurrent neural network mechanisms (\textit{e.g.} neural turing machines, LSTMs, and RNNs with fast weights) were evaluated on a synthetic version of associative recall~\citep[inter alia.]{graves2014neural,ba2016using, zhang2017learning}. These works use a formulation of associative recall very similar to the single-query associative recall in our work.
Since the rise of large language models, several works have argued that LLMs ability to perform \textit{in-context} learning is due, at least in part, to the associative recall capabilities of attention~\citep{elhage2021mathematical, olsson2022context}. 

\subsection{Efficient Language Modeling Architectures}
\label{sec:related-eff}
We first briefly review the efficiency motivations for the recent excitement around gated convolution architectures. While attention requires compute that scales as $\calO(N^2)$ in sequence length $N$, convolutions scale as $\calO(N \log N)$~\citep{cooley1965algorithm}. Ideal, inference complexity is $\calO(1)$ in sequence length, as provided by recurrent neural networks. State-space models can be computed either as a convolution or recurrence, to achieve both sub-quadratic training and constant inference complexity in sequence length~\citep{gu2021efficiently}. Architectures that use \textit{implicit} convolutional filters \citep{poli2023hyena}, can be converted to an SSM via a simple distillation step~\citep{massaroli2023laughing}.

\subsection{Input-Dependence in Sequence Modeling Architectures}
\label{sec:related-inp}
In an input-dependent sequence model, the way tokens are aggregated across the sequence is controlled by the data, not just the model parameters. We highlight several prior works related to our study that explore input-dependent sequence models.
\begin{itemize}
    \item \textbf{Attention} \citep{vaswani2018attention} achieves input-dependent sequence mixng since the $Q$, $K$, and $V$ terms are achieved via linear combinations of the input $x$. For instance, $Q = xW_Q$ for some learned weight matrix $W_Q$.
    \item \textbf{Input-Independent Convolution Architectures}. Given the quadratic scaling of attention, \cite{gu2021efficiently} propose S4, an architecture that use long convolutions to process the input. CKConv \cite{romero2022ckconv} uses a convolution filter that is implicitly parametrized by an MLP of the relative positions of the observations in a time-series. A line of subsequent work improved upon S4 by changing the parametrization \citep{gupta2022diagonal, gu2022parameterization, mehta2022long, ma2022mega, fu2023simple} or changing the SISO convolutions to MIMO convolutions \citep{smith2023simplified}. The convolution filter in each of these architectures is input-independent. The linear convolution layer alone cannot perform MQAR \citep{dao2022hungry}.
    \item \textbf{Input-Dependence via Gating}. Since pure convolution architectures with input-independent filters struggle to perform associative recall, an important task in in-context learning \citep{olah2022mechanistic}, subsequent work proposed to introduce input-dependence by adding a \textit{gating} or element-wise multiplication operation to the architecture, where both the inputs to the operator are defined in terms of the input \citep{dao2022hungry, poli2023hyena}. The multiplication is generally $y = \sigma(Wx) \odot x$. The filters remain input-independent. These architectures demonstrate promising results on the associative recall synthetics proposed in prior work \cite{dao2022hungry, poli2023hyena} and provide large downstream improvements over S4.
    
    Prior work suggests the models match attention in quality \citep{poli2023hyena, fu2023monarch, peng2023rwkv} and this has led to increasing use of these architectures \citep{nguyen2023hyenadna, toews2023transformers}, 
    however our results suggest there is still a sizable gap to attention. We show the gap is largely due to the models' recall abilities and capture this behavior in a novel synthetic task. We theoretically show that dimension scales in sequence length when solving MQAR with gated convolutions.

    Selective S4 \cite{wang2023selective} learns which positions to mask prior to passing the input to the next long convolution (S4) layer. This resembles gating (which also controls which information flows forwards to the next layer). 

    \item \textbf{Input-Dependent Convolution Architectures} Prior work introduces convolution architectures with input-dependent filters \citep{yang2020condconv, kosma2023time}. For instance, \cite{yang2020condconv} parametrizes the convolution filter as a linear combination of $n$ projections of the input. Liquid S4 \cite{hasani2022liquid} is also motivated by introducing input-dependence to S4. To accomplish this, the work proposes including correlation terms between the input tokens during state mixing. We evaluate Liquid S4 and find the architecture lags attention on the Pile by 4.4 perplexity points in \Cref{table:liquid-s4}.
    \input{Tables/liquid_ssm}

    \item \textbf{Recurrent Architectures} The pure long-convolution architectures can also be computed as recurrences. Letting $s_i$ represent the hidden state at time $i$,  $u_i$ be the input at time $i$, and $A, B, C$ be projection matrices, the recurrence computes: 
    $$s_{i+1} = f(As_i + Bu_i)$$
    $$y_{i} = Cs_i$$
    In a linear RNN like S4, the $A$, $B$, and $C$ matrices are input-independent and $f$ is an identify function. However, recent work proposes input-dependent RNNs (e.g. RetNet \cite{sun2023retentive}), where a subset of these matrices is input-dependent (e.g. $A = xW_a$). Note that linear attention mechanisms can also be expressed as input-dependent RNNs \citep{katharopoulos2020transformers}. While these architectures improve over gated convolutions in {\Task} synthetic experiments and downstream quality on the Pile (\Cref{app:synth-results}), we observe and prove that the required RNN hidden state dimensionality grows with the number of key-value pairs that the model needs to recall in the sequence (\Cref{app:retnet-proof}). In contrast, attention complexity does not scale with the number of key-value pairs to recall. 
    \item \textbf{Hybrid Architectures} Finally, we can combine architectural components that provide different capabilities. Prior work proposes architectures that hybridize sub-quadratic layers and attention \citep{dao2022hungry, ma2022mega, ren2023sparse}, but does motivate this choice from a mechanistic design perspective. Further, H3 without attention \cite{dao2022hungry} and MEGA \citep{ma2022mega}, which uses blocked attention, underperform attention on language modeling. 
    
    SeqBoat \citep{ren2023sparse} shares closest resemblance to our learned selection module. SeqBoat introduces a different module that learns where to use attention however, their model can end up using full attention at a layer, providing quadratic scaling in the worst case. Different from SeqBoat, we use an auxiliary loss to encourage sparsity in the selected attention positions and select the top-$k$ positions to remain sub-quadratic. 
\end{itemize}

Overall, our work finds that previously proposed sub-quadratic models do not efficiently solve the {$\Task$} task.

%% file: Tables/liquid_ssm.tex
\begin{table}[h!]
\centering
\small
\begin{tabular}{@{}lccccc@{}}
\toprule
 &  & \textbf{Overall} & \multicolumn{2}{c}{\textbf{Slices}} & \textbf{\% of gap due to} \\ 
Model & Param (M) & & AR Hits & Other Tokens &  \\
\midrule
Attention & 125 & \textbf{12.37} (2.52) & \textbf{2.32 (0.84)} & 14.04 (2.64) & ---  \\
Liquid S4 & 145 & \textbf{16.80} (2.82) & \textbf{22.75 (3.12)} & 16.42 (2.80) &  52.6\%  \\
\bottomrule
\end{tabular}
\caption{Validation perplexity on the Pile after pretraining on 5B tokens. We report log perplexity with negative log-likelihood in parentheses. AR vs. non-AR token slices are defined in \Cref{sec:sec3-downstream}.}
\label{table:liquid-s4}
\end{table}

%% file: Sections/appendix/code.tex
\section{Code Listing}
\label{app:implementation}

In this section, we include code listings for the {\Coyote} operator. We begin with the a standard {\Coyote} implementation that uses an explicitly-parameterized long convolution. Below, we also provide the implementation for \Coyote with an implicitly-parameterized long convolution. 
\begin{lstlisting}[language=Python,style=mystyle,caption={\textbf{Explicit {\Coyote} implementation.} Implementation of the BaseConv layer with explicitly-parameterized long convolutions. The implemtnation is 19 lines excluding comments and whitespace.}]
import torch

def fft_conv(u: torch.Tensor, k: torch.Tensor):
    """
    Args:
        u (torch.Tensor): (batch_size, d_model, seq_len)
        k (torch.Tensor): (d_model, l_max)
    Return: 
        y (torch.Tensor): (batch_size, d_model, seq_len)
    """
    seqlen = u.shape[-1]
    fft_size = 2 * seqlen
    k_f = torch.fft.rfft(k, n=fft_size) / fft_size
    u_f = torch.fft.rfft(u.to(dtype=k.dtype), n=fft_size)
    y = torch.fft.irfft(u_f * k_f, n=fft_size, norm="forward")[..., :seqlen]
    return y

class BaseConv(torch.nn.Module):
    
    def __init__(self, d_model: int, l_max: int, **kwargs):
        super().__init__()
        self.d_model, l_max = d_model, l_max, 
        self.projection = torch.nn.Linear(self.d_model,  self.d_model)
        self.filter = torch.nn.Parameter(torch.randn(self.d_model, l_max), requires_grad=True)      

    def forward(self, u: torch.Tensor):
        """
        Args:
            u (torch.Tensor): (batch_size, d_model, seq_len)
        Return:
            y (torch.Tensor): (batch_size, d_model, seq_len)
        """
        u_conv = fft_conv(u.transpose(1, 2), self.filter).transpose(1, 2)
        u_proj = self.projection(u)
        y = u_conv * u_proj
        return y + u

\end{lstlisting}

In some of our experiments, we interleave {\Coyote} layers that use explicit short convolution filters (like those in \texttt{torch.nn.Conv1d}) and with \Coyote layers that use implicit long convolution filters (like those described in \citet{poli2023hyena}). 

Below we include the code for {\Coyote} with an implicit convolution.
\label{app:implementation}
\begin{lstlisting}[language=Python,style=mystyle,caption={\textbf{Implicit {\Coyote} implementation.} Implementation of the {\Coyote} layer with implicitly-parameterized long convolutions. (The implementation is 34 lines excluding comments and whitespace).}]
class PositionalEmbedding(nn.Module):
    def __init__(self, emb_dim: int, seq_len: int, **kwargs):
        """Complex exponential positional embeddings for implicit long convolution filters."""
        super().__init__()
        t = torch.linspace(0, 1, seq_len)[None, :, None]  # 1, L, 1
        bands = (emb_dim - 1) // 2
        t_rescaled = torch.linspace(0, seq_len - 1, seq_len)[None, :, None]
        w = 2 * math.pi * t_rescaled / seq_len  # 1, L, 1
        f = torch.linspace(1e-4, bands - 1, bands)[None, None]
        z = torch.exp(-1j * f * w)
        z = torch.cat([t, z.real, z.imag], dim=-1)
        self.z = nn.Parameter(z, requires_grad=False)

    def forward(self, L):
        return self.z[:, :L]

class BaseImplicitConv(nn.Module):
    """
    BaseConv with implicit filter parameterized by an MLP.

    Args:
        d_model (int): Number of expected features in input and output.
        l_max (int): The maximum sequence length.
        d_emb (int, optional): Dimension of the positional embeddings. Must be odd and $\geq$ to 3 (time, sine and cosine). Defaults to 3.
        d_hidden (int, optional): The number of features in the hidden layer of the MLP. Defaults to 16.
    """

    def __init__(self, d_model: int, l_max: int, d_emb: int=3, d_hidden: int = 16,):
        """
        Long convolution with implicit filter parameterized by an MLP.
        """
        super().__init__()
        self.pos_emb = PositionalEmbedding(d_emb, l_max)
        self.filter_mlp = nn.Sequential(nn.Linear(d_emb, d_hidden), torch.nn.ReLU(), nn.Linear(d_hidden, d_model))
        self.projection = torch.nn.Linear(d_model, d_model)

    def forward(self, u: torch.Tensor, *args, **kwargs):
        """
        Args:
            u (torch.Tensor): (batch_size, seq_len, d_model)
        Return:
            y (torch.Tensor): (batch_size, seq_len, d_model)
        """
        filter = self.filter_mlp(self.pos_emb(u.shape[1])).transpose(1, 2)
        u_conv = fft_conv(u.transpose(1, 2), filter).transpose(1, 2).to(dtype=u.dtype)
        u_proj = self.projection(u)
        y = u_conv * u_proj
        return y + u

\end{lstlisting}

%% file: Sections/appendix/experiment.tex
We use A100 80GB Nvidia GPUs to run all experiments. We use the reference training infrastructure from \url{https://github.com/EleutherAI/gpt-neox} for all pretraining runs. The Pile data is tokenized using the GPT2BPETokenizer and all models see the data in the same order.

\label{sec:appendix_experiment_details}
Below we provide details on the hyperparameters and settings for training each architecture studied in the paper on the real-world Pile data. In \Cref{app:exp-measure-ar} we summarize and justify our method for measuring the quality gap due to associative recall between the convolution architectures and attention. In \Cref{app:exp-gated-convs}, we provide details on the pure gated convolution architectures, attention baseline, and hybrid architectures studied in the paper. 

\subsection{Measuring the {$\Task$} Gap on Real Language Data}
\label{app:exp-measure-ar}
Here we summarize our method for computing the amount of quality gap between the convolution and attention models that is ascribed to associative recall capability (e.g., in \Cref{table:ppl-slices}):
\begin{enumerate}
    \item Given an input sequence, we identify recurring bigrams (i.e. bigrams that have already appeared in the sequence at a prior position). Since bigrams that appear frequently during training may be memorized by the model, rather than requiring the model to perform recall at inference-time, we only measure AR log-probabilities with respect to bigrams that are seen fewer than a threshold number of times during training. The threshold used in all the experiments in our submission is $1,250$ training occurrences in the 10B tokens of pretraining data.
    \item We measure the log-probability assigned to the true bigram completion. This bigram completion is referred to as an AR Hit in our work. This protocol assumes that the model can produce the completion by recalling the prior occurrence of the bigram in the sequence. 
    \item For the model being evaluated $m$, and the attention model $M$, we measure the \% of the quality gap between $m$ and $M$ ascribed to associative recall capability as follows. 
    Let the average log-probability for all AR Hits across validation sequences be $l_{H}^m$ and $l_{H}^M$ for $m$ and $M$ respectively. Let the average log-probabilities of \textit{all tokens} in the validation sequences be $l^m$ and $l^M$ respectively. Let $p_{H}$ be the proportion of AR Hit tokens in the validation data. As the final gap ascribed to AR, we report:
    $$\min (\frac{(l_{H}^m - l_{H}^M) p_{H}}{l^m - l^M}, 1.0)$$
    Shown above, if $m$ is better than attention ($M$) overall and $M$ is better than $m$ at AR, we ascribe 100\% of the gap to AR. 
\end{enumerate}

We briefly discuss two important decisions in this protocol. First, we only measure \textbf{explicit bigrams}, i.e. bigrams are identified based on token ids in the sequence. However, intuitively, models may also perform associative recall between related \textit{concepts} produced by a contextual language model. For instance, language may contain bigrams in which one word is swapped by a synonym. As another example, a model may see a sentence such as ``The iPhone is outside my budget so I instead purchased an Android phone. ... It was much \_'' and predict ``cheaper'' for the blank, recalling that the sequence is discussing cost. Our work does not measure such fuzzy (more abstract) recall instances. 

Next, we measure the gap based on \textbf{log-probabilities} rather than perplexity. This is simply because we want to make our metrics independent of the number of tokens in each of the slices of the validation set. Approximately 6.4\% of validation tokens are AR Hits with the threshold set to consider bigrams seen less than $1,250 \times$ during training.

\subsection{Gated Convolution Downstream Architectures}
\label{app:exp-gated-convs}
We evaluate over $4$ previously proposed architectures as well as \Coyote, the theoretically ``canonical'' representation for gated convolutions, introduced in Section \ref{sec:theory-explanation}, for a total of 14 training runs. Here we provide details on the hyperaparamters and configurations used for training each architecture. We also provide details on the FLOPs computation. 

\begin{itemize}
    \item \textbf{Attention \citep{vaswani2018attention, touvron2023llama}}
We train using the the specifications in Table \ref{tab:attn-training-details}. The parameters are sourced from the Transformer implementation in \url{https://github.com/EleutherAI/gpt-neox}.

    \item \textbf{Hyena \citep{poli2023hyena}} We train using the specifications in Table \ref{tab:hyena-training-details}. The parameters are sourced from the Appendix of \cite{poli2023hyena} and the implementation is sourced from the provided reference at \url{https://github.com/HazyResearch/safari}.
    \item \textbf{H3 \citep{dao2022hungry}} We train using the specifications in \Cref{tab:h3-training-details}. The hyperparameters and implementation use the reference at \url{https://github.com/HazyResearch/H3}.

    \item \textbf{RWKV \citep{peng2023rwkv}}
We train using the specifications in Table \ref{tab:rwkv-training-details}. The parameters are sourced from the Appendix of \cite{peng2023rwkv} and the details provided in the reference implementation at \url{https://github.com/BlinkDL/RWKV-LM}. We specifically evaluate RWKV-V4. 
    \item \textbf{Pure long convolution}
We train using the specifications in Table \ref{tab:lc-training-details}.  We evaluate a simple long convolution based model with \textit{no gating} as a reference point. While this is a generic architecture, we use the reference implementation and initialziations/regularizations from recent work \cite{fu2023simple}. The implementation is provided at \url{https://github.com/HazyResearch/safari}.

    \item \textbf{{\Coyote}}
We train using the specifications in Table \ref{tab:coyote-training-details}. The implementation, amounting to 19 lines of PyTorch, is shown in \Cref{app:implementation}.
\end{itemize}

We provide the equations used to compute the FLOPs for each model, letting $D$ be the model width, $H$ the head dimension, $L$ the depth, $N$ the sequence length, $V$ the vocabulary size, and $B$ the batch size. FLOPs equations for different architecture types are provided in \Cref{tab:attn-flops} (attention), \Cref{tab:hyena-flops} (Hyena and adapted for H3), \Cref{tab:lc-flops} (pure long convolution), and \Cref{tab:coyote-flops} ({\Coyote}). We compute FLOPs for RWKV as in the Appendix of \citep{peng2023rwkv}, based on the number of linear layer parameters, plus input and language modeling head FLOPs.

\paragraph{Hybrid Architectures} In the main paper and appendix, we evaluate a series of architectures with a hybrid of gated convolution and non gated convolution layers. For these architectures, we use the same number of layers as the pure gated convolution architecture (as reported in \Cref{app:exp-gated-convs}). The hybridized (non gated convolution) layers are inserted as replacements of the gated convolution layer. For each architecture, we simply evenly intersperse the two types of layers.  We evaluate with the following replacement layers in this work:
\begin{itemize}
    \item \textbf{Full attention} following the specification of attention in \Cref{app:exp-gated-convs}.
    \item \textbf{Random selection} We evaluate sparse attention randomly applied to tokens as introduced in \cref{sec:sec5-method}.
    \item \textbf{Programmatic selection} We apply sparse attention on AR hit tokens as introduced in \cref{sec:sec5-method}. When processing the input sequence, we can causally determine if a bigram of raw token ids is repeated in the sequence to construct the attention pattern. 
    \item \textbf{Learned selection} We \textit{learn} the positions on which to use attention by introducing a simple linear layer with sigmoid activation to the architecture that is trained to output high scores if attention should be applied to the token. We can take the top-$k$ positions to control the computational complexity of the layer. This approach is introduced in \cref{sec:sec5-method}.
\end{itemize}
We use these protocols to validate that input-dependence suffices to address the associative recall gap and highlight that there are varied approaches to incorporating input-dependence in an architecture.

%% file: Sections/appendix/mqar_downstream.tex
\section{Extended Downstream Analysis of {$\Task$}}
\label{app:mqar-downsream}

In this section, we provide additional analysis of how {$\Task$} manifests in real language data. We extend our discussion of associative recall on the Pile \cite{pile}. We also perform analysis on sources in RedPajama including ArXiv and StackOverflow \cite{together2023redpajama}.

\subsection{Additional Discussion of {$\Task$} in the Pile}
The Pile is a widely popular language modeling corpus \cite{pile}. 
\begin{enumerate}
    \item First, in \cref{sec:pile-examples}, we provide several real examples of \textit{how} associative recall occurs in the Pile to demonstrate it's role in language modeling. We color code the tokens to highlight differences in the next token predictions from Attention, RWKV and Hyena models. 
    \item Next, in \cref{sec:pile-distribution}, we explain \textit{where} associative recall hits tend to occur in sequences. This is useful to guide the design of sequence mixers --- if associative recall tends to occur in specific places, the input-dependent sequence mixer may not need to compute all $N^2$ token-to-token interactions for every sequence.
\end{enumerate}

\subsubsection{Pile Examples}

For the examples below, the legend is: tokens are colored as \textcolor{ao(english)}{both correct}, \textcolor{red}{both incorrect}, \textcolor{blue}{Attention correct and Hyena incorrect}, \textcolor{orange}{Attention incorrect and Hyena correct}. For tokens where the models disagree, the predictions of each model are provided in parentheses.

\label{sec:pile-examples}
\vspace{10pt}

\newfloat{Example}{tbhp}{lop}[section]
\definecolor{both}{HTML}{009B55}
\mdfdefinestyle{example}{
  linecolor=black,
  outerlinewidth=0.5pt,
  roundcorner=3pt,
  innertopmargin=\baselineskip,
  innerbottommargin=\baselineskip,
  innerrightmargin=20pt,
  innerleftmargin=20pt,
  backgroundcolor=gray!10
}

\input{Sections/appendix/examples/doc76423-ar-bergey}

\input{Sections/appendix/examples/doc37815-ar-pixar}
\input{Sections/appendix/examples/doc37815-ar-david}

\input{Sections/appendix/examples/doc26640-ar-subspace}
\input{Sections/appendix/examples/doc1904-ar-sequential}
\input{Sections/appendix/examples/doc130926-ar-southern}

\subsubsection{Distribution of {$\Task$} Hit Positions in Sequences}
\label{sec:pile-distribution}
In Figure \ref{fig:pile_gaps}, we compute and plot the distances between AR hits and its prior bigram (key-value) occurrence in the sequence across the Pile training data. The distances follow a power law distribution where most AR hits are within 100 token positions from the prior bigram occurrence, and a long tail of AR hits requires long-range interactions.

\paragraph{Implications} This suggests that architectures which compute token-to-token interactions within windows, where the window size is less than the full sequence length, may suffice to handle {$\Task$} in real data. For instance, sliding window attention may help with AR \cite[inter alia.]{beltagy2020longformer, sun2023retentive}. Prior work leverages the observation that local attention captures most of the token dependencies to improve efficiency \cite{sainbayar2019adaptive}.

\begin{figure}[h!]
    \centering
    \includegraphics[width=0.5\linewidth]{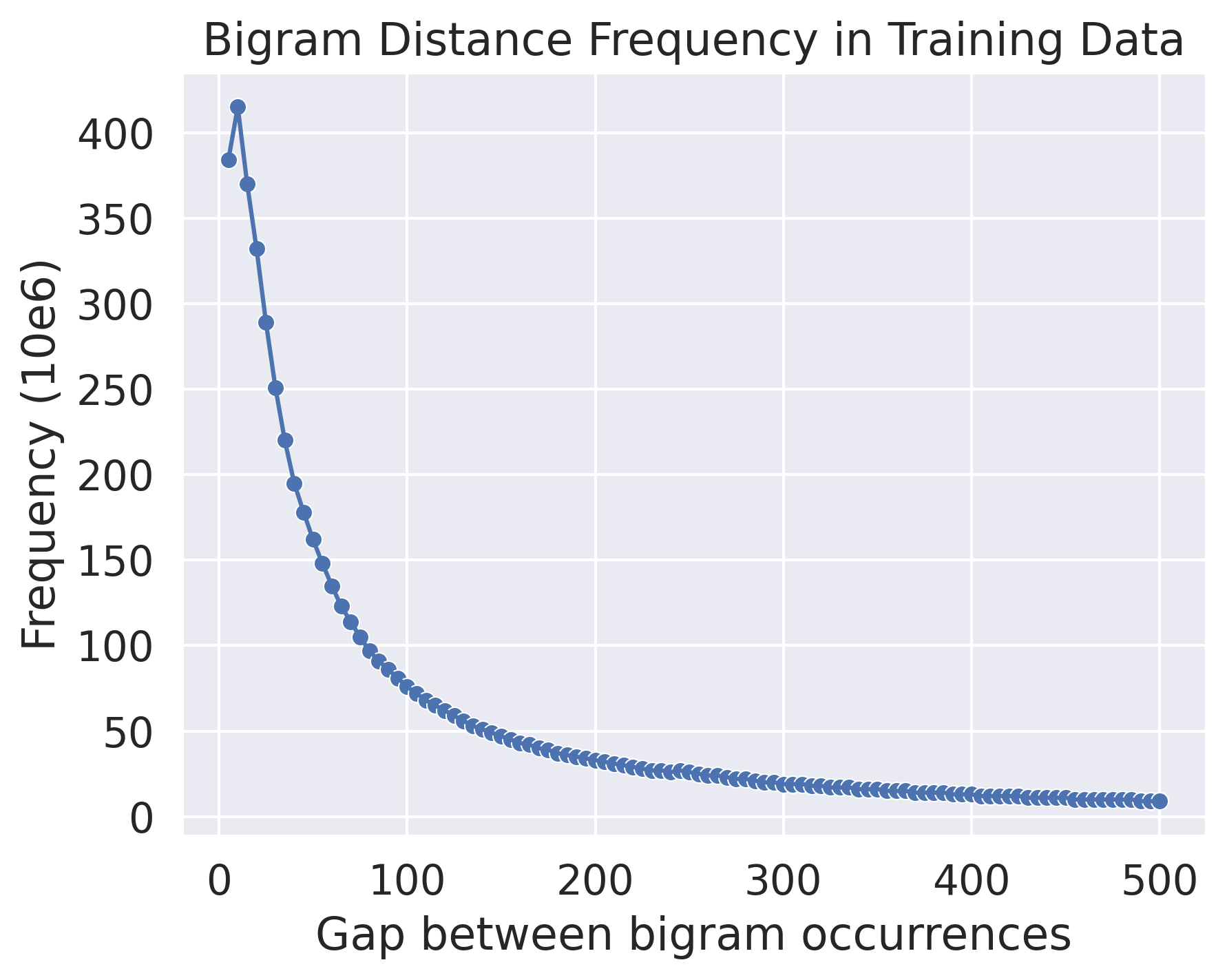}
    \caption{Across the Pile training data, we measure the distances between $n$-grams and their prior occurrences in the provided context. We plot the frequency across distances, finding it follows a power law distribution.}
    \label{fig:pile_gaps}
    \vspace{-2mm}
\end{figure}

\subsection{Analyzing {$\Task$} across Varied Language Distributions}

We use the HuggingFace sample of the popular RedPajama language modeling corpus \cite{together2023redpajama} to understand how the prevalence of {$\Task$} hits and and ways in which the hits appear vary across language distributions. We analyze text sequences from the following sources: ArXiv papers, Books, C4, Common Crawl, GitHub, Stack Exchange, and Wikipedia. \newline

First we profile the prevalence of {$\Task$} hits by computing the number of repeated bigram occurrences per sequence. Each sequence is $2,048$ tokens in length and exclude bigrams containing NLTK stop words or punctuation from the computation. The prevalence of hits and the distances between hit tokens and the prior bigram occurrence in context vary across distributions as shown in Figure \ref{fig:rp-data} and \ref{fig:rp-data-dist}. \newline
\input{Figures/Main/Figure5/figure}

Arxiv, Github, and Stack Exchange contain relatively structured or richly formatted langauge data. We observe that the prevalence of {$\Task$} hits is relatively high in these three sources in Figure \ref{fig:rp-data}. We inspect why these documents contain many bigrams. In Arxiv, we find domain-specific terminology, groups of related citations, Latex commands, and mathematical equations are frequently repeated within the same document. In GitHub and Stack Exchange, we find function and variable names are frequently reused. Meanwhile, the C4, CC, Books, and Wikipedia bigrams are highly variable. The topics and documents in these corpora are relatively diverse and unstructured. 
\vspace{3mm}
\input{Sections/appendix/rp_examples/arxiv}
\input{Sections/appendix/rp_examples/github}

%% file: Sections/appendix/examples/doc76423-ar-bergey.tex
\begin{mdframed}[style=example]

\textcolor{red}{~while}
\textcolor{red}{~lunch}
\textcolor{both}{-ing}
\textcolor{red}{~at}
\textcolor{both}{~the}
\textcolor{red}{~Ma}
\textcolor{both}{-ison}
\textcolor{red}{~Ber}
\textcolor{red}{-gey}
\textcolor{red}{~b}
\textcolor{both}{-ist}
\textcolor{both}{-ro}
\textcolor{red}{~near}
\textcolor{red}{~his}
\textcolor{red}{~apartment}
\textcolor{red}{:}
\textcolor{red}{~he}
\textcolor{both}{~had}
\textcolor{both}{~been}
\textcolor{red}{~mus}
\textcolor{both}{-ing}
\textcolor{orange}{~about}\texttt{(rwkv=~about,~attn=~on)}
\textcolor{both}{~the}
...\texttt{(723 tokens)}...
\textcolor{both}{~the}
\textcolor{red}{~young}
\textcolor{red}{~waitress}
\textcolor{both}{-'s}
\textcolor{red}{~sigh}
\textcolor{red}{~at}
\textcolor{both}{~the}
\textcolor{red}{~Ma}
\textcolor{both}{-ison}
\textcolor{blue}{~Ber}\texttt{(rwkv=~Bl,~attn=~Ber)}
\textcolor{blue}{-gey}\texttt{(rwkv=-nd,~attn=-gey)}
\end{mdframed}
\vspace{-7pt}
\captionof{Example}{\textbf{Comparing next token predictions of a 350M parameter Attention model and a 350M parameter RWKV model.} In this example, the models need to perform associative recall to correctly predict ``Bergey" in the 4-gram Ma-ison Ber-gey. The previous mention of the 4-gram was more than 700 tokens earlier in the passage.}
\vspace{10pt}

%% file: Sections/appendix/examples/doc37815-ar-pixar.tex
\begin{mdframed}[style=example]
\textcolor{red}{~The}
\textcolor{blue}{~second}\texttt{(rwkv=~first,~attn=~second)}
\textcolor{red}{~section}
\textcolor{both}{~is}
\textcolor{red}{~all}
\textcolor{both}{~about}
\textcolor{red}{~Pixar}
\textcolor{red}{~Fest}
\textcolor{red}{,}
\textcolor{red}{~and}
\textcolor{blue}{~the} \texttt{(rwkv=~third,~attn=~the)}
\textcolor{red}{~final}
\textcolor{both}{~section}
\textcolor{both}{~is}
\textcolor{blue}{~all}\texttt{(rwkv=~about,~attn=~all)}
\textcolor{both}{~about}
\textcolor{red}{~Pixar}
\textcolor{red}{~Pier}
\textcolor{both}{.}
...\texttt{(480 tokens)}...
\textcolor{blue}{-If}\texttt{(rwkv=-Disney,~attn=-If)}
\textcolor{red}{~there}
\textcolor{red}{~wasn}
\textcolor{both}{-âĢ}
\textcolor{both}{}
\textcolor{both}{-t}
\textcolor{red}{~enough}
\textcolor{red}{~Pixar}
\textcolor{red}{~at}
\textcolor{red}{~Disneyland}
\textcolor{both}{,}
\textcolor{red}{~Pixar}
\textcolor{blue}{~Fest}\texttt{(rwkv=~would,~attn=~Fest)}
\textcolor{blue}{~is}\texttt{(rwkv=~at,~attn=~is)}
\textcolor{red}{~coming}
\textcolor{both}{~to}
\textcolor{red}{~the}
\textcolor{both}{~Disneyland}
\textcolor{both}{~Resort}
\textcolor{red}{~on}
\textcolor{red}{~April}
\textcolor{red}{~13}
\textcolor{red}{,}
\textcolor{red}{~2018}
\textcolor{both}{.}
\end{mdframed}
\vspace{-7pt}
\captionof{Example}{\textbf{Comparing next token predictions of a 350M parameter Attention model and a 350M parameter RWKV model.} In this example, the models need to perform associative recall to correctly predict ``Fest" in the bigram Pixar Fest.}
\vspace{10pt}

%% file: Sections/appendix/examples/doc37815-ar-david.tex
\begin{mdframed}[style=example]
\textcolor{red}{~David}
\textcolor{red}{~Mus}
\textcolor{red}{-a}
\textcolor{red}{~P}
\textcolor{red}{-id}
\textcolor{red}{-cock}
\textcolor{red}{~is}
\textcolor{red}{~an}
\textcolor{red}{~indigenous}
\textcolor{red}{~English}
\textcolor{red}{~revert}
\textcolor{blue}{~to}\texttt{(rwkv=-ant,~attn=~to)}
\textcolor{red}{~Islam}
\textcolor{orange}{,}\texttt{(rwkv=,,~attn=~who)}
\textcolor{red}{~who}
\textcolor{red}{~formed}
\textcolor{both}{~the}
\textcolor{red}{~Islamic}
\textcolor{red}{~Party}
\textcolor{both}{~of}
\textcolor{red}{~Britain}
\textcolor{both}{~in}
\textcolor{red}{~1989}
\textcolor{red}{,}
\textcolor{red}{~at}
\textcolor{red}{~London}
\textcolor{red}{~Central}
\textcolor{orange}{~Mosque}\texttt{(rwkv=~Mosque,~attn=~University)}
\textcolor{orange}{.}\texttt{(rwkv=.,~attn=~in)}
...(\texttt{711 tokens})...
\textcolor{both}{-I}
\textcolor{red}{~have}
\textcolor{red}{~just}
\textcolor{red}{~found}
\textcolor{red}{~out}
\textcolor{both}{~that}
\textcolor{red}{~David}
\textcolor{blue}{~Mus}\texttt{(rwkv=~Cameron,~attn=~Mus)}
\textcolor{both}{-a}
\textcolor{blue}{~P}\texttt{(rwkv=~(,~attn=~P)}
\textcolor{blue}{-id}\texttt{(rwkv=-asha,~attn=-id)}
\textcolor{blue}{-cock}\texttt{(rwkv=-ham,~attn=-cock)}
\textcolor{blue}{~is}\texttt{(rwkv=~(,~attn=~is)}
\textcolor{red}{~speaking}
\textcolor{both}{~at}
\textcolor{red}{~a}
\textcolor{red}{~Yorkshire}
\textcolor{red}{~Forum}
\textcolor{red}{~debate}
\textcolor{red}{~today}
\textcolor{red}{~in}
\textcolor{red}{~Bradford}
\textcolor{red}{~with}
\end{mdframed}
\vspace{-7pt}
\captionof{Example}{\textbf{Comparing next token predictions of a 350M parameter Attention model and a 350M parameter RWKV model.} In this example, the models need to perform associative recall to correctly predict a middle and last name. Note that the middle and last names span several tokens.}
\vspace{10pt}

%% file: Sections/appendix/examples/doc26640-ar-subspace.tex
\begin{mdframed}[style=example]
\textcolor{both}{-The}
\textcolor{red}{~most}
\textcolor{both}{~common}
\textcolor{red}{~map}
\textcolor{blue}{~elements}\texttt{(hyena=~in,~attn=~elements)}
\textcolor{red}{~in}
\textcolor{red}{~Sub}
\textcolor{red}{-Space}
\textcolor{both}{~are}
\textcolor{red}{~prizes}
\textcolor{both}{,}
\textcolor{red}{~or}
\textcolor{both}{~"}
\textcolor{blue}{-g}\texttt{(hyena=-points,~attn=-g)}
\textcolor{red}{-reens}
\textcolor{blue}{"}\texttt{(hyena=-",,~attn=")}
\textcolor{red}{~(}
\textcolor{red}{-for}
\textcolor{red}{~their}
\textcolor{red}{~green}
\textcolor{red}{~color}
\textcolor{red}{-).}
\textcolor{both}{~}
\textcolor{red}{~Pri}
\textcolor{both}{-zes}
\textcolor{red}{~allow}
\textcolor{both}{~players}
\textcolor{both}{~to}
\textcolor{red}{~upgrade}
\textcolor{both}{~their}
\textcolor{red}{~ships}
\textcolor{red}{~and}
\textcolor{blue}{~gain}\texttt{(hyena=~other,~attn=~gain)}
\textcolor{red}{~special}
\textcolor{red}{~weapons}
\textcolor{red}{~or}
\textcolor{orange}{~abilities}\texttt{(hyena=~abilities,~attn=~upgrades)}
\textcolor{both}{.}
\textcolor{both}{~}
\textcolor{red}{~While}
\textcolor{red}{~prizes}
\textcolor{both}{~are}
\textcolor{red}{~generally}
\textcolor{red}{~plent}
\textcolor{red}{-ifully}
\textcolor{red}{~scattered}
\textcolor{both}{~throughout}
\textcolor{both}{~the}
\textcolor{orange}{~map}\texttt{(hyena=~map,~attn=~game)}
\textcolor{both}{,}
\textcolor{red}{~the}
\textcolor{red}{~upgrades}
\textcolor{red}{~or}
\textcolor{blue}{~abilities}\texttt{(hyena=~bonuses,~attn=~abilities)}
\textcolor{red}{~they}
\textcolor{red}{~award}
\textcolor{both}{~are}
\textcolor{red}{~randomly}
\textcolor{red}{~selected}
\textcolor{red}{~by}
\textcolor{both}{~the}
\textcolor{red}{~zone}
\textcolor{red}{.}
\textcolor{red}{}
\textcolor{both}{}
\textcolor{red}{-Energy}
\textcolor{red}{~}
\textcolor{both}{}
\textcolor{red}{-Rather}
\textcolor{both}{~than}
\textcolor{red}{~dealing}
\textcolor{both}{~with}
\textcolor{red}{~ammunition}
\textcolor{red}{~counts}
\textcolor{red}{~and}
\textcolor{red}{~hit}
\textcolor{both}{~points}
\textcolor{red}{~separately}
\textcolor{both}{,}
\textcolor{red}{~Sub}
\textcolor{both}{-Space}
\textcolor{red}{~combines}
\textcolor{red}{~both}
\textcolor{red}{~of}
\textcolor{both}{~these}
\textcolor{both}{~elements}
\textcolor{both}{~into}
\textcolor{both}{~a}
\textcolor{both}{~single}
\textcolor{red}{~unit}
\textcolor{red}{~of}
\textcolor{red}{~measure}
\textcolor{red}{:}
\textcolor{red}{~energy}
\textcolor{both}{.}
\textcolor{both}{~}
\textcolor{red}{~Each}
\textcolor{red}{~ship}
\textcolor{red}{~is}
\textcolor{both}{~equipped}
\textcolor{both}{~with}
\textcolor{both}{~a}
\textcolor{red}{~certain}
\textcolor{both}{~amount}
\textcolor{both}{~of}
\textcolor{both}{~energy}
\textcolor{both}{,}
\textcolor{red}{~from}
\textcolor{orange}{~which}\texttt{(hyena=~which,~attn=~the)}
\textcolor{both}{~it}
\textcolor{red}{~must}
\textcolor{red}{~draw}
\textcolor{red}{~its}
\textcolor{red}{~health}
\textcolor{red}{~as}
\textcolor{both}{~well}
\textcolor{both}{~as}
\textcolor{both}{~its}
\textcolor{red}{~weapons}
\textcolor{red}{~power}
\textcolor{both}{.}
...\texttt{(1{,}526 tokens)}...
\textcolor{red}{~Speed}
\textcolor{red}{~Zone}
\textcolor{red}{~proved}
\textcolor{both}{~to}
\textcolor{both}{~be}
\textcolor{red}{~less}
\textcolor{red}{~popular}
\textcolor{both}{~than}
\textcolor{red}{~the}
\textcolor{red}{~Jack}
\textcolor{both}{-pot}
\textcolor{red}{-/}
\textcolor{red}{-Running}
\textcolor{red}{,}
\textcolor{red}{~Chaos}
\textcolor{both}{,}
\textcolor{red}{~or}
\textcolor{red}{~"}
\textcolor{red}{-flag}
\textcolor{both}{"}
\textcolor{red}{~zone}
\textcolor{red}{~games}
\textcolor{red}{~and}
\textcolor{red}{~support}
\textcolor{red}{~was}
\textcolor{red}{~discontinued}
\textcolor{red}{~shortly}
\textcolor{both}{~after}
\textcolor{red}{~Sub}
\textcolor{blue}{-Space}\texttt{(hyena=-space,~attn=-Space)}
\textcolor{red}{~went}
\textcolor{red}{~to}
\textcolor{red}{~retail}
\textcolor{both}{.}
\end{mdframed}
\vspace{-7pt}
\captionof{Example}{\textbf{Comparing next token predictions of a 350M parameter Attention model and a 355M parameter Hyena model.} In this example, the models need to perform associative recall to correctly predict the name of the game ``SubSpace". Note that both models correctly perform the recall when there is a short gap between tokens, but only Attention correctly performs recall when the gap is greater than 1{,}000 tokens.}
\vspace{10pt}

%% file: Sections/appendix/examples/doc1904-ar-sequential.tex
\begin{mdframed}[style=example]
\textcolor{red}{~Thus}
\textcolor{red}{~far}
\textcolor{both}{,}
\textcolor{red}{~no}
\textcolor{red}{~systematic}
\textcolor{both}{~study}
\textcolor{both}{~has}
\textcolor{both}{~been}
\textcolor{red}{~published}
\textcolor{both}{~on}
\textcolor{red}{~is}
\textcolor{orange}{-olation}\texttt{(hyena=-olation,~attn=-ot)}
\textcolor{red}{~via}
\textcolor{red}{~sequential}
\textcolor{both}{~centrif}
\textcolor{both}{-ug}
\textcolor{both}{-ation}
\textcolor{red}{,}
\textcolor{both}{~and}
\textcolor{red}{~no}
\textcolor{red}{~systematic}
\textcolor{red}{~analysis}
\textcolor{red}{~is}
...\texttt{(491 tokens)}...
\textcolor{red}{Fig}
\textcolor{both}{.}
\textcolor{orange}{~1}\texttt{(hyena=~1,~attn=-Âł)}
\textcolor{both}{-is}
\textcolor{both}{-olation}
\textcolor{red}{~via}
\textcolor{blue}{~sequential}\texttt{(hyena=~centrif,~attn=~sequential)}
\textcolor{both}{~centrif}
\textcolor{both}{-ug}
\textcolor{both}{-ation}
\textcolor{both}{.}
\end{mdframed}
\vspace{-7pt}
\captionof{Example}{\textbf{Comparing next token predictions of a 350M parameter Attention model and a 355M parameter Hyena model.}  In this example, the models need to perform associative recall to correctly predict the technique ``sequential centrifugation".}
\vspace{10pt}

%% file: Sections/appendix/examples/doc130926-ar-southern.tex
\begin{mdframed}[style=example]
\textcolor{red}{Miss}
\textcolor{orange}{-ouri}\texttt{(hyena=-ouri,~attn=-iss)}
\textcolor{red}{~Southern}
\textcolor{red}{~has}
\textcolor{red}{~had}
\textcolor{red}{~14}
\textcolor{red}{~Major}
\textcolor{both}{~League}
\textcolor{red}{~Baseball}
\textcolor{red}{~draft}
\textcolor{red}{~selections}
\textcolor{red}{~since}
\textcolor{both}{~the}
\textcolor{red}{~draft}
\textcolor{red}{~began}
\textcolor{both}{~in}
\textcolor{red}{~1965}
\textcolor{both}{.}
...\texttt{(149 tokens)}...
\textcolor{red}{~Fred}
\textcolor{red}{~G}
\textcolor{both}{.}
\textcolor{red}{~Hughes}
\textcolor{orange}{~Stadium}\texttt{(hyena=~Stadium,~attn=~Field)}
\textcolor{both}{~(}
\textcolor{both}{-opened}
\textcolor{both}{~in}
\textcolor{red}{~1975}
\textcolor{both}{)}
\textcolor{both}{~is}
\textcolor{blue}{~named}\texttt{(hyena=~the,~attn=~named)}
\textcolor{both}{~after}
\textcolor{blue}{~former}\texttt{(hyena=~the,~attn=~former)}
\textcolor{red}{~J}
\textcolor{both}{-op}
\textcolor{both}{-lin}
\textcolor{red}{~Globe}
\textcolor{red}{~publisher}
\textcolor{orange}{~and}\texttt{(hyena=~and,~attn=~Fred)}
\textcolor{red}{~Missouri}
\textcolor{blue}{~Southern}\texttt{(hyena=~State,~attn=~Southern)}
\textcolor{red}{~board}
\textcolor{red}{~of}
\textcolor{red}{~reg}
\textcolor{both}{-ents}
\textcolor{red}{~member}
\end{mdframed}
\vspace{-7pt}
\captionof{Example}{\textbf{Comparing next token predictions of a 350M parameter Attention model and a 355M parameter Hyena model.} In this example, the models need to perform associative recall to correctly predict the name of the university ``Missouri Southern".}
\vspace{10pt}

%% file: Figures/Main/Figure5/figure.tex
\begin{figure}
    \includegraphics[width=\linewidth]{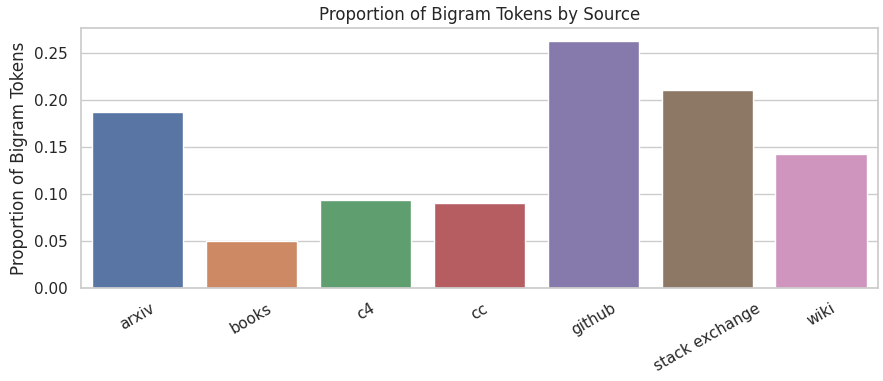}
    \caption{{$\Task$} hits by sub-source of the RedPajama language corpus \cite{together2023redpajama}. Hits are measured  using the repeated bigrams in each sequence of length .}
    \label{fig:rp-data}
    \vspace{-2mm}
\end{figure}

\begin{figure}
    \includegraphics[width=\linewidth]{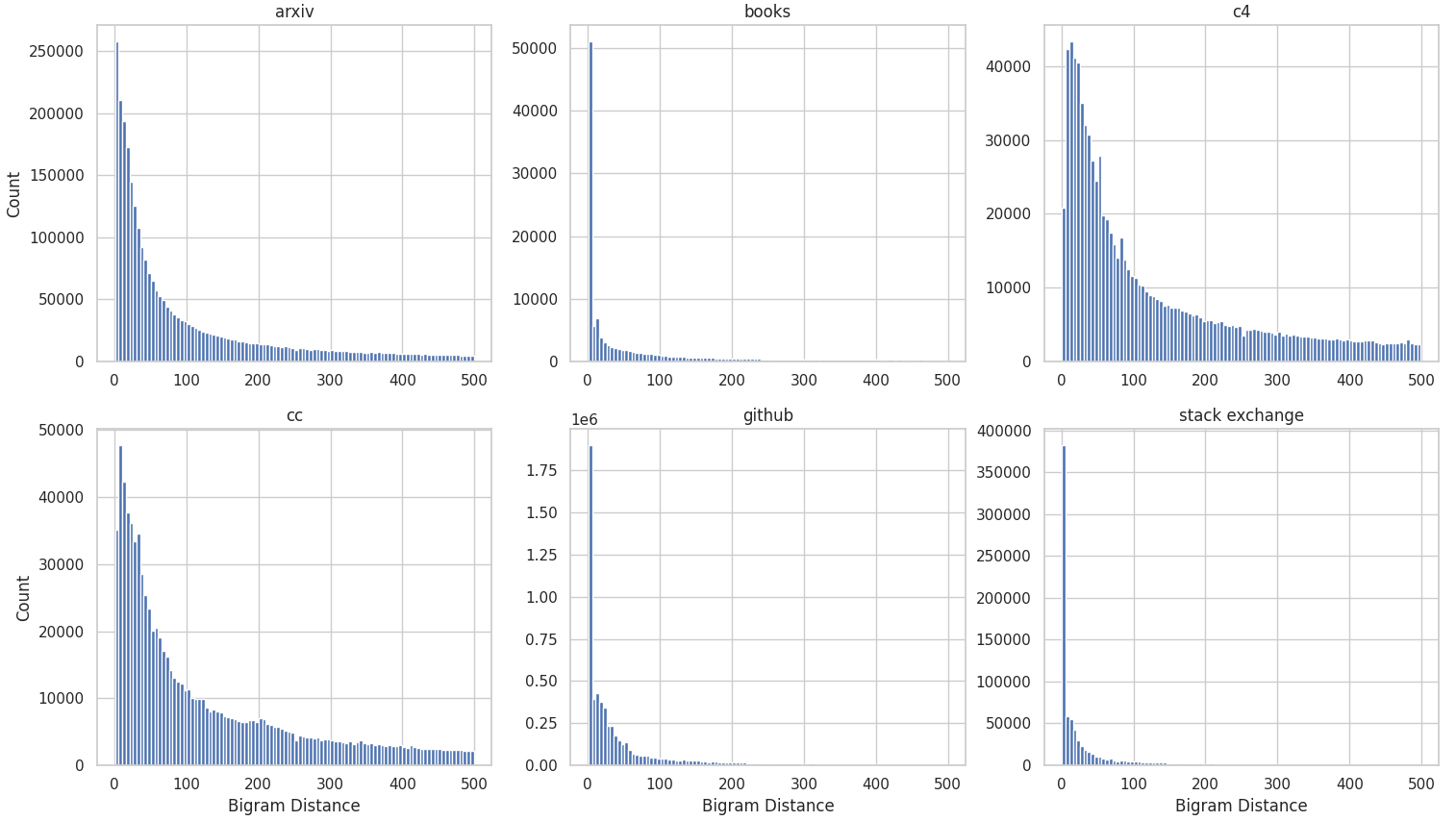}
    \caption{We measure the distance between $n$-grams and their prior occurrences in the provided context across sequences from each source. We plot the frequency across distances, finding it follows a power law distribution.}
    \label{fig:rp-data-dist}
    \vspace{-2mm}
\end{figure}

%% file: Sections/appendix/rp_examples/arxiv.tex
\begin{mdframed}[style=example]
If the placement $V$ of the hyperplanes is not generic, then the induced subdivision $\Sigma$ is not a triangulation. 
However, the above bijection is unaffected and, in
particular, the fine type of a vertex of $CD_A$ can be read off the corresponding facet of $\Sigma$. 
We define the \textcolor{red}{\emph{crosscut complex}} \textcolor{red}{$\ ccut{\Sigma}$} $\subseteq 2^{[n] \times
[d]}$ to be the unique simplicial complex with the same vertices-in-facets incidences as
the \textcolor{red}{polyhedral complex} $\Sigma$.  
The \textcolor{blue}{crosscut complex} is a standard notion in
combinatorial topology and can be defined in more generality (see Bj\"{o}rner).  
The following observation is immediate from the definitions.\newline

\begin{prop}
  For $V$ an ordered sequence of $n$ points in $T^{d-1}$ let $A = A(V)$ be the
  corresponding tropical arrangement and let $\Sigma$ be the regular subdivision of
  $\Delta_{n-1} \times \Delta_{d-1}$ induced by $V$. 
  Then the fine cotype ideal $\ fcI$ is Alexander dual to the Stanley-Reisner ideal of \textcolor{blue}{$\ ccut{\Sigma}$}.\newline
\end{prop}

The \textcolor{blue}{crosscut complex} encodes the information of which collections
of vertices lie in a common face. 
Hence, the \textcolor{blue}{crosscut complex} is a purely combinatorial object and does not see the affine structure of the
underlying \textcolor{blue}{polyhedral complex}.
\end{mdframed}
\vspace{-7pt}
\captionof{Example}{ArXiv sequence where the initial bigram occurrence is highlighted in \textcolor{red}{red} and the repeated bigram with the {$\Task$} hit is highlighted in \textcolor{blue}{blue}.}
\vspace{10pt}

%% file: Sections/appendix/rp_examples/github.tex

\begin{lstlisting}[language=Java,style=mystyle]
/**
     * Adds a directed edge to the graph from pt1 to pt2. Precondition: Both
     * GeographicPoints have already been added to the graph
     *
     * @param from The starting point of the edge
     * @param to The ending point of the edge
     * @param roadName The name of the road
     * @param roadType The type of the road
     * @param length The length of the road, in km
     * @throws IllegalArgumentException If the points have not already been
     * added as nodes to the graph, if any of the arguments is null, or if the
     * length is less than 0.
     */
    public void addEdge(GeographicPoint from, GeographicPoint to,
            String roadName, String roadType, double length)
            throws IllegalArgumentException {

        if (from == null || to == null || roadName == null
                || roadType == null || length < 0 || !mapNodes.containsKey(from)
                || !mapNodes.containsKey(to)) {
            throw new IllegalArgumentException();
        }

        if (!mapNodes.get(from).hasEdge(to)) {
            mapNodes.get(from).addNeighbor(to, roadName, roadType, length);
            numEdges++;
        }
    }
\end{lstlisting}
\vspace{-7pt}
\captionof{Example}{GitHub sequence where substrings such as ``IllegalArgumentException'', ``roadType'' (and other variable names), ``GeographicPoint'' (and other data types), ``containsKey'' (and other function calls) are repeated throughout.}
\vspace{10pt}

%% file: Sections/appendix/mqar_framework.tex
\section{Synthetic {$\Task$} Experimental Details}

In this paper, we propose {$\Task$} as a useful tool to help explain gaps between three popular sequence modeling layers --- Transformers or Sparse Transformers, Convolutions, and Recurrences. In this section, we detail the motivation behind the design of {$\Task$} and include extended experiments on additional architectures. We provide a procedure for generating {$\Task$} synthetic data, which can help in the development of new architectures. 

\label{app:synthetic}
\subsection{{$\Task$} Generation Procedure}
Here we provide additional discussion on the properties of our {$\Task$} synthetic data. The objective of the synthetic analysis is to help explain the differences in language modeling behavior between different classes of language modeling architectures, as observed on real-world data \cref{sec:sec3-downstream}. Synthetic recall tasks were used in the development of Hyena \cite{poli2023hyena} and H3 \cite{dao2022hungry}, and \cite{olsson2022context} to study Transformer in-context learning behavior. 
\begin{enumerate}
    \item \textbf{Attention}: Attention can  perform recall trivially (\cref{prop: attention-ar}). The bound is independent of the sequence length $N$.
    \item \textbf{Convolutions}: The gated convolution models use input-independent filters. We show in \cref{sec:theory-explanation} that the dimensionality for solving recall grows with the input sequence length. Real language modeling can require performing $\calO(N)$ recalls in one forward pass. Our intuition is that it is difficult to efficiently compute \textbf{all token-interaction distances} with such a filter.
    \item \textbf{Recurrences}: Recurrent models include a single hidden state as an information bottleneck. We show in \cref{app:retnet-proof} that the gated recurrence requires $\Omega(N)$ bits to solve ${\Task}$ for $d \leq \sqrt{N}$ for model dimension $d$. Our intuition is that it is difficult to store \textbf{large numbers of key-value pairs} seen in the sequence in low-dimensional hidden states.  
\end{enumerate}

Motivated by this analysis, we propose Procedure \ref{alg:synthetic} for generating synthetic {$\Task$} data. This procedure allows toggling two simple properties of the synthetic dataset, specified below, to reveal the differences between the sequence modeling architecture families. We also discuss deviations from the prior work that uses synthetic recall data.
\begin{enumerate}
    \item \textbf{Size of the Key-Value Map}: The number of unique key-value pairs appearing in each example in the dataset. In Procedure \ref{alg:synthetic}, we use the input parameter $D$ to toggle this value.
    \item \textbf{Number of Gaps in the Data}: The number of unique token-interaction distances required to perform the recall task. Prior work assumes there is a single query token that requires recall per example, failing to test whether the architecture can support multiple token-interaction distances. In Procedure \ref{alg:synthetic}, we use the parameters $N$ (sequence length), $D$ (size of the key-value map per example), and $\alpha$ to toggle this property. We enforce that the token-interaction distances appearing in the synthetic examples follow a power-law distribution specified by $\alpha$ (based on Figure \ref{fig:pile_gaps}). Keeping $\alpha$ and $D$ constant while varying $N$ changes the number of unique token-interaction distances appearing in the example.
\end{enumerate}

Finally, noting that the vocabulary size appears in the theoretical bounds for all three architecture families (\cref{sec:theory-explanation}), we increase the vocabulary size to be much larger than the model dimension as typically seen in language modeling (30k - 50k tokens). Prior work uses vocab sizes $\leq 40$ tokens.

\begin{algorithm}[H]
	\caption{{$\Task$} Synthetic Procedure}
	\begin{algorithmic}[1]
		\Require Vocabulary $C$, Sequence length $N$, Power-law parameter $\alpha$, Number of Key-Value Pairs $D$\newline 
        \textbf{Output: Synthetic sequence}
		\State Let the first half of $C$ be keys $K$ and the second half be values $V$. 
        \State Pair each key token $k \in K$ with a \textit{random} value token $v \in V$.
        \State Sub-select $D$ random key-value pairs to include in the sequence.
        \State Place the $D$ key-value pairs at the start of the sequence (i.e. consuming the first $2D$ positions).
        \State Place a second occurrence of each $d \in D$ at a distance from the first occurrence in the sequence. The distance for each $d \in D$ is selected at random from the positions $[2D..N]$, where the probability of choosing each position follows the power law distribution specified by $\alpha$.
        \State Output the synthetic sequence.
	\end{algorithmic}
    \label{alg:synthetic}
\end{algorithm}

\subsection{Training Details}
We first describe the architectures evaluated on {$\Task$} synthetic data and then the training hyperparameters. We evaluate the following four architecture categories in the experiments: 
\begin{enumerate}
    \item \textbf{Attention} Standard GPT-2 style multi-headed Transformer architecture with learned positional embeddings \cite{brown2020language}. Attention, the core building block of Transformers, is defined in \Cref{sec:prelim-architectures}. The architecture for synthetics has 1 attention head.
    \item \textbf{Gated convolutions} (Hyena~\cite{poli2023hyena}, RWKV~\cite{peng2023rwkv}). Gating, convolutions, and the class of gated convolutions are defined in \Cref{sec:prelim-architectures}.
    \item \textbf{Gated recurrences} (RetNet~\cite{sun2023retentive}). RetNet proposes computing attention over chunks of the sequence and combining this with a recurrence over the sequence. We evaluate this architecture (``Chunked RNN'') with chunk sizes of $32$ and $8$ in Figure \ref{fig:kv_storage}. 
    
    In RetNet, hidden states are updated as: $$S_n = \gamma S_{n-1} + A_n^T V_n$$
    Outputs at each timestep are defined as $$O_n = C_nS_n$$
    where $\gamma \in \mathbb{R}^1$, $A = xW_a$ for input $x \in \mathbb{N \times d}$ and learned weight matrix $W_a \in \mathbb{R}^{d \times d}$, $C = xW_c$ for $W_c \in \mathbb{R}^{d \times d}$, $V = xW_v$ for $W_v \in \mathbb{R}^{d \times d}$. 

    RetNet similar to a state-space-model (SSM) \cite{gu2021efficiently}, except for that the matrices $A$ and $C$ are \textit{input-dependent} (i.e. they are functions of the input $x$). We note that RetNet is a special case of the recently proposed Mamba \cite{mamba2023} and GateLoop \cite{gateloop2023} architectures, which replace the $\gamma$ term in RetNet with yet another input-dependent matrix.

    The RetNet model is formally defined in \Cref{sec:retnet-model}.
    
    \item \textbf{Sparse attention} Several works have proposed methods for efficiently computing attention \cite{tay2022efficient}. We evaluate two classical methods: sliding window attention \cite{beltagy2020longformer, zaheer2020bigbird} and blocked window attention \cite{child2019sparse, qiu2019blockwise}. 
    
    Consider a window size $w$. Sliding window attention permits each token to attend to the prior $w$ tokens in the sequence. Blocked window attention first splits the sequence into blocks of size $w$ tokens, and then computes standard causal attention within the block. \newline
\end{enumerate}

For all synthetic runs in the main paper and the appendix, we use the following training protocol:
\begin{itemize}
    \item \textbf{Optimizer and schedule}: Weight decay $0.1$, warmup duration $10\%$, linear warmup, AdamW optimizer. For each run, we sweep the learning rates in $\mathrm{np.logspace(-4, -2, 4)}$. We train for $64$ epochs.   
    \item \textbf{Training duration}: The global batch size is $8$ for input sequence length or model dimension $\geq 512$, $16$ for input sequence length or model dimension $\geq 256$, and $64$ otherwise.
    \item \textbf{Width and depth}: For all synthetic experiments, we use exactly two layers (each with one sequence mixer and one MLP, interleaved with layer normalization).  The model dimension, sequence length, and number of KV pairs are varied based on the relevant experiment (see \Cref{app:synth-results}).
    \item \textbf{Position information}: No position embeddings are used for the pure convolution runs (input-dependent nor input-independent filters). Position embeddings are used for the runs with any attention variant.
    \item \textbf{Data}: We train and evaluate each model on $100,000$ and $3,000$ data points respectively. The data and data order for all runs are constant.
\end{itemize}

\section{Extended Results on {$\Task$} Across Architectures}
\label{app:synth-results}

In this section, we provide further validation that  {$\Task$} is a useful diagnostic tool to explain the behaviors of a wide variety of architecture families. We show:
\begin{enumerate}
    \item \textbf{Gated convolutions} In the main paper, we claim the gated convolution models with input-independent filters use larger dimensionality than attention as the number of token-interaction distances required for the task increases ($N$, holding $\alpha$ and $D$ constant in Procedure \ref{alg:synthetic}).
    \item \textbf{Gated recurrences} In this section, we show gated recurrrent models use larger dimensionality than attention to solve the task as the number of unique key-value pairs to store ($D$ in Procedure \ref{alg:synthetic}) increases.
    \item \textbf{Sparse attention} In this section, we validate that sliding window attention helps close the {$\Task$} gap when the token-interaction distances are relatively short, and degrades on longer distances. Meanwhile, blocked window attention struggles to close the gap (intuitively tokens earlier in a block are able to interact with few other tokens in the sequence). Concretely, in \cref{app:mqar-downsream}, we observe that repeated bigrams in the Pile and RedPajama corpora tend to occur within neighboring tokens in the context.
\end{enumerate}

We finally show the architectures that perform well on {$\Task$} also perform well on downstream real-world AR on the Pile. We use the experimental protocols outlined in \cref{app:synthetic} for the synthetic experiments and \cref{app:expdetails} for the downstream experiments in this section.

\subsection{Synthetic Experiments}

\paragraph{Experiment 1: Increasing the number of token-interaction distances per example.}
To construct the synthetic dataset following Procedure \ref{alg:synthetic}, we set $\alpha=0.1$ and $|C|=8,192$ for all experiments. We vary $N \in \{64, 128, 256, 512\}$ to increase the number of token-interaction distances that will appear in the examples and we vary the model dimension $\in \{64, 128, 256, 512\}$. The results of this experiment are shown in \cref{fig:capacity}, validating that the gated convolutiosn use larger dimensionality than attention as the number of token-interaction distances required for the task increases.

\paragraph{Experiment 2: Increasing the number of key-value pairs occurring per example.} To construct the synthetic dataset following Procedure \ref{alg:synthetic}, we set $\alpha=0.1$ and $|C|=8,192$ for all experiments. We fix $N = 256$ and vary $D \in \{\frac{N}{32}, \frac{N}{16}, \frac{N}{8}, \frac{N}{4}\}$ to increase the number of key-value pairs occurring per example. For each number of key-value pairs per example, we additionally vary the model dimension $\in \{64, 128, 256\}$. The results are shown in \cref{fig:kv_storage}. We validate that as the number of associative recall keys and values exceeds the window size, the model again uses increased dimensionality relative to full $O(N^2)$ attention to solve the task. While \cref{fig:kv_storage} highlights RetNet, we note that it is a special case of the recently released Mamba \cite{mamba2023} and GateLoop \cite{gateloop2023} architectures. Corresponding to this experiment, we theoretically show in \cref{app:retnet-proof} that the gated recurrence uses $\Omega(N)$ bits to solve ${\Task}$ for $d \leq \sqrt{N}$. Intuitively, it is difficult to store a large number of key-value pairs in low dimensional hidden states.
\input{Figures/Main/Figure4/figure}

\paragraph{Experiment 3: Increasing the range of token-interaction distancs per example.} To construct the synthetic dataset following Procedure \ref{alg:synthetic}, we set $\alpha=0.1$ and $|C|=8,192$ for all experiments. We fix $N = 256$ and vary $D \in \{\frac{N}{32}, \frac{N}{16}, \frac{N}{8}, \frac{N}{4}\}$ to increase the number of key-value pairs occurring per example and the range of token-interaction distances per example. For each number of key-value pairs per example, we additionally vary the model dimension $\in \{64, 128, 256\}$.

The results are shown in \cref{fig:kv_storage}, validating that sliding window attention closes the gap to attention when the number of KV pairs is within the window size. I.e., $16$ KVs means that there are $16$ keys and $16$ values, for $32$ total tokens. Sliding attention with window size $32$ performs well up until $16$ KVs and degrades beyond this point relative to attention models of the same dimensionality. Sliding window attention with window size $8$ does not perform well on any setting, as expected. Blocked attention also not perform well -- intuitively tokens at the beginning of a block lack sufficient prior context to attend to. 
\input{Figures/Main/Figure6/figure}

\subsection{Demonstrating {$\Task$} as a Tool for Architecture Development}

In the main paper, we provide experiments showing that there is gap between the gated convolution and attention models on synthetic {$\Task$} data and correspondingly in the downstream Pile results. In \Cref{table:full-ppl-slices}, we show that the trends of RetNet and sparse attention on {$\Task$} also carry to the downstream results. We hope {$\Task$} synthetic data may be useful in future architectural development.

\input{Tables/full_downstream_ar}

%% file: Figures/Main/Figure4/figure.tex
\begin{figure}
    \includegraphics[width=\linewidth]{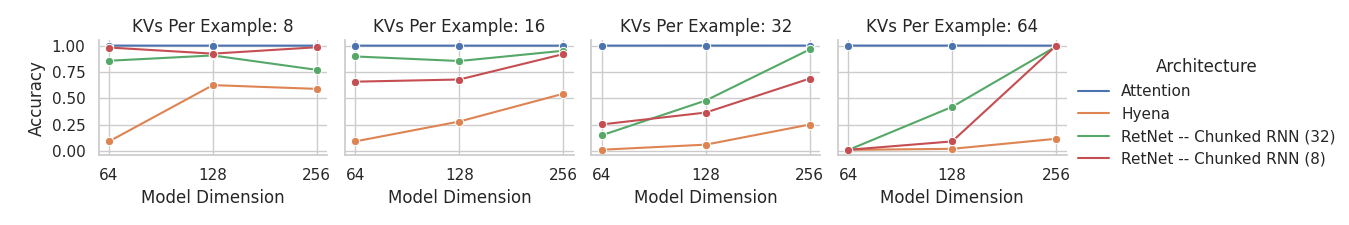}
    \includegraphics[width=\linewidth]{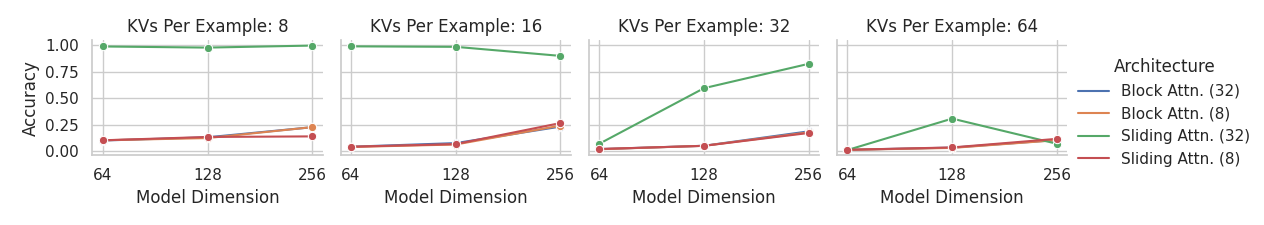}
    \caption{{$\Task$} quality as we vary the model dimension for synthetic datasets that contain different numbers of key-value pairs per example. (\textbf{Top}) We focus on the evaluation of the input-dependent gated recurrent RetNet~\cite{sun2023retentive} architecture. The architecture combines chunked attention with recurrence and we evaluate at two different chunk sizes (8 and 32). For reference, we include gated convolution (Hyena) and attention. (\textbf{Bottom}) We evaluate two efficient attention variants -- sliding window and blocked window attention -- at two window sizes (8 and 32).}
    \label{fig:kv_storage}
    \vspace{-2mm}
\end{figure}

%% file: Figures/Main/Figure6/figure.tex

%% file: Tables/full_downstream_ar.tex
\begin{table}[]
\centering
\small
\begin{tabular}{@{}lcccccc@{}}
\toprule
 &  &  & \textbf{Overall} & \multicolumn{2}{c}{\textbf{Slices}} & \textbf{\% gap due to} \\ 

Model  & Param (M) & TFLOPs  & & AR Hits & Other Tokens & AR Hits\\
\midrule
Attention  & 73 & 1.52 & \textbf{12.99 (2.56)} & \textbf{2.41 (0.88)} & 14.76 (2.69) & --- \\
Long Conv  & 76 & 1.20 & 20.28 (3.01) & 40.25 (3.70) & 19.25 (2.96) & 44.4\%\\
H3  & 72 & 1.33 & 15.78 (2.76) & 13.86 (2.63) & 15.94 (2.77) & 63.2\% \\
Hyena      & 72 & 1.34  & 15.13 (2.72) & 9.00 (2.20) & 15.74 (2.76) & 60.8\%  \\
RWKV       & 72 & 1.89 &   16.10 (2.78) & 14.11 (2.65) & 16.26 (2.79) & 57.9\%\\
\midrule
Attention & 125 & 2.46 & \textbf{11.01} (2.40) & \textbf{2.16 (0.77)} & 12.45 (2.52) & ---  \\
Long Conv & 128 & 1.74 &  16.98 (2.83) & 25.62 (3.24) & 16.46 (2.80) & 40.1\% \\
H3        & 168 & 2.55 &12.06 (2.49) & 6.75 (1.91) & 12.60 (2.53) & 88.4\% \\
RWKV      & 169 & 2.08 &  11.64 (2.45) & 5.70 (1.74) & 12.29 (2.51) & 100.0\% \\
RetNet (128) & 152 & 2.60 & 11.15 (2.41) & 3.01 (1.10) & 12.45 (2.55) & 100\% \\
Hyena     & 160 & 2.41 &   11.60 (2.45) & 5.00 (1.61) & 12.28 (2.51) & 100\% \\
 + Slide Attn. (64) & 159 & 2.28 &  11.57 (2.45) &  5.23 (1.65) &  12.30 (2.51) & 100\% \\
 + Slide Attn. (256) & 159 & 2.29 & \textbf{10.65} (2.37) &  3.42 (1.23) &  11.52 (2.45) & 100\% \\
\midrule
Attention & 360 & 6.23 &  \textbf{9.44} (2.25) & \textbf{1.98 (0.69)} & 10.62 (2.36) &  ---  \\
Long Conv & 360  & 4.08 & 13.13 (2.57) & 13.27 (2.59) & 13.12 (2.57) & 40.5\% \\
H3  & 357 & 4.85 & 10.38 (2.34) & 4.81 (1.57) & 11.00 (2.40) & 65.8\% \\
Hyena     & 358 & 5.03 & 10.07 (2.31) & 3.83 (1.34) & 10.75 (2.38) & 98.2\% \\
RWKV      & 351 & 4.31 & 9.79 (2.28) & 3.82 (1.34) & 10.51 (2.35) & 100.0\%  \\
\bottomrule
\end{tabular}
\caption{\textbf{Language modeling validation perplexity on the Pile.} After pretraining on 10B tokens of Pile data, we report log perplexity with negative log-likelihood in parentheses.  We report overall scores, and for the AR vs. non-AR token slices defined in \Cref{sec:sec3-downstream}. FLOPs are computed for inputs of 2048 tokens based on the equations in Appendix \ref{app:expdetails}. 
}
\label{table:full-ppl-slices}
\vspace{-5mm}
\end{table}

%% file: Sections/appendix/scaling.tex
\section{{\Task} Perplexity Gap and Model Size}
\label{sec:scaling}

In this section, we investigate how the associative recall gap changes as we increase the model size to a billion parameters and beyond. Overall, the results suggest that the larger models improve in memorizing bigrams from the training data, but do not rapidly improve in in-context learning (perplexity on rare bigrams).

Below, we pretrain at the 1.4Bn parameter scale and observe that the 70M parameter attention model is a full perplexity point better on AR than this 1.4Bn Hyena model (Table \ref{table:ppl-slices}). We discuss these results further in (\Cref{sec:sec3-downstream}).
\input{Tables/downstream_large}

\paragraph{Open-source 7B parameter models}
We next evaluate the RWKV (gated convolution) \cite{peng2023rwkv} and Llama 2 (attention) pretrained models at the 7B parameter scale \cite{touvron2023llama}.\footnote{We specifically evaluate the RWKV-Raven model downloaded from \url{https://huggingface.co/docs/transformers/model_doc/rwkv} and the Llama 2 model downloaded from \url{https://github.com/facebookresearch/llama}.} These are both popular models that took a significant amount of effort to train, towards maximizing quality. We find that there is a gap between RWKV and attention at the 7B scale and that it increases as the model needs to conduct more recalls per sequence ($P$ below).

We summarize the experimental protocol below. 
\begin{enumerate}
    \item \textbf{Justifying the experimental protocol.} Since frequent bigrams may just be memorized by the model and not require in-context recall, our work measures AR quality on infrequent bigrams in validation sequences (\Cref{fig:main}, \Cref{app:exp-measure-ar}). We do not have access to the custom training data mixtures used in training RWKV or Llama 2 to measure bigram frequencies, so we use a synthetic test to fairly measure the AR capabilities. 
    \item \textbf{Evaluation dataset.} We construct the  synthetic dataset following \Cref{alg:synthetic}, where each sequence contains $P$ token pairs, or ``bigrams''. Each bigram, containing a “key” token and a “value” token, appears twice in the sequence. On the second occurrence of a “key” token, the model should look back to the prior occurrence to output the corresponding “value” token. We measure the AR perplexity of the model based on its ability to predict the correct “value” token as the next word for these repeated keys. The sequences are constructed using the models' vocabulary tokens. 
    \item \textbf{Evaluation details.} We evaluate (inference only, no training) on sequence lengths of $1024$ tokens. We measure the AR perplexity when the sequences contain $P \in \{16, 32, 64, 128, 256\}$ key-value pairs, using $1000$ samples per $P$ value. The tokens that do not contain a key or value are simply filled with a fixed token id (so this token is repeated frequently in the sequence). We plot perplexity for AR and non-AR tokens (fixed token) vs. $P$. We find RWKV quality degrades with $P$ on the AR slice (blue line), while all other lines remain flat. MQAR remains problematic at scale.
\end{enumerate}
\begin{figure}
    \centering
    \includegraphics[width=0.6\linewidth]{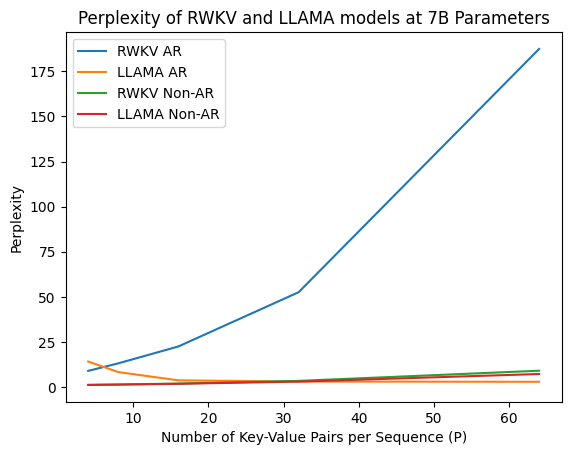}
    \caption{Perplexity of RWKV 7B and Llama2 7B parameter models on synthetic MQAR data as a function of the number of recalls required per input / number of key-value pairs per input ($P$). Inputs are constructed using each model's full vocabulary.}
    \label{fig:synth-7b}
    \vspace{-2mm}
\end{figure}

Overall our results suggest that simply scaling the model size may not close the MQAR quality gap between the gated convolutions and attention.

\paragraph{Measuring AR Gap with Scale} In the main paper, we measure the AR perplexity on downstream data based on bigrams that are seen $< 1,250\times$ during pretraining. However, we hypothesize that larger models may memorize a larger number of bigram pairs seen in the training dataset, but do not rapidly gain the capability to perform associative recall as well as attention in-context.  In-context learning is defined as learning from examples \textit{provided in context} \citep{xie2021incontext}.

Concretely, the gap between the Hyena / RWKV and attention models at the 350m scale is 1.85  / 1.84 perplexity points when we focus on bigrams seen $< 1,250\times$ during pretraining (\Cref{table:ppl-slices}). If we instead focus on bigrams seen $1\times$ during pretraining, the gap to attention quality is 12.0 / 13.2 perplexity points respectively. The gated convolutions appear to struggle in the regime of rare bigrams that require the model to use the context. 

%% file: Tables/downstream_large.tex
\begin{table}[h!]
\centering
\small
\begin{tabular}{@{}lcccccc@{}}
\toprule
 &  &  & \textbf{Overall} & \multicolumn{2}{c}{\textbf{Slices}} & \textbf{\% of gap due to} \\ 

Model  & Param (B) & TFLOPs  & & AR Hits & Other Tokens & AR Hits\\
\midrule
Attention  & 1.4 & 1.52 & 8.19 (2.10) & 1.91 (0.65) & 9.86 (2.29) & --- \\
Hyena  & 1.4 & 1.20 & 9.65 (2.27) & 3.43 (1.23) & 11.01 (2.40) & 40.3\%\\
\bottomrule
\end{tabular}
\caption{\textbf{Large-scale language model validation perplexity on the Pile.} After pretraining on 50B tokens of Pile data, we report log perplexity with negative log-likelihood in parentheses.  We report overall scores, and for the AR vs. non-AR token slices defined in \Cref{sec:sec3-downstream}. FLOPs are computed for inputs of 2048 tokens based on the equations in Appendix \ref{app:expdetails}.
}
\label{table:ppl-slices-large}
\end{table}

%% file: Sections/appendix/intro/setup.tex
\ignore{
\label{sec:gcm-intro-append}
We demonstrate the representation power of a newly defined layer, {\em Coyote}, with respect to similar gated convolution layers such as {\em Hyena} \citep{poli2023hyena}, and Attention \citep{vaswani2018attention}. We are interested in how the differences in structure between {\em Coyote} and Attention translates to solving what we define as the {\em multiple-query associative recall} problem.  In particular, we prove the following: 
\begin{itemize}
    \item general expressivity of \Coyote via arithmetic circuits,
    \item a sub-quadratic solution to mu associative recall.
    \item 
\end{itemize}
}
\subsubsection{Notation}
We denote the all \(1\) row vector of size $k$, given by \(\begin{bmatrix}1&1&\ldots&1&1\end{bmatrix}\),  and the all \(0\) row vector of size $k$, given by \(\begin{bmatrix}0&0&\ldots&0& 0\end{bmatrix}\), as \(\bm{1}^{k}\) and \(\bm{0}^{k}\), respectively. We also construe the standard basis vector $\mathbf{e}_i$ as a column vector in these notes, and adhere to the following matrix indexing convention: $\tbf{M}[i,j]$ is the entry in the $i$th row and the $j$th column,  $\tbf{M}[i,:] \in \F^{1 \times n}$ denotes the $i$th row, and $\tbf{M}[:,j] \in \F^{m \times 1}$ denotes the $j$th column of $\tbf{M} \in \F^{m \times n}$. We then use $\bm{1}^{m \times n}, \mathbf{0}^{m \times n} \in \F^{m \times 1}$ to denote the matrix of all $1$s and $0$s, respectively.

Next, we denote the {\em Hadamard product} of vectors $\tbf{u}, \tbf{v} \in \F^n$ as $\tbf{u}\odot \tbf{v}$; the operation can be extended to matrices by applying the Hadamard product column-wise across the matrices. This is commonly referred to as {\em(element-wise) gating}. For vectors $\tbf{u}, \tbf{v} \in \F^n$, we also denote their {\em linear (or acyclic) convolution} as $\tbf{u} \ast \tbf{v}$ and {\em cyclic convolution} as $\tbf{u} \circledast \tbf{v}$.

\paragraph{Polynomial Notation.} 

Because convolution is intimately tied to operations on polynomials, it is convenient to use them to discuss the inputs and outputs of gated convolution models. Let us define maps $\poly,\poly^* : \F^n \to \F[X]/(X^{n}) $ such that
\[
\begin{aligned}
\poly(\bm{u}) &= \sum_{i=0}^{n-1} \bm{u}[i]X^{i},\text{ and } \poly^*(\bm{u})  = \sum_{i=0}^{n-1} \bm{u}[i]X^{n-1-i}.
\end{aligned}
\]
This allows us to map between vectors and polynomial.
Accordingly, we also define $\coeff:\F[X]/(X^{n+1}) \to \F^n$ as the map converting polynomials back to vectors:
\(
 \coeff(\bm{u}(X)) = \bm{u} 
\)
with $\bm{u}[i]$ defined as the coefficient in $\bm{u}(X)$ at degree $i$.

These operations allow us to interpret the convolution of vectors  in terms of polynomial multiplication~\citep{heideman}. More specifically, we have 
\[
\begin{aligned}
\bm{u} \ast \bm{v} &= \coeff\paren{\bm{u}(X) \cdot \bm{v}(X) \mod{X^{n}}},\text{ and}\\
\bm{u} \circledast \bm{v}
&= \coeff\paren{\bm{u}(X) \cdot \bm{v}(X) \mod{X^{n}-1}}.\\
\end{aligned}
\]
We can similarly interpret the Hadamard product of vectors $\bm{u} \odot \bm{v}$ as the Hadamard product of polynomials $\bm{u}(X) \odot \bm{v}(X)$:
\[
\bm{u} \odot \bm{v} = \coeff\paren{\bm{u}(X) \odot \bm{v}(X)} = \coeff\paren{\sum_{i=0}^{n-1} (\bm{u}[i] \cdot \bm{v}[i]) \cdot X^i}.
\]
\paragraph{Arithmetic Circuit Notation.}  We briefly introduce the notation of arithmetic circuits~\citep{volkovich2016guide}, the focus of \cref{sec: arithmetic}. An {\em arithmetic circuit} $\calC$ with variables $X\triangleq \{x_1, x_2, \ldots, x_n\}$ over a field $\F$ is interpreted as a directed acyclic graph, where the input nodes are labelled by either the variables from $X$ or constants from $\F$ and the internal nodes are labelled by $+$ or $\times$ with the output being the polynomial computed at the output node. 
    
We shall also refer to the {\em size} of the circuit as the number of nodes, the {\em depth} of the circuit as the length of the longest path between an input node and the output node, and the {\em width} of the circuit as the number of parallel operations in the circuit, or `wires' which will be intersected by a horizontal `cut' through the circuit. Moreover, the {\em degree} of a circuit is defined as the degree of the polynomial computed by the circuit. We summarize this with the following definition:

\begin{definition}
\label{def: circuit-tuple}
   An arithmetic circuit $\calC$ is an {\em $(n,s,\circuitDegree,w)$-circuit} if $\calC$ is an $n$-variate arithmetic circuit of size $s$ and of depth at most $\circuitDegree$, and width $w$.
   \end{definition}

\paragraph{Model Notation.} Now we introduce the notation we will be using for defining layers. In what follows, we denote $\hyenaInput \in \R^{\inputDim}$ as the model input; $\inputLength$ as the sequence length; 
$\hyenaDepth$ as the number of stacked layers, indexed by $\hyenaLayerIndex$;
and $\headDim$ as the input (embedding) dimension.

\subsubsection{Summary of the Results}
The outline of our results are as follows: In \cref{sec:models_appen} we introduce \emph{gated convolution models} and define H3, Hyena, RWKV, RetNet, and {\Coyote}. In \cref{sec: primitives} we introduce a set of primitive operations that {\Coyote} can implement. We use these as tools in the subsequent proofs. 
Then, in \cref{sec: arithmetic}, we show that a general arithmetic circuit of size $s$ and degree at most $\circuitDegree$ can be simulated by {\Coyote}. We use the above results to show all models built from gating and convolutions can be simulated by {\Coyote}. This helps us analyze the \textit{class} of gated convolutions, beyond a specific proposal (e.g. Hyena).

Next, we study the representational power that gated convolutions and attention use to solve ${\Task}$. In \cref{sec: gen-ar}, we derive a {\Coyote} model inspired from dyadic intervals to show the dimensionality for gated convolutions (with input-independent filters) solve ${\Task}$. Next, we analyze the dimensionality for a {\Coyote} model with data-dependent kernels to solve ${\Task}$ in \cref{sec: data-dep-ar}.

%% file: Sections/appendix/intro/layer.tex
\subsection{Gated Attention-Free Models}
\label{sec:models_appen}
We now present formal definitions of gated convolution and recurrence models with respect to the 5-tuple of parameters $\paren{\inputLength,\hyenaDepth,\headDim,\innerN,\innerD}$.

\begin{definition}
\label{def: gated-conv}
   An $\modelTuple{\inputLength}{\hyenaDepth}{\headDim}{\innerN}{\innerD}{\text{Gated Convolution Model}}$ is a stacked sequence to sequence model with $L$ layers such that:

   \begin{enumerate}
     \item  Input and output are $\inputLength\times \headDim$ matrices, 
    \item  Each layer's operations consist of element-wise gating, convolution, linear projection,  and
    \item All the individual gated convolution layers take in $\innerN \times \innerD$ matrices and output $\innerN \times \innerD$ matrices. We refer to the tuple $\innerDim$ as the \emph{inner dimension} of the model.  
\end{enumerate}
   \end{definition}

We define the Hyena, RWKV, RetNet, and \Coyote\ layers to make step 2 more concrete. We also assume that the input $\hyenaInput \in \R^{\inputDim}$ is embedded into $\hyenaInput' \in \R^{\innerN \times \innerD}$  such that

\[
    \hyenaInput'[n,t] = \begin{cases}
    \hyenaInput[n,t] \ \ \text{ if } n < \inputLength, \ t < \headDim \ \\
    0 \ \ \text{ otherwise. }
    \end{cases}
\]

The output from the last layer $\bm{z} \in \R^{\innerN \times \innerD}$ is transformed into output $\bm{y} \in R^{\inputDim}$ by extracting the top left $\inputDim$ entries in $\bm{z}$.

Next, we define the class of weight matrices that we will use in the linear projections in the models:

\begin{definition}\label{def: W-kmat}
    The linear projection ${\tt Linear}_{m,m}$ has its matrix representation as the weight matrix\footnote{I.e. ${\tt Linear}_{m,m}(\bm{z})=\bm{W}\cdot\bm{z}$.}$\mathbf{W} \in \mathbb{R}^{m \times m}$ taken to be a {\bf K-matrix} for $\mathbf{W} \in (\calB\calB^{*})^{\text{poly-}\log{m}}_{\text{poly-}\log{m}}$ (\cref{def: kaleidoscope}). Consequently, each matrix $\mathbf{W}$ has $\tilde{\calO}(m)$ parameters and runtime for matrix vector multiplication, and allows us to represent general linear transformations with low-depth linear arithmetic circuits~\citep{dao2020kaleidoscope}. We will also need linear maps ${\tt Linear}_{m,n}$ for $m<n$, where each take the corresponding square matrices from ${\tt Linear}_{m,m}$ and note that such matrices  has $\tilde{\calO}(n)$ parameters and runtime for matrix vector multiplication.
\end{definition}
\begin{remark}
    We note that the weight matrix $\mathbf{W} \in \mathbb{R}^{d \times d}$ above is taken to be a dense matrix in the experiments. However, for our theoretical results, we restrict the linear projection in our models as per \cref{def: W-kmat}. In the rest of the paper, unless mentioned otherwise all linear projections will follow \cref{def: W-kmat}.
\end{remark}

Next, we define three popular gated convolution models that we study in our work: Hyena \cite{poli2023hyena}, RWKV \cite{peng2023rwkv}, and RetNet \cite{sun2023retentive}.

\subsubsection{The Hyena Layer}
\label{sec:hyena-model}
We will now outline the Hyena layer~\citep{hyena}. Hyena takes a sequence $\bm{u} \in \R^{\inputLength\times \headDim}$ as input and produces $\hyenaDepth+1$ projections $p^1, \ldots, p^L, v$ by passing $\bm{y}$ though a linear layer and applying a short convolution afterwards.
The algorithm then recursively 
performs a point-wise multiplication of the projection with the convolution of the filter $h^l$ with the previous output. We summarize this process in \cref{algo: hyena}. 

\begin{algorithm}[H]
		\caption{$\texttt{Projection}\paren{\hyenaInput,\convFilter}$}
		\begin{algorithmic}[1]\label{algo: projection}
            \Require Input sequence $\bm{u} \in \R^{\inputLength \times \headDim}$, a short convolution filter $\convFilter  \in \R^{\inputLength}$.
		  \State In parallel for $0 \leq  n <\inputLength : \hat{\bm{z}}[n,:] \gets \LinearD{\headDim}{(\hyenaDepth+1)\headDim}{\hyenaInput[n,:]}$
    so that  $\hat{\bm{z}} \in \R^{\inputLength \times (\hyenaDepth+1)\headDim}$
            \State In parallel for $0 \leq t < (\hyenaDepth + 1)\headDim: \bm{z}[:,t] \gets \convFilter \ast \hat{\bm{z}}[:,t]$
            \State Reshape and split $\bm{z} \in \R^{\inputLength \times (\hyenaDepth+1)\headDim}$ into $\bm{p}^{1}, \ldots, \bm{p}^{\hyenaDepth}, \bm{v}$, where $\bm{p}^{\ell}, \bm{v} \in \R^{\inputLength \times \headDim}$ for $\ell \in [\hyenaDepth]$.
            \State \Return $\bm{p}^{1}, \ldots, \bm{p}^{\hyenaDepth}, \bm{v}$.
		\end{algorithmic}
\end{algorithm}
\begin{algorithm}[H]
		\caption{$\Hyena{\hyenaInput}{\convFilter,\convFilter_s}$}
		\begin{algorithmic}[1]\label{algo: hyena}
            \Require Input sequence $\bm{u} \in \R^{\inputLength \times \headDim}$, set of convolution filters $\convFilter^1,\dots,\convFilter^L \in \R^{\inputDim}$, short convolution filter $\convFilter_s  \in \R^{\inputLength}$.
		  \State$\bm{p}^{1}, \ldots, \bm{p}^{\hyenaDepth}, \bm{v} \gets \texttt{Projection}(\bm{u}, \convFilter_s)$. 
            \State $\bm{z}^0 \gets \bm{v}$
            \For{$\ell = 1, \ldots, L$}
                \State In parallel for $0 \leq t <\headDim: \bm{z}^{\ell}[:,t] \gets \bm{p}^\ell[t,:] \odot \paren{\convFilter^\ell[:,t] \ast \bm{z}^{\ell-1}[:,t]}$.
            \EndFor
            \State \Return $\bm{z}^{L}$
		\end{algorithmic}
\end{algorithm}
\begin{remark}
    In \cref{algo: hyena}, $\hyenaDepth$ is used to denute the nuber of recursive applications of the Hadamard product and convolutions, not the number of layers  Note that asymptotically, the recursive step does not make a difference.
\end{remark}
Henceforth, we will refer to a model consisting of $\hyenaDepth$ Hyena layers is a gated convolution model with associated tuple $\paren{\inputLength,\hyenaDepth,\headDim,\inputLength,(\hyenaDepth+1)\headDim}$ as $\hyenaTupleN$.

\subsubsection{The RWKV Layer}
\label{sec:rwkv-model}
We will now describe the RWKV layer~\citep{peng2023rwkv}. RWKV is typically referred to as an RNN, but, like some other recurrent models (\textit{e.g.} S4~\citep{gu2021efficiently}), it can also be viewed as a convolution. Here, we present the convolutional view of RWKV. We will see that it is closely related to the Hyena layer. 

RWKV takes a sequence $\bm{u} \in \R^{\inputLength\times \headDim}$ as input and applies a short convolution. Next, it produces three projections $q, k, v$ by passing $\bm{u}$ through a linear layer. Then, it performs a convolution sandwiched by element-wise multiplication. We summarize this process in \cref{algo: rwkv}. 

\begin{algorithm}[H]
\caption{$\texttt{RWKVProjection}\paren{\hyenaInput,\convFilter}$}
		\begin{algorithmic}[1]\label{algo: rwkvprojection}
            \Require Input sequence $\bm{u} \in \R^{\inputLength \times \headDim}$, a short convolution filter $\convFilter  \in \R^{\inputLength}$.
            \State In parallel for $0 \leq t < \headDim: \hat{\bm{z}}[:,t] \gets \convFilter \ast \bm{u}[:,t]$
		  \State In parallel for $0 \leq  n <\inputLength : \bm{z}[n,:] \gets \LinearD{\headDim}{3\headDim}{\hat{\bm{z}}[n,:]}$
    so that  $\bm{z} \in \R^{\inputLength \times 3\headDim}$
            \State Reshape and split $\bm{z} \in \R^{\inputLength \times 3\headDim}$ into $\bm{q}, \bm{k}, \bm{v} \in \R^{\inputLength \times \headDim}$.
            \State \Return $\bm{q}, \bm{k}, \bm{v}$.
		\end{algorithmic}
\end{algorithm}
In practice, the short convolution filter $\convFilter_s$ is restricted to a length two filter with $\convFilter_s[0] = \mu$ and $\convFilter_s[1] = 1 - \mu$, where $\mu \in [0,1]$ is learned parameter. In the RWKV paper, the short convolution is referred to as a ``time shift". 

\begin{algorithm}[H]
\caption{$\texttt{RWKV}\paren{\hyenaInput, \convFilter,\convFilter_s}$}
		\begin{algorithmic}[1]\label{algo: rwkv}
            \Require Input sequence $\bm{u} \in \R^{\inputLength \times \headDim}$, set of convolution filters $\convFilter^1,\dots,\convFilter^L \in \R^{\inputDim}$, short convolution filter $\convFilter_s  \in \R^{\inputLength}$.
		  \State$\bm{q}, \bm{k}, \bm{v} \gets \texttt{Projection}(\bm{u}, \convFilter_s)$.
            \State In parallel for $0 \leq t <\headDim: \bm{z}^{\ell}[:,t] \gets \sigma(\bm{q}[:,t]) \odot \paren{\convFilter^\ell[:,t] \ast (\softmax(\bm{k}) \odot \bm{v})[:,t]}$.
            \State \Return $\bm{z}^{L}$
		\end{algorithmic}
\end{algorithm}

In practice, the long convolution filter $\convFilter$ is also restricted to $\convFilter[:, t] = e^{w(t-1)}$, where $w\in \mathbb{R}$ is a learnable parameter that controls how quickly the magnitudes of the filter decreases as $t$ grows.

To summarize, the differences between Hyena and RWKV are: (1) Hyena applies the short convolution \textit{after} the linear projection whereas RWKV applies it \textit{before}, (2) RWKV includes non-linear activations (sigmoid and softmax) while Hyena does not, (3) RWKV and Hyena use different parameterizations for the convolutional filters, and (4) Hyena recursively performs a point-wise multiplication of the projections with the convolution filter whereas RWKV performs this operation only once (though, in practice, Hyena uses a recurrence depth of one, making them equivalent in this regard).

\subsubsection{The Retnet Layer}
\label{sec:retnet-model}
In this section, we introduce the RetNet layer~\cite{sun2023retentive}. To this end, we take an input sequence $\bm{u} \in \R^{N \times d}$ and project it using learnable weight matrices. We then compute the states $\bm{z}^n$ recurrently as in line \ref{step: recurrence-retnet} and the output ${\tt Out}[n,:]$ as in line \ref{step: output-retnet} for each $n \in [N]$ (see \cref{algo: retnet} below).
\begin{algorithm}[H]
\caption{$\texttt{RetNet}\paren{\hyenaInput, \convFilter,\convFilter_s}$}
		\begin{algorithmic}[1]\label{algo: retnet}
            \Require Input sequence $\bm{u} \in \R^{\inputLength \times \headDim}$, a scalar $\gamma \in \R$, learnable weight matrices $\tbf{W}_A, \tbf{W}_C, \tbf{W}_V \in  \R^{d\times d}$. 
		    \State $\tbf{A}, \tbf{C}, \tbf{V} \gets \bm{u}\tbf{W}_A, \bm{u}\tbf{W}_C, \bm{u}\tbf{W}_V$ so that $\tbf{A}, \tbf{C}, \tbf{V} \in \R^{N \times d}$.
            \State Initialize the output ${\tt Out} \in \R^{N \times d}$
            \State Initialize the state $\bm{z}^0 \gets \paren{\tbf{A}[0,:]}^{\top}\tbf{V}[0,:]$.
            \For{$1 \le n < N$}
                \State $\bm{z}^{n} \gets \gamma \bm{z}^{n-1} + \paren{\tbf{A}[n, :]}^{\top}\tbf{V}[n,:]$ for $\bm{z}^n \in \R^{d \times d}$ \label{step: recurrence-retnet}
                \State ${\tt Out}[n,:] \gets \tbf{C}[n,:]\bm{z}^{n}$ \label{step: output-retnet}
            \EndFor
            \State \Return ${\tt Out}$
		\end{algorithmic}
\end{algorithm}

\input{Sections/appendix/intro/coyote}

%% file: Sections/appendix/intro/coyote.tex
\subsubsection{\Coyote}
\label{sec:simplified-hyena-append}
Finally, we introduce the {\Coyote} here as follows:
\begin{equation}
\label{eq:simplified_hyena}
        \mathbf{Y} = (\hyenaInput\mathbf{W} + \bm{b}_1) \odot (\hyenaInput \ast \bm{h} + \bm{b}_2),
\end{equation}
with input $\hyenaInput \in \R^{\innerN \times \innerD}$, weight matrix $\mathbf{W} \in \mathbb{R}^{\innerD \times \innerD}$  and bias matrices $\bm{b}_i \in \mathbb{R}^{\innerN \times \headDim'}$ defining linear projections of the input sequence, and $\bm{h} \in \R^{\innerN \times \innerD}$ is the a set of the $\innerD$ mixed length filters.

The corresponding pseudocode for $\SHyenaName$ is as follows: 
\begin{algorithm}[H]
		\caption{$\SHyena{\hyenaInput}{\bm{W},\tbf{b}_1,\convFilter,\tbf{b}_2}$}
		\begin{algorithmic}[1]\label{algo: simp-hyena}
            \Require Input sequence $\hyenaInput \in \R^{\innerN \times \headDim'}$, linear mapping $\bm{W} \in \R^{\innerD \times \innerD}$ so that $\bm{W}\in (\calB\calB^{*})^{\text{poly}\log{\innerD}}_{\text{poly}\log{\innerD}}$, convolution filter $\convFilter \in \R^{\innerN \times \headDim'
            }$, bias matrices $\tbf{b}_1,\tbf{b}_2 \in \R^{\innerN \times \headDim'}$.
            \State In parallel for $0 \leq  n <\innerN: \bm{x}[n,:] = \LinearD{d'}{d'}{\hyenaInput[n,:]}$
            \State In parallel for $0 \leq t <\headDim': \bm{z}[:,t] = \convFilter[:,t] \ast \hyenaInput[:,t]$
            \Statex
            \State In parallel for $0 \leq t <\headDim': \bm{y}[:,t] \gets \paren{\bm{x}[:,t] + \bm{b}_1[:,t]} \odot \paren{\bm{z}[:,t] + \bm{b}_2[:,t]}$.\Comment{See \eqref{eq:simplified_hyena}}
            \State \Return $\bm{y}$
		\end{algorithmic}
\end{algorithm}

Note that the convolution in \cref{algo: simp-hyena} is not limited to causal convolution and allows circular convolution as well. For simplicity, we use $\ast$ here and will disambiguate when required. We will start by specifying the parameter count and complexity of a single layer of \Coyote\ below.
\begin{proposition}\label{prop: single-baseconv}
    A single layer of \Coyote\ requires $\tilde{O}(\innerN\innerD)$ parameters and runtime.
\end{proposition}
\begin{proof}
    We know the parameters and runtime for the linear part of \Coyote\ via \cref{def: W-kmat} is $\tilde{\calO}(\innerN\innerD)$. Further, each convolution operation requires $\calO(\innerN)$ parameters and $\calO(\innerN\log{\innerN})$ runtime, and we employ $\innerD$ such operations. Finally the Hadamard product step takes $O(nd)$ time.
\end{proof}

Similar to Hyena, we will refer to a model consisting of $\hyenaDepth$ $\SHyenaName$ layers as $\coyoteTupleN$. In our experiments, we extend $\SHyenaName$ by adding an MLP after \cref{algo: simp-hyena}. For simplicity we will denote $\SHyena{\hyenaInput}{\convFilter,\bm{b}_1,\bm{b}_2}$ as $\SHyena{\hyenaInput}{\convFilter}$ when $\bm{b}_1 = \bm{b}_2 = \tbf{0}$. 

We will now show that there exists a \Coyote\ model that can emulate each of the basic operations in \cref{algo: simp-hyena}.
\begin{lemma}\label{lmm:primitives}
    The  functions $\LinearD{d}{d}{\hyenaInput}$ (\cref{def: W-kmat}), with $d,d'$ defined as in \cref{algo: simp-hyena},
    convolution with filter $\convFilter \in \R^{\inputDim}$, and element-wise gating can be computed with \cref{algo: simp-hyena} via a $\coyoteTuple{\inputLength}{1}{\headDim}{\inputLength}{\innerD}$.
\end{lemma}
\begin{proof}
For each operation from \cref{def: gated-conv} and \cref{algo: simp-hyena}:
\begin{enumerate}
    \item For any input $\hyenaInput \in \R^{\innerN \times \innerD}$, $\LinearD{d}{d'}{\hyenaInput}$ with matrix representation $\bm{W} \in \R^{\innerN \times \innerD}$ can be performed by a  single $\SHyenaName$ layer computing $\SHyena{\bm{y}}{\bm{W},\convFilter,\bm{b}_1,\bm{b}_2}$ with $\bm{b}_1$ and $\bm{b}_2$ being the matrix of all $0$s and all $1$s, respectively while and the convolution with the zero filter. That is, we have
    \[
    \mathbf{Y} = (\hyenaInput\mathbf{W} + \mathbf{0}^{N' \times d'}) \odot (\hyenaInput \ast \bm{0}^{N' \times d'} + \mathbf{1}^{N' \times d'}) = (\hyenaInput\mathbf{W}) \odot \mathbf{1}^{N' \times d'} = \hyenaInput\mathbf{W} = \LinearD{d}{d'}{\hyenaInput}.
    \]

    \item For any input $\hyenaInput \in \R^{\inputDim}$, convolution  with filter $\convFilter \in \R^{\inputDim}$ can be performed by a single $\SHyenaName$ layer computing $\SHyena{\bm{y}}{\bm{W},\convFilter,\bm{b}_1,\bm{b}_2}$ where $\tbf{W}, \bm{b}_2$ are all zeroes, and $\bm{b}_1$ is the matrix of all 1s so that we get
    \[
    \mathbf{Y} = (\hyenaInput\mathbf{0}^{N' \times d'} + \mathbf{1}^{N' \times d'}) \odot (\hyenaInput \ast \bm{h} + \mathbf{0}^{N' \times d'}) = \mathbf{1}^{N' \times d'} \odot (\hyenaInput \ast \bm{h}) = \hyenaInput \ast \bm{h}.
    \]
    \item We may compute element-wise gating between matrices $\hyenaInput,\bm{v} \in \R^{\inputDim}$, where $\bm{v}$ is some fixed factor, with a single layer computing $\SHyena{\bm{y}},\mathbf{0}^{N' \times d'}, {\bm{e}_0,\bm{v},\mathbf{0}^{N' \times d'}}$ where $\bm{e}_1$ is the identity filter, respectively, by \cref{def: gated-conv}.
    \[
    \mathbf{Y} = (\hyenaInput\mathbf{0}^{N' \times d'} + \bm{v}) \odot (\hyenaInput \ast \bm{e}_0 + \mathbf{0}^{N' \times d'}) =\bm{v} \odot \bm{u}.
    \]
\end{enumerate}
\end{proof}

%% file: Sections/appendix/circuit/primitives/setup.tex
\label{sec: primitives}
In this section, we will establish some additional basic primitives that we expect a gated convolution model to emulate: $\shiftN, \copyPrimitiveN$ and $\addPrimitiveN$. We specify them below:
\begin{enumerate}
    \item {\bf Shift} an sequential input of length $\inputLength$ up or down by $s$ entries:
    {\primitive{
    \boxed{\texttt{shift\_up($\bm{y}, s$), shift\_down($\bm{y}, s$)}}
    \begin{itemize}
        \item 
        {\bf Input:} $\bm{y} \in \R^{\inputDim}$, $s \geq 0 $
        \item {\bf Output: } $\bm{z} \in \R^{\inputDim}$ where $\bm{z}^+ =\texttt{shift\_down}(\bm{y}, s) $ and $\bm{z}^- =\texttt{shift\_up}(\bm{y}, s) $ 

        \[
            \bm{y} \equiv  \begin{pmatrix}
               \leftarrow \bm{y}_{0} \rightarrow \\
                \hline \\
                \vdots\\
                \hline \\
                \arrows{\bm{y}_{i-1}}\\
                \hline \\
                \arrows{\bm{y}_{i}}\\
                 \hline \\
                \vdots\\
                \hline \\
                \arrows{\bm{y}_{N-1}}
            \end{pmatrix}
             \qquad  \bm{z}^+ \equiv  \begin{pmatrix}
               \leftarrow \bm{0} \rightarrow \\
                \hline \\
                \vdots\\
                \hline \\
                \arrows{\bm{0}} \\
                \hline \\
               \arrows{\bm{y}_{0}}\\
                 \hline \\
                \vdots\\
                 \hline \\
               \arrows{ \bm{y}_{N-1-s}}
            \end{pmatrix}
             \qquad  \bm{z}^- \equiv  \begin{pmatrix}
               \leftarrow \bm{y}_{s} \rightarrow \\
                \hline \\
                \vdots\\
                \hline \\
               \arrows{ \bm{y}_{N-1}}\\
                \hline \\
               \arrows{\bm{0}}\\
                 \hline \\
                \vdots\\
                 \hline \\
               \arrows{\bm{0}}
            \end{pmatrix}
        \]
    \end{itemize}
    }}
    \item {\bf Add} a sequence $\bm{x} \in \R^{n \times d}$ to a running sum $\bm{S} \in  \R^{n \times d}$ for some $2n \le N'$ 
    with both $\bm{x}$ and $\bm{S}$ contained as subvectors in  $\bm{y} \in \R^{N \times d}$.
    {\primitive{
    {
     $\texttt{add}_{n}(\bm{y}:\bm{x}, \bm{S})$:
    }
    \begin{itemize}
        \item 
        {\bf Input:} sequence $\bm{y}$ containing $\bm{x}, \bm{S} \in \R^{n \times d}$ for $2n \le N$ such that $\bm{y}[0:n-1] \equiv \bm{x}, \bm{y}[n:2n-1] \equiv \bm{S}$ and $\bm{y}[2n: N-1] \equiv \bm{0}^{N-2n}$.
        \item {\bf Output: } $\bm{z} \in \R^{N \times d}$ containing the sum $\bm{y}+\bm{S}$ such that $\bm{y}[0:n-1] \equiv \bm{1}^n, \bm{y}[n:2n-1] \equiv \bm{S}+\bm{x}$ and $\bm{z}[2n: N-1] \equiv \bm{0}^{N-2n}$.
            \[
            \bm{y} \equiv  \begin{pmatrix}
               \arrows{ \bm{x}}\\
                 \hline \\
                \arrows{\bm{S}}\\
                \hline\\
                \longleftarrow \bm{0} \longrightarrow \\
                \hline \\
                \vdots\\
                \hline \\
                \arrows{\bm{0}}\\
            \end{pmatrix}
            \qquad\bm{z} \equiv  
             \begin{pmatrix}
               \arrows{ \bm{1}^n}\\
                 \hline \\
                \arrows{\bm{S}+\bm{x}}\\
                \hline\\
                \longleftarrow \bm{0} \longrightarrow \\
                \hline \\
                \vdots\\
                \hline \\
                \arrows{\bm{0}}\\
            \end{pmatrix}
            \]
    \end{itemize}}}
    
    \item {\bf Remember} $\bm{v} \in \R^{m \times d}$ as part of a sequence of input $\bm{y} \in \R^{\inputDim}$ while performing gated convolution {\em only} on $\bm{x} \in \R^{n \times d}$ for some $m,n \le \inputDim$. 
    {\primitive{
    {$\texttt{remember}_{n,m,s,t}(\bm{y:x,v,h,p})$:}
    \begin{itemize}
        \item 
        {\bf Input:}  sequence $\bm{y} \in \R^{\inputDim}$ containing $\bm{x} \in  \R^{n \times d}, \bm{v} \in \R^{m \times d}$, and modifiers $\bm{p, h} \in  \R^{n \times d}$ 
        such that $\bm{y}[0:n-1] \equiv \bm{x}, \bm{y}[n+s:n+s+m-1] \equiv \bm{v}$ and $\bm{y}[i] = 0$ otherwise  with $\bm{x} \ast \bm{h} \in \R^{(n+s) \times d}$ and $\bm{v} \ast \bm{h} \in \R^{(m+t) \times d}$. 
         \item {\bf Output: } $\bm{z} \in \R^{\inputDim}$ containing $\paren{\bm{p} \odot (\bm{x}\ast \bm{h})} \in \R^{(n+s)\times d}$, 
         such that:
            \[
            \bm{y} \equiv
                    \begin{pmatrix}       
                     \arrows{ \bm{x}} \\
                        \hline\\
                        \bm{0}^s\\
                        \hline\\
                        \leftarrow\bm{v}\rightarrow\\
                        \hline\\
                        \vdots\\
                        \hline\\
                        \bm{0}
                    \end{pmatrix} 
                    \qquad\qquad \bm{z} \equiv
                    \begin{pmatrix}       
                     \arrows{\bm{p} \odot (\bm{x} \ast \bm{h})} \\
                        \hline\\
                        \leftarrow\bm{v}\rightarrow\\
                        \hline\\
                        \vdots\\
                        \hline\\
                        \bm{0}
                    \end{pmatrix} 
            \] 
    \end{itemize}
    }}
\end{enumerate}
These primitives are building blocks of our proofs in the sequel. We will show that each of these primitives can be solved by some $\coyoteTupleN$ model with a small constant $\hyenaDepth$.

%% file: Sections/appendix/circuit/primitives/primitives.tex
\begin{proposition}[The Shift Primitive]
\label{prop: prim-shift}
For any $\bm{y} \in \R^{N \times d}$, there exist $\coyoteTuple{\inputLength}{1}{\headDim}{\inputLength}{\headDim}$ and $\coyoteTuple{\inputLength}{3}{\headDim}{\inputLength}{\headDim}$ that computes $\texttt{shift\_down}({\bm{y}}, {s})$ and $\texttt{shift\_up}({\bm{y}}, {s})$ for any $s \le N$.
\end{proposition}
\begin{proof}
Define the following kernel dependent on $s \le N$\:
\[
\bm{h}_s[n,:] \equiv
\begin{cases}
    \bm{1}^d &\text{if }n = s+1\\
    \bm{0}^d &\text{otherwise}.
\end{cases}
\]
We now deal with the down and up shifts separately:
\begin{enumerate}
    \item We define $\tbf{W}:= \mathbf{0}^{N \times d}, \bm{b}_1:= \mathbf{1}^{N \times d}, \bm{b}_2:= \mathbf{0}^{N \times d}$. Then, for input $\bm{y} \in \R^{N \times d}$, 
    $\SHyena{\bm{y}}{\mathbf{0}^{N \times d},\convFilter_{s},\mathbf{1}^{N \times d},\mathbf{0}^{N \times d}}$ for $\SHyenaName$ in \cref{algo: simp-hyena} is given by \eqref{eq:simplified_hyena} as
    \begin{align*}
        &\tbf{Y} \equiv   \bm{y} \ast \convFilter_s.
    \end{align*}
    Now, to perform \texttt{shift\_down}$(\bm{y}, s)$,
        we note that 
    \begin{align*}
        \tbf{Y}[:,j] = \bm{y}[:,j] \ast \convFilter_s[:,j] &= \coeff(\bm{y}[:,j](X)\cdot \bm{h}_s[:,j](X))\\
                                                         &=  \coeff\paren{\paren{\sum_{i=0}^{N-1}\bm{y}[i,j]\cdot X^i}\cdot X^s\mod{X^N}}\\
                                                         &= \coeff\paren{\sum_{i=0}^{N-1}\bm{y}[i,j]\cdot X^{i+s}\mod{X^N}}\\
                                                         &= \coeff\paren{\sum_{i=s}^{N-1+s}\bm{y}[i-s,j]\cdot X^{i}\mod{X^N}}\\
                                                         &= \coeff\paren{\sum_{i=s}^{N-1}\bm{y}[i-s,j]\cdot X^{i}},
    \end{align*}
    which implies that we exactly get what is specified in the output.
    \item We again define $\tbf{W}:= \mathbf{0}^{N \times d}, \bm{b}_1:= \mathbf{1}^{N \times d}, \bm{b}_2:= \mathbf{0}^{N \times d}$. Then, for input $\bm{y} \in \R^{N \times d}$, 
    $\SHyena{\bm{y}}{\mathbf{0}^{N \times d},\bm{e}_0,\mathbf{1}^{N \times d},\mathbf{0}^{N \times d}}$ for $\SHyenaName$ in \cref{algo: simp-hyena} is given in \eqref{eq:simplified_hyena} as
    \begin{align*}
        \tbf{Y}_0 \equiv   \bm{y} \circledast \bm{e}_0.
    \end{align*}
    Now, to perform \texttt{shift\_up}$(\bm{y}, s)$, as before, we first apply the circular convolution to reverse the input
    \begin{align*}
        \tbf{Y}_0[:,j] = \bm{y}[:,j] \circledast  \bm{e}_0 &= \coeff\paren{\sum_{i=0}^{N-1}\bm{y}[N-1-i,j]\cdot X^{i}},
    \end{align*}
    We then apply $\bm{Y}_1\equiv \texttt{shift\_down}(\bm{Y}_0, s)$ to get 
    \begin{align*}
    \bm{Y}_1[:,j] 
    &\equiv  \coeff\paren{\sum_{i=s}^{N-1}\bm{Y}_0[N-1-(i-s),j]\cdot X^{i}},\\
    &\equiv  \coeff\paren{\sum_{i=s}^{N-1}\bm{Y}_0[N-1-i+s,j]\cdot X^{i}}.
    \end{align*}
    Finally, we apply another circular convolution with the identity filter to replace $N-1-i$ with $i$ to get
     \begin{align*}
        \tbf{Y}_2[:,j] = \bm{Y}_1[:,j] \circledast  \bm{e}_0 &= \coeff\paren{\sum_{i=0}^{N-1}\bm{y}[i+s,j]\cdot X^{i}},
    \end{align*}
    Here, we note that we can compute both of these primitives in one and three layers, respectively~(see \cref{lem: stacking-layers}). 
\end{enumerate}

\end{proof}

Now, we present a $\SHyenaName$ model with two 
layers that implements the $\texttt{add}_n({\bm{y}:\bm{x}},{\bm{S}})$, the purpose of which is to add some window of computation $\bm{x}$ to a running sum $\bm{S}$. 
\begin{proposition}[The Running Sum Primitive]
\label{prop: prim-add}
     For any $\bm{x, S} \in \R^{n \times d}$ contained in some $\bm{y} \in \R^{N \times d}$, there exists a $\coyoteTuple{\inputLength}{2}{\headDim}{\inputLength}{\headDim}$ that computes $\texttt{add}_n({\bm{y}:\bm{x}},{\bm{S}})$
     for $\SHyenaName$ as in~\cref{algo: simp-hyena}.
\end{proposition}
\begin{proof}
We will show this for $d' = 1$ and the general case follows as we will explain at the end. We now specify the two layers that we use
\[
\begin{aligned}
    \bm{z}^1 &\equiv \SHyena{\bm{y}}{\mathbf{0}^{N \times 1},\convFilter^1, \bm{b}^1_1,\mathbf{0}^{N \times 1}} \equiv \bm{b}^1_1 \odot \paren{\bm{h}^1 \ast \bm{y}}\\
    \bm{z} &\equiv \SHyena{\bm{z}^1}{\mathbf{0}^{N \times 1},\convFilter^2, \bm{b}^2_1,\bm{b}^2_1} \equiv \bm{b}^2_1 \odot \paren{\bm{h}^2 \ast \bm{y} + \bm{b}^2_2},
\end{aligned}
\]
where we will specify the kernels as we go along. Let us start by defining the kernel and the bias for the first layer as 
    \[
    \bm{h}^{1} 
    \equiv
    \begin{pmatrix}
        \bm{e}_0\\
        \hline\\
        \bm{e}_0\\
        \hline\\
        \bm{0}^n \\
        \hline\\
        \cdots\\
        \hline\\
        \bm{0}^n
    \end{pmatrix},
    \quad \quad
    \bm{b}_1 
     \equiv 
    \begin{pmatrix}
        \bm{0}^n\\
        \hline\\
        \bm{1}^n\\
        \hline\\
        \bm{0}^n\\
        \hline\\
        \cdots\\
        \hline\\
        \bm{0}^n
    \end{pmatrix}.
    \]
    Let us first compute $\bm{h}^1 \ast \bm{y}$ as follows:
    \begin{align*}
     \bm{h}^1(X) \cdot \bm{y}(X) 
     &= (X^n + 1)\cdot (\bm{S}(X)\cdot X^{n} + \bm{x}(X))\\
     &= \bm{S}(X) \cdot X^{2n} + (\bm{S}+\bm{x})(X)\cdot X^{n} +\bm{x}(X).    
    \end{align*}
    We then have
    \[
    \bm{z}_1 \equiv \bm{b}_1^1 \odot \paren{\bm{h}^1 \ast \bm{y}}
     \equiv  \begin{pmatrix}
        \bm{0}^n\\
        \hline\\
        \bm{1}^n\\
        \hline\\
        \bm{0}^n\\
        \hline\\
        \cdots\\
        \hline\\
        \bm{0}^n
    \end{pmatrix} \odot 
    \begin{pmatrix}
        \bm{x}\\
        \hline\\
        \bm{S}+\bm{x}\\
        \hline\\
        \bm{S}\\
        \hline\\
        \cdots\\
        \hline\\
        \bm{0}^n
    \end{pmatrix} \equiv 
    \begin{pmatrix}
        \bm{0}^n\\
        \hline\\
        \bm{S}+\bm{x}\\
        \hline\\
        \bm{0}^n\\
        \hline\\
        \cdots\\
        \hline\\
        \bm{0}^n
    \end{pmatrix}
    \]
    \paragraph{Resetting for Next Phase.} We now use the next layer to reset for the next phase. Here, we need the first vector to be $\bm{1}^n$ in order to start adding the next vector. We thus use the kernel and the biases $\bm{h}^{2}, \bm{b}^2_1, \bm{b}^2_2$ defined as 
    \[
    \bm{h}^{2} \equiv 
    \begin{pmatrix}
        \bm{e}_0\\
        \hline\\
        \bm{0}^n\\
        \hline\\
        \bm{0}^n \\
        \hline\\
        \cdots\\
        \hline\\
        \bm{0}^n
    \end{pmatrix},\quad\quad
    \bm{b}^2_1 \equiv 
    \begin{pmatrix}
        \bm{1}^n\\
        \hline\\
        \bm{1}^n\\
        \hline\\
        \bm{0}^n \\
        \hline\\
        \cdots\\
        \hline\\
        \bm{0}^n
    \end{pmatrix},\quad\quad
    \bm{b}^2_2 \equiv 
    \begin{pmatrix}
        \bm{1}^n\\
        \hline\\
        \bm{0}^n\\
        \hline\\
        \bm{0}^n \\
        \hline\\
        \cdots\\
        \hline\\
        \bm{0}^n
    \end{pmatrix}.
    \]
    Explicitly, for the second layer, we compute the result of the convolution in terms of polynomials as follows:
    \begin{align*}
     \bm{h}^{2}(X) \cdot \bm{z}^{1}(X)
     = 1 \cdot (\bm{S}+\bm{x})(X)\cdot X^{n} = (\bm{S}+\bm{x})(X)\cdot X^{n}.
    \end{align*}
    Thus, the output for the second layer is given by
    \[
    \bm{z} \equiv \bm{b}_1^2 \odot \paren{\bm{h}^2 \ast \bm{z}^1 + \bm{b}^2_2}
     \equiv \begin{pmatrix}
        \bm{1}^n\\
        \hline\\
        \bm{1}^n\\
        \hline\\
        \bm{0}^n \\
        \hline\\
        \cdots\\
        \hline\\
        \bm{0}^n
    \end{pmatrix} \odot 
    \paren{
    \begin{pmatrix}
        \bm{0}^n\\
        \hline\\
        \bm{S}+\bm{x}\\
        \hline\\
        \bm{0}^n\\
        \hline\\
        \cdots\\
        \hline\\
        \bm{0}^n
    \end{pmatrix}
    + 
    \begin{pmatrix}
        \bm{1}^n\\
        \hline\\
        \bm{0}^n\\
        \hline\\
        \bm{0}^n\\
        \hline\\
        \cdots\\
        \hline\\
        \bm{0}^n
    \end{pmatrix}} \equiv  
   \begin{pmatrix}
        \bm{1}^n\\
        \hline\\
        \bm{1}^n\\
        \hline\\
        \bm{0}^n \\
        \hline\\
        \cdots\\
        \hline\\
        \bm{0}^n
    \end{pmatrix}  \odot 
    \begin{pmatrix}
        \bm{1}^n\\
        \hline\\
        \bm{S}+\bm{x}\\
        \hline\\
        \bm{0}^n\\
        \hline\\
        \cdots\\
        \hline\\
        \bm{0}^n
    \end{pmatrix} \equiv 
    \begin{pmatrix}
        \bm{1}^n\\
        \hline\\
        \bm{S}+\bm{x}\\
        \hline\\
        \bm{0}^n\\
        \hline\\
        \cdots\\
        \hline\\
        \bm{0}^n
    \end{pmatrix} .
    \]
     Therefore, we have used two \Coyote{} layers to add $\bm{x}$ to the running sum $\bm{S}$ and reset for the next phase.
     Here, we note that the only operations we perform and are convolutions and Hadamard product and they generalize in the obvious way to $d>1$.
\end{proof}

Next, we show that a five layer $\SHyenaName$ model can perform gated convolution on windows of the input (without changing the rest of the input).

\begin{proposition}[The Remembering Primitive]
\label{prop: prim-remember}
For any $\bm{x} \in \R^{n \times d}, \bm{v} \in \R^{m \times d}$ contained in some $\bm{y} \in \R^{N \times d}$ for some $n+m+s+t \le N$ so that for $\bm{h} \in \R^{n \times d}$ and $\bm{p} \in \R^{(n+s) \times d}$ with $\bm{x} \ast \bm{h} \in \R^{(n+s) \times d}$ and $\bm{v} \ast \bm{h} \in \R^{(m+t) \times d}$, there exists a $\coyoteTuple{\inputLength}{5}{\headDim}{\inputLength}{\headDim}$ that computes ${\tt remember}(\bm{y}:\bm{x},\bm{v},\bm{h}, \bm{p})$
for $\SHyenaName$ as in~\cref{algo: simp-hyena}.
\end{proposition}
\begin{proof}
We will again show this for $d' = 1$ and the general case should follow. 
We now specify the first two layers that we use
\[
\begin{aligned}
    \bm{z}^1 &\equiv \SHyena{\bm{y}}{\mathbf{0}^{N \times 1},\convFilter^1, \bm{b}^1_1,\mathbf{0}^{N \times d}} \equiv \bm{b}^1_1 \odot \paren{\bm{h}^1 \ast \bm{y}}\\
    \bm{z}^2 &\equiv \SHyena{\bm{z}^1}{\mathbf{0}^{N \times 1},\convFilter^2, \bm{b}^2_1,\mathbf{0}^{N \times d}} \equiv \bm{b}^2_1 \odot \paren{\bm{h}^2 \ast \bm{y}},
\end{aligned}
\]
The kernel $\bm{h}^1$ and the bias $\bm{b}^1_1$ for the first layer are then given by
    \[
    \bm{h}^{1} 
    \equiv
    \begin{pmatrix}
        \bm{h}\\
        \hline\\
        \bm{0}^m \\
        \hline\\
        \bm{e}_{s+t}\\
        \hline\\
        \bm{0}^n\\
        \hline\\
        \bm{0}^n\\
        \hline\\
        \cdots\\
        \hline\\
        \bm{0}^n\\
    \end{pmatrix},
    \quad \quad
    \bm{b}_1^1
     \equiv 
     \begin{pmatrix}
        \bm{p}\\
        \hline\\
        \bm{0}^{m+t}\\
        \hline\\
        \bm{0}^n \\
        \hline\\
        \bm{0}^s \\
        \hline\\
        \bm{1}^m\\
        \hline\\
        \cdots\\
        \hline\\
        \bm{0}^n\\
    \end{pmatrix}.
    \]
    where recall that $\bm{x} \ast \bm{h} \in \R^{(n+s) \times d}$ and $\bm{v} \ast \bm{h} \in \R^{(m+t) \times d}$.
    
We now want to first specify the result of applying the first kernel:
\begin{align*}
    \paren{\bm{h}^1 \ast \bm{y}} &= \mathrm{coeff}\paren{(\bm{h}(X)+X^{n+m+s+t}) \cdot \paren{\bm{v}(X) \cdot X^{n+s} + \bm{x}(X)}}\\
    &= \mathrm{coeff}\paren{\bm{h} \ast \bm{v} (X) \cdot X^{n+s} + \bm{h} \ast \bm{x} (X) + \bm{v}(X) \cdot X^{2n+2s+m+t} + \bm{x}(X)\cdot X^{n+m+s+t}}
\end{align*}
We then have 
    \[
    \bm{z}_1 \equiv \bm{b}_1^1 \odot \paren{\bm{h}^1 \ast \bm{y}}
     \equiv 
    \begin{pmatrix}
        \bm{p}\\
        \hline\\
        \bm{0}^{m+t}\\
        \hline\\
        \bm{0}^n \\
        \hline\\
        \bm{0}^s \\
        \hline\\
        \bm{1}^m\\
        \hline\\
        \cdots\\
        \hline\\
        \bm{0}^n\\
    \end{pmatrix} \odot 
    \begin{pmatrix}
        \bm{h}\ast \bm{x}\\
        \hline\\
        \bm{h} \ast \bm{v}\\
        \hline\\
        \bm{x} \\
        \hline\\
        \bm{0}^s \\
        \hline\\
        \bm{v}\\
        \hline\\
        \cdots\\
        \hline\\
        \bm{0}^n\\
    \end{pmatrix} \equiv  
    \begin{pmatrix}
        \bm{p} \odot (\bm{h}\ast \bm{x})\\
        \hline\\
        \bm{0}^{m+t}\\
        \hline\\
        \bm{0}^n \\
        \hline\\
        \bm{0}^s \\
        \hline\\
        \bm{v}\\
        \hline\\
        \cdots\\
        \hline\\
        \bm{0}^n\\
    \end{pmatrix}.
    \]
We now describe the second kernel $\bm{h}^2$ and the bias matrix $\bm{b}^2_1$ as follows:
\[
    \bm{h}^{2} 
    \equiv
    \begin{pmatrix}
        \bm{e}_0 \\
        \hline\\
        \bm{0}^{m+t}\\
        \hline\\
        \bm{e}^{0}\\
        \hline\\
        \bm{0}^{n+s}\\
        \hline\\
        \bm{0}^m\\
        \hline\\
        \cdots\\
        \hline\\
        \bm{0}
    \end{pmatrix},
    \quad \quad
    \bm{b}_1^2 
     \equiv 
    \begin{pmatrix}
        \bm{0}^{n+s} \\
        \hline\\
        \bm{0}^{m+t}\\
        \hline\\
        \bm{0}^{n}\\
        \hline\\
        \bm{1}^{n+s}\\
        \hline\\
        \bm{1}^m\\
        \hline\\
        \cdots\\
        \hline\\
        \bm{0}
    \end{pmatrix}
    \]
This yields the following convolution computation:
    \begin{align*}
    \bm{h}^2 \odot \bm{z}^1 
    &\equiv \mathrm{coeff}\paren{\paren{X^{m+n+s+t}+1} \cdot \paren{\bm{v}(X) \cdot X^{2n+2s+m+t} + \paren{\bm{p} \odot (\bm{h}\ast \bm{x})}(X)}}\\
    &\equiv \mathrm{coeff}(\bm{v}(X) \cdot X^{3n+3s+2m+2t} + \bm{v}(X) \cdot X^{2n+2s+m+t}\\
    &\quad + \paren{\bm{p} \odot (\bm{h}\ast \bm{x})}(X) \cdot X^{m+n+s+t} + \paren{\bm{p} \odot (\bm{h}\ast \bm{x})}(X))
    \end{align*}
Thus we have
     \[
    \bm{z}^2 \equiv \bm{b}_2^1 \odot \paren{\bm{h}^2 \ast \bm{z}^1}
     \equiv 
    \begin{pmatrix}
        \bm{0}^{n+s} \\
        \hline\\
        \bm{0}^{m+t}\\
        \hline\\
        \bm{0}^{n}\\
        \hline\\
        \bm{1}^{n+s}\\
        \hline\\
        \bm{1}^m\\
        \hline\\
        \cdots\\
        \hline\\
        \bm{0}
    \end{pmatrix}
    \odot 
    \begin{pmatrix}
         \bm{p} \odot (\bm{h}\ast \bm{x}) \\
        \hline\\
        \bm{0}^{m+t}\\
        \hline\\
        \bm{0}^{n}\\
        \hline\\
        \bm{p} \odot (\bm{h}\ast \bm{x})\\
        \hline\\
        \bm{v}\\
        \hline\\
        \cdots\\
        \hline\\
        \bm{0}
    \end{pmatrix} \equiv  
    \begin{pmatrix}
        \bm{0}^{n+s} \\
        \hline\\
        \bm{0}^{m+t}\\
        \hline\\
        \bm{0}^{n}\\
        \hline\\
        \bm{p} \odot (\bm{h}\ast \bm{x})\\
        \hline\\
        \bm{v}\\
        \hline\\
        \cdots\\
        \hline\\
        \bm{0}
    \end{pmatrix}
    \]
We now shift this up by $2n+s+m+t$ entries using the primitive operation defined in \cref{prop: prim-shift} that costs three additional layers so that we end up with 
\[
    \bm{z} \equiv 
    \begin{pmatrix}
        \bm{p} \odot (\bm{h} \ast \bm{u})\\
        \hline\\
        \bm{v}\\
        \hline\\
        \cdots\\
        \hline\\
        \bm{0} 
    \end{pmatrix}
    \]
Again, we note that the only operations we perform and are convolutions and Hadamard product and they generalize in the obvious way to $d>1$.
\end{proof}

Finally, we show that these primitives may be composed by 'stacking' models with matching inner dimension $\innerDim$.

\begin{lemma}\label{lem: stacking-layers}
     For $f,g:\R^{N\times d}\to \R^{N\times d}$ that have $(N,L_1,d,N',d')$ and $(N,L_2,d,N',d')$  $\SHyenaName$ models then their composition $f\circ g$ has an $(N,L_1+L_2,d,N',d')$ $\SHyenaName$ model which can be computed by performing their models in succession, or 'stacking'. 
\end{lemma}
\begin{proof}
This result follows from noting that for any $f(\hyenaInput)$ which requires $L_1$ layers to compute and that we can compute $f\circ g(\hyenaInput) = g(f(\hyenaInput))$ using the $\SHyenaName$ model with $L_2$ layers, yielding $L_1+L_2$ layers in total. 
\end{proof}

%% file: Sections/appendix/circuit/primitives/equivalence.tex
\subsubsection{\Coyote-Hyena Equivalence}
\label{sec: coyote-hyena-equiv}
We show that the equivalence between $\SHyenaName$ and Hyena by showing that each layer can simulate the other's computation using a constant number of layers.

\begin{proposition}
    \label{prop: coyote-hyena-equiv}
     For any input $\hyenaInput \in \R^{\inputDim}$ and $(N,L,d,N',d)-$Hyena such that $\bm{z}_{\mathrm{Hyena}} \equiv \mathrm{Hyena}(\bm{u})$ with a set of filters $\bm{h}^{\ell}$ and linear projections $\bm{p}^{\ell}$ as per \cref{def: W-kmat} for $\ell \in [L]$, there exists a $(N,5L,d,N'+N,d)$-\Coyote model such that $\bm{z}_{\mathrm{Hyena}} \equiv \Coyote(\bm{u})$.

     Similarly, for any input $\bm{u}_{\Coyote} \in \R^{\inputDim}$ and $(N,L,d,N',d)-$Coyote such that $\bm{z}_{\Coyote} \equiv \Coyote(\bm{u})$ with a set of filters $\bm{h}^{\ell}$ for $\ell \in [L]$, there exists a series of Hyena layers such that we have
     \[\underbrace{\mathrm{Hyena}\paren{\mathrm{Hyena}\paren{\ldots \mathrm{Hyena}(\bm{u}_{\Coyote, \bm{h}})}}}_{L\text{ layers}} \equiv \bm{z}_{\Coyote}.\]
\end{proposition}
\begin{proof}
    For the input $\bm{u}_{\mathrm{Hyena}} \in \R^{N \times d}$, the output of the $\ell$th layer $\bm{z}^{\ell}_{\mathrm{Hyena}} \in \R^{N' \times d'}$ for Hyena is given by (see \cref{algo: hyena})
\begin{align*}
    \bm{z}^{\ell}_{\mathrm{Hyena}} \equiv \bm{p}^{\ell}_{\mathrm{Hyena}} \odot (\bm{h}^l \ast \bm{z}^{\ell - 1}),
\end{align*}
where $\bm{p}^{\ell}_{\mathrm{Hyena}} \equiv \texttt{Linear}(\bm{u}_{\mathrm{Hyena}}) \in \R^{N' \times d}$. 
Now, using the original input $\bm{u}_{\mathrm{Hyena}} \in \R^{N \times d}$ to Hyena, we define the following input for \Coyote using one layer:
\[
\bm{u}_{\Coyote} \equiv 
               \begin{pmatrix}
                        \bm{u}_{\mathrm{Hyena}}\\
                        \hline\\
                        \bm{0}^{(N'-N) \times d}\\
                        \hline\\
                        \bm{u}_{\mathrm{Hyena}}
                \end{pmatrix}
\]
Then, we simply use the ${\tt remember}_{N, N, N'-N, N'-N}(\bm{u}_{\Coyote}:\bm{u}_{\mathrm{Hyena}},\bm{u}_{\mathrm{Hyena}},\bm{h}^{\ell}_{\mathrm{Hyena}},\bm{p}^{\ell}_{\mathrm{Hyena}})$ primitive for {\Coyote}. Consequently, this allows us to ``remember'' the input $\bm{u}_{\mathrm{Hyena}}$ in the output of the previous \Coyote layer $\bm{z}^{\ell-1}_{\Coyote}$. We then use this to retrieve $\bm{p}^{\ell}_{\mathrm{Hyena}} \equiv \mathrm{linear}(\bm{u}_{\mathrm{Hyena}})$ with the projection used for $\Coyote$ given by
\[
p^{\ell}_{\Coyote} \equiv \texttt{Linear}(\bm{z}^{\ell-1}_{\Coyote})
\equiv \begin{pmatrix}
                        \bm{1}^{N \times d}\\
                        \hline\\
                        \bm{p}^{\ell}_{\mathrm{Hyena}}.
                \end{pmatrix}
\]
Overall, the output of the $\ell$th layer for {\Coyote} is given by
\[
\bm{z}^{\ell}_{\Coyote} \equiv 
                \begin{pmatrix}
                        \bm{p}^{\ell}_{\mathrm{Hyena}} \odot \paren{\bm{h}^{\ell}_{\mathrm{Hyena}} \ast \bm{u}_{\mathrm{Hyena}}}\\
                        \hline\\
                        \bm{0}^{(M-N) \times d}\\
                        \hline\\
                        \bm{u}_{\mathrm{Hyena}}
                \end{pmatrix}
                \equiv \begin{pmatrix}
                        \bm{z}^{\ell}_{\mathrm{Hyena}}\\\hline\\
                        \bm{0}^{(M-N) \times d}\\
                        \hline\\
                        \bm{u}_{\mathrm{Hyena}}
                \end{pmatrix}
\]
Hence, we can reproduce the output of the $\ell$th layer of Hyena using five layers of {\Coyote} after augmenting the input and using the remembering primitive (\cref{prop: prim-remember}) with internal dimension $N'+N$.

Now, for the input $\bm{u}_{\Coyote} \in \R^{N \times d}$, the output of the $\ell$th layer for {\Coyote} is given by 
\begin{align*}
    &z^{\ell}_{\Coyote} \equiv \texttt{Linear}(z^{\ell - 1}_{\Coyote}) \odot \mathrm{conv}(h^l, z^{\ell - 1}_{\Coyote}).
\end{align*}
Here, we show inductively that simply using $\ell$-many Hyena models recursively simulates $z^{\ell}_{\Coyote}$. For $\ell = 1$, we have
\[
\mathrm{Hyena}(\bm{u}_{\Coyote, \bm{h}}) \equiv \texttt{Linear}(\bm{u}_{\Coyote}) \odot (\bm{h}^{1}\ast \bm{u}_{\Coyote}) \equiv \bm{z}^{1}_{\Coyote}.
\]
We now assume that $(\ell-1)$-many recursive Hyena models produce $z^{(\ell-1)}_{\Coyote}$. For the $\ell$th layer, we then have
\begin{align*}
&\mathrm{Hyena}\paren{\mathrm{Hyena}\paren{\ldots \mathrm{Hyena}(\bm{u}_{\Coyote, \bm{h}})}} \\
    \equiv\quad  &\mathrm{Hyena}\paren{z^{(l-1)}_{\Coyote}}\\
    \equiv\quad  &\mathrm{linear}\paren{z^{(l-1)}_{\Coyote}} \odot \mathrm{conv}\paren{h^l, z^{(l-1)}_{\Coyote}} \\
    \equiv\quad  &z^{\ell}_{\Coyote}.
\end{align*}
\end{proof}

\ignore{
\begin{proof}
The gated convolution step of \cref{algo: hyena} gives the following result with respect to input $\hyenaInput \in \R^{N \times d}$, and convolution filters $\bm{h} \in  \R^{N \times d}$
\begin{align*}
    \bm{y}_H = \bm{p}^{1} \odot \paren{\bm{h} \ast \bm{z}^{0}}
\end{align*}

where  $\bm{p}^{1} = \Projection{\hyenaInput,\convFilter_s}[:,:,0]$ and $\bm{z}^{0} = \Projection{\hyenaInput,\convFilter_s}[:,:,1]$ as in \cref{algo: projection}. Specifically we have  

\[\bm{p}^{1}= \paren{\convFilter_s \ast \LinearD{\headDim}{2\headDim}{\bm{u}}}[:,0:\headDim-1] \qquad \qquad \bm{z}^{0}= \paren{\convFilter_s \ast \LinearD{\headDim}{2\headDim}{\bm{u}}}[:,
\headDim:2\headDim-1]\]

We note that $\bm{y}_H$ can be computed from $\hyenaInput$ via the following sequence of operations:

\begin{enumerate}
    \item Compute $\bm{z} =\convFilter_s \ast \LinearD{\headDim}{2\headDim}{\bm{u}}$ where $\bm{z} \in \R^{\inputLength \times 2\headDim}$, giving us
    \[ \bm{z} =  \begin{pmatrix}
                \bm{p}^1\\
                \hline \\
                 \bm{z}^0
            \end{pmatrix} \]
    \item Compute $\shiftN\texttt{\_down}({\bm{z}},{\inputLength})$ to get
    \[ \bm{z}_{\shiftN} =  \begin{pmatrix}
                \bm{0}\\
                \hline \\
                 \bm{p}^1
            \end{pmatrix} \]
\AR{Where is $\bm{z}_{\shiftN}$ used later on?}

    \item Compute $\copyPrimitive{\bm{z}}{\bm{p^1,z^0,h,p^1}}$ \AR{There seem to be two issues here. (i) How do you get $\bm{z}^0$ as input from $\bm{z}_{\shiftN}$ (which presumably the input to this step?). (ii) The way I interpreted the remember operator the last two inputs are independent of $\bm{y}$: so we cannot have $\bm{p}^1$ as the last input since it depends on $\bm{y}$}to get
    \[ \bm{y} =  \begin{pmatrix}
                \bm{p}^1 \\
                \hline \\
                 \bm{p}^{1} \odot \paren{\bm{h} \ast \bm{z}^{0}}
            \end{pmatrix}, \]
            as desired.
\end{enumerate}

By \cref{lmm:primitives,prop: prim-remember,prop: prim-shift}, completing each step requires $\coyoteTuple{\inputLength}{1}{\headDim}{2\inputLength}{\innerD}$.
With \cref{prop: prim-stack}, we can stack \AR{Since $\bm{z}_{\shiftN}$ does not feed into the remember step, I do not see how you can stack things.} these models to define $\coyoteTuple{\inputLength}{3}{\headDim}{2\inputLength}{\innerD}$ such that its output $\bm{y}_C = \bm{y}_H$ as desired.

Now we show that for any $\bm{y}_C = \SHyena{\hyenaInput}{\bm{h}}$ defined on $\coyoteTuple{\inputLength}{1}{\headDim}{\innerN}{\innerD}$, there exists a $\hyenaTuple{\inputLength}{1}{\headDim}{N}{2\headDim}$ such that 
      $\bm{y}_C = \bm{y}_H$. By \cref{eq:simplified_hyena}, for any $\bm{u} \in \R^{N \times d}$ the output for $\SHyenaName$ with no bias matrices is given as
\begin{align*}
    &\bm{y}_C \equiv \tbf{W}\hyenaInput \odot \paren{\convFilter \ast \bm{u}}.
\end{align*}

Because the final step of \cref{algo: hyena} is a gated convolution, we may get $\bm{y}_C = \bm{y}_H$ by setting $\Projection{\hyenaInput, \convFilter_s} = \tbf{W}\hyenaInput$, thereby ignoring the convolution step in $\ProjectionN$. Then the statement follows.

\end{proof}
}

%% file: Sections/appendix/circuit/linear_ac.tex
\subsection{Linear Arithmetic Circuits}
\label{sec: linear_circuit}

In this section we show the relation between linear arithmetic circuits and $\Coyote$. We recall a few definitions from ~\citep{dao2020kaleidoscope}. 

\begin{definition}[Linear Arithmetic Circuit~\citep{bürgisser1996algebraic}]
    An arithmetic circuit is called a {\em linear arithmetic circuit} if it only uses addition, subtraction and scalar multiplication. Further, every multiplication has a fixed constant from $\F$ as at least one of its two inputs. In other words, all gates in the circuit are linear functions of their inputs (i.e. of the form $ax + by$ for fixed constants $a, b \in \F$).
\end{definition}

\begin{definition}[Butterfly Matrices~\citep{dao2020kaleidoscope}]
\label{def: butterfly}
A {\em butterfly factor} of size $k \geq 2$ (denoted as $\mathbf{B}_k$ ) is a matrix of the form $\mathbf{B}_k=$ $\left[\begin{array}{ll}\mathbf{D}_1 & \mathbf{D}_2 \\ \mathbf{D}_3 & \mathbf{D}_4\end{array}\right]$ where each $\mathbf{D}_i$ is a $\frac{k}{2} \times \frac{k}{2}$ diagonal matrix. We restrict $k$ to be a power of 2 .

A {\em butterfly factor matrix} of size $n$ with block size $k$ (denoted as $\mathbf{B}_k^{(n)}$ ) is a block diagonal matrix of $\frac{n}{k}$ (possibly different) butterfly factors of size $k$ :
$$
\mathbf{B}_k^{(n)}=\operatorname{diag}\left(\left[\mathbf{B}_k\right]_1,\left[\mathbf{B}_k\right]_2, \ldots,\left[\mathbf{B}_k\right]_{\frac{n}{k}}\right)
$$
Finally, a {\em butterfly matrix} of size $n$ (denoted as $\mathbf{B}^{(n)}$ ) is a matrix that can be expressed as a product of butterfly factor matrices: $\mathbf{B}^{(n)}=\mathbf{B}_n^{(n)} \mathbf{B}_{\frac{n}{2}}^{(n)} \ldots \mathbf{B}_2^{(n)}$. Equivalently, we may define $\mathbf{B}^{(n)}$ recursively as a matrix that can be expressed in the following form:
$$
\mathbf{B}^{(n)}=\mathbf{B}_n^{(n)}\left[\begin{array}{cc}
{\left[\mathbf{B}^{\left(\frac{n}{2}\right)}\right]_1} & 0 \\
0 & {\left[\mathbf{B}^{\left(\frac{n}{2}\right)}\right]_2}
\end{array}\right]
$$
(Note that $\left[\mathbf{B}^{\left(\frac{n}{2}\right)}\right]_1$ and $\left[\mathbf{B}^{\left(\frac{n}{2}\right)}\right]_2$ may be different.)
\end{definition}

From \cref{def: butterfly}, we observe that size $n$ butterfly factor is comprised of three vectors $\bm{d, d^+, d^-} \in \R^n$ such that 
\begin{align*}
    \bm{d} &= \paren{\diag^{-1}\paren{\tbf{D}_1} , \diag^{-1}\paren{\tbf{D}_4}}, \\
    \bm{d}^{+} &= \paren{\bm{0}^{\tfrac n2}, \diag^{-1}\paren{\tbf{D}_2}}, \text{ and } \\
    \bm{d}^{-} &= \paren{\diag^{-1}\paren{\tbf{D}_3} , \bm{0}^{\tfrac n2}},
\end{align*}

where $\diag^{-1}(\tbf{D}): \R^{n \times n} \mapsto \R^n$ is the mapping from diagonal matrices to the vector of its diagonal entries. Let us define $\tbf{D_1, D_2, D_3} \in \R^{n \times n}$ as $\diag\paren{\bm{d}}, \diag\paren{\bm{d^+}}$, and $\diag\paren{\bm{d^-}}$ respectively. Then we note that

\begin{equation}
\label{eq: butterfly-split}
    \tbf{D}_1 \equiv \left[\begin{array}{ll}\mathbf{D}_1 & \mathbf{0} \\ \mathbf{0} & \mathbf{D}_4\end{array}\right] \qquad \tbf{D}_2\bm{S}^{\tfrac n2} \equiv \left[\begin{array}{ll}\mathbf{0} & \mathbf{D}_2 \\ \mathbf{0} & \mathbf{0}\end{array}\right] \qquad \bm{S}^{\tfrac n2}\tbf{D}_3 \equiv \left[\begin{array}{ll}\mathbf{0} & \mathbf{0} \\ \mathbf{D}_3 & \mathbf{0}\end{array}\right]
\end{equation}

where $\tbf{S}^k \in \F^{n \times n}$ is a shift matrix for $i \in [n/2]$. This gives us the following proposition:

\begin{proposition}
    \label{prop: butterfly-decomposition}
    For any powers of 2, $n = k \geq 2$,
     any butterfly factor matrix $\tbf{B}^{(n)}_k$ is equivalent to 
    \[ \tbf{B}^{(n)}_k= \tbf{S}^{\tfrac k2}\tbf{D}_3 +\tbf{D}_2\tbf{S}^{\tfrac n2} + \tbf{D}_1\]
    where $\tbf{D_3, D_2, D_1, S^{\frac{n}{2}}}$ are defined as in \eqref{eq: butterfly-split}. 
\end{proposition}

We use \cref{prop: butterfly-decomposition} to show that butterfly matrices can easily be computed by $\SHyenaName$ .
\begin{lemma}\label{prop: butterfly-hyena}
For any $n,\headDim \geq 2$, $k \geq 1$, and arbitrary vector $\bm{x} \in \R^{n\headDim}$:
\begin{enumerate}[label=(\arabic*)]
      \item  there exists a $\coyoteTupleN$ that  can represent $\tbf{B}^{(nd)}_{k}\cdot \bm{x}$ with $N = n, \innerN = \calO(\inputLength)$, $\hyenaDepth = \calO(1)$, and $\innerD = \calO(\headDim)$, and 
      \item  there exists a $\coyoteTupleN$ that  can represent $\tbf{B}^{(nd)}\cdot \bm{x}$ with $N = n, \innerN = \calO(\inputLength)$, $\hyenaDepth = \calO(\log nd)$, and $\innerD = \calO(\headDim)$.
\end{enumerate}
    
\end{lemma}
\begin{proof}

\begin{enumerate}[label=(\arabic*)]
\item   Given $\bm{x} \in \R^{n\headDim}$, construct $\hyenaInput \in \R^{n\times d}$ 
where $\bm{x}$ is the row-major form of $\hyenaInput$. We show that $\Coyote$ can compute $\tbf{B}_{nd}\cdot \bm{x}$ column by column.

Let $\tbf{A}=\tbf{S}^{\tfrac k2}\tbf{D}'_3, \tbf{C} = \tbf{D}'_2\tbf{S}^{\tfrac n2}$, and $\tbf{D} = \tbf{D}_1$ for  $\tbf{D}_i,\tbf{S}^{\tfrac{k}{2}} \in \R^{nd \times nd}$ for $1\le i\le 3$ 
as defined in \cref{prop: butterfly-decomposition}. We take $\bm{d}_1 = \bm{1}^{nd}\tbf{D}, \ \bm{d}_2 = \bm{1}^{nd}\tbf{C}_2, \bm{d}_3 = \bm{1}^{nd}\tbf{A}$, which extracts the diagonal entries of $\tbf{D}_i$. With this we construct $\tbf{D}'_i \in \R^{n\times d}$ where $\bm{d}_i$ is the row major form of $\bm{D}'_i$. This implies that
    \[
    \tbf{D}_i\bm{x} \equiv \tbf{D}'_i \odot \bm{u}.
    \]
    
Then we can decompose   $\tbf{B}_{nd}\cdot \bm{x}$ into 
    \[\tbf{B}_{nd}\bm{x} \equiv \tbf{D}_1 \odot \bm{u} +  \tbf{D}_2 \odot \bm{u} + \tbf{D}_3 \odot \bm{u}.\]

   By \cref{lmm:primitives}, each Hadamard product $\tbf{A} \odot \bm{u}, \tbf{B} \odot \bm{u}, \tbf{C} \odot \bm{u}$ can be trivially be performed with a single layer $\Coyote$ model. Let each of these model outputs be denoted $\bm{y_1,y_2,y_3}$, respectively. Finally all that remains is to compute the $\bm{y_1 + y_2 + y_3}$. We achieve this using layers of add primitives\footnote{Recall that $\texttt{add}_{n}(\bm{y}: \bm{x}, \bm{S})$ adds the subvector $\bm{x}$ to $\bm{S}$ for the input $\bm{y}$.}:
   \[
   \begin{aligned}
     &\texttt{add}_{n}(\bm{y}^1: \bm{y}_1, \bm{0})\\
     &\texttt{add}_{n}(\bm{y}^2: \bm{y}_2, \bm{y}_1)\\
     &\texttt{add}_{n}(\bm{y}^3: \bm{y}_3, \bm{y}_1+\bm{y}_2),
   \end{aligned}
   \]
   where using by \cref{prop: prim-add} and \cref{lem: stacking-layers}, this requires six more layers, and we get
   \[
   \bm{y}^3 \equiv \bm{y_1 + y_2 + y_3} \equiv \tbf{B}_{nd}\bm{x}.
   \]
   Then we can construct the $\coyoteTupleN$ as desired with $L = O(1)$ layers. 

\item From \cref{def: butterfly}, $\mathbf{B}^{(nd)}=\mathbf{B}_{nd}^{(nd)} \mathbf{B}_{\frac{nd}{2}}^{(nd)} \ldots \mathbf{B}_2^{(nd)}$. From (1), $\Coyote$ can compute any butterfly matrix by simulating the $\log(nd)$ butterfly factor matrices which comprise $\mathbf{B}^{(nd)}$. With \cref{lem: stacking-layers}, this creates a $\Coyote$ with $5 \cdot \log(nd) = \calO( \log(nd))$ layers. \cref{lem: stacking-layers}
\end{enumerate}
 
\end{proof}
Butterfly matrices comprise the kaleidoscope hierarchy, which we define below:

\begin{definition}[The Kaleidoscope Hierarchy~\citep{dao2020kaleidoscope}]
\label{def: kaleidoscope}
\quad
\begin{itemize}
    \item Define $\calB$ as the set of all matrices that can be expressed in the form $\tbf{B}^{(n)}$ (for some $n$).
    \item Define $\paren{\cal{BB}^*}$ as the set of matrices $\tbf{M}$ of the form $\tbf{M} = \tbf{M}_1\tbf{M}_2^*$ for some $\tbf{M_1}, \tbf{M}_2 \in \calB$.
    \item Define $\paren{\cal{BB}^*}^w$ as the set of matrices $\tbf{M}$ that can be expressed as $\tbf{M} = \tbf{M}_w\ldots \tbf{M}_2\tbf{M}_1$, with each $\tbf{M}_i \in \paren{\cal{BB}^*} (1 \le i \le w).$ (The notation $w$ represents width.)
    \item Define $\paren{\cal{BB}^*}^w_e$ as the set of $n \times n$ matrices $\tbf{M}$ that can be expressed as $\tbf{M} = \tbf{SES}^{\top}$ for some $en \times en$ matrix $\tbf{E} \in \paren{\cal{BB}^*}^w$, where $\tbf{S} \in \F^{n \times en} = \begin{bmatrix}
        \tbf{I}_n & 0 &\ldots & 0
    \end{bmatrix}]$ (i.e. $\tbf{M}$ is the upper-left corner of $\tbf{E}$). (The notation $e$ represents expansion relative to $n$.)
\end{itemize}  
\end{definition}

We similarly show how $\Coyote$ can simulate any kaleidoscope matrix.

\begin{lemma}
\label{lmm: kaleido-coyote}
  Given $n,\headDim \geq 2$,  $e > 0$ for any $nd \times nd$ matrix $ \tbf{M} \in \paren{\cal{BB}^*}^{w}_{e}$, and $\bm{x} \in \R^{n\headDim}$  there exists a $\coyoteTupleN$ that  can represent $\tbf{M}\cdot \bm{x}$ 
  with $N = n, \hyenaDepth = \calO(w\log(end))$,
  $\innerN= en,$ 
  and $\  \innerD=d$. 
\end{lemma}
\begin{proof}
By \cref{def: kaleidoscope}, $\tbf{M}$ can be decomposed with respect to size $end \times end$ matrix 
\[\tbf{E} = \tbf{E}_1 \cdot \tbf{E}_2 \cdots \tbf{E}_{w}.\] 

 Further, any $\tbf{E}_i \in \paren{\cal{BB}^*}$ can be expressed as a product of $2 \log {end}$ 
 butterfly factor matrices. Then by \cref{prop: butterfly-hyena} and \cref{lem: stacking-layers} we can compute  $\tbf{E}_i\bm{x'}$ in by stacking $2 \log {end}$ $\coyoteTuple{n}{d}{L}{en}{d}$ models each with $L =\calO(1)$. Because $\tbf{E}$ has width $w$, \cref{lem: stacking-layers} implies that composing with each $\tbf{E}_i$ for $1 \leq i \leq w$ constructs a final model with $\calO(w\log(end))$ layers.
\end{proof}
Finally, the kaleidoscope hierarchy is related to linear arithmetic circuits via the following result. We note that in \citep{dao2020kaleidoscope} it is assumed that $w = s$, yet inspection of the proof yields the following stronger result:

\begin{theorem}[~\citep{dao2020kaleidoscope}]\label{thm: k-linear-ac}
Let $\tbf{M}$ be an $n \times n$ matrix such that multiplication of $\tbf{M}$ times an arbitrary vector $\bm{u}$ can be
be represented as $(n,s,\circuitDegree,w)$-linear arithmetic circuit $\calC$. Then, $\tbf{M} \in \paren{\cal{BB}^*}^{\calO(\circuitDegree)}_{\calO(w/n)}$.
\end{theorem}

 We combine \cref{thm: k-linear-ac} and \cref{lmm: kaleido-coyote} to show that $\Coyote$ can compute any linear arithmetic circuit with polylogarithmic factors in $\circuitDegree$.

\begin{corollary}\label{cor: lin-arith}
 For any $(nd,s,\circuitDegree,w)$-linear arithmetic circuit $\calC$ that can be represented by a matrix $\tbf{M} \in \R^{n\headDim \times n\headDim}$ multiplied by a vector $\bm{x} \in \R^{n\headDim}$, there exists an equivalent $\coyoteTuple{n}{\circuitDegree'}{\headDim}{w}{d}$ with  
 $\circuitDegree'=\calO(\Delta\log(w))$ such that  $\tbf{M}\bm{x} = \SHyena{\hyenaInput}{\convFilter}$ where $\bm{x}$ is the row major form of $\hyenaInput \in \R^{n \times d}$.
\end{corollary}

%% file: Sections/appendix/circuit/general_ac.tex
\subsection{General Arithmetic Circuits}
\label{sec: arithmetic}
We are now ready to prove the result that yields the equivalency between arithmetic circuits and \Coyote.
\begin{theorem}
    For any {\em $(n\headDim,s,\circuitDegree,w)$-arithmetic circuit} $\calC$, there exists an equivalent $\coyoteTuple{\inputLength}{\circuitDegree'}{\headDim}{\innerN}{\innerD}$ with $N = n, \circuitDegree'=\calO(\Delta\log{w})$, $N'=\calO(w), d'= d$ that simulates $\calC$.
    \label{thm: gen-ac}
\end{theorem}
For the reader's convenience, we begin with a proof sketch and then provide the details afterwards.
\begin{proof}[Proof Sketch of \cref{thm: gen-ac}]
Let us layer $\calC$ so that each layer $\calC_{\ell}$ for $\ell \in [L_{\calC}]$ either only has linear gates or multiplication gates, where the number of such layers $L_{\calC} = \calO(\Delta)$. The composition of all $\calC_{\ell}$ layers results in $\calC$. We use $\bm{z}^{\ell+1} \in \R^{w}$ to denote the output of the $\ell$-th layer $\calC_{\ell}$ which feeds as the input to the $(\ell+1)$-th layer $\calC_{\ell+1}$. Here, we note that if we can simulate each $\calC_{\ell}$ with \Coyote, then we can simulate the entire layered circuit $\calC$ due to \cref{lem: stacking-layers}. 

Now, if the layer $\calC_{\ell}^{\mathrm{lin}}$ is a linear layer (with only addition gates), then it can be represented by a matrix $\bm{M} \in \R^{w \times w}$ multiplied by $\bm{z}^{\ell} \in \R^{w}$ (We can append with $0$s if necessary so that the input from the previous gates can be written as $w$-length vector). 
Thus, we can apply \cref{cor: lin-arith} to simulate $\calC_{\ell}^{\mathrm{lin}}$ with an equivalent $\coyoteTuple{n}{\log{w}}{\headDim}{\calO(\log{w})}{d}$ model.

Next, if $\calC_{\ell}^{\mathrm{mult}}$ instead consists of only the multiplication gates. Then, we note here that the output $\bm{z}^{\ell}$ may not exactly equal the input to $\calC_{\ell}^{\mathrm{mult}}$. Nevertheless, we can apply a $\calO(w)$ sparse 
linear map $\tbf{R} \in \R^{w\times w}$ so that $\tbf{R}\bm{z}^{\ell}$ yields vectors $\bm{v}^1, \bm{v}^2,$ and $\bm{v}^3$, where $\bm{v}_1$ constitutes the ``first" input to all the multiplication gates and $\bm{v}_2$ constitutes all the ``second" inputs while $\bm{v}^3$ consists of all entries needed as inputs in the subsequent layers. That is, for the $i$th gate in $\calC_{\ell}^{\mathrm{mult}}$, we compute $\bm{v}^1_i \cdot \bm{v}^2_i$. This implies that for all the gates in $\calC_{\ell}^{\mathrm{mult}}$, we can simply compute $\bm{v}_1\odot \bm{v}_2$. To this end, we use the ${\tt remember}$ primitive with constant number of layers from \cref{prop: prim-remember} to define a $\coyoteTuple{n}{\calO(\log{w})}{\headDim}{w}{d}$ model that remembers $\bm{v}^3$ while performing the Hadamard product of $\bm{v}^1$ with $\bm{v}^2$. 

Overall, we can then collect all the resulting \Coyote\ layers and compose them as in \cref{lem: stacking-layers} to simulate $\calC$. Overall, the number of layers used is given by $\calO(\Delta\log{w})$ while the internal dimension remains fixed at $w$.
\end{proof}
Using the outline in the proof sketch, we now delve into a detailed analysis of the arithmetic circuit $\calC$ to an equivalent \Coyote model.
\begin{proof}[Proof of \cref{thm: gen-ac}]
Let $\bm{u} \in \R^{nd}$ be the input to the arithmetic circuit $\calC$ with depth $\Delta$ and width $w$. We begin by rearranging the circuit into layers of addition and multiplication gates. That is, each layer $\calC^{\ell}$ has either all addition gates or all multiplication gates. This allows us to readily apply the results from \cref{sec: linear_circuit}. Note here that we can assert that the number of such layers $L_{\calC} = \calO(\Delta)$. Moreover, we construe the input to each such circuit as a vector of length ${w}$ by appending with extra 0s if necessary so that the composition of the layers results in a circuit equivalent to $\calC$. See \cref{fig: circuit-decomp} for an example of such decomposition

\begin{figure}[!h]
    \centering
\tikzset{every picture/.style={line width=0.75pt}} 

\begin{tikzpicture}[x=0.75pt,y=0.75pt,yscale=-1,xscale=1]

\draw   (229.67,105.17) .. controls (229.67,97.71) and (235.71,91.67) .. (243.17,91.67) .. controls (250.62,91.67) and (256.67,97.71) .. (256.67,105.17) .. controls (256.67,112.62) and (250.62,118.67) .. (243.17,118.67) .. controls (235.71,118.67) and (229.67,112.62) .. (229.67,105.17) -- cycle ;
\draw   (285,104.33) .. controls (285,96.97) and (290.97,91) .. (298.33,91) .. controls (305.7,91) and (311.67,96.97) .. (311.67,104.33) .. controls (311.67,111.7) and (305.7,117.67) .. (298.33,117.67) .. controls (290.97,117.67) and (285,111.7) .. (285,104.33) -- cycle ;
\draw   (357.67,103.17) .. controls (357.67,95.71) and (363.71,89.67) .. (371.17,89.67) .. controls (378.62,89.67) and (384.67,95.71) .. (384.67,103.17) .. controls (384.67,110.62) and (378.62,116.67) .. (371.17,116.67) .. controls (363.71,116.67) and (357.67,110.62) .. (357.67,103.17) -- cycle ;
\draw   (413,103.33) .. controls (413,95.97) and (418.97,90) .. (426.33,90) .. controls (433.7,90) and (439.67,95.97) .. (439.67,103.33) .. controls (439.67,110.7) and (433.7,116.67) .. (426.33,116.67) .. controls (418.97,116.67) and (413,110.7) .. (413,103.33) -- cycle ;
\draw   (264.67,183.17) .. controls (264.67,175.71) and (270.71,169.67) .. (278.17,169.67) .. controls (285.62,169.67) and (291.67,175.71) .. (291.67,183.17) .. controls (291.67,190.62) and (285.62,196.67) .. (278.17,196.67) .. controls (270.71,196.67) and (264.67,190.62) .. (264.67,183.17) -- cycle ;
\draw   (320,182.33) .. controls (320,174.97) and (325.97,169) .. (333.33,169) .. controls (340.7,169) and (346.67,174.97) .. (346.67,182.33) .. controls (346.67,189.7) and (340.7,195.67) .. (333.33,195.67) .. controls (325.97,195.67) and (320,189.7) .. (320,182.33) -- cycle ;
\draw   (383.67,183.17) .. controls (383.67,175.71) and (389.71,169.67) .. (397.17,169.67) .. controls (404.62,169.67) and (410.67,175.71) .. (410.67,183.17) .. controls (410.67,190.62) and (404.62,196.67) .. (397.17,196.67) .. controls (389.71,196.67) and (383.67,190.62) .. (383.67,183.17) -- cycle ;
\draw   (247.67,261.17) .. controls (247.67,253.71) and (253.71,247.67) .. (261.17,247.67) .. controls (268.62,247.67) and (274.67,253.71) .. (274.67,261.17) .. controls (274.67,268.62) and (268.62,274.67) .. (261.17,274.67) .. controls (253.71,274.67) and (247.67,268.62) .. (247.67,261.17) -- cycle ;
\draw   (303,261.33) .. controls (303,253.97) and (308.97,248) .. (316.33,248) .. controls (323.7,248) and (329.67,253.97) .. (329.67,261.33) .. controls (329.67,268.7) and (323.7,274.67) .. (316.33,274.67) .. controls (308.97,274.67) and (303,268.7) .. (303,261.33) -- cycle ;
\draw   (365.67,261.17) .. controls (365.67,253.71) and (371.71,247.67) .. (379.17,247.67) .. controls (386.62,247.67) and (392.67,253.71) .. (392.67,261.17) .. controls (392.67,268.62) and (386.62,274.67) .. (379.17,274.67) .. controls (371.71,274.67) and (365.67,268.62) .. (365.67,261.17) -- cycle ;
\draw    (243.17,118.67) -- (278.17,169.67) ;
\draw    (298.33,118) -- (278.17,169.67) ;
\draw    (298.33,118) -- (397.17,169.67) ;
\draw    (243.17,118.67) -- (333.33,169) ;
\draw    (426.33,116.67) -- (333.33,169) ;
\draw    (371.17,117.33) -- (397.17,169.67) ;
\draw    (243.17,118.67) .. controls (229.33,186.83) and (229.33,189.83) .. (261.17,247.67) ;
\draw    (333.33,195.67) -- (316.33,248) ;
\draw    (397.17,196.67) -- (379.17,247.67) ;
\draw    (278.17,196.67) -- (263.5,246.75) ;
\draw    (333.33,195.67) -- (379.17,247.67) ;
\draw    (426.33,116.67) .. controls (425.33,152.83) and (427,252.25) .. (425,305.25) ;
\draw    (233.5,66.92) -- (244.17,92.67) ;
\draw    (256.5,67.92) -- (244.17,92.67) ;
\draw    (287.5,64.92) -- (298.17,90.67) ;
\draw    (310.5,65.92) -- (298.17,90.67) ;
\draw    (360,63.42) -- (370.67,89.17) ;
\draw    (383,64.42) -- (370.67,89.17) ;
\draw    (416,63.92) -- (426.67,89.67) ;
\draw    (439,64.92) -- (426.67,89.67) ;
\draw    (261.17,274.67) -- (261.5,307.08) ;
\draw    (316.33,274.67) -- (316.67,307.08) ;
\draw    (379.17,274.67) -- (379.5,307.08) ;
\draw    (278.17,196.67) -- (316.33,248) ;

\draw (323,41.4) node [anchor=north west][inner sep=0.75pt]    {$\dotsc $};
\draw (232,96.73) node [anchor=north west][inner sep=0.75pt]    {$\times $};
\draw (289,95.73) node [anchor=north west][inner sep=0.75pt]    {$\times $};
\draw (361,93.73) node [anchor=north west][inner sep=0.75pt]    {$\times $};
\draw (416,93.73) node [anchor=north west][inner sep=0.75pt]    {$\times $};
\draw (270,173.73) node [anchor=north west][inner sep=0.75pt]    {$+$};
\draw (326,172.73) node [anchor=north west][inner sep=0.75pt]    {$+$};
\draw (391,175.73) node [anchor=north west][inner sep=0.75pt]    {$+$};
\draw (251,251.73) node [anchor=north west][inner sep=0.75pt]    {$\times $};
\draw (306,251.73) node [anchor=north west][inner sep=0.75pt]    {$\times $};
\draw (369,251.73) node [anchor=north west][inner sep=0.75pt]    {$\times $};
\draw (331,309.29) node [anchor=north west][inner sep=0.75pt]    {$\dotsc $};
\draw (454.67,94.96) node [anchor=north west][inner sep=0.75pt]    {$\mathcal{C}^{\ell}_{\mathrm{mult}}$};
\draw (453.33,174.96) node [anchor=north west][inner sep=0.75pt]    {$\mathcal{C}^{\ell+1}_{\mathrm{lin}}$};
\draw (452.67,252.29) node [anchor=north west][inner sep=0.75pt]    {$\mathcal{C}^{\ell+2}_{\mathrm{mult}}$};
\end{tikzpicture}
\caption{An example decomposition of an arithmetic circuit as layers of only addition or multiplication gates.}
\label{fig: circuit-decomp}
\end{figure}
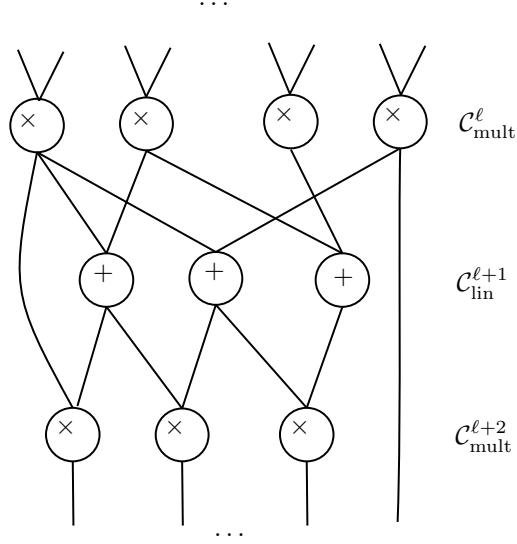

Let $\bm{z}^{\ell} \in \R^{w}$ denote the input to each layer $\calC_{\ell}$ for $\ell \in [L_{\calC}]$ with $\bm{z}^1 \equiv (\bm{u}, \bm{0}^{w-nd})$. It is important to note here that we may {\it not} have $\bm{z}^{\ell+1} \equiv \calC_{\ell}(\bm{z}^{\ell})$ since the inputs for gates at the $(\ell+1)$-th layer may come from any of the previous layers. 
We now handle the case of layers with addition gates, multiplication gates, and the intermediate stage separately.

\paragraph{Addition Gates.} Let $\calC_{\ell}^{\mathrm{lin}}$ denote an arbitrary linear layer which only contains the addition gates that takes in $\bm{z}^{\ell}$ as input. We know that there exists a matrix $\tbf{M} \in \R^{w \times w'}$ such that we have
\[
\calC_{\ell}^{\mathrm{lin}}(\bm{z}^{\ell}) \equiv \tbf{M}\bm{z}^{\ell}.
\]
Here, note here that we may need entries from the vector $\bm{z}^{\ell}$ in subsequent layers. Let $s$ be the total number of such entries, then for at each index for such entries $i \in [w]$, we have the corresponding index $i_s \in [s]$. We then append the $i$th standard row vector into the matrix $\tbf{M}$ to get the matrix $\tbf{M}' \in \R^{w \times (w'+s)}$ so that we have
\[
\paren{\tbf{M}'\bm{z}^{\ell}}[j] = \begin{cases}
    \paren{\tbf{M}\bm{z}^{\ell}}[j] &\text{if } j < w'\\
     \bm{z}^{\ell}[i] &\text{if } j = w'+i_s
\end{cases}
\]
Here, note that we must have $w'+s \le w$ as $w$ is the width of the circuit. If needed, we then append the matrix $\tbf{M}'$ with all-zero rows $\bm{0}^w$ so that we get the matrix $\tbf{M}'' \in \R^{w \times w}$. Note here that we have thus preserved all entries needed in subsequent layers and incorporated the output of the $\ell$th layer with the output $\tbf{M}''\bm{z}^{\ell} \equiv \bm{z}^{\ell+1}$ serving as the input of the $(\ell+1)$-th layer. Applying \cref{cor: lin-arith} then yields an equivalent $\coyoteTuple{n}{\calO(\log{w})}{\headDim}{w}{d}$ model.

\paragraph{Multiplication Gates.} 
Next, we deal with the layer $\calC^{\ell}_{\mathrm{mult}}$ of multiplication gates by collecting the two inputs to each multiplication gates as $\bm{v}_1^{\ell}, \bm{v}_2^{\ell}$. Note that for the input $\bm{z}^{\ell} \in \R^{w}$ from the previous layer, we will again have entries that need to be used in the subsequent layers that we will denote as $\bm{v}_3^{\ell}$. We thus need to permute the entries of $\bm{z}^{\ell}$ to get the vector $[\bm{v}_1^{\ell} :: \bm{v}_2^{\ell} :: \bm{v}_3^{\ell}]$ so that we can remember $\bm{v}_3^{\ell}$ while taking the Hadamard product $\bm{v}_1^{\ell} \odot \bm{v}_2^{\ell}$. To this end, we can achieve this permutation of entries using a $\calO(w)$-sparse linear matrix\footnote{Each row of $\tbf{R}^{\ell}$ has exactly one non-zero entry.} $\tbf{R}^{\ell} \in \R^{w \times w}$ with which is equivalently represented by an  $(w,w,1,w)$-linear arithmetic circuit that simply moves the appropriate wires from the previous layer. This can again be achieved by a equivalent $\coyoteTuple{n}{\calO(\log{w})}{\headDim}{w}{d}$ model. That is, $\tbf{R}^{\ell}\bm{z}^{\ell}$ has the following form:
\[
\tbf{R}^{\ell}\bm{z}^{\ell} \equiv 
                \begin{pmatrix}
                        \bm{v}_1^{\ell}\\
                        \hline\\
                        \bm{v}_2^{\ell}\\
                        \hline\\
                        \bm{v}_3^{\ell}
                \end{pmatrix}.
\]
Next, we can now define a $\coyoteTuple{n}{1}{\headDim}{n}{d}$ which extracts $\bm{v}_1^{\ell}$ as the projection with the input to the remember primitive given by $\bm{y}^{\ell} \equiv (\bm{v}_2^{\ell}, \bm{v}_3^{\ell}, \bm{0})$. We then specify the {\tt remember} primitive as ${\tt remember}\paren{\bm{y}^{\ell}: \bm{v}_2^{\ell}, \bm{v}_3^{\ell}, \bm{e}_0, \bm{v}_1^{\ell}}$ which computes the following
\begin{align*}
    \bm{z}^{\ell+1} 
    &\equiv \begin{pmatrix}
                        \bm{v}_1^{\ell} \odot (\bm{e}_0 \ast \bm{v}_2^{\ell})\\
                        \hline\\
                        \bm{v}_3^{\ell}\\
                        \hline\\
                        \bm{0}
                \end{pmatrix}    \\  
    &\equiv \begin{pmatrix}
                        \bm{v}_1^{\ell} \odot \bm{v}_2^{\ell} \\
                        \hline\\
                        \bm{v}_3^{\ell} \\
                        \hline\\
                        \bm{0}
                \end{pmatrix}    \\ 
    &\equiv \calC^{\ell}_{\mathrm{mult}}(\bm{z}^{\ell}) \in \R^{w}.
\end{align*}
Using \cref{prop: prim-remember}, we know that this requires a $\coyoteTuple{n}{\calO(\log{w})}{\headDim}{w}{d}$ which remembers the entries that will serve as inputs in subsequent layers while performing the Hadamard product $\bm{v}^1 \odot \bm{v}^2$.

Overall, we can now stack all the resulting  \Coyote\ layers and compose them as in \cref{lem: stacking-layers} to simulate $\calC$. Overall, the number of layers blows up by $\calO(\Delta\log{w})$ while the internal dimension remains fixed at $w$.
\end{proof}

%% file: Sections/appendix/circuit/retnet_equivalency.tex
\subsection{The Retnet Reduction and Space Complexity for AR}
\label{app:retnet-proof}

We are now in a position to show that there exists an equivalent \Coyote\ model that can simulate a Retnet layer from \cref{algo: retnet}. Recall that for an input sequence $\bm{u} \in \R^{N \times d}$, we project the input using weight matrices $\tbf{W}_A, \tbf{W}_C, \tbf{W}_V\in \R^{d \times d}$ to $\tbf{A}, \tbf{C}, \tbf{V} \in  \R^{N \times d}$ and compute the following recurrence (see line \ref{step: recurrence-retnet} in \cref{algo: retnet}):
\begin{equation}
    \label{eq: retnet-recurrence}
    \bm{z}^{n} \equiv \gamma \bm{z}^{n-1} + \paren{\tbf{A}[n,:]}^{\top}\tbf{V}[n,:],
\end{equation}
where $\bm{z}^n \in \R^{d \times d}$. We unroll this recurrence to express each state $\bm{z}^n$ in terms of $\bm{u}$ with coefficients given by $\gamma$ and the weight matrices $\tbf{W}_A, \tbf{W}_C, \tbf{W}_V$ as follows:
\begin{align*}
    \bm{z}^{0} 
    &\equiv \paren{\paren{\bm{u}\tbf{W}_A}[0,:]}^{\top}\paren{\bm{u}\tbf{W}_V}[0,:]\\
    &\equiv \paren{\bm{u}[0,:]\tbf{W}_A}^{\top}\paren{\bm{u}[0,:]}\tbf{W}_V\\
    &\equiv \tbf{W}_A^{\top}\paren{\bm{u}[0,:]}^{\top}\paren{\bm{u}[0,:]}\tbf{W}_V\\
    &\equiv \tbf{W}_A^{\top} \paren{\bm{u}^{\top}[:,0] \bm{u}[0,:]}\tbf{W}_V.
\end{align*}
Similarly, we have
\begin{align*}
    \bm{z}^{1} 
    &\equiv \gamma \bm{z}^{0} + \tbf{W}_A^{\top} \paren{\bm{u}^{\top}[:,1] \bm{u}[1,:]}\tbf{W}_V\\
    &\equiv \gamma \paren{\tbf{W}_A^{\top} \paren{\bm{u}^{\top}[:,1] \bm{u}[0,:]}\tbf{W}_V} + \tbf{W}_A^{\top} \paren{\bm{u}^{\top}[:,1] \bm{u}[1,:]}\tbf{W}_V \\
    &\equiv \tbf{W}_A^{\top} \paren{\gamma\paren{\bm{u}^{\top}[:,0] \bm{u}[0,:]} + \bm{u}^{\top}[:,1] \bm{u}[1,:]}\tbf{W}_V \\
\end{align*}
We can then generalize this to $n \in [N]$ to get
\begin{equation}
\label{eq: retnet-poly}
    \bm{z}^{n} \equiv \tbf{W}_A^{\top} \paren{\sum_{i = 0}^{n} \gamma^{n-i}\bm{u}^{\top}[:,i] \bm{u}[i,:]}\tbf{W}_V.
\end{equation}
The above helps us infer that $\bm{z}^n$ can be expressed as a polynomial in $\{\bm{u}[0,0], \bm{u}[0,1], \ldots, \bm{u}[N-1, d-1]\}$. For each polynomial, there exists an arithmetic circuit that computes the polynomial in a natural way, whence we can apply \cref{thm: gen-ac} to get an equivalent \Coyote\ model.
\begin{corollary}
\label{cor: retnet-coyote}
    For a {\tt RetNet} model with $\calO(d^2)$ parameters and $N$ layers, there exists an equivalent \Coyote model that uses $\calO(Nd)$ parameters and $\calO(N\log{d})$ layers.
\end{corollary}
\begin{proof}
     We will start by describing the arithmetic circuit that computes the states $\bm{z}^{n} \in \R^{d \times d}$ here. Indeed, each state requires exactly two alternative layers with only multiplication gates and only addition gates, respectively. First, we note that for each $n \in N$, we have
     \[
     \bm{z}^n \equiv \gamma\bm{z}^{n-1} + \paren{\tbf{A}[n,:]}^{\top}\tbf{V}[n,:].
     \]
     Since we can compute $\bm{z}^0$ in $\calO(1)$ layers and computing $\gamma \cdot \bm{z}^{n-1}$ needs exactly one layer with each entry from $\bm{z}^{n-1}$ serving as an input to a multiplication gate along with $\gamma$, the depth of the circuit is $\calO(N)$ with width $d^2$. We can thus apply \cref{thm: gen-ac} to get an $(\calO(N), \calO(N\log{d}), d^2, \calO(d^2), \calO(d^2))-$\Coyote\ model that can compute $\bm{z}^n$ for each $n \in N$.
\end{proof}

\paragraph{The Space Complexity of AR}
In the context of the Retnet Model and associative recall, it is worth exploring the space complexity of the associative recall (AR) problem:
\begin{displayquote}
     The AR problem takes key-value pairs $\{\bm{k}_i, \bm{v}_i\}_{i = 0}^{n-1}$ along with a query $\bm{q}$ appended at the end as input and the goal is to output $\bm{v}_i$ if $\bm{q} = \bm{k}_i$ for some $i \in [0, N-1]$.
\end{displayquote} 
Indeed, we will first provide a lower-bound for {\em any} model that purports to solve AR. To this end, we will require a randomized communication complexity lower bound result for the {\em index problem}:
\begin{displayquote}
    The index problem has two agents, Alice and Bob, where Alice has a string $\bm{x} \in \{0,1\}^n$ and Bob has an index $i \in [n]$, and the goal for the players is to output the $i$-th entry $\bm{x}_i$. Moreover, we also require the communication to be {\em one-way}: only Alice is allowed to send a single message to Bob and Bob needs to output the answer.
\end{displayquote} 
We will make use of the following lower-bound result.
\begin{theorem}[\citep{jayram2008one}]
\label{thm: space-index}
    The one-way randomized communication complexity\footnote{The randomized communication complexity of function $f$ is defined as $\min_{\pi} \norm{\pi}$, where $\pi$ ranges over all randomized protocols that can solve $f$ with probability of success at least $2/3$.} of the index problem for sending an $n$-length bit string is $\Omega(n)$.
\end{theorem}
We now use \cref{thm: space-index} to provide a lower bound on the number of bits required by the Retnet model to solve AR.
\begin{corollary}
\label{cor: space-ar}
    The RetNet model requires $\Omega(N)$-bits to solve AR for $d \le \sqrt{N}$.
\end{corollary}
\begin{proof}
    Consider an instance $(\bm{x}, i)$ of the index problem with $\bm{x} \in \{0,1\}^N$. We now describe the corresponding instance of the AR problem:
    \begin{equation}
        \label{eq: ar-index}
        \{i, \bm{x}_i\}_{i = 0}^{N-1}, i.
    \end{equation}
    Next, consider the following one-way protocol for solving the index problem using the RetNet model. Alice with their access of $\bm{x} \in \{0,1\}^N$ generate an input for AR (without the query) as in \eqref{eq: ar-index}. Alice then runs the RetNet model on $\{i, \bm{x}_i\}_{i = 0}^{N-1}$ and sends the memory content of running the RetNet model to Bob. This should include the output $\bm{z}^{N-1}$ as we can reasonably assume that both have access weight matrices $\tbf{W}_A, \tbf{W}_C, \tbf{W}_V$ and the scalar $\gamma$. Since we assume that this model solves AR, the output ${\tt Out}[N, :] = \bm{x}_i$ should contain the associated value of $i$. Here, Bob can compute ${\tt Out}[N, :]$ by using the memory content sent by Alice along with the term $\gamma \bm{z}^{N-1}$:
    \[
    \bm{x}_i = {\tt Out}[N, :] = \tbf{C}[N, :]\bm{z}^N = \tbf{C}[N, :]\paren{\gamma \bm{z}^{N-1} +  (\tbf{A}[N,:])^{\top}\tbf{V}[N, :]}.
    \]
    That is, the total number of bits that are communicated in this protocol is $O(d^2)$. For $d \le \sqrt{N}$, have shown that a one-way communication protocol exists for solving the index problem exists that uses $o(N)$ communication complexity. This contradicts \cref{thm: space-index} and hence, we conclude that the ReNet model solving AR also needs $\Omega(N)$ bits.
\end{proof}

%% file: Sections/appendix/general_ar/setup_iclr.tex
\subsubsection{Introduction}\label{sec: intro-general-ar}
In this section, we consider a general version of the associative recall problem~\citep{ba2016using}.
\paragraph{Setup.}
Next, we will redefine the {\em multiple-query associative recall} problem ($\Task$) from \cref{def: general-AR} to a slightly more general problem: 
\begin{displayquote}
    Suppose we are given an input sequence 
$
\bm{u}[0 \cdots N-1] \triangleq \left\{\paren{\bm{k}_0, \bm{v}_0, \bm{q}_0}, \ldots, \paren{\bm{k}_{\frac{N}{3}-1}, \bm{v}_{\frac{N}{3}-1}, \bm{q}_{\frac{N}{3}-1}}\right\}
$
with each $\bm{k}_i, \bm{v}_i, \bm{q}_i \in C$ is a token drawn from a vocabulary of size $c = |C|$.
Our goal is then to check, for each $1 \le i \le \frac{N}{3}-1$, whether there exists $0 \le j < i$ such that $\bm{q}_i \equiv \bm{k}_j$, and if so, output $\bm{v}_{j}$. 
\end{displayquote}
Here, we note that it suffices to have $d \approx \log(c)$ so that $\bm{k}_i, \bm{v}_i, \bm{q}_i$ is embedded in $\{0,1\}^{d}$. However, we will specify the specific embedding being used for the results below. 
 Here, we construe the tokens $\bm{k}_i, \bm{q}_i$ and $\bm{v}_i$ to be the {\em keys}, the {\em queries}, and the {\em associated values}. Indeed, it might be helpful to think of the input $\bm{u}$ as a streaming sequence of key-value pairs for which we sequentially employ standard associative recall for every key that shows up in the sequence so far. 
 
 To see that the above generalizes \cref{def: general-AR}, considers a sequence of length $\frac N3$: $\bm{u}[0 \cdots N-1]:= \{\bm{x}_0, \ldots, \bm{x}_{N-1}\}$, where each $\bm{x}_i\in C$. The goal of \cref{def: general-AR} is then to check, for each $1 \le i < N-1$, whether there exists $0 \le j < i$ such that $\bm{x}_i \equiv \bm{x}_j$, and if so, output $\bm{x}_{j+1}$, and continue otherwise. We can reduce this problem to the above general formulation by taking the following sequence of tuples as the input $\{(\bm{x}_i, \bm{x}_{i+1}, \bm{x}_i)\}$. 

\begin{remark}
    As noted above, this version is more general than \cref{def: general-AR}. Thus, the results proven in the sequel, which are proven for the above general \Task\, can be ported (with constant blowup in parameters) to get the results corresponding results for \cref{def: general-AR}.
\end{remark}

\subsubsection{${\Task}$ Solution via Attention}
Before describing how \Coyote{} solves the multiple-query associative recall problem, we discuss how Attention solves it trivially using pairwise inner-products. To this end, we will specify how the input is presented to attention.
\begin{remark}\label{rem: attn-input}
We note that the input for the multiple-query associative recall problem $\bm{u} \in \{0,1\}^{N\times d}$ has designated indices for the keys, queries, and values in the sequence. We gather these indices below:
    \begin{equation}
    \label{eq: kqv-indices}
    \begin{aligned}
    \calK &= \{i \in \{0, \ldots, N-1\} \lvert\ i \equiv 0 \mod{3}\},\\
    \calV &= \{i \in \{0, \ldots, N-1\} \lvert\ i \equiv 1 \mod{3}\},\\
    \calQ &= \{i \in \{0, \ldots, N-1\} \lvert\ i \equiv 2 \mod{3}\},.
    \end{aligned}
    \end{equation}
    The input $\bm{u} \in \R^{N \times d}$ to Attention for $d =3c$ is then given by
    \[
    \bm{u}[i,:] \equiv \begin{cases}
      [\bm{k}_i:\bm{0}^c:\bm{0}^c] &\text{ if } i \in \calK \\
      [\bm{0}^c:\bm{v}_i:\bm{0}^c] &\text{ if }  i \in \calV\\
  [\bm{0}^c:\bm{0}^c:\bm{q}_i] &\text{ if }  i \in \calQ
    \end{cases}
    \]
    Here, each $\bm{k}_i, \bm{v}_i, \bm{q}_i$ is embedded as a one-hot encoding in $\{0,1\}^{c}$. 
\end{remark}
Without softmax, the output for an attention layer $\tbf{O} \in \R^{N \times d}$ is given by 
\begin{equation}
    \label{eq: attention}
        \tbf{O} \equiv \paren{\tbf{Q}\tbf{K}^{\top}}\tbf{V},
\end{equation}
where $\tbf{Q},\tbf{K},\tbf{V} \in \R^{N \times d}$ 
are defined as $\bm{u}\tbf{W}_Q, \bm{u}\tbf{W}_K, \bm{u}\tbf{W}_V$ for $ \bm{u} \in \R^{N \times d}$. 
Instead of position embeddings, we use ALiBi, a popular technique that biases the attention scores $\tbf{Q}\tbf{K}^{\top}$ with a lower-triangular Toeplitz matrix $\tbf{B} \in \mathbb{R}^{N \times N}$~\citep{press2021train}. The values in this matrix are controlled by a fixed hyperparameter so they do not count towards the number of parameters in the model. 
\begin{algorithm}[H]
		\caption{$\texttt{ALiBi-without-softmax}\paren{\hyenaInput[0\cdots N-1], \tbf{O}_{\mathrm{prev}}[0\cdots N-1], \tbf{B})}$}
		\begin{algorithmic}[1]\label{algo: seq-gen-ar}
            \Require Input sequence $\hyenaInput[0 \cdots N-1] \triangleq \{\paren{\bm{k}_i, \bm{v}_i, \bm{q}_i}\}_{i=0}^{\frac{N}{3}-1}
            $ with each $\bm{k}_i, \bm{v}_i, \bm{q}_i \in \{0,1\}^{3c}$, previous layer's output $\tbf{O}_{\mathrm{prev}}[0\cdots N-1] \in \R^{N \times 3c}$, and linear bias $\tbf{B} \in \R^{N \times N}$.
		    \State Add $\bm{u}_{\mathrm{curr}} \gets \hyenaInput + \tbf{O}_{\mathrm{prev}}$ as an input to this layer.
            \State $\tbf{K}, \tbf{Q}, \tbf{V}\gets \bm{u}_{\mathrm{curr}}\tbf{W}_Q, \bm{u}_{\mathrm{curr}}\tbf{W}_K, \bm{u}_{\mathrm{curr}}\tbf{W}_V$.
            \State $\tbf{O} \gets \paren{\tbf{Q}\tbf{K}^{\top} + \tbf{B}}\tbf{V}$
            \State \Return $\tbf{O}$ as the output of this layer.
		\end{algorithmic}
\end{algorithm}
\begin{proposition}
\label{prop: app-attention}
    Given an input $\hyenaInput\in \{0,1\}^{N \times d}$ (encoded as in \cref{rem: attn-input}) where $d = 3c$, 
    Attention with linear biases (even without using soft-max) solves {$\Task$} for $\hyenaInput$ using {${\calO}(c^2)$} parameters, ${\calO}(Nc^2 + N^2c)$ time complexity and ${\calO}(1)$ layers. 
\end{proposition}

\begin{proof}
    We use two layers of attention.
    We will start by specifying the projection matrices for the first layer $\tbf{W}_Q^1,\tbf{W}^1_K,\tbf{W}^1_V \in \R^{d \times d}$ as:
    \[
    \tbf{W}^1_K \equiv \tbf{W}^1_Q \equiv \tbf{0},\quad
    \tbf{W}_V^1 \equiv \begin{pmatrix}
        \tbf{0} & \tbf{0} & \tbf{0} \\
        \tbf{0} & \tbf{I_{c \times c}} & \tbf{0} \\
        \tbf{0} & \tbf{0} & \tbf{0} 
    \end{pmatrix}
    \]
         
    Above, $\tbf{W_V^1}$ is meant to isolate the $\bm{v}_i$ embeddings. For the first layer, we then have $\tbf{Q}^1 \equiv
    \tbf{K}^1 \equiv \tbf{0}$, and  
    \begin{equation*}
    \label{eq: kqv-projections}
    \tbf{V}^1[i,:] := \bm{u}^1[i,:] \tbf{W}_V^1 \equiv 
    \begin{cases}
    [\bm{0}^{c}:\bm{v}_i:\bm{0}^c] &\text{if }i \in \calV,\\
    \tbf{0}^d &\text{otherwise}
    \end{cases},
    \end{equation*}
    where $\calK, \calQ, \calV$ are defined as in \eqref{eq: kqv-indices}. The output for the first layer is given by the following: 
     \begin{align*}
        \tbf{O}^1[i,:] 
        &= \paren{\paren{\tbf{Q}\tbf{K}^{\top}+\tbf{B}^1}\tbf{V}}[i,:] \\
        &= \tbf{B}^1\tbf{V}[i,:] \\
        &=
        \begin{cases}
             [\bm{0}^{c}:\bm{v}_i:\bm{0}^c] &\text{if }i \in \calK\\
            \bm{0}^d &\text{otherwise}
        \end{cases},
    \end{align*}
    The $\tbf{QK^T}$ is ignored ($ \equiv \tbf{0}$) and we isolate the shifted sequence by setting the bias matrix $\tbf{B}$ appropriately. In particular, the last equality follows from the fact that $\tbf{B}^1$ is an up-shift matrix that shifts each row of $\tbf{V}$ by 1.
    We apply the residual connection at the end of the standard Transformer block to insert the $\bm{k}_i$ adjacent to the $\bm{v}_i$. For the second layer, the input $\bm{u}^2 \in \R^{N \times d}$ is given by
\[
\bm{u}^2[i,:] \equiv \begin{cases}
            [\bm{k}_i: \bm{v}_i:\bm{0}^{c}] &\text{if }i \in \calK\\
            [\bm{0}^{c}: \bm{v}_i:\bm{0}^{c}] &\text{if }i \in \calV\\
            [\bm{0}^{c}: \bm{0}^{c}:\bm{q}_i] &\text{if }i \in \calQ
        \end{cases},
\]
Further, we take the following projection matrices:
\[
    \tbf{W}^2_K \equiv\begin{pmatrix}
        \tbf{I}_{c \times c} & \tbf{0} & \tbf{0} \\
        \tbf{0} & \tbf{0} & \tbf{0} \\
        \tbf{0} & \tbf{0} & \tbf{0} 
    \end{pmatrix}, 
    \tbf{W}^2_Q \equiv \begin{pmatrix}
        \tbf{0} & \tbf{0} & \tbf{0} \\
        \tbf{0} & \tbf{0} & \tbf{0} \\
        \tbf{I}_{c \times c} & \tbf{0} & \tbf{0} 
    \end{pmatrix},
    \tbf{W}^2_V \equiv \begin{pmatrix}
        \tbf{0} & \tbf{0} & \tbf{0} \\
        \tbf{I}_{c \times c} & \tbf{0} & \tbf{0} \\
        \tbf{0} & \tbf{0} & \tbf{0} 
    \end{pmatrix}
\]
We then have the following matrices as input to attention after applying projection:
\begin{equation}
    \begin{aligned}
    \tbf{Q}^2[i,:] &:= \bm{u}^2[i,:]\tbf{W}^2_K \equiv \begin{cases}
    [\bm{q}_i:\bm{0}^{2c}] &\text{if }i \in \calQ,\\
    \tbf{0}^d &\text{otherwise}
    \end{cases},\\
    \tbf{K}^2[i,:] &:=\bm{u}^2[i,:]\tbf{W}^2_Q \equiv  \begin{cases}
    [\bm{k}_i:\bm{0}^{2c}] &\text{if }i \in \calK,\\
    \tbf{0}^d &\text{otherwise}
    \end{cases},\\
    \tbf{V}^2[i,:] &:= \bm{u}^2[i,:]\tbf{W}^2_V \equiv 
    \begin{cases}
    [\bm{v}_i:\bm{0}^{2c}] &\text{if }i \in \calK \cup \calV,\\
    \tbf{0}^d &\text{otherwise}
    \end{cases}.
    \end{aligned}
    \label{eq: shifted-value}
    \end{equation}
    Here, note that the values have been shifted to the corresponding key position.
     Next, we compute the term in the parenthesis for the second layer as
        \begin{align*}
            \paren{\tbf{Q}^2\tbf{K}^{2\top}}[i,j] = \paren{\tbf{Q}^2\tbf{K}^{2\top}}[i,j]
            &= \angles{\tbf{Q}^2[i,:], \tbf{K}^{2\top}[:,j]} \\
            &= \angles{\tbf{Q}^2[i,:], \tbf{K}^2[j,:]} \\
            &= \begin{cases}
                \angles{\bm{q}_i, \bm{k}_j} &\text{if }i \in \calQ, j \in \calK \\
                0 &\text{otherwise}
            \end{cases}\\
            &= \begin{cases}
                1 &\text{if }i \in \calQ, j \in \calK, \bm{q}_i \equiv \bm{k}_j \equiv \bm{e}_k\text{ for some $k$}\\
                0 &\text{otherwise}.
            \end{cases}
        \end{align*}
        Finally, we compute the output as follows:
         \begin{align*}
            \tbf{O}^2[i,:] 
            &= \paren{\paren{\tbf{Q}^2\tbf{K}^{2\top}+{\tbf{B}^2}}\tbf{V}^2}[i,:] \\
            &= \paren{\paren{\tbf{Q}^2\tbf{K}^{2\top}+{\tbf{0}}}\tbf{V}^2}[i,:] \\
            &= \paren{\tbf{Q}^2\tbf{K}^{2\top}}[i,:]\cdot\tbf{V}^2  \\
            &= \sum_{j = 0}^{N-1} \paren{\tbf{Q}^2\tbf{K}^{2\top}}[i,j]\cdot \tbf{V}^2[j,:]  \\
            &= \sum_{j \in \calK} \paren{\tbf{Q}^2\tbf{K}^{2\top}}[i,j]\cdot [\bm{v}_j:\bm{0}^{2c}] \\
            &=
            \begin{cases}
                [\bm{v}_j:\bm{0}^{2c}] &\text{if }j \in \calK, i \in \calQ,  \bm{q}_i \equiv \bm{k}_j \\
                \bm{0}^d &\text{otherwise}
            \end{cases},
        \end{align*}
    where we use the fact that for each index $j \in \calK$, the matrix $\tbf{V}^2$ contains the associated value from \eqref{eq: shifted-value}.    Thus, for each query $\bm{q}_i$, we solve the associated value problem yielding a match for the $j$th key.

    In total, we only need $\calO(c^2)$-many parameters to perform these multiplications and the linear bias in the first layer is a hyperparameter that is static and unlearned; the time complexity comes from the multiplication of $\tbf{QK^{\top}}$ in ${\calO}(N^2c)$, and projections in ${\calO}(Nc^2)$. Finally, we only need $\calO(1)$ layers for this solution.
\end{proof}

In the sequel, we develop a parallel algorithm to solve the multiple-query associative recall problem with  $\calO(Nd\cdot\log^2{N})$ work complexity and $\calO(d \cdot \log^2{N})$ time. We then convert the algorithm into a \Coyote{} model via the route of arithmetic circuits, which then solves the multiple-query associative recall problem with $\tilde{\calO}(1)$ layers and $\tilde{\calO}(Nd)$ parameters.

%% file: Sections/appendix/general_ar/seq-ar_iclr.tex
\subsubsection{Initial Attempt: A Sequential Algorithm}
We will first discuss the algorithm that simply uses an associative array to solve the multiple-query associative recall problem. Specifically, we want to use a data structure that allows for logarithmic insertion and membership query. Here, we do not specify a choice but data structures including self-balancing binary search trees which allow for $\calO(\log{N}\cdot d)$ ${\tt insert}$ and ${\tt find}$ operations for $d$-bit entries should be sufficient.

\begin{algorithm}[H]
		\caption{$\texttt{Sequential-MQ-AR}\paren{\hyenaInput}[0\cdots N-1]$}
		\begin{algorithmic}[1]\label{algo: seq-gen-ar}
            \Require Input sequence $\hyenaInput[0 \cdots N-1] \triangleq \{\paren{\bm{k}_i, \bm{v}_i, \bm{q}_i}\}_{i=0}^{\frac{N}{3}-1}
            $ with each $\bm{k}_i, \bm{v}_i, \bm{q}_i \in \{0,1\}^{\headDim}$.
		    \State Initialize an associative array with ${\tt insert}$ and ${\tt find}$ and an output array $\texttt{out} \gets [].$
            \For{$i \in \{0, \ldots, \frac{N}{3}-1$\}}
                \State $(\bm{k}_j, \bm{v}_j) \gets {\tt find}(\bm{q}_i)$ \Comment{Query for $\bm{q}_i$ in the data structure.}
                \If{$\bm{k}_j$ is not {\tt null}}
                    \State Add $\bm{v}_{j}$ to $\texttt{out}$.
                \EndIf
                \State ${\tt insert}(\bm{k}_i, \bm{v}_{i})$ \Comment{Add the key-value pair to the data structure.}
            \EndFor
            \State \Return $\texttt{out}$.
		\end{algorithmic}
\end{algorithm}
\begin{proposition}
\label{prop: seq-gen-ar}
    \cref{algo: seq-gen-ar} solves the multiple-query associative recall problem $\paren{\Task}$ in $\calO(dN\log{N})$ time for an input sequence $\bm{u} \in \{0,1\}^{\inputLength \times \headDim}$.
\end{proposition}
\begin{proof}
    For any $i \in \{0,\ldots, \frac{N}{3}-1\}$, we know that both insertion and lookup operations take $\calO(\log(i)\cdot d)$ time. Overall, the runtime of the algorithm is 
    \[\sum_{i=0}^{\frac{N}{3}-1} \calO(\log(i)\cdot d) = \calO(\log(N!)\cdot d) = \calO(N \log{N}\cdot d).\]
\end{proof}

%% file: Sections/appendix/general_ar/pbs_iclr.tex
\subsubsection{Algorithm via Parallel Binary Search}
\label{sec: pram-algo}
Our plan is to convert the algorithm for solving the multiple-query associative recall problem in the RAM model into an arithmetic circuit, which by \cref{thm: gen-ac} will lead to a \Coyote{} model that solves the multiple-query associative recall problem. With respect to \cref{algo: seq-gen-ar}, it may be the case that the arithmetic circuit has a large number of layers $\Omega(N)$. Unfortunately, this would imply that the resulting \Coyote{} model may have near quadratic complexity in $N$. Instead, we now initiate our effort into designing a \Coyote{} model  with both small enough number of parameters and number of layers. Here, we will first subdivide the problem using dyadic intervals into $\calO(N)$ subproblems and reduce each such subproblem into a {\em multiple search problem}~\citep{akl1990parallel}. To this end, we briefly introduce the multiple search problem below.
\begin{displayquote}
    Given two array of numbers $A \triangleq a_0\le \ldots \le a_{n-1}$ and $B \triangleq (b_0\le \ldots \le b_{m-1})$ with $n \le m$, for each $a_j \in A$, the goal is to find the smallest element in $B$ that is larger than or equal to $a_j$.
\end{displayquote} 
The multiple search problem is solved by a {\em parallel binary search} ({\tt pbs}) algorithm in \citep{akl1990parallel} with work complexity $\calO(n \cdot \log{m})$ and time $\calO(\log{n}\log{m})$. Specifically, for sorted arrays $A[0 \cdots n-1]$ and $B[0 \cdots m-1]$, ${\tt pbs}$ constructs the array $C[0\cdots n-1]$ defined as
\begin{equation}
        C[i] \triangleq 
        \begin{cases}
            \min_{0 \le j < m} \{j\lvert\ A[i] \le B[j]\} &\text{if } A[i] \le B[m-1]\\
            m &\text{otherwise.}
        \end{cases} 
        \label{eq: pbs-output}
\end{equation}
The algorithm itself runs in exclusive-read exclusive-write (EREW) PRAM model—no two processors are allowed to read from or write into the same memory location at the same time. 

We now augment the algorithm copied from \citep{akl1990parallel} for our purposes below.
\begin{algorithm}[H]
		\caption{$\texttt{pbs-key-values}\paren{\bm{q}[s \cdots t],\ \bm{k}[x \cdots y], n, m}$}
		\begin{algorithmic}[1]\label{algo: pbs-key-values}
            \Require sorted arrays $\bm{q}[s \cdots t]:= \{\bm{q}_i\}_{i = s}^t,\ \bm{k}[x\cdots y]:= \{(j, \bm{k}_j)\}_{j = x}^y$.
            \State Initialize $n$ processors denoted $P_0, P_1, \ldots, P_{n-1}$  
             \Comment{$\{$Sequential steps are assumed to be executed by $P_s$.$\}$}
            \State Initialize the output array $C := [m]_{i=s}^t$.
            \If{$s \le t$} 
                \State $\mathrm{mid} \gets \floor{(s+t)/2}$
                \If{$\bm{q}[\mathrm{mid}] \le \bm{k}[x][1]$}\label{step: comp-1}
                    \For{$i := s$ to $\mathrm{mid}$ in parallel}
                        \State $C[i] \gets j$ \Comment{Step executed in parallel by $P_i$}
                    \EndFor
                    \State $\texttt{pbs-key-values}\paren{\bm{q}[\mathrm{mid}+1 \cdots t],\ \bm{k}[x \cdots y]}$ \label{step: rec-1}
                \Else
                    \If{$\bm{q}[\mathrm{mid}] > \bm{k}[y][1]$}\label{step: comp-2}
                        \For{$i := \mathrm{mid}$ to $t$ in parallel}
                            \State $C[i] \gets {y+1}$ \Comment{Step executed in parallel by $P_i$}
                        \EndFor
                    \State $\texttt{pbs-key-values}\paren{\bm{q}[s \cdots \mathrm{mid}-1],\ \bm{k}[x \cdots y]}$\label{step: rec-2}
                    \Else
                        \Comment{C[mid] is determined using sequential binary search}
                        \State $z \gets \min_{x \le j \le y} \{j\lvert\ \bm{q}[\mathrm{mid}] \le \bm{k}[j][1]\}$
                        \State $C[\mathrm{mid}] \gets z$
                        \State {\bf do} steps \ref{step: left} and \ref{step: right} in parallel
                            \State \quad $\texttt{pbs-key-values}\paren{\bm{q}[s \cdots \mathrm{mid}-1],\ \bm{k}[x \cdots z-1]}$] \label{step: left}
                            \State \quad $\texttt{pbs-key-values}\paren{\bm{q}[\mathrm{mid}+1 \cdots t],\ \bm{k}[z \cdots y]}$ \label{step: right}
                    \EndIf
                \EndIf
            \EndIf
            \State \Return $C$.            
		\end{algorithmic}
\end{algorithm}

Let $\Sigma$ be the set $\{0,1\}$ and denote the set of binary strings of size $n$ as $\Sigma^n$. We define $\mathrm{prefix}(\bm{x})$ for $n$-bit strings as the set of all initial substrings of $\bm{x} \in \Sigma^n$ which includes the empty string and $\bm{x}$ itself. Next, let $\mathrm{dec}: \{0,1\}^n \to \N$ be the decimal representation of an $n$-bit string $\bm{x}$ with $\bm{x}[0]$ denoting the least significant bit. We also use ${\tt sort}(A)$ as a procedure that sorts an array $A$. Finally, wlog, we assume that $N$ is a power of $2$. We are now ready to present a parallel algorithm that solves the multiple-query associative recall problem below.

\begin{algorithm}[H]
		\caption{$\texttt{Parallel-MQAR}\paren{\hyenaInput[0 \cdots N-1]}$}
		\begin{algorithmic}[1]\label{algo: pram-gen-ar} 
            \Require Input sequence $\hyenaInput[0 \cdots N-1] \triangleq \{\paren{\bm{k}_i, \bm{v}_i, \bm{q}_i}\}_{i=0}^{\frac{N}{3}-1}
            $ with each $\bm{k}_i, \bm{v}_i, \bm{q}_i \in \{0,1\}^{\headDim}$.
            \State Initialize $\frac{N}{3}\log\paren{\frac{N}{3}}$ processors denoted $P_0, \ldots, P_{\frac{N}{3}\log\paren{\frac{N}{3}}-1}$.
		    \State Initialize the index and output array $\texttt{idx, val} \gets [].$
            \For{$k := \{0, \ldots, \log\paren{\frac{N}{3}}-1\}$}
                \For{$\bm{x} := \{\bm{x} \in \Sigma^{\log\paren{\frac{N}{3}}-k}\lvert\ \bm{x}[\log\paren{\frac{N}{3}}-k-1] = 0\}$} 
                    \State {\em \{All the steps below are executed in parallel by $\{\{\{P_{i}^{\bm{x}, k}\}_{\bm{x}}\}_{i \in [0, 2k-1]}\}_{k}$\}}
                    \State $I_k^{\bm{x}}\gets \{\bm{y} \in \Sigma^{\log\paren{\frac{N}{3}}} \lvert\ \bm{x} \in \mathrm{prefix}(\bm{y})\}.$\label{step: i-def} 
                    \State $\bm{k}_{\mathrm{sorted}}^{k\bm{x}}, I_{\mathrm{permuted}}^{k\bm{x}} \gets {\tt sort}\paren{\{\bm{k}_{\mathrm{dec}(i)}\}_{i \in I^{\bm{x}}_k}}$ 
                    \label{step: k-sort}\Comment{$I_{\mathrm{permuted}}^{k\bm{x}} := \{(j, \mathrm{dec}(i))\lvert\  \bm{k}_{\mathrm{sorted}}^{k\bm{x}}[j] \equiv \bm{k}_{\mathrm{dec}(i)}\}$}
                    \State $\bm{x}[\log\paren{\frac{N}{3}}-k-1] \gets 1$
                    \State $J_k^{\bm{x}} \gets \{\bm{y} \in \Sigma^{\log\paren{\frac{N}{3}}} \lvert\ \bm{x} \in \mathrm{prefix}(\bm{y})\}.$\label{step: j-def}   
                    \State $\bm{q}_{\mathrm{sorted}}^{k\bm{x}}, J_{\mathrm{permuted}}^{k\bm{x}} \gets {\tt sort}\paren{\{\bm{q}_{\mathrm{dec}(j)}\}_{j \in J^{\bm{x}}_k}}$  \Comment{$J_{\mathrm{permuted}}^{k\bm{x}} := \{(\mathrm{dec}(j), k)\lvert\  \bm{q}_{\mathrm{sorted}}^{k\bm{x}}[k] \equiv \bm{q}_{\mathrm{dec}(j)}\}$ }
                    \label{step: q-sort}
                    \State $C_k \gets \texttt{pbs-key-values}\paren{\bm{q}_{\mathrm{sorted}}^{k\bm{x}},\ \bm{k}_{\mathrm{sorted}}^{k\bm{x}}, 2^{k}, 2^{k} } $ \label{step: pbs-call}
                    \For{$j \in  J^{\bm{x}}_k$}
                        \If{$C_k[\mathrm{dec}(j)] \neq 2^k$}
                        \State $\bm{c}^{k\bm{x}}_j \gets C_k[J_{\mathrm{permuted}}^{k\bm{x}} (\mathrm{dec}(j))]$
                            \If{$\bm{c}^{k\bm{x}}_j \neq 2^{k}$} \Comment{cf. \eqref{eq: pbs-output}}
                                \State Add $I_{\mathrm{permuted}}^{k\bm{x}} (\bm{c}^{k\bm{x}}_j)$ to $\texttt{idx}[\mathrm{dec}(j)]$.
                                \label{step: check-m} 
                            \EndIf
                        \EndIf
                    \EndFor
                \EndFor
            \EndFor
            \For {$i \in \{1, \ldots, \frac{N}{3}-1\}$} \label{step: iter-N}
                \State \{{\it Executed in parallel by $P_i$.}\} \label{step: iter-N}
                \If{$\exists\ j \in {\tt idx}[i]$}
                    \State Add $\bm{v}_{j+1}$ to $\texttt{val}$
                    \label{step: add-val}
                \EndIf
            \EndFor
            \State \Return $\texttt{val}$.
		\end{algorithmic}
\end{algorithm}
\begin{remark}
In lines \ref{step: k-sort} and \ref{step: q-sort}, we keep track of the sorted permutation of the indices of keys and queries, respectively. This helps us in the retrieval of the index of the matching key as in line \ref{step: check-m}.
\end{remark}

\subsubsection{Correctness and Complexity}
\begin{proposition}
\label{prop: pram-gen-ar}
    \cref{algo: pram-gen-ar} solves the multiple-query associative recall problem with work complexity $\calO(Nd \cdot \log^2{N})$ and time $\calO(d \cdot \log^2{N})$.
\end{proposition}
\begin{proof}
    The correctness of {\tt pbs} implies the correctness of \cref{algo: pram-gen-ar} if we can show that, for each $1 \le i < \frac{N}{3}$, we check for a match among the keys $\{\bm{k}_j\}_{j \in [i-1]}$. To this end, for each $1 \le i < \frac{N}{3}$, let the set of all iterator indices associated with an index $i$ be defined as  $K_i \triangleq \{(k,x)\lvert i \in J_k^{\bm{x}}\}$ with $J_k^{\bm{x}}$ as noted in line \ref{step: j-def}. Then, we define the corresponding set for keys as
    \(
    \calI_i \triangleq \bigcup_{(k,\bm{x}) \in K_i} I_k^{\bm{x}}\text{ with $I_k^{\bm{x}}$s defined as in line \ref{step: i-def}}.
    \) 
    That is, for all the calls to {\tt pbs-key-values} that $i$ is part of (given by $K_i$) where the algorithm checks for a match among the keys in $\calI_i$, it then suffices to show that $\calI_i = [i-1]$. 
    
    Here, first note that if some index $j \in I_k^{\bm{x}}\subseteq \calI_i$ for some $\bm{x} \in \Sigma^{\log\paren{\frac{N}{3}}-k}$, then, by definition, $\bm{x} \in \mathrm{prefix}(\mathrm{bin}(j))$. Here, let $\bm{x}^1 := \bm{x}\lvert_{\bm{x}[\log\paren{\frac{N}{3}}-k] = 1}$ where we set the $(\log\paren{\frac{N}{3}}-k)$-th index of $\bm{x}$ to be 1. Consequently, as we have $i \in J^{\bm{x}}_k$ for the same $k$ and $\bm{x}$ as in $I_k^{\bm{x}}$ ({\it cf.} line \ref{step: i-def}), we must have $\bm{x}^1 \in \mathrm{prefix}(\mathrm{bin(i)})$. Thus we get $j < i$, whence we can claim that $\calI_i \subseteq [i-1]$. 
    
    For the other direction, for any $i$, let $b$ denote the position of the most significant bit in $\mathrm{bin}(i)$ which differs from $\mathrm{bin}(j)$ for any $j \in [i-1]$. Then, there must exist a binary string that is in the prefix set of  both $\mathrm{bin}(i)$ and $\mathrm{bin}(i)$. That is, there exists $\bm{x} \in \mathrm{prefix}(\mathrm{bin}(i)) \cap \mathrm{prefix}(\mathrm{bin}(j))$ with $\bm{x} \in \Sigma^{b}$. Thus, we then must have $\mathrm{bin}(j) \in I^{\bm{x}}_{\log\paren{\frac{N}{3}}-b}$ and $\mathrm{bin}(i) \in J^{\bm{x}}_{\log\paren{\frac{N}{3}}-b}$ with $\bm{x}$ as the corresponding witness. Hence, we have $[i-1] \subseteq \calI_i$. 
    
    Overall, we have shown that $\calI_i = [i-1]$. Since this holds for all $1 \le i < \frac{N}{3}$, we can conclude that \cref{algo: pram-gen-ar} solves the multiple-query associative recall problem. 
    
    Next, it is easy to see that we execute lines \ref{step: i-def} to \ref{step: check-m} $\sum_{k=0}^{\log\paren{\frac{N}{3}}-1} \frac{\frac{N}{3}}{2^{k+1}}$-many times. 
    We note that sorting $n$ values each of size $d$ can be done with work complexity $n\log{n}\cdot d$. We note that, at each instance, we are sorting sorting $2^k$ values. Meanwhile, remembering $2^k$ sorted permutation of indices can be done in linear time using arrays.
    Moreover, each call to {\tt pbs-key-values} has $n = m = 2^k$ which has work complexity $n \log{m}$. Finally, we know that the work complexity of lines \ref{step: iter-N} to \ref{step: add-val} is $\calO(N)$. Thus, the overall work complexity of \cref{algo: pram-gen-ar} is
    \begin{equation}
    \label{eq: work-complexity}
        d \cdot \sum_{k=0}^{\log\paren{\frac{N}{3}}-1} \frac{\frac{N}{3}}{2^{k+1}} \calO(2^k\cdot \log{2^k}) = d \cdot \calO(N) \cdot \sum_{k=0}^{\log{N/3}-1} \calO(k) = \calO(Nd \cdot \log^2{N}).
    \end{equation}

    We will now analyze the depth of \cref{algo: pram-gen-ar}. We know that the depth of computation for \cref{algo: pbs-key-values} is $\calO(\log{n}\log{m})$ for input sizes. Moreover, we have $\calO(1)$ depth for the computation in \ref{step: iter-N} to \ref{step: add-val} as each entry in ${\tt idx}$ can have at most one entry. Since the nested for loops iterating over $k$s and the associated $\bm{x}$s runs in parallel, the depth of \cref{algo: pram-gen-ar} is dominated by the largest depth among all calls to 
    \texttt{pbs-key-values} and to {\tt sort}. The largest such call to \texttt{pbs-key-values} is of size $n = m = 2^{\log\paren{\frac{N}{3}}-1} = N/6$ which yields a depth of $d \cdot \log^2{\frac{N}{3}}$. Moreover, using sorting networks~(\cref{def: sorting-networks}), we know that the largest depth is for sorting $\frac{N}{6}$ values of size $d$ given by $d \cdot \Theta(\log{N})$~(\cref{lem: sorting-networks}). Thus, we can conclude that \cref{algo: pram-gen-ar} takes $\calO(d \cdot \log^2{N})$,
    time where $N$ is the length of the input.
\end{proof}

%% file: Sections/appendix/general_ar/reduction-to-ac_iclr.tex
\subsubsection{Conversion to Arithmetic Circuit}
We will convert \cref{algo: pram-gen-ar} to an arithmetic circuits modularly. In particular, after writing out an explicit circuit for \cref{algo: pbs-key-values}, we will uses this circuit as a black-box along with circuits for sorting networks.
\paragraph{Circuit for \texttt{pbs-key-values}.}
We will denote the corresponding arithmetic circuit for \cref{algo: pbs-key-values}  as \texttt{pbs-key-values} as well with the input gates comprising of each entry from $\bm{q}[s \cdots t]$ and $\bm{k}[x \cdots y]$ and the $i$-th output gate yielding the value $C[i]$ as in \cref{algo: pbs-key-values}.

Here, we first convert the comparisons for the {\tt if} statements in \cref{algo: pbs-key-values}. To this end, we briefly introduce comparators.
\begin{definition}[Comparators]
    \label{def: comparators}
    A {\em comparator} is a device with inputs $\bm{x}$ and $\bm{y}$ and outputs $\bm{x}'$ and $\bm{y}'$, that performs the following function:
    \[
    \begin{aligned}
       \bm{x}'  &=  \min(\bm{x}, \bm{y}),\\
       \bm{y}'  &=  \max(\bm{x}, \bm{y}).
    \end{aligned}
    \]
\end{definition}

Using comparators, we can use bit-wise {\tt XOR} and {\tt AND} to  define the result of the comparisons in lines \ref{step: comp-1} and \ref{step: comp-2} as the following fixed variables: 
\begin{equation}
    \label{eq: comparison-gates}
    \begin{aligned}
    \ell_x &:= \mathbbm{1}\{\bm{q}[\mathrm{mid}] \leq \bm{k}_{x}\},\\
    g_{y} &:= \mathbbm{1}\{\bm{q}[\mathrm{mid}] > \bm{k}_{y}\},\\
    z &:= \texttt{bin-search}({\{\bm{k}_{i}\}_{i=s}^{t} ,\ \bm{q}[\mathrm{mid}]}).
    \end{aligned}
\end{equation}
This then allows us to infer the index of the key array for each recursive call to $\texttt{pbs-key-values}$ in lines \ref{step: rec-1}, \ref{step: rec-2}, \ref{step: left}, and \ref{step: right} from \cref{algo: pbs-key-values}. Specifically, let $z_s$ and $z_t - 1$ denote the {\em starting and ending indices} for the keys as inputs to the recursive calls in \cref{algo: pbs-key-values}) below:
\begin{align}
    &\texttt{pbs-key-values}(\displaystyle \bm{q}[\mathrm{mid} +1\cdots t] ,\ \bm{k}[z_{t} \cdots y]); 
    \text{ (lines \ref{step: rec-1} and \ref{step: right})}
    \label{eq: rec-1}\\
    &\texttt{pbs-key-values}(\bm{q}[ s\cdots \mathrm{mid} -1],\ \bm{k}[ x\ \cdots z_{s} -1]);
    \text{ (lines \ref{step: rec-2} and \ref{step: left})}
    \label{eq: rec-2}
\end{align}
Here, $z_t$ and $z_s$ can assume values dependent on the results of the comparisons in lines \ref{step: comp-1} and \ref{step: comp-2}. Specifically, we have
\begin{align*}
z_t &= \ell_x \cdot x + (1-\ell_x)(1-g_{y})\cdot z 
= \begin{cases}
    x & \text{if }\bm{q}[\mathrm{mid}] \leq \bm{k}_{x}\text{ (line \ref{step: rec-1})}  \\
    z & \text{if }\bm{q}[\mathrm{mid}] \in (\bm{k}_{x}, \bm{k}_y]\text{ (line \ref{step: right})}\\
    0 & \text{otherwise}
\end{cases},\\
z_s &= g_{y} \cdot (y+1) + (1-\ell_x)(1-g_{y})\cdot z 
= \begin{cases}
    y+1 & \text{if }\bm{q}[\mathrm{mid}] > \bm{k}_{y}\text{ (line \ref{step: rec-2})}\\
    z & \text{if }\bm{q}[\mathrm{mid}] \in (\bm{k}_{x}, \bm{k}_y]\text{ (line \ref{step: left})}\\
    0 & \text{otherwise}
\end{cases}.
\end{align*}
Here, $z_s$ or $z_t$ getting a value $0$ signifies that the branch is dead, and we do not execute the recursive call.

Finally, let the arrays $C_t[\mathrm{mid} +1\cdots t]$ and $C_s[ s\cdots \mathrm{mid} -1]$ denote the outputs to the recursive calls in \eqref{eq: rec-1} and \eqref{eq: rec-2}, respectively. We can then succinctly express the outputs for each index of the output array $C$ as 
\begin{equation}
    C[i] = 
    \begin{cases}
        \ell_x \cdot x + z_s\cdot (C_1[i]) & i \in [s \cdots \mathrm{mid}-1] \\
        g_{y} \cdot (y+1) + z_t \cdot (C_2[i]) & i \in [\mathrm{mid}+1 \cdots t] \\
        \ell_x \cdot x + (1-\ell_x)(1-g_{y})\cdot z + g_{y} \cdot (y+1)  & i = \mathrm{mid} \\
    \end{cases}
\end{equation}    

We can thus state the circuit schematically in \cref{fig: pbs-circuit}.

Now, before accounting for the complexity of the circuit for \cref{algo: pbs-key-values}, we must first assert the complexity of the comparators that we use in \cref{fig: pbs-circuit}.

\begin{figure}[!h]
\tikzset{every picture/.style={line width=0.75pt}} 
\begin{tikzpicture}[x=0.75pt,y=0.75pt,yscale=-0.75,xscale=0.75]

\draw   (184,117.01) -- (255.33,117.01) -- (255.33,134.51) -- (184,134.51) -- cycle ;
\draw   (395,118.4) -- (470,118.4) -- (470,134.84) -- (395,134.84) -- cycle ;
\draw   (270,119.09) -- (380,119.09) -- (380,147.01) -- (270,147.01) -- cycle ;
\draw    (206,90) -- (206,117.01) ;
\draw [shift={(206,117.01)}, rotate = 270] [color={rgb, 255:red, 0; green, 0; blue, 0 }  ][line width=0.75]    (10.93,-3.29) .. controls (6.95,-1.4) and (3.31,-0.3) .. (0,0) .. controls (3.31,0.3) and (6.95,1.4) .. (10.93,3.29)   ;
\draw    (411,86) -- (206,117.01) ;
\draw [shift={(206,117.01)}, rotate = 341] [color={rgb, 255:red, 0; green, 0; blue, 0 }  ][line width=0.75]    (10.93,-3.29) .. controls (6.95,-1.4) and (3.31,-0.3) .. (0,0) .. controls (3.31,0.3) and (6.95,1.4) .. (10.93,3.29)   ;
\draw    (411,86) -- (314,119.01) ;
\draw [shift={(314,119.01)}, rotate = 325] [color={rgb, 255:red, 0; green, 0; blue, 0 }  ][line width=0.75]    (10.93,-3.29) .. controls (6.95,-1.4) and (3.31,-0.3) .. (0,0) .. controls (3.31,0.3) and (6.95,1.4) .. (10.93,3.29)   ;
\draw    (411,86) -- (425,118.01) ;
\draw [shift={(425,118.01)}, rotate = 195] [color={rgb, 255:red, 0; green, 0; blue, 0 }  ][line width=0.75]    (10.93,-3.29) .. controls (6.95,-1.4) and (3.31,-0.3) .. (0,0) .. controls (3.31,0.3) and (6.95,1.4) .. (10.93,3.29)   ;
\draw    (347,88.01) -- (425,118.01) ;
\draw [shift={(425,118.01)}, rotate = 245] [color={rgb, 255:red, 0; green, 0; blue, 0 }  ][line width=0.75]    (10.93,-3.29) .. controls (6.95,-1.4) and (3.31,-0.3) .. (0,0) .. controls (3.31,0.3) and (6.95,1.4) .. (10.93,3.29)   ;
\draw    (206,134) -- (206,158.51) ;
\draw [shift={(206,158.51)}, rotate = 270] [color={rgb, 255:red, 0; green, 0; blue, 0 }  ][line width=0.75]    (10.93,-3.29) .. controls (6.95,-1.4) and (3.31,-0.3) .. (0,0) .. controls (3.31,0.3) and (6.95,1.4) .. (10.93,3.29)   ;
\draw    (425,136) -- (425,158.51) ;
\draw [shift={(425,158.51)}, rotate = 270] [color={rgb, 255:red, 0; green, 0; blue, 0 }  ][line width=0.75]    (10.93,-3.29) .. controls (6.95,-1.4) and (3.31,-0.3) .. (0,0) .. controls (3.31,0.3) and (6.95,1.4) .. (10.93,3.29)   ;
\draw   (196.75,167.51) .. controls (196.75,162.54) and (200.89,158.51) .. (206,158.51) .. controls (211.11,158.51) and (215.25,162.54) .. (215.25,167.51) .. controls (215.25,172.48) and (211.11,176.51) .. (206,176.51) .. controls (200.89,176.51) and (196.75,172.48) .. (196.75,167.51) -- cycle ;
\draw   (415.75,167.51) .. controls (415.75,162.54) and (419.89,158.51) .. (425,158.51) .. controls (430.11,158.51) and (434.25,162.54) .. (434.25,167.51) .. controls (434.25,172.48) and (430.11,176.51) .. (425,176.51) .. controls (419.89,176.51) and (415.75,172.48) .. (415.75,167.51) -- cycle ;
\draw   (211.75,301.51) .. controls (211.75,296.54) and (215.89,292.51) .. (221,292.51) .. controls (226.11,292.51) and (230.25,296.54) .. (230.25,301.51) .. controls (230.25,306.48) and (226.11,310.51) .. (221,310.51) .. controls (215.89,310.51) and (211.75,306.48) .. (211.75,301.51) -- cycle ;
\draw   (496.75,258.51) .. controls (496.75,253.54) and (500.89,249.51) .. (506,249.51) .. controls (511.11,249.51) and (515.25,253.54) .. (515.25,258.51) .. controls (515.25,263.48) and (511.11,267.51) .. (506,267.51) .. controls (500.89,267.51) and (496.75,263.48) .. (496.75,258.51) -- cycle ;
\draw   (326,258.51) .. controls (326,253.54) and (330.14,249.51) .. (335.25,249.51) .. controls (340.36,249.51) and (344.5,253.54) .. (344.5,258.51) .. controls (344.5,263.48) and (340.36,267.51) .. (335.25,267.51) .. controls (330.14,267.51) and (326,263.48) .. (326,258.51) -- cycle ;
\draw   (285.75,219.51) .. controls (285.75,214.54) and (289.89,210.51) .. (295,210.51) .. controls (300.11,210.51) and (304.25,214.54) .. (304.25,219.51) .. controls (304.25,224.48) and (300.11,228.51) .. (295,228.51) .. controls (289.89,228.51) and (285.75,224.48) .. (285.75,219.51) -- cycle ;
\draw   (133.75,252.51) .. controls (133.75,247.54) and (137.89,243.51) .. (143,243.51) .. controls (148.11,243.51) and (152.25,247.54) .. (152.25,252.51) .. controls (152.25,257.48) and (148.11,261.51) .. (143,261.51) .. controls (137.89,261.51) and (133.75,257.48) .. (133.75,252.51) -- cycle ;

\draw    (285.75,219.51) -- (216,205.51) ;
\draw [shift={(285.75,219.51)}, rotate = 190] [color={rgb, 255:red, 0; green, 0; blue, 0 }  ][line width=0.75]    (10.93,-3.29) .. controls (6.95,-1.4) and (3.31,-0.3) .. (0,0) .. controls (3.31,0.3) and (6.95,1.4) .. (10.93,3.29)   ;

\draw    (418,206.51) -- (304.25,219.51) ;
\draw [shift={(304.25,219.51)}, rotate = 350] [color={rgb, 255:red, 0; green, 0; blue, 0 }  ][line width=0.75]    (10.93,-3.29) .. controls (6.95,-1.4) and (3.31,-0.3) .. (0,0) .. controls (3.31,0.3) and (6.95,1.4) .. (10.93,3.29)   ;

\draw    (315,147) -- (331,203.51) ;

\draw    (339,225.51) -- (335.25,249.51) ;
\draw [shift={(335.25,249.51)}, rotate = 280] [color={rgb, 255:red, 0; green, 0; blue, 0 }  ][line width=0.75]    (10.93,-3.29) .. controls (6.95,-1.4) and (3.31,-0.3) .. (0,0) .. controls (3.31,0.3) and (6.95,1.4) .. (10.93,3.29)   ;

\draw    (331,203.51) .. controls (338,193.51) and (346,225.51) .. (339,225.51) ;

\draw    (295,228.51) -- (326,258.51) ;
\draw [shift={(326,258.51)}, rotate = 220] [color={rgb, 255:red, 0; green, 0; blue, 0 }  ][line width=0.75]    (10.93,-3.29) .. controls (6.95,-1.4) and (3.31,-0.3) .. (0,0) .. controls (3.31,0.3) and (6.95,1.4) .. (10.93,3.29)   ;

\draw    (206,134) -- (143,243.51) ;
\draw [shift={(143,243.51)}, rotate = 230] [color={rgb, 255:red, 0; green, 0; blue, 0 }  ][line width=0.75]    (10.93,-3.29) .. controls (6.95,-1.4) and (3.31,-0.3) .. (0,0) .. controls (3.31,0.3) and (6.95,1.4) .. (10.93,3.29)   ;

\draw    (425,136) -- (506,249.51) ;
\draw [shift={(506,249.51)}, rotate = 235] [color={rgb, 255:red, 0; green, 0; blue, 0 }  ][line width=0.75]    (10.93,-3.29) .. controls (6.95,-1.4) and (3.31,-0.3) .. (0,0) .. controls (3.31,0.3) and (6.95,1.4) .. (10.93,3.29)   ;

\draw    (133,231.51) -- (143,243.51) ;
\draw [shift={(143,243.51)}, rotate = 300] [color={rgb, 255:red, 0; green, 0; blue, 0 }  ][line width=0.75]    (10.93,-3.29) .. controls (6.95,-1.4) and (3.31,-0.3) .. (0,0) .. controls (3.31,0.3) and (6.95,1.4) .. (10.93,3.29)   ;

\draw    (525,231.51) -- (506,249.51) ;
\draw [shift={(506,249.51)}, rotate = 317] [color={rgb, 255:red, 0; green, 0; blue, 0 }  ][line width=0.75]    (10.93,-3.29) .. controls (6.95,-1.4) and (3.31,-0.3) .. (0,0) .. controls (3.31,0.3) and (6.95,1.4) .. (10.93,3.29)   ;

\draw    (215.25,167.51) -- (235,149.51) ;
\draw [shift={(215.25,167.51)}, rotate = 317] [color={rgb, 255:red, 0; green, 0; blue, 0 }  ][line width=0.75]    (10.93,-3.29) .. controls (6.95,-1.4) and (3.31,-0.3) .. (0,0) .. controls (3.31,0.3) and (6.95,1.4) .. (10.93,3.29)   ;

\draw    (400,148.51) -- (415.75,167.51) ;
\draw [shift={(415.75,167.51)}, rotate = 230] [color={rgb, 255:red, 0; green, 0; blue, 0 }  ][line width=0.75]    (10.93,-3.29) .. controls (6.95,-1.4) and (3.31,-0.3) .. (0,0) .. controls (3.31,0.3) and (6.95,1.4) .. (10.93,3.29)   ;

\draw   (430.75,300.51) .. controls (430.75,295.54) and (434.89,291.51) .. (440,291.51) .. controls (445.11,291.51) and (449.25,295.54) .. (449.25,300.51) .. controls (449.25,305.48) and (445.11,309.51) .. (440,309.51) .. controls (434.89,309.51) and (430.75,305.48) .. (430.75,300.51) -- cycle ;
\draw    (506,267.51) -- (449.25,300.51) ;
\draw [shift={(449.25,300.51)}, rotate = 325] [color={rgb, 255:red, 0; green, 0; blue, 0 }  ][line width=0.75]    (10.93,-3.29) .. controls (6.95,-1.4) and (3.31,-0.3) .. (0,0) .. controls (3.31,0.3) and (6.95,1.4) .. (10.93,3.29)   ;

\draw    (335.25,267.51) -- (430.75,300.51) ;
\draw [shift={(430.75,300.51)}, rotate = 200] [color={rgb, 255:red, 0; green, 0; blue, 0 }  ][line width=0.75]    (10.93,-3.29) .. controls (6.95,-1.4) and (3.31,-0.3) .. (0,0) .. controls (3.31,0.3) and (6.95,1.4) .. (10.93,3.29)   ;

\draw    (230.25,301.51) -- (335.25,267.51) ;
\draw [shift={(230.25,301.51)}, rotate = 340] [color={rgb, 255:red, 0; green, 0; blue, 0 }  ][line width=0.75]    (10.93,-3.29) .. controls (6.95,-1.4) and (3.31,-0.3) .. (0,0) .. controls (3.31,0.3) and (6.95,1.4) .. (10.93,3.29)   ;

\draw    (211.75,301.51) -- (143,261.51) ;
\draw [shift={(211.75,301.51)}, rotate = 210] [color={rgb, 255:red, 0; green, 0; blue, 0 }  ][line width=0.75]    (10.93,-3.29) .. controls (6.95,-1.4) and (3.31,-0.3) .. (0,0) .. controls (3.31,0.3) and (6.95,1.4) .. (10.93,3.29)   ;

\draw    (221,310.51) -- (221,352.51) ;
\draw [shift={(221,352.51)}, rotate = 270] [color={rgb, 255:red, 0; green, 0; blue, 0 }  ][line width=0.75]    (10.93,-3.29) .. controls (6.95,-1.4) and (3.31,-0.3) .. (0,0) .. controls (3.31,0.3) and (6.95,1.4) .. (10.93,3.29)   ;

\draw    (440,309.51) -- (440,353.51) ;
\draw [shift={(440,353.51)}, rotate = 270] [color={rgb, 255:red, 0; green, 0; blue, 0 }  ][line width=0.75]    (10.93,-3.29) .. controls (6.95,-1.4) and (3.31,-0.3) .. (0,0) .. controls (3.31,0.3) and (6.95,1.4) .. (10.93,3.29)   ;

\draw   (162.42,469.51) .. controls (162.42,464.54) and (166.56,460.51) .. (171.67,460.51) .. controls (176.78,460.51) and (180.92,464.54) .. (180.92,469.51) .. controls (180.92,474.48) and (176.78,478.51) .. (171.67,478.51) .. controls (166.56,478.51) and (162.42,474.48) .. (162.42,469.51) -- cycle ;
\draw   (162.75,420.51) .. controls (162.75,415.54) and (166.89,411.51) .. (172,411.51) .. controls (177.11,411.51) and (181.25,415.54) .. (181.25,420.51) .. controls (181.25,425.48) and (177.11,429.51) .. (172,429.51) .. controls (166.89,429.51) and (162.75,425.48) .. (162.75,420.51) -- cycle ;
\draw   (337.17,326.51) .. controls (337.17,321.54) and (341.31,317.51) .. (346.42,317.51) .. controls (351.53,317.51) and (355.67,321.54) .. (355.67,326.51) .. controls (355.67,331.48) and (351.53,335.51) .. (346.42,335.51) .. controls (341.31,335.51) and (337.17,331.48) .. (337.17,326.51) -- cycle ;
\draw   (90.94,352.51) -- (340,352.51) -- (340,370.51) -- (90.94,370.51) -- cycle ;
\draw   (350,353.51) -- (625.67,353.51) -- (625.67,368.51) -- (350,368.51) -- cycle ;

\draw    (143,261.51) -- (337.17,326.51) ;
\draw [shift={(337.17,326.51)}, rotate = 200] [color={rgb, 255:red, 0; green, 0; blue, 0 }  ][line width=0.75]    (10.93,-3.29) .. controls (6.95,-1.4) and (3.31,-0.3) .. (0,0) .. controls (3.31,0.3) and (6.95,1.4) .. (10.93,3.29)   ;

\draw    (355.67,326.51) -- (434,307.51) ;
\draw [shift={(355.67,326.51)}, rotate = 345] [color={rgb, 255:red, 0; green, 0; blue, 0 }  ][line width=0.75]    (10.93,-3.29) .. controls (6.95,-1.4) and (3.31,-0.3) .. (0,0) .. controls (3.31,0.3) and (6.95,1.4) .. (10.93,3.29)   ;

\draw    (172,429.51) -- (171.67,460.51) ;
\draw [shift={(171.67,460.51)}, rotate = 270] [color={rgb, 255:red, 0; green, 0; blue, 0 }  ][line width=0.75]    (10.93,-3.29) .. controls (6.95,-1.4) and (3.31,-0.3) .. (0,0) .. controls (3.31,0.3) and (6.95,1.4) .. (10.93,3.29)   ;

\draw    (171.67,371.51) -- (172,411.51) ;
\draw [shift={(172,411.51)}, rotate = 270] [color={rgb, 255:red, 0; green, 0; blue, 0 }  ][line width=0.75]    (10.93,-3.29) .. controls (6.95,-1.4) and (3.31,-0.3) .. (0,0) .. controls (3.31,0.3) and (6.95,1.4) .. (10.93,3.29)   ;

\draw   (528.75,420.51) .. controls (528.75,415.54) and (532.89,411.51) .. (538,411.51) .. controls (543.11,411.51) and (547.25,415.54) .. (547.25,420.51) .. controls (547.25,425.48) and (543.11,429.51) .. (538,429.51) .. controls (532.89,429.51) and (528.75,425.48) .. (528.75,420.51) -- cycle ;
\draw    (537.67,371.51) -- (538,411.51) ;
\draw [shift={(538,411.51)}, rotate = 270] [color={rgb, 255:red, 0; green, 0; blue, 0 }  ][line width=0.75]    (10.93,-3.29) .. controls (6.95,-1.4) and (3.31,-0.3) .. (0,0) .. controls (3.31,0.3) and (6.95,1.4) .. (10.93,3.29)   ;

\draw    (435,205.51) .. controls (748,229.51) and (675,404.51) .. (547.25,420.51) ;
\draw [shift={(547.25,420.51)}, rotate = 345] [color={rgb, 255:red, 0; green, 0; blue, 0 }  ][line width=0.75]    (10.93,-3.29) .. controls (6.95,-1.4) and (3.31,-0.3) .. (0,0) .. controls (3.31,0.3) and (6.95,1.4) .. (10.93,3.29)   ;

\draw    (206,175.51) -- (206,198.51) ;
\draw [shift={(206,198.51)}, rotate = 270] [color={rgb, 255:red, 0; green, 0; blue, 0 }  ][line width=0.75]    (10.93,-3.29) .. controls (6.95,-1.4) and (3.31,-0.3) .. (0,0) .. controls (3.31,0.3) and (6.95,1.4) .. (10.93,3.29)   ;

\draw   (196.75,207.51) .. controls (196.75,202.54) and (200.89,198.51) .. (206,198.51) .. controls (211.11,198.51) and (215.25,202.54) .. (215.25,207.51) .. controls (215.25,212.48) and (211.11,216.51) .. (206,216.51) .. controls (200.89,216.51) and (196.75,212.48) .. (196.75,207.51) -- cycle ;

\draw    (198,205.51) .. controls (-104,246.51) and (78.33,400.51) .. (162.75,420.51) ;
\draw [shift={(162.75,420.51)}, rotate = 200] [color={rgb, 255:red, 0; green, 0; blue, 0 }  ][line width=0.75]    (10.93,-3.29) .. controls (6.95,-1.4) and (3.31,-0.3) .. (0,0) .. controls (3.31,0.3) and (6.95,1.4) .. (10.93,3.29)   ;

\draw    (426,175.51) -- (426,198.51) ;
\draw [shift={(426,198.51)}, rotate = 270] [color={rgb, 255:red, 0; green, 0; blue, 0 }  ][line width=0.75]    (10.93,-3.29) .. controls (6.95,-1.4) and (3.31,-0.3) .. (0,0) .. controls (3.31,0.3) and (6.95,1.4) .. (10.93,3.29)   ;

\draw   (416.75,207.51) .. controls (416.75,202.54) and (420.89,198.51) .. (426,198.51) .. controls (431.11,198.51) and (435.25,202.54) .. (435.25,207.51) .. controls (435.25,212.48) and (431.11,216.51) .. (426,216.51) .. controls (420.89,216.51) and (416.75,212.48) .. (416.75,207.51) -- cycle ;

\draw (235.33,177.91) node [anchor=north west][inner sep=0.75pt]  [font=\tiny]  {$1$};

\draw    (215,203) -- (235.33,187.91) ;
\draw [shift={(215,203)}, rotate = 325] [color={rgb, 255:red, 0; green, 0; blue, 0 }  ][line width=0.75]    (10.93,-3.29) .. controls (6.95,-1.4) and (3.31,-0.3) .. (0,0) .. controls (3.31,0.3) and (6.95,1.4) .. (10.93,3.29)   ;

\draw (383.33,176.91) node [anchor=north west][inner sep=0.75pt]  [font=\tiny]  {$1$};

\draw    (417,203) -- (390.33,186.91) ;
\draw [shift={(417,203)}, rotate = 210] [color={rgb, 255:red, 0; green, 0; blue, 0 }  ][line width=0.75]    (10.93,-3.29) .. controls (6.95,-1.4) and (3.31,-0.3) .. (0,0) .. controls (3.31,0.3) and (6.95,1.4) .. (10.93,3.29)   ;

\draw    (538,429.51) -- (538.67,459.51) ;
\draw [shift={(538.67,459.51)}, rotate = 270] [color={rgb, 255:red, 0; green, 0; blue, 0 }  ][line width=0.75]    (10.93,-3.29) .. controls (6.95,-1.4) and (3.31,-0.3) .. (0,0) .. controls (3.31,0.3) and (6.95,1.4) .. (10.93,3.29)   ;

\draw    (346.42,335.51) -- (344.67,510.51) ;
\draw [shift={(344.67,510.51)}, rotate = 270] [color={rgb, 255:red, 0; green, 0; blue, 0 }  ][line width=0.75]    (10.93,-3.29) .. controls (6.95,-1.4) and (3.31,-0.3) .. (0,0) .. controls (3.31,0.3) and (6.95,1.4) .. (10.93,3.29)   ;

\draw    (143,261.51) .. controls (60.33,351.51) and (61.33,412.51) .. (162.42,469.51) ;
\draw [shift={(162.42,469.51)}, rotate = 215] [color={rgb, 255:red, 0; green, 0; blue, 0 }  ][line width=0.75]    (10.93,-3.29) .. controls (6.95,-1.4) and (3.31,-0.3) .. (0,0) .. controls (3.31,0.3) and (6.95,1.4) .. (10.93,3.29)   ;

\draw    (171.67,478.51) -- (171.33,509.51) ;
\draw [shift={(171.33,509.51)}, rotate = 270] [color={rgb, 255:red, 0; green, 0; blue, 0 }  ][line width=0.75]    (10.93,-3.29) .. controls (6.95,-1.4) and (3.31,-0.3) .. (0,0) .. controls (3.31,0.3) and (6.95,1.4) .. (10.93,3.29)   ;

\draw   (529.42,468.51) .. controls (529.42,463.54) and (533.56,459.51) .. (538.67,459.51) .. controls (543.78,459.51) and (547.92,463.54) .. (547.92,468.51) .. controls (547.92,473.48) and (543.78,477.51) .. (538.67,477.51) .. controls (533.56,477.51) and (529.42,473.48) .. (529.42,468.51) -- cycle ;

\draw    (506,267.51) .. controls (687,371.51) and (660,384.51) .. (547.92,468.51) ;
\draw [shift={(547.92,468.51)}, rotate = 320] [color={rgb, 255:red, 0; green, 0; blue, 0 }  ][line width=0.75]    (10.93,-3.29) .. controls (6.95,-1.4) and (3.31,-0.3) .. (0,0) .. controls (3.31,0.3) and (6.95,1.4) .. (10.93,3.29)   ;

\draw    (538.67,477.51) -- (538.33,508.51) ;
\draw [shift={(538.33,508.51)}, rotate = 270] [color={rgb, 255:red, 0; green, 0; blue, 0 }  ][line width=0.75]    (10.93,-3.29) .. controls (6.95,-1.4) and (3.31,-0.3) .. (0,0) .. controls (3.31,0.3) and (6.95,1.4) .. (10.93,3.29)   ;

\draw (184,120) node [anchor=north west][inner sep=0.75pt]  [font=\tiny]  {$\bm{q}[\mathrm{mid}] \leq \bm{k}_{x}$};
\draw (396,120) node [anchor=north west][inner sep=0.75pt]  [font=\tiny]  {$\bm{q}[\mathrm{mid}] > \bm{k}_{y}$};
\draw (275,118.09) node [anchor=north west][inner sep=0.75pt]  [font=\scriptsize] [align=left] {\texttt{bin-search}};
\draw (270,128) node [anchor=north west][inner sep=0.75pt]  [font=\tiny]  {$({\{\bm{k}_{i}\}_{i=s}^{t} ,\ \bm{q}[\mathrm{mid}]})$};

\draw (-70, 70.0) node [anchor=north west][inner sep=0.75pt]  [font=\tiny]   {$\texttt{pbs-key-values}\paren{\bm{q}[s \cdots t],\ \bm{k}[x \cdots y], n, m}$};

\draw   (-75,65.5) -- (175,65.5) -- (175,85.5) -- (-75,85.5) -- cycle ;

\draw   (-100,60.0) -- (680,60.0) -- (680,550) -- (-100,550) -- cycle ;

\draw (357,70.0) node [anchor=north west][inner sep=0.75pt] [font=\scriptsize]   {$\bm{q}_{s}$};
\draw (397,67.0) node [anchor=north west][inner sep=0.75pt]  [font=\scriptsize]  {$\bm{q}[\mathrm{mid}]$};
\draw (458,70.0) node [anchor=north west][inner sep=0.75pt]  [font=\scriptsize]  {$\bm{q}_{t}$};
\draw (375,78.0) node [anchor=north west][inner sep=0.75pt]  [font=\scriptsize]  {$.....$};
\draw (202,70.0) node [anchor=north west][inner sep=0.75pt]   [font=\scriptsize] {$\bm{k}_{x}$};
\draw (331,70.0) node [anchor=north west][inner sep=0.75pt] [font=\scriptsize]   {$\bm{k}_{y}$};
\draw (225,78.0) node [anchor=north west][inner sep=0.75pt] [font=\scriptsize]   {$.......................$};
\draw (433,78.0) node [anchor=north west][inner sep=0.75pt]  [font=\scriptsize]  {$.....$};
\draw (121,222.4) node [anchor=north west][inner sep=0.75pt]  [font=\scriptsize]  {$x$};
\draw (526,228.4) node [anchor=north west][inner sep=0.75pt]  [font=\scriptsize]  {$y+1$};
\draw (286.75,210.91) node [anchor=north west][inner sep=0.75pt]    {$\times $};
\draw (326.75,249.91) node [anchor=north west][inner sep=0.75pt]    {$\times $};
\draw (134.75,243.91) node [anchor=north west][inner sep=0.75pt]    {$\times $};
\draw (497.75,250.91) node [anchor=north west][inner sep=0.75pt]    {$\times $};
\draw (235.33,137.91) node [anchor=north west][inner sep=0.75pt]  [font=\tiny]  {$-1$};
\draw (383.33,136.91) node [anchor=north west][inner sep=0.75pt]  [font=\tiny]  {$-1$};
\draw (213,293.4) node [anchor=north west][inner sep=0.75pt]    {$+$};
\draw (432,292.4) node [anchor=north west][inner sep=0.75pt]    {$+$};
\draw (205,340.91) node [anchor=north west][inner sep=0.75pt] [font=\scriptsize] {$z_{s}$};
\draw (90.94,356.07) node [anchor=north west][inner sep=0.75pt]  [font=\tiny] [align=left] {\texttt{pbs-key-values}($\displaystyle \bm{q}[ s\cdots \mathrm{mid} -1,\ \bm{k}[ z_{s} \cdots y]$)};
\draw (350.94,355.07) node [anchor=north west][inner sep=0.75pt]  [font=\tiny] [align=left] {\texttt{pbs-key-values}($\displaystyle \bm{q}[\mathrm{mid} +1\cdots t] ,\ \bm{k}[ x\ \cdots z_{t} -1]$)};
\draw (208,137.4) node [anchor=north west][inner sep=0.75pt]  [font=\tiny]  {$\ell_{x}$};
\draw (410,137.4) node [anchor=north west][inner sep=0.75pt]  [font=\tiny]  {$g_{y}$};
\draw (306,148.4) node [anchor=north west][inner sep=0.75pt]  [font=\tiny]  {$z$};
\draw (446,340.91) node [anchor=north west][inner sep=0.75pt]  [font=\tiny]  {$z_{t} $};
\draw (339,318.4) node [anchor=north west][inner sep=0.75pt]    {$+$};
\draw (163.75,412.91) node [anchor=north west][inner sep=0.75pt]    {$\times $};
\draw (529.75,412.91) node [anchor=north west][inner sep=0.75pt]    {$\times $};
\draw (172.25,489.91) node [anchor=north west][inner sep=0.75pt]    {$\dotsc $};
\draw (521,522.4) node [anchor=north west][inner sep=0.75pt]    {$...........$};
\draw (198,159.4) node [anchor=north west][inner sep=0.75pt]    {$\times$};

\draw (198,199.4) node [anchor=north west][inner sep=0.75pt]    {$+$};

\draw (417,159.4) node [anchor=north west][inner sep=0.75pt]    {$\times$};

\draw (417,199.4) node [anchor=north west][inner sep=0.75pt]    {$+$};

\draw (164,462.4) node [anchor=north west][inner sep=0.75pt]    {$+$};
\draw (530,460.4) node [anchor=north west][inner sep=0.75pt]    {$+$};
\draw (509.25,488.91) node [anchor=north west][inner sep=0.75pt]    {$\dotsc $};
\draw (323,515.4) node [anchor=north west][inner sep=0.75pt]  [font=\scriptsize]  {$C[\mathrm{mid}]$};
\draw (455,515.4) node [anchor=north west][inner sep=0.75pt]  [font=\scriptsize]  {$C[\mathrm{mid} -1]$};
\draw (120,515.4) node [anchor=north west][inner sep=0.75pt]  [font=\scriptsize]  {$C[ s]$};
\draw (205,515.4) node [anchor=north west][inner sep=0.75pt]  [font=\scriptsize]  {$C[\mathrm{mid} -1]$};
\draw (148,522.4) node [anchor=north west][inner sep=0.75pt]    {$...........$};
\draw (575,515.06) node [anchor=north west][inner sep=0.75pt]  [font=\scriptsize]  {$C[ t]$};
\end{tikzpicture}
\caption{$\texttt{pbs-key-values}\paren{\bm{q}[s \cdots t],\ \bm{k}[x \cdots y], n, m}$ as a circuit with recursive calls and subprocedures as ``black boxes.''}
\label{fig: pbs-circuit}
\end{figure}
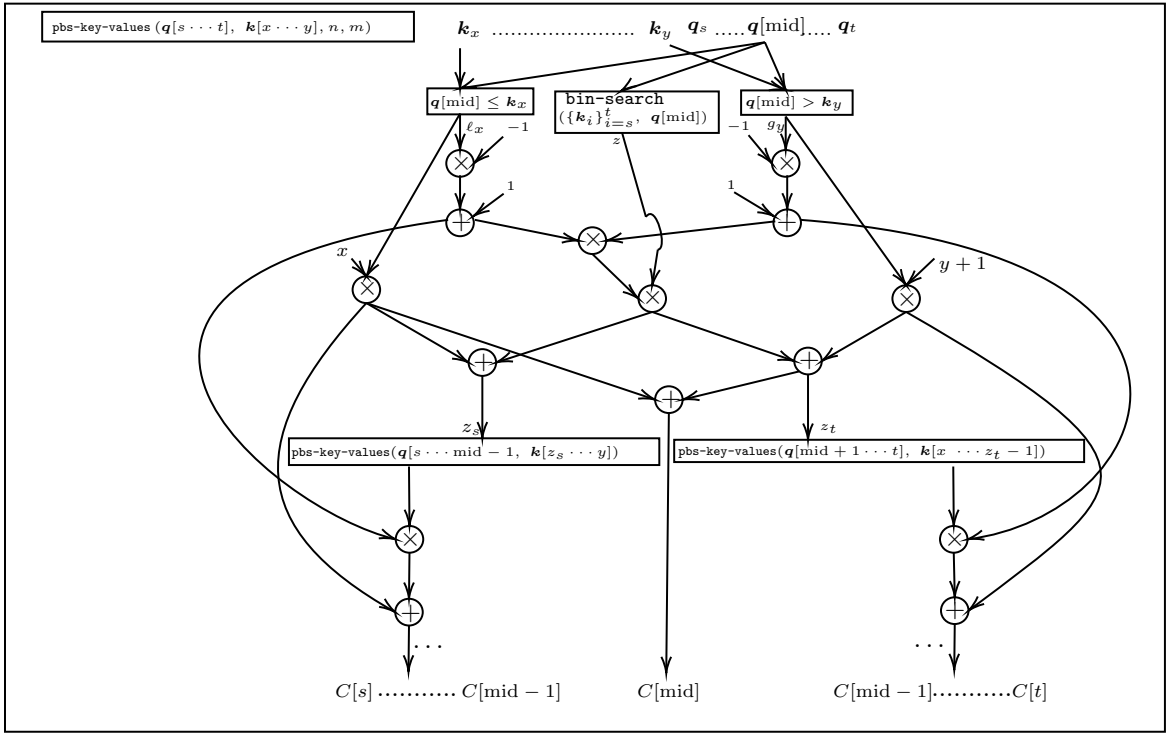
\begin{lemma}[\cite{cormen2022introduction}]
    \label{lem: comparators}
    For binary strings $\bm{x}, \bm{y} \in \Sigma^{d}$ of length $d$, there exists a comparison network of size $\calO(d)$, width $\calO(d)$, and depth $\calO(\log{d})$.
\end{lemma}

We use this lemma to compute the complexity of the arhmetic circuit from \cref{fig: pbs-circuit}.
\begin{proposition}
    \label{prop: pbs-circuit}
    There exists an $(\calO((n+m)d), \calO(nd\cdot(\log{m}+\log{n})), \calO(\log{n}(\log{m}+\log{n}\log{d})), \calO(nd))$-arithmetic circuit\footnote{Recall that a $(n,s,\Delta, w)$-arithmetic circuit is an $n$-variate circuit with size $s$, depth at most $\Delta$, and width $w$.} equivalent to {\tt pbs-key-values} (\cref{algo: pbs-key-values}) with inputs $\bm{q}$ and $\bm{k}$ of lengths $n$ and $m$ with $d$-bit entries. 
\end{proposition}
\begin{proof}
The size of the circuit for \texttt{pbs-key-values} should equal the work-complexity of \cref{algo: pbs-key-values} but we also need to account for the comparison gates in \eqref{eq: comparison-gates}. Further, the circuit for binary search also has a size of $\calO(d \cdot n)$ instead of $\calO(d \cdot \log{n})$ work as in \cref{algo: pbs-key-values}. Using \cref{lem: comparators}, along with the fact that the comparison gates and the binary search are used at most $\calO(\log{n})$ times, we deduce that we are adding $\calO(nd\log{n})$ size to the circuit in addition to the work complexity of the parallel algorithm. Thus, the overall size of the arithmetic circuit for \texttt{pbs-key-values} is $\calO(dn\log{m}+nd\log{n})$. Further, the depth of the circuit here is determined by the runtime of \cref{algo: pbs-key-values} along with the depth of the comparison gates and binary search $\calO(\log{n}\log{d})$, and finally, at most $n$ processors are used for the parallel algorithm in \cref{algo: pbs-key-values}, which yields the width of the circuit.
\end{proof}

\paragraph{Circuit for \texttt{Parallel-MQAR}.}
Now, we will call the circuit for \texttt{Parallel-MQAR} with the same name while the input gates contain the inputs of \cref{algo: pram-gen-ar}. Indeed, we can directly ``translate" \cref{algo: pram-gen-ar} to an arithmetic circuit as the values for $I_{k}^{\bm{x}}$ and $I_{k}^{\bm{x}}$ for each $\bm{x}$ and $k$ are predetermined from $N$. Thus, we start by placing the corresponding sorting networks which feeds into the $\texttt{pbs-key-values}\paren{\bm{q}[s \cdots t],\ \bm{k}[x \cdots y], n, m}$ circuit for \cref{algo: pbs-key-values} in \cref{fig: pbs-circuit} so that the output values from the calls to \texttt{pbs-key-values} result in the checks as in line \ref{step: check-m} of \cref{algo: pram-gen-ar}. That is, we get outputs $C_k[J_{\mathrm{permuted}}^{k\bm{x}}(\mathrm{dec}(i))]$ from each call to \texttt{pbs-key-values}. We can then use a comparison gate to check if this value equals $2^k$, and if not, we have found a match $C_k[J_{\mathrm{permuted}}^{k\bm{x}}(\mathrm{dec}(i))]$ for the query $\bm{q}_i$ which results in the output of the associated value $\bm{v}_{C_k[J_{\mathrm{permuted}}^{k\bm{x}}(\mathrm{dec}(i))]+1}$, exactly as in \cref{algo: pram-gen-ar}. That is, we first define the following variable as the output of the comparison gate:
\begin{equation}
    \label{eq: comparison-gate-neq}
    \bm{c}_{\mathrm{dec}(i)}^k := \mathbbm{1}\{C_k[J_{\mathrm{permuted}}^{k\bm{x}}(\mathrm{dec}(i))] \neq 2^k\}.
\end{equation}
Here, as $C_k[J_{\mathrm{permuted}}^{k\bm{x}}(\mathrm{dec}(i))] \neq 2^k$ implies that $I_{\mathrm{permuted}}^{k\bm{x}}\paren{C_k[J_{\mathrm{permuted}}^{k\bm{x}}(\mathrm{dec}(i))]}$ equals the index of the matching key $\bm{k}_j$ corresponding to the query $\bm{q}_i$, the $i$th output is then simply given by $\bm{c}_{\mathrm{dec}(i)}^k \cdot I_{\mathrm{permuted}}^{k\bm{x}}\paren{C_k[J_{\mathrm{permuted}}^{k\bm{x}}(\mathrm{dec}(i))]}$, where the $0$ output implies that there does not exist a matching key.

Here, we also briefly introduce the the sorting networks that we use to sort the keys and queries:
\begin{definition}[Informal]
\label{def: sorting-networks}
{\em Sorting networks} are circuits with gates and wires where the gates of the circuit are comparators (\cref{def: comparators}) connecting two wires. Each such circuit can perform sorting on a fixed number of values.  
\end{definition}

We can then show the circuit schematically as in \cref{fig: pram-gen-ar-circuit}.
\begin{figure}[!h]
    \centering
\begin{tikzpicture}[x=0.6pt,y=0.7pt,yscale=-1,xscale=1]

\draw   (22.33,134.58) .. controls (22.33,130.67) and (25.5,127.51) .. (29.4,127.51) -- (129.93,127.51) .. controls (133.84,127.51) and (137,130.67) .. (137,134.58) -- (137,155.78) .. controls (137,159.68) and (133.84,162.84) .. (129.93,162.84) -- (29.4,162.84) .. controls (25.5,162.84) and (22.33,159.68) .. (22.33,155.78) -- cycle ;

\draw   (143.96,136.04) .. controls (143.96,132.62) and (146.74,129.84) .. (150.16,129.84) -- (240.8,129.84) .. controls (244.22,129.84) and (247,132.62) .. (247,136.04) -- (247,154.64) .. controls (247,158.07) and (244.22,160.84) .. (240.8,160.84) -- (150.16,160.84) .. controls (146.74,160.84) and (143.96,158.07) .. (143.96,154.64) -- cycle ;
\draw   (20,131.18) .. controls (20,126.57) and (23.73,122.84) .. (28.33,122.84) -- (276.67,122.84) .. controls (281.27,122.84) and (285,126.57) .. (285,131.18) -- (285,156.18) .. controls (285,160.78) and (281.27,164.51) .. (276.67,164.51) -- (28.33,164.51) .. controls (23.73,164.51) and (20,160.78) .. (20,156.18) -- cycle ;
\draw   (410.33,137.98) .. controls (410.33,134.04) and (413.53,130.84) .. (417.47,130.84) -- (505.87,130.84) .. controls (509.81,130.84) and (513,134.04) .. (513,137.98) -- (513,159.38) .. controls (513,163.32) and (509.81,166.51) .. (505.87,166.51) -- (417.47,166.51) .. controls (413.53,166.51) and (410.33,163.32) .. (410.33,159.38) -- cycle ;
\draw   (520.33,137.78) .. controls (520.33,133.95) and (523.44,130.84) .. (527.27,130.84) -- (613.07,130.84) .. controls (616.9,130.84) and (620,133.95) .. (620,137.78) -- (620,158.58) .. controls (620,162.41) and (616.9,165.51) .. (613.07,165.51) -- (527.27,165.51) .. controls (523.44,165.51) and (520.33,162.41) .. (520.33,158.58) -- cycle ;
\draw   (401,132.84) .. controls (401,127.87) and (405.03,123.84) .. (410,123.84) -- (669.33,123.84) .. controls (674.3,123.84) and (678.33,127.87) .. (678.33,132.84) -- (678.33,159.84) .. controls (678.33,164.81) and (674.3,168.84) .. (669.33,168.84) -- (410,168.84) .. controls (405.03,168.84) and (401,164.81) .. (401,159.84) -- cycle ;
\draw    (313.33,54.84) -- (111.23,122.21) ;
\draw [shift={(109.33,122.84)}, rotate = 341.57] [color={rgb, 255:red, 0; green, 0; blue, 0 }  ][line width=0.75]    (10.93,-3.29) .. controls (6.95,-1.4) and (3.31,-0.3) .. (0,0) .. controls (3.31,0.3) and (6.95,1.4) .. (10.93,3.29)   ;
\draw    (371.33,55.84) -- (577.42,119.26) ;
\draw [shift={(579.33,119.84)}, rotate = 197.1] [color={rgb, 255:red, 0; green, 0; blue, 0 }  ][line width=0.75]    (10.93,-3.29) .. controls (6.95,-1.4) and (3.31,-0.3) .. (0,0) .. controls (3.31,0.3) and (6.95,1.4) .. (10.93,3.29)   ;
\draw    (50,163) -- (62.62,227.55) ;
\draw [shift={(63,229.51)}, rotate = 258.94] [color={rgb, 255:red, 0; green, 0; blue, 0 }  ][line width=0.75]    (10.93,-3.29) .. controls (6.95,-1.4) and (3.31,-0.3) .. (0,0) .. controls (3.31,0.3) and (6.95,1.4) .. (10.93,3.29)   ;
\draw    (111,164) -- (64.18,227.9) ;
\draw [shift={(63,229.51)}, rotate = 306.23] [color={rgb, 255:red, 0; green, 0; blue, 0 }  ][line width=0.75]    (10.93,-3.29) .. controls (6.95,-1.4) and (3.31,-0.3) .. (0,0) .. controls (3.31,0.3) and (6.95,1.4) .. (10.93,3.29)   ;
\draw   (12.7,235.27) .. controls (12.7,231.46) and (15.79,228.37) .. (19.6,228.37) -- (234.14,228.37) .. controls (237.95,228.37) and (241.04,231.46) .. (241.04,235.27) -- (241.04,255.95) .. controls (241.04,259.76) and (237.95,262.84) .. (234.14,262.84) -- (19.6,262.84) .. controls (15.79,262.84) and (12.7,259.76) .. (12.7,255.95) -- cycle ;
\draw   (5,231.04) .. controls (5,225.96) and (9.12,221.84) .. (14.2,221.84) -- (276.8,221.84) .. controls (281.88,221.84) and (286,225.96) .. (286,231.04) -- (286,258.64) .. controls (286,263.72) and (281.88,267.84) .. (276.8,267.84) -- (14.2,267.84) .. controls (9.12,267.84) and (5,263.72) .. (5,258.64) -- cycle ;
\draw   (405,236.16) .. controls (405,231.94) and (408.42,228.51) .. (412.65,228.51) -- (623.35,228.51) .. controls (627.57,228.51) and (631,231.94) .. (631,236.16) -- (631,259.12) .. controls (631,263.35) and (627.57,266.77) .. (623.35,266.77) -- (412.65,266.77) .. controls (408.42,266.77) and (405,263.35) .. (405,259.12) -- cycle ;
\draw   (402,232.96) .. controls (402,228.07) and (405.97,224.11) .. (410.86,224.11) -- (681.14,224.11) .. controls (686.03,224.11) and (690,228.07) .. (690,232.96) -- (690,259.53) .. controls (690,264.42) and (686.03,268.39) .. (681.14,268.39) -- (410.86,268.39) .. controls (405.97,268.39) and (402,264.42) .. (402,259.53) -- cycle ;
\draw    (480.27,168.99) -- (496.46,226.59) ;
\draw [shift={(497,228.51)}, rotate = 254.3] [color={rgb, 255:red, 0; green, 0; blue, 0 }  ][line width=0.75]    (10.93,-3.29) .. controls (6.95,-1.4) and (3.31,-0.3) .. (0,0) .. controls (3.31,0.3) and (6.95,1.4) .. (10.93,3.29)   ;
\draw    (557.67,168.51) -- (498.42,227.1) ;
\draw [shift={(497,228.51)}, rotate = 315.32] [color={rgb, 255:red, 0; green, 0; blue, 0 }  ][line width=0.75]    (10.93,-3.29) .. controls (6.95,-1.4) and (3.31,-0.3) .. (0,0) .. controls (3.31,0.3) and (6.95,1.4) .. (10.93,3.29)   ;
\draw    (521.6,267.74) -- (522.64,294.51) ;
\draw [shift={(522.72,296.51)}, rotate = 267.77] [color={rgb, 255:red, 0; green, 0; blue, 0 }  ][line width=0.75]    (10.93,-3.29) .. controls (6.95,-1.4) and (3.31,-0.3) .. (0,0) .. controls (3.31,0.3) and (6.95,1.4) .. (10.93,3.29)   ;
\draw   (143.74,329.38) .. controls (143.74,326.15) and (146.36,323.53) .. (149.59,323.53) -- (294.87,323.53) .. controls (298.1,323.53) and (300.71,326.15) .. (300.71,329.38) -- (300.71,346.92) .. controls (300.71,350.15) and (298.1,352.76) .. (294.87,352.76) -- (149.59,352.76) .. controls (146.36,352.76) and (143.74,350.15) .. (143.74,346.92) -- cycle ;
\draw   (135,325.5) .. controls (135,320.81) and (138.81,317) .. (143.5,317) -- (330.5,317) .. controls (335.19,317) and (339,320.81) .. (339,325.5) -- (339,351.01) .. controls (339,355.7) and (335.19,359.51) .. (330.5,359.51) -- (143.5,359.51) .. controls (138.81,359.51) and (135,355.7) .. (135,351.01) -- cycle ;
\draw    (199,265) -- (199.94,325.51) ;
\draw [shift={(200,327.51)}, rotate = 268.24] [color={rgb, 255:red, 0; green, 0; blue, 0 }  ][line width=0.75]    (10.93,-3.29) .. controls (6.95,-1.4) and (3.31,-0.3) .. (0,0) .. controls (3.31,0.3) and (6.95,1.4) .. (10.93,3.29)   ;
\draw    (166,353) -- (166,419.51) ;
\draw [shift={(166,421.51)}, rotate = 270] [color={rgb, 255:red, 0; green, 0; blue, 0 }  ][line width=0.75]    (10.93,-3.29) .. controls (6.95,-1.4) and (3.31,-0.3) .. (0,0) .. controls (3.31,0.3) and (6.95,1.4) .. (10.93,3.29)   ;
\draw    (199,265) .. controls (13,300.51) and (-28.17,396.68) .. (155.83,431.68) ;
\draw [shift={(155.83,431.68)}, rotate = 190.77] [color={rgb, 255:red, 0; green, 0; blue, 0 }  ][line width=0.75]    (10.93,-3.29) .. controls (6.95,-1.4) and (3.31,-0.3) .. (0,0) .. controls (3.31,0.3) and (6.95,1.4) .. (10.93,3.29)   ;
\draw   (155.83,431.68) .. controls (155.83,426.06) and (160.39,421.51) .. (166,421.51) .. controls (171.61,421.51) and (176.17,426.06) .. (176.17,431.68) .. controls (176.17,437.29) and (171.61,441.84) .. (166,441.84) .. controls (160.39,441.84) and (155.83,437.29) .. (155.83,431.68) -- cycle ;
\draw    (166,441.84) -- (165.37,478.84) ;
\draw [shift={(165.33,480.84)}, rotate = 270.98] [color={rgb, 255:red, 0; green, 0; blue, 0 }  ][line width=0.75]    (10.93,-3.29) .. controls (6.95,-1.4) and (3.31,-0.3) .. (0,0) .. controls (3.31,0.3) and (6.95,1.4) .. (10.93,3.29)   ;
\draw   (1.33,7) -- (699.33,7) -- (699.33,517.51) -- (1.33,517.51) -- cycle ;
\draw   (9,14) -- (215.33,14) -- (215.33,30.51) -- (9,30.51) -- cycle ;

\draw (218,40.4) node [anchor=north west][inner sep=0.75pt]  [font = \scriptsize]  {$( \bm{k}_{1} ,\ \bm{v}_{1} ,\ \bm{q}_{1}) ,\ \dotsc \dotsc \dotsc \dotsc ,\ ( \bm{k}_{\frac{N}{3}-1} ,\ \bm{v}_{\frac{N}{3}-1} ,\ \bm{q}_{\frac{N}{3}-1})$};
\draw (260,158.4) node [anchor=north west][inner sep=0.75pt]    {$\dotsc $};
\draw (250,126.4) node [anchor=north west][inner sep=0.75pt]  [font=\tiny]  {$k=0$};
\draw (642.63,163.4) node [anchor=north west][inner sep=0.75pt]    {$\dotsc $};
\draw (623.63,128.4) node [anchor=north west][inner sep=0.75pt]  [font=\tiny]  {$k=\log \frac{N}{3}$};
\draw (643.63,138.4) node [anchor=north west][inner sep=0.75pt]  [font=\tiny]  {$-1$};
\draw (300,146.4) node [anchor=north west][inner sep=0.75pt]    {$\dotsc \dotsc \dotsc \dotsc $};
\draw (453.72,134.88) node [anchor=north west][inner sep=0.75pt]  [font=\tiny] [align=left] {\tt sort};
\draw (419,145.09) node [anchor=north west][inner sep=0.75pt]  [font=\tiny]  {$(\{\bm{q}_{\mathrm{dec}( j)}\}_{j\in J_{k}^{x}})$};
\draw (556.35,132.88) node [anchor=north west][inner sep=0.75pt]  [font=\tiny] [align=left] {\tt sort};
\draw (526.33,148.64) node [anchor=north west][inner sep=0.75pt]  [font=\tiny]  {$(\{\bm{k}_{\mathrm{dec}( i)}\}_{i\in I_{k}^{x}})$};
\draw (250.37,226.4) node [anchor=north west][inner sep=0.75pt]  [font=\tiny]  {$k=0$};
\draw (185.72,135.88) node [anchor=north west][inner sep=0.75pt]  [font=\tiny] [align=left] {\tt sort};
\draw (150,145.09) node [anchor=north west][inner sep=0.75pt]  [font=\tiny]  {$(\{\bm{k}_{\mathrm{dec}( i)}\}_{i\in I_{0}^{x}})$};
\draw (210,288.4) node [anchor=north west][inner sep=0.75pt]    {$\dotsc $};
\draw (178,273.4) node [anchor=north west][inner sep=0.75pt]  [font=\tiny]  {$C_{0}$};
\draw (143.46,334.75) node [anchor=north west][inner sep=0.75pt]  [font=\scriptsize]  {$C_{0}[J_{\mathrm{permuted}}^{k\bm{x}}(\mathrm{dec}( i))] \neq 2^{0}$};

\draw (170,363.4) node [anchor=north west][inner sep=0.75pt]  [font=\tiny]  {$\bm{c}^{0}_{\mathrm{dec}( i)}$};

\draw (303,325.51) node [anchor=north west][inner sep=0.75pt]  [font=\tiny]  {$i \in I^{\bm{x}}_0$};

\draw (350,335.51) node [anchor=north west][inner sep=0.75pt]  [font=\tiny]  {$k \in \{0, \ldots, \log{\frac{N}{3}}-1\}$};

\draw (350,350.51) node [anchor=north west][inner sep=0.75pt]    {$\dotsc \dotsc \dotsc \dotsc $};

\draw (157,425.74) node [anchor=north west][inner sep=0.75pt]    {$\times $};
\draw (181,479.4) node [anchor=north west][inner sep=0.75pt]    {$\dotsc \dotsc \dotsc \dotsc $};
\draw (60.79,389.64) node [anchor=north west][inner sep=0.75pt]    {$\dotsc $};
\draw (37,360.4) node [anchor=north west][inner sep=0.75pt]  [font=\tiny]  {$I_{\mathrm{permuted}}^{k\bm{x}}\paren{\bm{c}^{0}_{\mathrm{dec}( i)}}$};
\draw (161,486.74) node [anchor=north west][inner sep=0.75pt]  [font=\tiny]  {$\bm{v}_{j}$};
\draw (185,419.4) node [anchor=north west][inner sep=0.75pt]    {$\dotsc \dotsc \dotsc \dotsc $};
\draw (125,16.4) node [anchor=north west][inner sep=0.75pt]  [font=\tiny]  {$(\hyenaInput[ 0\cdots N-1])$};
\draw (11,17) node [anchor=north west][inner sep=0.75pt]  [font=\tiny] [align=left] {\tt Parallel-General-AR};
\draw (67.82,135.53) node [anchor=north west][inner sep=0.75pt]  [font=\tiny] [align=left] {\tt sort};
\draw (33.79,144.6) node [anchor=north west][inner sep=0.75pt]  [font=\tiny]  {$(\{\bm{q}_{\mathrm{dec}( j)}\}_{j\in J_{0}^{x}})$};
\draw (304,262.4) node [anchor=north west][inner sep=0.75pt]    {$\dotsc \dotsc \dotsc \dotsc $};
\draw (659.86,260.66) node [anchor=north west][inner sep=0.75pt]    {$\dotsc $};
\draw (633.51,227.46) node [anchor=north west][inner sep=0.75pt]  [font=\tiny]  {$k=\log{\frac{N}{3}}$};
\draw (655.51,237.46) node [anchor=north west][inner sep=0.75pt]  [font=\tiny]  {$-1$};
\draw (481.85,232.75) node [anchor=north west][inner sep=0.75pt]  [font=\tiny] [align=left] {\tt pbs-key-values};
\draw (405.97,247.65) node [anchor=north west][inner sep=0.75pt]  [font=\tiny]  {$(\mathrm{\tt sort}\{\bm{q}_{\mathrm{dec}( j)}\}_{j\in J_{k}^{x}} ,\mathrm{\tt sort}\{\bm{k}_{\mathrm{dec}( i)}\}_{i\in I_{k}^{x}} \ )$};
\draw (82.61,233.05) node [anchor=north west][inner sep=0.75pt]  [font=\tiny] [align=left] {\tt pbs-key-values};
\draw (20.16,246.73) node [anchor=north west][inner sep=0.75pt]  [font=\tiny]  {$(\mathrm{\tt sort}\{\bm{q}_{\mathrm{dec}( j)}\}_{j\in J_{k}^{x}} ,\mathrm{\tt sort}\{\bm{k}_{\mathrm{dec}( i)}\}_{i\in I_{k}^{x}} \ )$};
\end{tikzpicture}
    \caption{\texttt{Parallel-MQAR}$(\hyenaInput[0 \cdots N-1])$ as a circuit that includes sorting networks and the circuit for \texttt{pbs-key-values} as subroutines.}
    \label{fig: pram-gen-ar-circuit}
\end{figure}
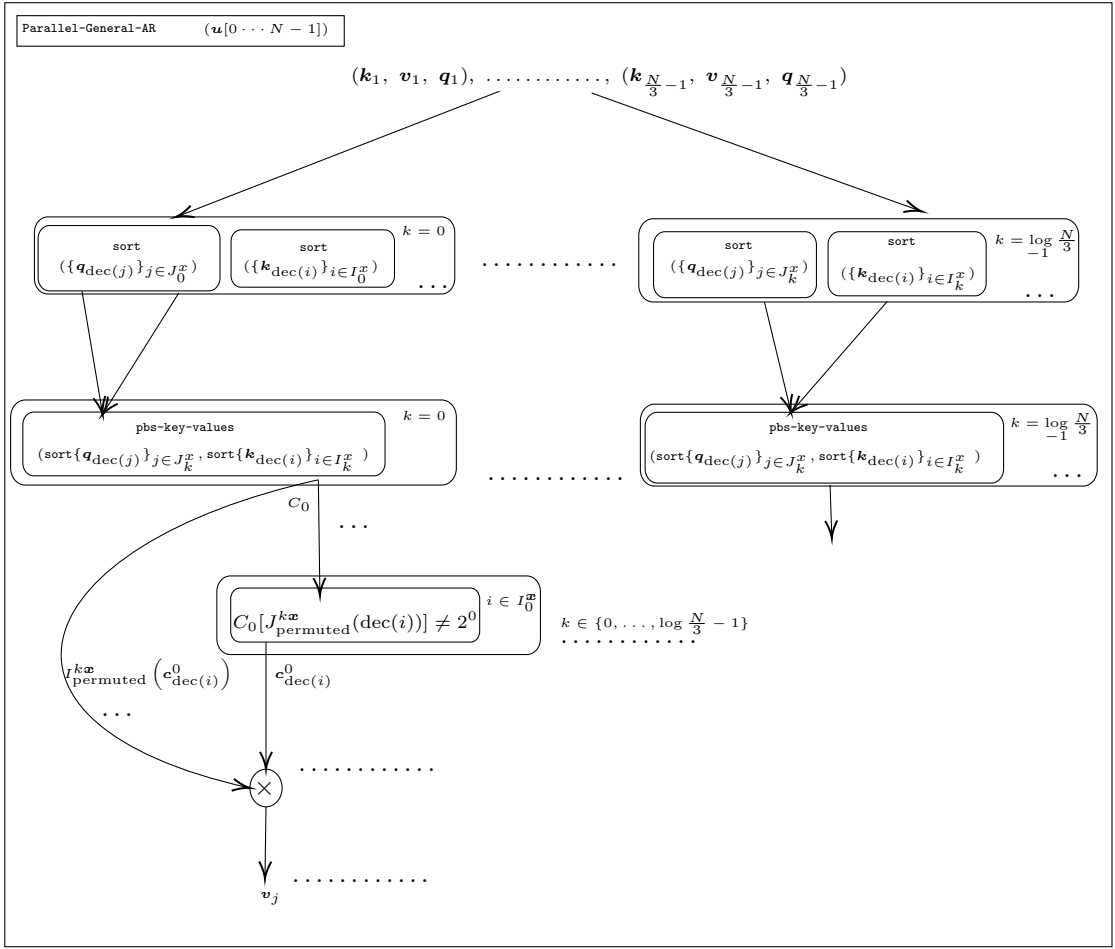

We now dilineate the complexity of the circuit, starting with the complexity of the sorting networks.
\begin{lemma}[\cite{ajtai19830}]
    \label{lem: sorting-networks}
    Let $A$ be an array with $d$-bit entries of size $n$. Then, one can implement a sorting network to sort the array $A$ with size $\calO(d \cdot n\log{n})$ and depth $\calO(\log{d}\log{n})$.
\end{lemma}
\begin{proposition}\label{prop: pram-circuit}
There exists an {\em $(\inputLength\cdot d, \calO(Nd\cdot \log^2{N}), \calO(\log{d}\log^2{N}), \calO(Nd\log{N}))$-arithmetic circuit} that solves the multiple-query associative recall problem.
\end{proposition}
\begin{proof}
We note here that for each $k$, there are $\frac{\frac{N}{3}}{2^{k+1}}$ parallel problems of size $2^k$ for both the sorting networks and the ${\tt pbs-key-values}$ circuit. Using \cref{lem: sorting-networks}, the cumulative size of these sorting networks is $\calO(d \cdot N\log^2{N})$ (see \eqref{eq: work-complexity}) with overall depth $\calO(\log{d}\log{N})$.

Similarly, the next layer again runs $\sum_{k=0}^{\log{\frac{N}{3}}-1}\frac{\frac{N}{3}}{2^{k+1}}$-many circuits for ${\tt pbs-key-values}$ each of which has size $\calO(2^kd(\log{2^k} + \log{2^k})) = O(d\cdot 2^k\log{2^k})$, depth $\calO(\log^2{2^k}\log{d})$ and width $\calO(2^k)$ (\cref{prop: pbs-circuit}). Again, the cumulative size of this layer is given by $\calO(Nd\cdot\log^2{N})$~(see \eqref{eq: work-complexity}).
Since we run each of these circuits in parallel, the depth of this layer is again $\calO(\log{d}\log^2(N))$ while the width is $\calO(N \cdot \log{N})$. 

Finally, we perform $\frac{N}{3}\log{\frac{N}{3}}$ comparisons at the end of $d$-bit strings in parallel which results in size $\calO(N\log{N}\cdot d)$, depth $\calO(\log{d})$ and width $\calO(N\log{N}\cdot d)$~(\cref{lem: comparators}). Therefore, the resulting arithmetic circuit has size $\calO(d \cdot N\log^2{N} + Nd \cdot \log^2{N} + N\log{N}\cdot d) = O(Nd\log^2{N})$, depth $\calO(\log{d}\log^2{N})$ and width $\calO(Nd\log{N})$.
\end{proof}

%% file: Sections/appendix/general_ar/ac-to-coyote_iclr.tex
\subsubsection{The Resulting \Coyote{} Model}
\label{sec:general-AR-coyote}
As we have an arithemtic circuit for solving the multiple-query associative recall problem, we can now invoke \cref{thm: gen-ac} to claim that there is a corresponding \Coyote{} model that solves the multiple-query associative recall problem with $\tilde{\calO}(N\log{c})$ parameters and $\tilde{\calO(1)}$ layers.
\begin{theorem}
    \label{thm: gen-ar-coyote}
     There exists a $\coyoteTuple{\inputLength d}{\tilde{\calO}(1)}{\tilde{\calO}(1)}{\tilde{\calO}(N d)}{\tilde{\calO}(1)}$ solves the multiple-query associative recall problem.
\end{theorem}
\begin{proof}
    Directly applying \cref{thm: gen-ac} yields a \Coyote\ model with the number of layers $\calO(\log{d} \cdot \log^2{N} \cdot \log{Nd\log{N}}) = \calO(1)$ layers while the claim on the input and inner dimensions follow trivially.
\end{proof}

%% file: Sections/appendix/general_ar/data-dependency/setup.tex
 \subsection{Data-Dependent Convolutions}
 \label{sec: data-dep-ar}
\subsubsection{Introduction}
In this section, we are again concerned with solving the multiple-query associative recall problem (\cref{sec: intro-general-ar}). However, in contrast to \cref{sec: pram-algo}, which yields a circuit that is unchanged and works for {\em all} inputs, we instead take the viewpoint of adapting the sequence mixing weights (\cref{sec:ar-background}) with respect to the particular sequence that the model receives as input. More specifically, we take the distance between the tokens in the sequence as a measure for designing data-dependent convolutions. 
\paragraph{Setup.}
To formally setup the problem, as in our discussion of designing a parallel algorithm, we consider the following problem description of the multiple-query associative recall problem. 
\begin{displayquote}
    Suppose we are given an input
$
\bm{u}[0 \cdots N-1] \triangleq \left\{\paren{\bm{k}_0, \bm{v}_0, \bm{q}_0}, \ldots, \paren{\bm{k}_{\frac{N}{3}-1}, \bm{v}_{\frac{N}{3}-1}, \bm{q}_{\frac{N}{3}-1}}\right\}
$
with each $\bm{k}_i, \bm{v}_i, \bm{q}_i \in C$. Here, each token is embedded using the standard one-hot encoding in $\{0,1\}^c$ (i.e. we assume $d=c$).\footnote{Our arguments do need $c=d$. However we do not need $d=N$ but we made this simplification for ease of presentation.} Our goal is again to check, for each $1 \le i \le \frac{N}{3}-1$, whether there exists $0 \le j < i$ such that $\bm{q}_i \equiv \bm{k}_j$, and if so, output $\bm{v}_{j}$. 

Here, we define the interaction distance between the $i$th query $\bm{q}_i$ and the matching key $\bm{k}_j$ as $i-j$. We then also assume that {\em number of distinct interaction distances} is bounded by $t$. 
\end{displayquote}
We can identify these distances using an autocorrelation~\citep{chatfield1995the}, which has an elegant underlying formulation. We will briefly introduce the relevant mathematical machinery in the context of elucidating the data-dependent model that we seek to develop in the sequel. 

\paragraph{Autocorrelations}
We introduce autocorrelation convolutions. Let $\tilde{\bm{u}}[t]:= \bm{u}[-t]$, then the cross correlation of two vectors $\bm{u}$ and $\bm{v}$ is given by 
\[
\bm{u} \star \bm{v} \triangleq \tilde{\bm{u}} \ast \bm{v}.
\]
The {\em autocorrelation} of a vector $\bm{u} \in \R^{n}$ is the cross correlation of $\bm{u}$ with itself. Moreover, in terms of polynomials, we have
$\tilde{\bm{u}}(X) = X^{n-1}\cdot \bm{u}(1/X)$. Thus, in analogy with our interpretation of convolution in terms of polynomial multiplication, we characterize the autocorrelation of a vector $\bm{u} \in \R^n$, given by $\tbf{w} \in \R^n$ as follows:
\begin{equation*}
    \label{eq: auto-cor}
    \tbf{w} = \coeff\paren{\bm{u}(X) \cdot \tilde{\bm{u}}(X) \mod{X^{n}-1}}.
\end{equation*}

\subsubsection{\Coyote\ with kernels generated using Auto-Correlation}
\label{app-theory-autocorrelation}
We are now ready to describe the model that solves the multiple-query associative recall problem using data-dependent kernels derived using auto-correlations. 
\paragraph{The Input-Dependent Kernels.}
There are several potential strategies to identify the best token-interaction distances for an input using a convolution. Here we will focus on an autocorrelation convolution for exposition. Autocorrelation allows us to identify the top $t$ distinct shifts of the sequence that result in highly overlapping values (e.g. matches between query and keys in our associative recall setting). We can then construct convolution filters that perform each of these $t$ shifts. 
That is, we define a function ${\tt Top}$ such that ${\tt Top}\paren{u \star u, t}$ returns a list of the top $t$ shifts $\{s_{\ell}\}_{\ell \in [t]}$. We then use these top $t$ distances $\{s_{\ell}\}_{\ell \in [t]}$ to define the following two kernels:
\begin{equation}
\label{eq: data-dependent-kernels}
\begin{aligned}
\tbf{h}^{k}(X) &\equiv \sum_{\ell \in [t]} X^{s_{\ell} + (\ell -1) \cdot \frac{N}{3}},\\
\tbf{h}^{v}(X) &\equiv \sum_{\ell \in [t]} X^{s_{\ell} - 1 + (\ell -1)\cdot \frac{N}{3}}.\\
\end{aligned}
\end{equation}
Here, we note that we only only have two kernels as we will assume $N'=t\frac{N}{3}$ and the shift will be done in "parallel." Obviously, one can instead define $t$ distinct shift kernels but then there is a cost of $\calO(t)$ in the number of layers.

\paragraph{Projections.}
We define the following projections $\tbf{K}, \tbf{Q}, \tbf{V} \in \{0,1\}^{N\times d}$ of the input that we shall use below using \eqref{eq: kqv-indices}.
\begin{equation}
\label{eq: kqv-projections}
\begin{aligned}
\tbf{K}[i,:] &:= 
\begin{cases}
\bm{u}[i,:] &\text{if }i \in \calK,\\
\tbf{0}^d &\text{otherwise}
\end{cases},\\
\tbf{Q}[i,:] &:= \begin{cases}
\bm{u}[i,:] &\text{if }i \in \calQ,\\
\tbf{0}^d &\text{otherwise}
\end{cases},\\
\tbf{V}[i,:] &:= \begin{cases}
\bm{u}[i,:] &\text{if }i \in \calV,\\
\tbf{0}^d &\text{otherwise}
\end{cases},\\
\end{aligned}
\end{equation}

Finally, we present the \Coyote{} model that solves the multiple-query associative recall problem with input-dependent kernels using $\calO(1)$ layer and $\calO(t \cdot Nd)$-many parameters.
\begin{theorem}\label{thm: input-dep-genar}
    There exists a $\Coyote$ model with gated data-dependent convolutions that solves the multiple-query associative recall problem on inputs from $\{0,1\}^{3N\times c}$ with the total number of distinct interaction distances bounded by $t$ in $\calO(1)$ layers and $\calO(t\cdot Nc)$ total parameters.
\end{theorem}
\begin{proof}
    We note that we have the input dimension $d = c$. Now, for any input sequence $\bm{u} \in \{0,1\}^{N\times d}$, we get the input-dependent kernels as in \eqref{eq: data-dependent-kernels} using autocorrelation of the input. We will now outline the following computations for the \Coyote{} layers:
    \begin{align}
        \bm{y}  &= {\tt Linear}_{\tbf{Q}}(\bm{u}) \odot \paren{\tbf{h}^{K} * {\tt Linear}_{\tbf{K}}(\bm{u})}\nonumber \\
                &= \tbf{Q} \odot \paren{\tbf{h}^{K} * \tbf{K}} \label{eq: y-data-ind}\\
        \bm{z}  &= {\tt Linear}_{\tbf{E}}(\bm{y}) \odot \paren{\tbf{h}^{V} * {\tt Linear}_{\tbf{V}}(\bm{u})}\nonumber \\
                &= \tbf{E} \odot \paren{\tbf{h}^{V} * \tbf{V}},  \label{eq: z-data-ind}   
    \end{align}
    where we have the linear projections ${\tt Linear}_{\tbf{Q}}(\bm{u}) = \tbf{Q}$, ${\tt Linear}_{\tbf{K}}(\bm{u}) = \tbf{K}$, ${\tt Linear}_{\tbf{V}}(\bm{u}) = \tbf{V}$ and ${\tt Linear}_{\tbf{E}}(\bm{y}) = \tbf{E}$ defined as 
    \begin{equation}
    \label{eq: E-def}
        \tbf{E}[i,:] := {\tt Linear}_{\tbf{E}}(\bm{y})[i,:] = 
        \begin{cases}
            \tbf{1}^d &\text{if }\exists\ j \in [d]\text{ such that }\bm{y}[i,j] = 1\\
            \tbf{0}^{d} &\text{otherwise}
        \end{cases}.
    \end{equation}
    Here, we will first present the argument for the special case when we have $t = 1$ as that will help us elucidate the general case. To this end, as the kernels from \eqref{eq: data-dependent-kernels} for $t=1$ are given by
    \begin{equation}
    \label{eq: data-dependent-kernels-t1}
    \begin{aligned}
    \tbf{h}^{k}(X) &\equiv X^{s_{1}};\\
    \tbf{h}^{v}(X) &\equiv X^{s_{1} - 1},
    \end{aligned}
    \end{equation}
    we observe that convolving with these kernels $\tbf{h} \ast \bm{y}$ is equivalent to operating with the following primitives (\cref{sec: primitives}):
    \begin{equation}
    \label{eq: data-dependent-kernels-prim1}
    \begin{aligned}
    &{\tt shift\_down}(\bm{y}, s_{1});\\
    &{\tt shift\_down}(\bm{y}, s_{1}-1).
    \end{aligned}
    \end{equation}
    We note that we shift down instead of shifting up as the index of the top-left entry is $(0,0)$. We can then write down the computations performed in \eqref{eq: y-data-ind} and \eqref{eq: z-data-ind} as follows:
    \begin{align}
        \bm{y}  &= \tbf{Q} \odot {\tt shift\_down}(\tbf{K}, s_{1}) \label{eq: y-data-ind-t1}\\
        \bm{z}  &= \tbf{E} \odot {\tt shift\_down}(\tbf{V}, s_{1}-1),  \label{eq: z-data-ind-t1}   
    \end{align}
    We begin by examining $\bm{y}$ below:
    \begin{align}
        \bm{y}[i,:] 
        &= \paren{\tbf{Q} \odot {\tt shift\_down}(\tbf{K}, s_1)}[i,:]\nonumber \\
        &= \tbf{Q}[i,:] \odot {\tt shift\_down}(\tbf{K}, s_1)[i,:] \label{eq: gated-row-independency}  \\
        &= \paren{
        \begin{cases}
            \bm{u}[i,:] &\text{if }i \in \calQ\\
            \tbf{0}^d &\text{otherwise}
        \end{cases}
        } \odot \paren{
        \begin{cases}
            \bm{u}[i-s_1,:] &\text{if }i - s_1 \in \calK,\\
            \tbf{0}^d &\text{otherwise}
        \end{cases}
        } \label{eq: key-indices-shift}\\
        &= \begin{cases}
            \bm{u}[i,:] \odot \bm{u}[i-s_1,:] &\text{if }i \in \calQ \text{ and }i-s_1 \in \calK,\\
            \tbf{0}^d &\text{otherwise}
        \end{cases} \nonumber
   \end{align}
    Here, we use the fact that the Hadamard product is row-independent in \eqref{eq: gated-row-independency}, and the definitions of the projections from \eqref{eq: kqv-projections} in \eqref{eq: key-indices-shift}. Examining the $j$th entry, we get
    \begin{align*}
        \bm{u}[i,j] \odot \bm{u}[i-s_1,j] &= 
        \begin{cases}
            1 & \text{if }i \in \calQ, i-s_1 \in \calK\text{ and }\bm{q}_{i} \equiv \bm{k}_{i-s_1} \equiv \tbf{e}_j\\
            0 & \text{otherwise.}
        \end{cases} 
   \end{align*}
   That is, we can express
   \begin{align}
        \bm{y}[i,:] 
        &= \begin{cases}
            \tbf{e}_j &\text{if }i \in \calQ, i-s_1 \in \calK\text{ and }\bm{q}_{i} \equiv \bm{k}_{i-s_1} \equiv \tbf{e}_j\\
            \tbf{0}^d &\text{otherwise}
        \end{cases} \label{eq: y-t1}.
   \end{align}
   Consequently, as per the definition in \eqref{eq: E-def}, we get
   \begin{align}
       \tbf{E}[i,:] = 
        \begin{cases}
            \tbf{1}^d &\text{if } i \in \calQ, i-s_1 \in \calK\text{ and }\bm{q}_{i} \equiv \bm{k}_{i-s_1} \\
            \tbf{0}^{d} &\text{otherwise}
        \end{cases} \label{eq: E-t1}
   \end{align}
   We can now finally specify the output $\bm{z}$ from \eqref{eq: z-data-ind-t1} as follows:
   \begin{align}
        \bm{z}[i,:] 
        &= \paren{\tbf{E} \odot {\tt shift\_down}(\tbf{V}, s_1-1)}[i,:]\nonumber \\
        &= \tbf{E}[i,:] \odot {\tt shift\_down}(\tbf{V}, s_1-1)[i,:] \label{eq: gated-row-independency-2}  \\
        &= \paren{
        \begin{cases}
            \tbf{1}^d &\text{if } i \in \calQ, i-s_1 \in \calK\text{ and }\bm{q}_{i} \equiv \bm{k}_{i-s_1} \\
            \tbf{0}^{d} &\text{otherwise}
        \end{cases}
        } \odot \paren{
        \begin{cases}
            \bm{u}[i-s_1+1,:] &\text{if }i - s_1 + 1 \in \calV,\\
            \tbf{0}^d &\text{otherwise}
        \end{cases}
        } \label{eq: value-indices-shift}\\
        &= \begin{cases}
            \bm{u}[i-s_1+1,:] &\text{if }i \in \calQ, i-s_1 \in \calK, i - s_1 + 1 \in \calV \text{ and }\bm{q}_{i} \equiv \bm{k}_{i-s_1}\\
            \tbf{0}^d &\text{otherwise}
        \end{cases} \nonumber\\
        &= \begin{cases}
            \bm{v}_{i-s_1} &\text{if }\bm{q}_{i} \equiv \bm{k}_{i-s_1}\\
            \tbf{0}^d &\text{otherwise}
        \end{cases} \label{eq: output-t1}
   \end{align}
   Again, we use the fact that the Hadamard product is row-independent in \eqref{eq: gated-row-independency-2}, and the definitions of the projections from \eqref{eq: kqv-projections} in \eqref{eq: value-indices-shift}. Overall, we have solved associative recall for all queries that have interaction distance exactly equal to $s_1$. 

   In order to generalize this to arbitrary $t \le \frac{N}{3}$, we first increase the internal dimension so that the input to the kernels $\bm{u}' \in \R^{(N\cdot t)\times d}$ in \eqref{eq: data-dependent-kernels} and the projections $\tbf{K',Q',V'} \in \R^{(N\cdot t)\times d}$ are given by 
   \[
   \bm{u}' \equiv 
   \begin{pmatrix}
                \tbf{0}^d\\
                \hline \\
                \vdots\\
                \hline\\
                \tbf{0}^d\\
                \hline\\
                 \bm{u}
    \end{pmatrix},
    \bm{K}' \equiv
   \begin{pmatrix}
                \tbf{0}^d\\
                \hline \\
                \vdots\\
                \hline\\
                \tbf{0}^d\\
                \hline\\
                 \tbf{K}
    \end{pmatrix}, 
    \tbf{Q}' \equiv 
   \begin{pmatrix}
               \tbf{Q}\\
                \hline \\
                \vdots\\
                \hline\\
                \tbf{Q}\\
                \hline\\
                 \tbf{Q}
    \end{pmatrix}, \tbf{V}' \equiv 
   \begin{pmatrix}
                \tbf{0}^d\\
                \hline \\
                \vdots\\
                \hline\\
                \tbf{0}^d\\
                \hline\\
                 \tbf{V}
    \end{pmatrix}, 
   \]
   We then observe that for $\tbf{h}^k_{\ell}(X) := X^{s_{\ell}}$ and $\tbf{h}^v_{\ell}(X) := X^{s_{\ell} - 1}$, we have
   \[
   \begin{aligned}
       \tbf{h}^{k}(X) &\equiv \sum_{\ell \in [t]} \tbf{h}^k_{\ell}(X)\cdot X^{(\ell -1) \cdot \frac{N}{3}},\\
        \tbf{h}^{v}(X) &\equiv \sum_{\ell \in [t]} \tbf{h}^k_{\ell}(X)\cdot X^{(\ell -1)\cdot \frac{N}{3}}.
   \end{aligned}
   \]
   In analogy with \eqref{eq: data-dependent-kernels-prim1}, we can then equivalently write
   \begin{align}
       \paren{\tbf{h}^{K} \ast \tbf{K}} 
       &\equiv 
       \begin{pmatrix}
                 \tbf{h}^{K}_{t}\\
                \hline \\
                \vdots\\
                \hline\\
                 \tbf{h}^{K}_{2}\\
                \hline\\
                 \tbf{h}^{K}_{1}
        \end{pmatrix}
        \ast
        \begin{pmatrix}
                \tbf{0}^d\\
                \hline \\
                \vdots\\
                \hline\\
                \tbf{0}^d\\
                \hline\\
                 \tbf{K}
    \end{pmatrix} \nonumber \\
    &\equiv \begin{pmatrix}
                 \tbf{h}^{K}_{t} \ast \tbf{K}\\
                \hline \\
                \vdots\\
                \hline\\
                 \tbf{h}^{K}_{2} \ast \tbf{K}\\
                \hline\\
                 \tbf{h}^{K}_{1} \ast \tbf{K}
    \end{pmatrix}
     \nonumber \\
     &\equiv \begin{pmatrix}
                {\tt shift\_down}(\tbf{K}, s_{t}) \\
                \hline \\
                \vdots\\
                \hline\\
                {\tt shift\_down}(\tbf{K}, s_{2})\\
                \hline\\
                {\tt shift\_down}(\tbf{K}, s_{1})
    \end{pmatrix}
     \nonumber.
   \end{align}
   Similarly, we also have
   \begin{align}
       \paren{\tbf{h}^{V} \ast \tbf{V}'} 
     &\equiv \begin{pmatrix}
                {\tt shift\_down}(\tbf{V}, s_{t}-1) \\
                \hline \\
                \vdots\\
                \hline\\
                {\tt shift\_down}(\tbf{V}, s_{2}-1)\\
                \hline\\
                {\tt shift\_down}(\tbf{V}, s_{1}-1)
    \end{pmatrix}
     \nonumber.
   \end{align}
   That is, the argument for $t=1$ now applies to each of the $t$ shifts as we now have ({\em cf.} \eqref{eq: y-data-ind}) 
   \begin{align*}
       \bm{y}' 
       &\equiv \tbf{Q}' \odot \paren{\tbf{h}^{V} \ast \tbf{V}}  \\
       &\equiv \begin{pmatrix}
               \tbf{Q}\\
                \hline \\
                \vdots\\
                \hline\\
                \tbf{Q}\\
                \hline\\
                 \tbf{Q}
    \end{pmatrix} \odot \begin{pmatrix}
                {\tt shift\_down}(\tbf{V}, s_{t}-1) \\
                \hline \\
                \vdots\\
                \hline\\
                {\tt shift\_down}(\tbf{V}, s_{2}-1)\\
                \hline\\
                {\tt shift\_down}(\tbf{V}, s_{1}-1)
    \end{pmatrix}  \\
    &\equiv \begin{pmatrix}
                \tbf{Q} \odot {\tt shift\_down}(\tbf{V}, s_{t}-1) \\
                \hline \\
                \vdots\\
                \hline\\
                \tbf{Q} \odot {\tt shift\_down}(\tbf{V}, s_{2}-1)\\
                \hline\\
                \tbf{Q} \odot {\tt shift\_down}(\tbf{V}, s_{1}-1)
    \end{pmatrix} \\
    &\equiv \begin{pmatrix}
                \bm{y}_t \\
                \hline \\
                \vdots\\
                \hline\\
                \bm{y}_2 \\
                \hline\\
                \bm{y}_1 \\
    \end{pmatrix},
   \end{align*}
   where, for each $\ell \in [t]$, we have ({\it cf.} \eqref{eq: y-t1})
   \[
   \bm{y}_{\ell}[i,:] \equiv \begin{cases}
            \tbf{e}_j &\text{if }i \in \calQ, i-s_{\ell} \in \calK\text{ and }\bm{q}_{i} \equiv \bm{k}_{i-s_{\ell}} \equiv \tbf{e}_j,\\
            \tbf{0}^d &\text{otherwise}
        \end{cases}.
   \]
   We then analogously get $\tbf{E}'$ as follows:
   \[
   \tbf{E}' \equiv {\tt Linear}_{\tbf{E}}(\bm{y}') \equiv
   \begin{pmatrix}
               {\tt Linear}_{\tbf{E}}(\bm{y}_t)\\
                \hline \\
                \vdots\\
                \hline\\
                {\tt Linear}_{\tbf{E}}(\bm{y}_2)\\
                \hline\\
                 {\tt Linear}_{\tbf{E}}(\bm{y}_1)
    \end{pmatrix}  \equiv  \begin{pmatrix}
               \tbf{E}_t\\
                \hline \\
                \vdots\\
                \hline\\
                \tbf{E}_2\\
                \hline\\
                 \tbf{E}_1
    \end{pmatrix}, 
   \]
   where, for each $\ell \in [t]$, we have ({\it cf.} \eqref{eq: E-t1})
   \[
   \tbf{E}_{\ell}[i,:] = 
        \begin{cases}
            \tbf{1}^d &\text{if } i \in \calQ, i-s_{\ell} \in \calK\text{ and }\bm{q}_{i} \equiv \bm{k}_{i-s_{\ell}} \\
            \tbf{0}^{d} &\text{otherwise}
        \end{cases}
   \]
   The output in the general case is then given by
   \begin{align}
       \bm{z}' 
       &\equiv \tbf{E}' \odot \paren{\tbf{h}^{V} * \tbf{V}'} \nonumber\\
       &\equiv \begin{pmatrix}
               \tbf{E}_t\\
                \hline \\
                \vdots\\
                \hline\\
                \tbf{E}_2\\
                \hline\\
                 \tbf{E}_1
    \end{pmatrix} \odot \begin{pmatrix}
                {\tt shift\_down}(\tbf{V}, s_{t}-1) \\
                \hline \\
                \vdots\\
                \hline\\
                {\tt shift\_down}(\tbf{V}, s_{2}-1)\\
                \hline\\
                {\tt shift\_down}(\tbf{V}, s_{1}-1)
    \end{pmatrix}  \nonumber \\
    &\equiv \begin{pmatrix}
                \tbf{E}_t \odot {\tt shift\_down}(\tbf{V}, s_{t}-1) \\
                \hline \\
                \vdots\\
                \hline\\
                 \tbf{E}_2 \odot {\tt shift\_down}(\tbf{V}, s_{2}-1)\\
                \hline\\
                 \tbf{E}_1 \odot {\tt shift\_down}(\tbf{V}, s_{1}-1)
    \end{pmatrix}  \nonumber \\
    &\equiv \begin{pmatrix}
                \bm{z}_t \\
                \hline \\
                \vdots\\
                \hline\\
                \bm{z}_2 \\
                \hline\\
                \bm{z}_1 \\
    \end{pmatrix}  \nonumber,
   \end{align}
   where, for each $\ell \in [t]$, we have ({\it cf.} \eqref{eq: output-t1})
   \[
   \bm{z}_{\ell}[i,:] \equiv 
        \begin{cases}
            \bm{v}_{i-s_{\ell}} &\text{if }\bm{q}_{i} \equiv \bm{k}_{i-s_{\ell}}\\
            \tbf{0}^d &\text{otherwise}
        \end{cases}
   \]
   Finally, we define the last output layer to compute 
   \(
   \bm{z}_{out}\equiv {\tt Linear}_{\mathrm{sum}}(\bm{z}') \equiv \sum_{\ell \in [t]} \bm{z}_{\ell}
   \)
   so that we have
   \[
   \bm{z}_{out}[i,:] \equiv 
        \begin{cases}
            \bm{v}_{i-s_{\ell}} &\text{if }\bm{q}_{i} \equiv \bm{k}_{i-s_{\ell}}\text{ for some }\ell \in [t]\\
            \tbf{0}^d &\text{otherwise}
        \end{cases}
   \]
To recall, we retrieved the top $t$ interaction distances of the input $\bm{u}$ using auto-correlation and defined the corresponding convolution kernels (\eqref{eq: data-dependent-kernels}). We then shifted djown the keys $\tbf{K}$ using the first kernel and gated with the corresponding queries $\tbf{Q}$ so that we got a match exactly when there exists a key that is at $s_{\ell}$ interaction distance from the corresponding query. After ``smearing'' this match to get $\tbf{E}$, we used it as a mask to retrieve the value in the next layer. Overall, since we have $t$ atomic kernels that perform $t$ shifts with each of these kernels using $\calO(Nd)$ parameters, we can conclude that the output solves the associative recall problem for all queries with exactly $\ell$ interaction distance from the corresponding keys for all $\ell \in [t]$ using $\calO(1)$ layers and $\calO(t \cdot Nc)$ parameters as we have $d = c$.
\end{proof}

%% file: Sections/appendix/hyperparameters/attention.tex
\begin{table}[h!]
    \caption{Attention Training Settings}
    \centering
    \begin{tabular}{rccc}
    \toprule
    {} & 73M & 125M & 360M  \\
    \midrule
    Optimizer & \multicolumn{3}{c}{Adam} \\
    Optimizer momentum & \multicolumn{3}{c}{$\beta_1, \beta_2=0.9, 0.95$} \\
    Optimizer eps & \multicolumn{3}{c}{$1e-8$} \\
    Precision &  \multicolumn{3}{c}{BFloat16} \\
    \midrule
    Learning rate decay & \multicolumn{3}{c}{Cosine} \\
    Learning rate (min, base) & \multicolumn{3}{c}{8e-5, 8e-4} \\
    Global batch size & \multicolumn{3}{c}{256} \\ 
    Training iterations & \multicolumn{3}{c}{20000} \\
    Warmup Duration (Linear) & \multicolumn{3}{c}{0.01} \\
    Weight decay & \multicolumn{3}{c}{0.1} \\
    \midrule
    Num Layers & 6 & 12 & 24 \\
    Hidden Size & 704  & 768 & 1024 \\
    FFN Width & \multicolumn{3}{c}{4} \\
    Position Embeddings & \multicolumn{3}{c}{Rotary} \\
    Weight Tying &  \multicolumn{3}{c}{True} \\
    \midrule
    Number of Heads ($H$) & 8 & 12 & 16 \\
    \bottomrule 
    \end{tabular}
    \label{tab:attn-training-details}
\end{table}

\begin{table}[h!]
    \caption{Attention FLOPs Computation}
    \centering
    \begin{tabular}{rc}
    \toprule
    {} & Equation  \\
    \midrule
    Input Layer & $B \times V \times N \times D$ \\
    \midrule
     Sequence Mixer $\mathbf{Q}, \mathbf{K}, \mathbf{V}$ Projections & $B \times N \times D \times D \times 3$ \\
     Sequence Mixer Attention & $B \times H \times H \times D + H \times N \times N + B \times N \times N \times D$ \\
     Sequence Mixer Output Projection & $B \times N \times D \times D$ \\
     \midrule
     Channel Mixer (FFN Width $4$) &  $B \times D \times D \times 8 \times \frac{2}{3} \times N$ \\
     \midrule
     Language Modeling Head & $B \times V \times N \times D$ \\
    \bottomrule 
    \end{tabular}
    \label{tab:attn-flops}
\end{table}

%% file: Sections/appendix/hyperparameters/hyena.tex
\begin{table}[h!]
    \caption{Hyena~\citep{poli2023hyena} Training Settings}
    \label{tab:hyena-training-details}
    \centering
    \begin{tabular}{rccc}
    \toprule
    {} & 72M & 158M & 358M  \\
    \midrule
    Optimizer & \multicolumn{3}{c}{Adam} \\
    Optimizer momentum & \multicolumn{3}{c}{$\beta_1, \beta_2=0.9, 0.98$} \\
    Precision &  \multicolumn{3}{c}{BFloat16} \\
    \midrule
    Learning rate decay & \multicolumn{3}{c}{Cosine} \\
    Learning rate (min, base) & \multicolumn{3}{c}{(1e-5, 8e-4)} \\
    Global batch size & \multicolumn{3}{c}{256} \\ 
    Training iterations & \multicolumn{3}{c}{20000} \\
    Warmup Duration (Linear) & \multicolumn{3}{c}{0.01} \\
    Weight decay & \multicolumn{3}{c}{0.1} \\
    \midrule
    Num Layers & 8 & 18 & 24 \\
    Hidden Size & 768  & 864 & 1024 \\
    FFN Width & \multicolumn{3}{c}{2} \\
    Position Embeddings & \multicolumn{3}{c}{None} \\
    Weight Tying &  \multicolumn{3}{c}{True} \\
    \midrule
    Short Conv. Filter Size & \multicolumn{3}{c}{3} \\
    Exp. Mod. Decay (Fast, Slow) & \multicolumn{3}{c}{0.3, 1.2} \\
    Filter Sine Freq. (w) & \multicolumn{3}{c}{14} \\
    Filter Order & \multicolumn{3}{c}{64} \\
    Filter Inner MLP & \multicolumn{3}{c}{2} \\
    Filter Weight Decay & \multicolumn{3}{c}{0} \\
    \bottomrule 
    \end{tabular}
\end{table}

\begin{table}[h!]
    \caption{Hyena FLOPs Computation}
    \label{tab:hyena-flops}
    \centering
    \begin{tabular}{rc}
    \toprule
    {} & Equation  \\
    \midrule
    Input Layer & $B \times V \times N \times D$ \\
    \midrule
     Sequence Mixer Input Projection & $B \times N \times D \times D \times 3 + B \times N \times 9 \times D$ \\
     Sequence Mixer Long Convolution & $10 \times N \times \log(N) \times D \times B$ \\
     Sequence Mixer Short Convolution & $3 \times B \times N \times D$ \\
     Sequence Mixer Implicit MLP (Order $64$) & $D \times 64$ \\
     Sequence Mixer Output Projection & $B \times N \times D \times D$ \\
     \midrule
     Channel Mixer (FFN Width $2$) &  $B \times D \times D \times 2 \times 2 \times N$ \\
     \midrule
     Language Modeling Head & $B \times V \times N \times D$ \\
    \bottomrule 
    \end{tabular}
\end{table}

%% file: Sections/appendix/hyperparameters/h3.tex
\begin{table}[h!]
    \caption{H3~\citep{fu2023simple} Training Settings}
    \label{tab:h3-training-details}
    \centering
    \begin{tabular}{rccc}
    \toprule
    {} & 72M & 168M & 357M  \\
    \midrule
    Optimizer & \multicolumn{3}{c}{Adam} \\
    Optimizer momentum & \multicolumn{3}{c}{$\beta_1, \beta_2=0.9, 0.95$} \\
    Precision &  \multicolumn{3}{c}{BFloat16} \\
    \midrule
    Learning rate decay & \multicolumn{3}{c}{Cosine} \\
    Learning rate (min, base) & \multicolumn{3}{c}{(8e-5, 8e-4)} \\
    Global batch size & \multicolumn{3}{c}{256} \\ 
    Training iterations & \multicolumn{3}{c}{20000} \\
    Warmup Duration (Linear) & \multicolumn{3}{c}{0.01} \\
    Weight decay & \multicolumn{3}{c}{0.1} \\
    Weight Tying &  \multicolumn{3}{c}{True} \\
    \midrule
    Num Layers & 8 & 19 & 33 \\
    Hidden Size & 720  & 864 & 1024 \\
    FFN Width & \multicolumn{3}{c}{2} \\
    Position Embeddings & \multicolumn{3}{c}{None} \\
    \midrule
    State Space Model State Size & \multicolumn{3}{c}{64} \\
    \bottomrule 
    \end{tabular}
\end{table}

%% file: Sections/appendix/hyperparameters/rwkv.tex
\begin{table}[h!]
    \caption{RWKV~\cite{peng2023rwkv} Training Settings}
    \centering
    \begin{tabular}{rccc}
    \toprule
    {} & 72M & 169M & 351M  \\
    \midrule
    Optimizer & \multicolumn{3}{c}{Adam} \\
    Optimizer momentum & \multicolumn{3}{c}{$\beta_1, \beta_2=0.9, 0.999$} \\
    Optimizer eps & \multicolumn{3}{c}{$1e-8$} \\
    Precision &  \multicolumn{3}{c}{BFloat16} \\
    \midrule
    Learning rate decay & \multicolumn{3}{c}{Cosine} \\
    Learning rate (min, base) & \multicolumn{3}{c}{1e-5, 8e-4} \\
    Global batch size & \multicolumn{3}{c}{256} \\ 
    Training iterations & \multicolumn{3}{c}{20000} \\
    Warmup Duration (Linear) & \multicolumn{3}{c}{0.01} \\
    Weight decay & \multicolumn{3}{c}{0.1} \\
    Weight Tying &  \multicolumn{3}{c}{False} \\
    \midrule
    Num Layers & 6 & 12 & 20 \\
    Hidden Size & 624 & 768 & 984 \\
    Position Embeddings & \multicolumn{3}{c}{None} \\
    Initialization & \multicolumn{3}{c}{From Reference Impl.} \\
    \bottomrule 
    \end{tabular}
    \label{tab:rwkv-training-details}
\end{table}

%% file: Sections/appendix/hyperparameters/retnet.tex
\begin{table}[h!]
    \caption{RetNet~\citep{sun2023retentive} Training Settings}
    \centering
    \begin{tabular}{rc}
    \toprule
    {} & 152M  \\
    \midrule
    Optimizer & \multicolumn{1}{c}{Adam} \\
    Optimizer momentum & \multicolumn{1}{c}{$\beta_1, \beta_2=0.9, 0.98$} \\
    Optimizer eps & \multicolumn{1}{c}{1.0e-8} \\
    Precision &  \multicolumn{1}{c}{BFloat16} \\
    \midrule
    Learning rate decay & \multicolumn{1}{c}{Linear} \\
    Learning rate (min, base) & \multicolumn{1}{c}{(1e-5, 8e-4)} \\
    Global batch size & \multicolumn{1}{c}{256} \\ 
    Training iterations & \multicolumn{1}{c}{20000} \\
    Warmup Duration (Linear) & \multicolumn{1}{c}{0.01} \\
    Weight decay & \multicolumn{1}{c}{0.05} \\
    Normalization & \multicolumn{1}{c}{Layernorm} \\
    Weight Tying &  \multicolumn{1}{c}{True} \\
    \midrule
    Num Layers & 12 \\
    Hidden Size & 768 \\
    Value Hidden Size & 1280 \\
    Window Size & 128 \\
    Position Embeddings & \multicolumn{1}{c}{xPos (relative position embeddings)} \\
    Initialization & DeepNet \\
    \bottomrule 
    \end{tabular}
    \label{tab:renet-training-details}
\end{table}

%% file: Sections/appendix/hyperparameters/longconv.tex
\begin{table}[h!]
    \caption{Simple Long Convolution~\cite{fu2023simple} Training Settings}
    \label{tab:lc-training-details}
    \centering
    \begin{tabular}{rccc}
    \toprule
    {} & 76M & 128M & 360M  \\
    \midrule
    Optimizer & \multicolumn{3}{c}{Adam} \\
    Optimizer momentum & \multicolumn{3}{c}{$\beta_1, \beta_2=0.9, 0.95$} \\
    Precision &  \multicolumn{3}{c}{BFloat16} \\
    \midrule
    Learning rate decay & \multicolumn{3}{c}{Linear} \\
    Learning rate (min, base) & \multicolumn{3}{c}{8e-5, 8e-4} \\
    Global batch size & \multicolumn{3}{c}{256} \\ 
    Training iterations & \multicolumn{3}{c}{20000} \\
    Warmup Duration (Linear) & \multicolumn{3}{c}{0.01} \\
    Weight decay & \multicolumn{3}{c}{0.1} \\
    \midrule
    Num Layers & 6 & 12 & 24 \\
    Hidden Size & 704  & 864 & 1024 \\
    FFN Width & \multicolumn{3}{c}{4} \\
    Position Embeddings & \multicolumn{3}{c}{-} \\
    Weight Tying &  \multicolumn{3}{c}{True} \\
    \midrule
    Channels & \multicolumn{3}{c}{1} \\
    Lam & \multicolumn{3}{c}{0.001} \\
    Kernel Dropout & \multicolumn{3}{c}{0.1} \\
    Kernel LR & \multicolumn{3}{c}{$5e-5$} \\
    Activation & \multicolumn{3}{c}{GeLU} \\
    Exponential Modulation & \multicolumn{3}{c}{True} \\
    \bottomrule 
    \end{tabular}
\end{table}

\begin{table}[h!]
    \caption{Simple Long Convolution FLOPs Computation}
    \label{tab:lc-flops}
    \centering
    \begin{tabular}{rc}
    \toprule
    {} & Equation  \\
    \midrule
    Input Layer & $B \times V \times N \times D$ \\
    \midrule
     Sequence Mixer Long Convolution & $10 \times N \times \log(N) \times D \times B$ \\
     Sequence Mixer Output Projection & $B \times N \times D \times D$ \\
     \midrule
     Channel Mixer (FFN Width $4$) &  $B \times D \times D \times 8 \times \frac{2}{3} \times N$ \\
     \midrule
     Language Modeling Head & $B \times V \times N \times D$ \\
    \bottomrule 
    \end{tabular}
\end{table}

%% file: Sections/appendix/hyperparameters/coyote.tex
\begin{table}[h!]
    \caption{{\Coyote} Training Settings}
    \label{tab:coyote-training-details}
    \centering
    \begin{tabular}{rcc}
    \toprule
    {} & 168M & 354M  \\
    \midrule
    Optimizer & \multicolumn{2}{c}{Adam} \\
    Optimizer momentum & \multicolumn{2}{c}{$\beta_1, \beta_2=0.9, 0.95$} \\
    Precision &  \multicolumn{2}{c}{BFloat16} \\
    \midrule
    Learning rate decay & \multicolumn{2}{c}{Cosine} \\
    Learning rate (min, base) & \multicolumn{2}{c}{8e-5, 8e-4} \\
    Global batch size & \multicolumn{2}{c}{256} \\ 
    Training iterations & \multicolumn{2}{c}{20000} \\
    Warmup Duration (Linear) & \multicolumn{2}{c}{0.01} \\
    Weight decay & \multicolumn{2}{c}{0.1} \\
    \midrule
    Num Layers & 30 & 48 \\
    Hidden Size   & 852 & 1080 \\
    FFN Width & \multicolumn{2}{c}{2} \\
    Position Embeddings & \multicolumn{2}{c}{-} \\
    Weight Tying &  \multicolumn{2}{c}{True} \\
    \midrule
    Short Conv. Filter Size & \multicolumn{2}{c}{3} \\
    Exp. Mod. Decay (Fast, Slow) & \multicolumn{2}{c}{0.3, 1.2} \\
    Filter Sine Freq. (w) & \multicolumn{2}{c}{14} \\
    Filter Order & \multicolumn{2}{c}{64} \\
    Filter Inner MLP & \multicolumn{2}{c}{2} \\
    Filter Weight Decay & \multicolumn{2}{c}{0} \\
    \bottomrule 
    \end{tabular}
\end{table}

\begin{table}[h!]
    \caption{{\Coyote} FLOPs Computation}
    \label{tab:coyote-flops}
    \centering
    \begin{tabular}{rc}
    \toprule
    {} & Equation  \\
    \midrule
    Input Layer & $B \times V \times N \times D$ \\
    \midrule
     Sequence Mixer Long Convolution & $10 \times N \times \log(N) \times 0.5(D) \times B$ \\
     Sequence Mixer Short Convolution & $B \times N \times 0.5(D)$ \\
     Sequence Mixer Implicit MLP (Order $64$) & $0.5(D) \times 64$ \\
     Sequence Mixer Linear Projection & $B \times N \times D \times D$ \\
     \midrule
     Channel Mixer (FFN Width $2$) &  $B \times D \times D \times 2 \times 2 \times N$ \\
     \midrule
     Language Modeling Head & $B \times V \times N \times D$ \\
    \bottomrule 
    \end{tabular}
\end{table}

%% file: main.bbl
\begin{thebibliography}{79}
\providecommand{\natexlab}[1]{#1}
\providecommand{\url}[1]{\texttt{#1}}
\expandafter\ifx\csname urlstyle\endcsname\relax
  \providecommand{\doi}[1]{doi: #1}\else
  \providecommand{\doi}{doi: \begingroup \urlstyle{rm}\Url}\fi

\bibitem[Fu et~al.(2023{\natexlab{a}})Fu, Dao, Saab, Thomas, Rudra, and R{\'e}]{dao2022hungry}
Daniel~Y. Fu, Tri Dao, Khaled~K. Saab, Armin~W. Thomas, Atri Rudra, and Christopher R{\'e}.
\newblock Hungry {H}ungry {H}ippos: Towards language modeling with state space models.
\newblock In \emph{International Conference on Learning Representations}, 2023{\natexlab{a}}.

\bibitem[Ma et~al.(2022)Ma, Zhou, Kong, He, Gui, Neubig, May, and Luke]{ma2022mega}
Xuezhe Ma, Chunting Zhou, Xiang Kong, Junxian He, Liangke Gui, Graham Neubig, Jonathan May, and Zettlemoyer Luke.
\newblock Mega: Moving average equipped gated attention.
\newblock \emph{arXiv preprint arXiv:2209.10655}, 2022.

\bibitem[Wang et~al.(2022)Wang, Yan, Gu, and Rush]{wang2022pretraining}
Junxiong Wang, Jing~Nathan Yan, Albert Gu, and Alexander~M Rush.
\newblock Pretraining without attention.
\newblock \emph{arXiv preprint arXiv:2212.10544}, 2022.

\bibitem[Poli et~al.(2023{\natexlab{a}})Poli, Massaroli, Nguyen, Fu, Dao, Baccus, Bengio, Ermon, and R{\'e}]{poli2023hyena}
Michael Poli, Stefano Massaroli, Eric Nguyen, Daniel~Y Fu, Tri Dao, Stephen Baccus, Yoshua Bengio, Stefano Ermon, and Christopher R{\'e}.
\newblock Hyena hierarchy: Towards larger convolutional language models.
\newblock \emph{arXiv preprint arXiv:2302.10866}, 2023{\natexlab{a}}.

\bibitem[Dauphin et~al.(2017)Dauphin, Fan, Auli, and Grangier]{dauphin2017language}
Yann~N Dauphin, Angela Fan, Michael Auli, and David Grangier.
\newblock Language modeling with gated convolutional networks.
\newblock In \emph{International conference on machine learning}, pages 933--941. PMLR, 2017.

\bibitem[Gu et~al.(2021)Gu, Goel, and R{\'e}]{gu2021efficiently}
Albert Gu, Karan Goel, and Christopher R{\'e}.
\newblock Efficiently modeling long sequences with structured state spaces.
\newblock \emph{arXiv preprint arXiv:2111.00396}, 2021.

\bibitem[Peng et~al.(2023)Peng, Alcaide, Anthony, Albalak, Arcadinho, Cao, Cheng, Chung, Grella, Kiran~GV, He, Hou, Kazienko, Kocon, and Kong]{peng2023rwkv}
Bo~Peng, Eric Alcaide, Quentin Anthony, Alon Albalak, Samuel Arcadinho, Huanqi Cao, Xin Cheng, Michael Chung, Matteo Grella, Kranthi Kiran~GV, Xuzheng He, Haowen Hou, Przemyslaw Kazienko, Jan Kocon, and Jiaming et~al. Kong.
\newblock Rwkv: Reinventing rnns for the transformer era.
\newblock \emph{arXiv:2305.13048}, 2023.

\bibitem[Fu et~al.(2023{\natexlab{b}})Fu, Arora, Grogan, Johnson, Eyuboglu, Thomas, Spector, Poli, Rudra, and Ré]{fu2023monarch}
Daniel~Y. Fu, Simran Arora, Jessica Grogan, Isys Johnson, Sabri Eyuboglu, Armin~W. Thomas, Benjamin Spector, Michael Poli, Atri Rudra, and Christopher Ré.
\newblock Monarch mixer: A simple sub-quadratic gemm-based architecture, 2023{\natexlab{b}}.

\bibitem[Graves et~al.(2014)Graves, Wayne, and Danihelka]{graves2014neural}
Alex Graves, Greg Wayne, and Ivo Danihelka.
\newblock Neural turing machines.
\newblock \emph{arXiv preprint arXiv:1410.5401}, 2014.

\bibitem[Ba et~al.(2016)Ba, Hinton, Mnih, Leibo, and Ionescu]{ba2016using}
Jimmy Ba, Geoffrey~E Hinton, Volodymyr Mnih, Joel~Z Leibo, and Catalin Ionescu.
\newblock Using fast weights to attend to the recent past.
\newblock \emph{Advances in neural information processing systems}, 29, 2016.

\bibitem[Elhage et~al.(2021)Elhage, Nanda, Olsson, Henighan, Joseph, Mann, Askell, Bai, Chen, Conerly, et~al.]{elhage2021mathematical}
Nelson Elhage, Neel Nanda, Catherine Olsson, Tom Henighan, Nicholas Joseph, Ben Mann, Amanda Askell, Yuntao Bai, Anna Chen, Tom Conerly, et~al.
\newblock A mathematical framework for transformer circuits.
\newblock \emph{Transformer Circuits Thread}, 1, 2021.

\bibitem[Olsson et~al.(2022)Olsson, Elhage, Nanda, Joseph, DasSarma, Henighan, Mann, Askell, Bai, Chen, et~al.]{olsson2022context}
Catherine Olsson, Nelson Elhage, Neel Nanda, Nicholas Joseph, Nova DasSarma, Tom Henighan, Ben Mann, Amanda Askell, Yuntao Bai, Anna Chen, et~al.
\newblock In-context learning and induction heads.
\newblock \emph{arXiv preprint arXiv:2209.11895}, 2022.

\bibitem[Lutati et~al.(2023)Lutati, Zimerman, and Wolf]{lutati2023focus}
Shahar Lutati, Itamar Zimerman, and Lior Wolf.
\newblock Focus your attention (with adaptive iir filters), 2023.

\bibitem[Hahn(2020)]{hahn2020theoretical}
Michael Hahn.
\newblock Theoretical limitations of self-attention in neural sequence models.
\newblock In \emph{Transactions of the Association for Computational Linguistics}, volume~8, 2020.

\bibitem[Merrill et~al.(2022)Merrill, Sabharwal, and Smith]{merrill2022saturated}
William Merrill, Ashish Sabharwal, and Noah~A. Smith.
\newblock Saturated transformers are constant-depth threshold circuits.
\newblock In \emph{Transactions of the Association for Computational Linguistics}, volume~10, 2022.

\bibitem[Keles et~al.(2023)Keles, Wijewardena, and Hegde]{keles2023on}
Feyza~Duman Keles, Pruthuvi~Mahesakya Wijewardena, and Chinmay Hegde.
\newblock On the computational complexity of self-attention.
\newblock In \emph{34th International Conference on Algorithmic Learning Theory}, volume 201, page 1–23, 2023.

\bibitem[Gao et~al.(2020)Gao, Biderman, Black, Golding, Hoppe, Foster, Phang, He, Thite, Nabeshima, Presser, and Leahy]{pile}
Leo Gao, Stella Biderman, Sid Black, Laurence Golding, Travis Hoppe, Charles Foster, Jason Phang, Horace He, Anish Thite, Noa Nabeshima, Shawn Presser, and Connor Leahy.
\newblock The {P}ile: An 800gb dataset of diverse text for language modeling.
\newblock \emph{arXiv preprint arXiv:2101.00027}, 2020.

\bibitem[Touvron et~al.(2023)Touvron, Martin, Stone, Albert, Almahairi, Babaei, Bashlykov, Batra, Bhargava, and Bhosale]{touvron2023llama}
Hugo Touvron, Louis Martin, Kevin Stone, Peter Albert, Amjad Almahairi, Yasmine Babaei, Nikolay Bashlykov, Soumya Batra, Prajjwal Bhargava, and Shruti Bhosale.
\newblock Llama 2: Open foundation and fine-tuned chat models.
\newblock \emph{arXiv:2307.09288}, 2023.

\bibitem[Vaswani et~al.(2017)Vaswani, Shazeer, Parmar, Uszkoreit, Jones, Gomez, Kaiser, and Polosukhin]{vaswani2018attention}
Ashish Vaswani, Noam Shazeer, Niki Parmar, Jakob Uszkoreit, Llion Jones, Aidan~N Gomez, Lukasz Kaiser, and Illia Polosukhin.
\newblock Attention is all you need.
\newblock volume~30, 2017.

\bibitem[Cooley and Tukey(1965)]{cooley1965algorithm}
James~W Cooley and John~W Tukey.
\newblock An algorithm for the machine calculation of complex fourier series.
\newblock \emph{Mathematics of computation}, 19\penalty0 (90):\penalty0 297--301, 1965.

\bibitem[Tay et~al.(2022)Tay, Dehghani, Bahri, and Metzler]{tay2022efficient}
Yi~Tay, Mostafa Dehghani, Dara Bahri, and Donald Metzler.
\newblock Efficient transformers: A survey.
\newblock \emph{ACM Computing Surveys}, 55\penalty0 (6):\penalty0 1--28, 2022.

\bibitem[Goel et~al.(2022)Goel, Gu, Donahue, and Ré]{goel2022its}
Karan Goel, Albert Gu, Chris Donahue, and Christopher Ré.
\newblock It's raw! audio generation with state-space models.
\newblock \emph{Proceedings of the 39 th International Conference on Machine Learning,}, 2022.

\bibitem[Zhang et~al.(2023)Zhang, Saab, Poli, Dao, Goel, and R{\'e}]{zhang2023effectively}
Michael Zhang, Khaled Saab, Michael Poli, Tri Dao, Karan Goel, and Christopher R{\'e}.
\newblock Effectively modeling time series with simple discrete state spaces.
\newblock \emph{International Conference on Learning Representations}, 2023.

\bibitem[Zhai et~al.(2021)Zhai, Talbott, Srivastava, Huang, Goh, Zhang, and Susskind]{zhai2021attention}
Shuangfei Zhai, Walter Talbott, Nitish Srivastava, Chen Huang, Hanlin Goh, Ruixiang Zhang, and Josh Susskind.
\newblock An attention free transformer.
\newblock \emph{arXiv preprint arXiv:2105.14103}, 2021.

\bibitem[Andonian et~al.(2023)Andonian, Anthony, Biderman, Black, Gali, Gao, Hallahan, Levy-Kramer, Leahy, Nestler, Parker, Pieler, Phang, Purohit, Schoelkopf, Stander, Songz, Tigges, Thérien, Wang, and Weinbach]{gpt-neox-library}
Alex Andonian, Quentin Anthony, Stella Biderman, Sid Black, Preetham Gali, Leo Gao, Eric Hallahan, Josh Levy-Kramer, Connor Leahy, Lucas Nestler, Kip Parker, Michael Pieler, Jason Phang, Shivanshu Purohit, Hailey Schoelkopf, Dashiell Stander, Tri Songz, Curt Tigges, Benjamin Thérien, Phil Wang, and Samuel Weinbach.
\newblock {GPT-NeoX: Large Scale Autoregressive Language Modeling in PyTorch}, 9 2023.
\newblock URL \url{https://www.github.com/eleutherai/gpt-neox}.

\bibitem[Hasani et~al.(2022)Hasani, Lechner, Wang, Chahine, Amini, and Rus]{hasani2022liquid}
Ramin Hasani, Mathias Lechner, Tsun-Huang Wang, Makram Chahine, Alexander Amini, and Daniela Rus.
\newblock Liquid structural state-space models.
\newblock \emph{arXiv preprint arXiv:2209.12951}, 2022.

\bibitem[Sun et~al.(2023)Sun, Dong, Huang, Ma, Xia, Xue, Wang, and Wei]{sun2023retentive}
Yutao Sun, Li~Dong, Shaohan Huang, Shuming Ma, Yuqing Xia, Jilong Xue, Jianyong Wang, and Furu Wei.
\newblock Retentive network: A successor to transformer for large language models, 2023.

\bibitem[Eyuboglu et~al.(2022)Eyuboglu, Varma, Saab, Delbrouck, Lee-Messer, Dunnmon, Zou, and R{\'e}]{Eyuboglu2022-qz}
Sabri Eyuboglu, Maya Varma, Khaled Saab, Jean-Benoit Delbrouck, Christopher Lee-Messer, Jared Dunnmon, James Zou, and Christopher R{\'e}.
\newblock Domino: Discovering systematic errors with cross-modal embeddings.
\newblock In \emph{International Conference on Learning Representations}, March 2022.

\bibitem[Willshaw et~al.(1969)Willshaw, Buneman, and Longuet-Higgins]{willshaw1969non}
David~J Willshaw, O~Peter Buneman, and Hugh~Christopher Longuet-Higgins.
\newblock Non-holographic associative memory.
\newblock \emph{Nature}, 222\penalty0 (5197):\penalty0 960--962, 1969.

\bibitem[Hopfield(1982)]{hopfield1982neural}
John~J Hopfield.
\newblock Neural networks and physical systems with emergent collective computational abilities.
\newblock \emph{Proceedings of the national academy of sciences}, 79\penalty0 (8):\penalty0 2554--2558, 1982.

\bibitem[B{\"u}rgisser et~al.(1996)B{\"u}rgisser, Lickteig, Clausen, and Shokrollahi]{bürgisser1996algebraic}
P.~B{\"u}rgisser, T.~Lickteig, M.~Clausen, and A.~Shokrollahi.
\newblock \emph{Algebraic Complexity Theory}.
\newblock Grundlehren der mathematischen Wissenschaften. Springer Berlin Heidelberg, 1996.
\newblock ISBN 9783540605829.
\newblock URL \url{https://books.google.com/books?id=dYcgjfXsYk8C}.

\bibitem[Dao et~al.(2020)Dao, Sohoni, Gu, Eichhorn, Blonder, Leszczynski, Rudra, and R{\'e}]{dao2020kaleidoscope}
Tri Dao, Nimit~S Sohoni, Albert Gu, Matthew Eichhorn, Amit Blonder, Megan Leszczynski, Atri Rudra, and Christopher R{\'e}.
\newblock Kaleidoscope: An efficient, learnable representation for all structured linear maps.
\newblock \emph{arXiv preprint arXiv:2012.14966}, 2020.

\bibitem[Fu et~al.(2023{\natexlab{c}})Fu, Epstein, Nguyen, Thomas, Zhang, Dao, Rudra, and R{\'e}]{fu2023simple}
Daniel~Y. Fu, Elliot~L. Epstein, Eric Nguyen, Armin~W. Thomas, Michael Zhang, Tri Dao, Atri Rudra, and Christopher R{\'e}.
\newblock Simple hardware-efficient long convolutions for sequence modeling.
\newblock \emph{arXiv preprint arXiv:2302.06646}, 2023{\natexlab{c}}.

\bibitem[Zaheer et~al.(2020)Zaheer, Guruganesh, Dubey, Ainslie, Alberti, Ontanon, Pham, Ravula, Wang, Yang, and et~al]{zaheer2020bigbird}
Manzil Zaheer, Guru Guruganesh, Avinava Dubey, Joshua Ainslie, Chris Alberti, Santiago Ontanon, Philip Pham, Anirudh Ravula, Qifan Wang, Li~Yang, and et~al.
\newblock Big bird: Transformers for longer sequences.
\newblock \emph{Proceedings of NeurIPS}, 2020.

\bibitem[Child et~al.(2019{\natexlab{a}})Child, Gray, Radford, and Sutskever]{child2019generating}
Rewon Child, Scott Gray, Alec Radford, and Ilya Sutskever.
\newblock Generating long sequences with sparse transformers.
\newblock \emph{arXiv preprint arXiv:1904.10509}, 2019{\natexlab{a}}.

\bibitem[Beltagy et~al.(2020)Beltagy, Peters, and Cohan]{beltagy2020longformer}
Iz~Beltagy, Matthew~E Peters, and Arman Cohan.
\newblock Longformer: The long-document transformer.
\newblock \emph{arXiv preprint arXiv:2004.05150}, 2020.

\bibitem[Ren et~al.(2023)Ren, Liu, Wang, Xu, Zhu, and Zhai]{ren2023sparse}
Liliang Ren, Yang Liu, Shuohang Wang, Yichong Xu, Chenguang Zhu, and ChengXiang Zhai.
\newblock Sparse modular activation for efficient sequence modeling.
\newblock \emph{Sparse Modular Activation for Efficient Sequence Modeling}, 2023.

\bibitem[Dao et~al.(2022)Dao, Fu, Ermon, Rudra, and R{\'e}]{dao2022flashattention}
Tri Dao, Daniel~Y. Fu, Stefano Ermon, Atri Rudra, and Christopher R{\'e}.
\newblock Flash{A}ttention: Fast and memory-efficient exact attention with {IO}-awareness.
\newblock In \emph{Advances in Neural Information Processing Systems}, 2022.

\bibitem[Dao(2023)]{dao2023flashattention2}
Tri Dao.
\newblock Flash{A}ttention-2: Faster attention with better parallelism and work partitioning.
\newblock 2023.

\bibitem[Katharopoulos et~al.(2020{\natexlab{a}})Katharopoulos, Vyas, Pappas, and Fleuret]{katharopoulos-et-al-2020}
A.~Katharopoulos, A.~Vyas, N.~Pappas, and F.~Fleuret.
\newblock Transformers are rnns: Fast autoregressive transformers with linear attention.
\newblock In \emph{Proceedings of the International Conference on Machine Learning (ICML)}, 2020{\natexlab{a}}.
\newblock URL \url{https://arxiv.org/abs/2006.16236}.

\bibitem[Olah(2022)]{olah2022mechanistic}
Chris Olah.
\newblock Mechanistic interpretability, variables, and the importance of interpretable bases: An informal note on some intuitions related to mechanistic interpretability, 2022.
\newblock URL \url{https://transformer-circuits.pub/2022/mech-interp-essay/index.html}.

\bibitem[Power et~al.(2022)Power, Burda, Edwards, Babuschkin, and Misra]{power2022grokking}
Alethea Power, Yuri Burda, Harri Edwards, Igor Babuschkin, and Vedant Misra.
\newblock Grokking: Generalization beyond overfitting on small algorithmic datasets.
\newblock \emph{arXiv:2201.02177}, 2022.

\bibitem[Cammarata et~al.(2020)Cammarata, Carter, Goh, Olah, Petrov, Schubert, Voss, Egan, and Lim]{cammarata2020thread}
Nick Cammarata, Shan Carter, Gabriel Goh, Chris Olah, Michael Petrov, Ludwig Schubert, Chelsea Voss, Ben Egan, and Swee~Kiat Lim.
\newblock Thread: circuits.
\newblock \emph{Distill}, 5\penalty0 (3):\penalty0 e24, 2020.

\bibitem[Wang and Eisner(2016)]{wang2016galactic}
Dingquan Wang and Jason Eisner.
\newblock The galactic dependencies treebanks: Getting more data by synthesizing new languages.
\newblock \emph{Transactions of the Association for Computational Linguistics}, 4:\penalty0 491--505, 2016.

\bibitem[White and Cotterell(2021)]{white2021examining}
Jennifer~C White and Ryan Cotterell.
\newblock Examining the inductive bias of neural language models with artificial languages.
\newblock \emph{arXiv preprint arXiv:2106.01044}, 2021.

\bibitem[Allen-Zhu and Li(2023)]{allen2023physics}
Zeyuan Allen-Zhu and Yuanzhi Li.
\newblock Physics of language models: Part 1, context-free grammar.
\newblock \emph{arXiv preprint arXiv:2305.13673}, 2023.

\bibitem[Ravfogel et~al.(2019)Ravfogel, Goldberg, and Linzen]{ravfogel2019studying}
Shauli Ravfogel, Yoav Goldberg, and Tal Linzen.
\newblock Studying the inductive biases of rnns with synthetic variations of natural languages.
\newblock \emph{arXiv preprint arXiv:1903.06400}, 2019.

\bibitem[Xie et~al.(2021)Xie, Raghunathan, Liang, and Ma]{xie2021incontext}
Sang~Michael Xie, Aditi Raghunathan, Percy Liang, and Tengyu Ma.
\newblock An explanation of in-context learning as implicit bayesian inference.
\newblock \emph{arXiv preprint arXiv:2111.02080}, 2021.

\bibitem[Kitaev et~al.(2020)Kitaev, Kaiser, and Levskaya]{kitaev2020reformer}
Nikita Kitaev, {\L}ukasz Kaiser, and Anselm Levskaya.
\newblock Reformer: The efficient transformer.
\newblock \emph{arXiv preprint arXiv:2001.04451}, 2020.

\bibitem[Feldman et~al.(1981)Feldman, Hinton, and Anderson]{feldman1981parallel}
JA~Feldman, GE~Hinton, and JA~Anderson.
\newblock Parallel models of associative memory, 1981.

\bibitem[Zhang and Zhou(2017)]{zhang2017learning}
Wei Zhang and Bowen Zhou.
\newblock Learning to update auto-associative memory in recurrent neural networks for improving sequence memorization.
\newblock \emph{arXiv preprint arXiv:1709.06493}, 2017.

\bibitem[Massaroli et~al.(2023)Massaroli, Poli, Fu, Kumbong, Romero, Parnichukun, Timalsina, McIntyre, Chen, Rudra, Zhang, R{\'e}, Ermon, and Bengio]{massaroli2023laughing}
Stefano Massaroli, Michael Poli, Daniel~Y Fu, Hermann Kumbong, David Romero, Rom Parnichukun, Aman Timalsina, Quinn McIntyre, Beidi Chen, Atri Rudra, Ce~Zhang, Christopher R{\'e}, Stefano Ermon, and Yoshua Bengio.
\newblock Laughing hyena distillery: Extracting compact recurrences from convolutions.
\newblock 2023.

\bibitem[Romero et~al.(2022)Romero, Kuzina, Bekkers, Tomczak, and Hoogendoorn]{romero2022ckconv}
David~W. Romero, Anna Kuzina, Erik~J. Bekkers, Jakub~M. Tomczak, and Mark Hoogendoorn.
\newblock Ckconv: Continuous kernel convolution for sequential data.
\newblock 2022.

\bibitem[Gupta et~al.(2022)Gupta, Gu, and Berant]{gupta2022diagonal}
Ankit Gupta, Albert Gu, and Jonathan Berant.
\newblock Diagonal state spaces are as effective as structured state spaces, 2022.

\bibitem[Gu et~al.(2022)Gu, Gupta, Goel, and Ré]{gu2022parameterization}
Albert Gu, Ankit Gupta, Karan Goel, and Christopher Ré.
\newblock On the parameterization and initialization of diagonal state space models, 2022.

\bibitem[Mehta et~al.(2022)Mehta, Gupta, Cutkosky, and Neyshabur]{mehta2022long}
Harsh Mehta, Ankit Gupta, Ashok Cutkosky, and Behnam Neyshabur.
\newblock Long range language modeling via gated state spaces, 2022.

\bibitem[Smith et~al.(2023)Smith, Warrington, and Linderman]{smith2023simplified}
Jimmy T.~H. Smith, Andrew Warrington, and Scott~W. Linderman.
\newblock Simplified state space layers for sequence modeling, 2023.

\bibitem[Nguyen et~al.(2023)Nguyen, Poli, Faizi, Thomas, Birch-Sykes, Wornow, Patel, Rabideau, Massaroli, Bengio, Ermon, Baccus, and Ré]{nguyen2023hyenadna}
Eric Nguyen, Michael Poli, Marjan Faizi, Armin Thomas, Callum Birch-Sykes, Michael Wornow, Aman Patel, Clayton Rabideau, Stefano Massaroli, Yoshua Bengio, Stefano Ermon, Stephen~A. Baccus, and Chris Ré.
\newblock Hyenadna: Long-range genomic sequence modeling at single nucleotide resolution, 2023.

\bibitem[toe(2023)]{toews2023transformers}
Transformers revolutionized ai. what will replace them?, 2023.
\newblock URL \url{https://www.forbes.com/sites/robtoews/2023/09/03/transformers-revolutionized-ai-what-will-replace-them/?sh=6ed698269c1f}.

\bibitem[Wang et~al.(2023)Wang, Zhu, Wang, Yu, Liu, Omar, and Hamid]{wang2023selective}
Jue Wang, Wentao Zhu, Pichao Wang, Xiang Yu, Linda Liu, Mohamed Omar, and Raffay Hamid.
\newblock Selective structured state-spaces for long-form video understanding.
\newblock In \emph{CVPR}, 2023.

\bibitem[Yang et~al.(2020)Yang, Bender, Le, and Ngiam]{yang2020condconv}
Brandon Yang, Gabriel Bender, Quoc~V. Le, and Jiquan Ngiam.
\newblock Condconv: Conditionally parametrized convolutions for efficient inference.
\newblock In \emph{33rd Conference on Neural Information Processing Systems (NeurIPS 2019)}, 2020.

\bibitem[Kosma et~al.(2023)Kosma, Nikolentzos, and Vazirgiannis]{kosma2023time}
Chrysoula Kosma, Giannis Nikolentzos, and Michalis Vazirgiannis.
\newblock Time-parameterized convolutional neural networks for irregularly sampled time series.
\newblock 2023.

\bibitem[Katharopoulos et~al.(2020{\natexlab{b}})Katharopoulos, Vyas, Pappas, and Fleuret]{katharopoulos2020transformers}
Angelos Katharopoulos, Apoorv Vyas, Nikolaos Pappas, and François Fleuret.
\newblock Transformers are rnns: Fast autoregressive transformers with linear attention, 2020{\natexlab{b}}.

\bibitem[Together(2023)]{together2023redpajama}
Together.
\newblock Redpajama: An open source recipe to reproduce llama training dataset, 2023.
\newblock URL \url{https://github.com/togethercomputer/RedPajama-Data}.

\bibitem[Sukhbaatar et~al.(2019)Sukhbaatar, Grave, Bojanowski, and Joulin]{sainbayar2019adaptive}
Sainbayar Sukhbaatar, Edouard Grave, Piotr Bojanowski, and Armand Joulin.
\newblock Adaptive attention span in transformers.
\newblock \emph{Association of Computational Linguistics}, 2019.

\bibitem[Brown et~al.(2020)Brown, Mann, Ryder, Subbiah, Kaplan, Dhariwal, Neelakantan, Shyam, Sastry, Askell, et~al.]{brown2020language}
Tom Brown, Benjamin Mann, Nick Ryder, Melanie Subbiah, Jared~D Kaplan, Prafulla Dhariwal, Arvind Neelakantan, Pranav Shyam, Girish Sastry, Amanda Askell, et~al.
\newblock Language models are few-shot learners.
\newblock \emph{Advances in neural information processing systems}, 33:\penalty0 1877--1901, 2020.

\bibitem[Anonymous(2023{\natexlab{a}})]{mamba2023}
Anonymous.
\newblock Mamba: Linear-time sequence modeling with selective state spaces.
\newblock 2023{\natexlab{a}}.
\newblock URL \url{https://openreview.net/forum?id=AL1fq05o7H}.

\bibitem[Anonymous(2023{\natexlab{b}})]{gateloop2023}
Anonymous.
\newblock Gateloop: Fully data-controlled linear recurrence for sequence modeling.
\newblock 2023{\natexlab{b}}.
\newblock URL \url{https://openreview.net/pdf?id=02Ug9N8DCI}.

\bibitem[Child et~al.(2019{\natexlab{b}})Child, Gray, Radford, and Sutskever]{child2019sparse}
Rewon Child, Scott Gray, Alec Radford, and Ilya Sutskever.
\newblock Generating long sequences with sparse transformers.
\newblock \emph{arXiv preprint arXiv:1904.10509}, 2019{\natexlab{b}}.

\bibitem[Qiu et~al.(2019)Qiu, Ma, Levy, tau Yih, Wang, and Tang]{qiu2019blockwise}
Jiezhong Qiu, Hao Ma, Omer Levy, Scott~Wen tau Yih, Sinong Wang, and Jie Tang.
\newblock Blockwise self-attention for long document understanding.
\newblock \emph{arXiv preprint arXiv:1911.02972}, 2019.

\bibitem[Heideman and Burrus(1988)]{heideman}
Michael~T Heideman and C~Sidney Burrus.
\newblock \emph{Multiplicative complexity, convolution, and the DFT}.
\newblock Springer, 1988.

\bibitem[Volkovich(2016)]{volkovich2016guide}
Ilya Volkovich.
\newblock A guide to learning arithmetic circuits.
\newblock In \emph{Conference on Learning Theory}, pages 1540--1561. PMLR, 2016.

\bibitem[Poli et~al.(2023{\natexlab{b}})Poli, Massaroli, Nguyen, Fu, Dao, Baccus, Bengio, Ermon, and R{\'e}]{hyena}
Michael Poli, Stefano Massaroli, Eric Nguyen, Daniel~Y Fu, Tri Dao, Stephen Baccus, Yoshua Bengio, Stefano Ermon, and Christopher R{\'e}.
\newblock Hyena hierarchy: Towards larger convolutional language models.
\newblock \emph{Proceedings of the 40th International Conference on Machine Learning}, 2023{\natexlab{b}}.

\bibitem[Jayram et~al.(2008)Jayram, Kumar, and Sivakumar]{jayram2008one}
Thathachar~S Jayram, Ravi Kumar, and Dandapani Sivakumar.
\newblock The one-way communication complexity of hamming distance.
\newblock \emph{Theory of Computing}, 4\penalty0 (1):\penalty0 129--135, 2008.

\bibitem[Press et~al.(2021)Press, Smith, and Lewis]{press2021train}
Ofir Press, Noah~A Smith, and Mike Lewis.
\newblock Train short, test long: Attention with linear biases enables input length extrapolation.
\newblock \emph{arXiv preprint arXiv:2108.12409}, 2021.

\bibitem[Akl and Meijer(1990)]{akl1990parallel}
Selim~G Akl and Henk Meijer.
\newblock Parallel binary search.
\newblock \emph{IEEE Transactions on Parallel \& Distributed Systems}, 1\penalty0 (02):\penalty0 247--250, 1990.

\bibitem[Cormen et~al.(2022)Cormen, Leiserson, Rivest, and Stein]{cormen2022introduction}
Thomas~H Cormen, Charles~E Leiserson, Ronald~L Rivest, and Clifford Stein.
\newblock \emph{Introduction to algorithms}.
\newblock MIT press, 2022.

\bibitem[Ajtai et~al.(1983)Ajtai, Koml{\'o}s, and Szemer{\'e}di]{ajtai19830}
Mikl{\'o}s Ajtai, J{\'a}nos Koml{\'o}s, and Endre Szemer{\'e}di.
\newblock An 0 (n log n) sorting network.
\newblock In \emph{Proceedings of the fifteenth annual ACM symposium on Theory of computing}, pages 1--9, 1983.

\bibitem[Chatfield(1995)]{chatfield1995the}
Chris Chatfield.
\newblock The analysis of time series: An introduction, fifth edition.
\newblock 1995.

\end{thebibliography}
